\documentclass[11pt,english]{article}

\usepackage{graphics}
\usepackage{amsmath,amsfonts,amssymb}
\usepackage{hyperref}
\hypersetup{hidelinks}
\usepackage{enumitem}

\usepackage{url}            
\usepackage{booktabs}       
\usepackage{nicefrac}       
\usepackage{mathtools}
\usepackage{stmaryrd}  
\usepackage{color}
\usepackage{algorithm, algorithmic}
\usepackage{comment}
\usepackage{subfig}
\usepackage{multirow, makecell}
\usepackage{multirow}
\usepackage{multicol}

\newenvironment{thmbis}[1]
{\renewcommand{\thetheorem}{\ref{#1}$'$}%
	\addtocounter{theorem}{-1}%
	\begin{theorem}}
	{\end{theorem}}

\usepackage{amsthm}

\usepackage[top=1in, bottom=1in, left=1.5in, right=1.5in]{geometry}
\usepackage[utf8]{inputenc} 
\usepackage[T1]{fontenc}    
\usepackage{lmodern}
\usepackage{nicefrac}       
\usepackage{booktabs}       
\usepackage{amsfonts}       
\usepackage{microtype}      
\usepackage{algorithm}
\usepackage{algorithmic}
\usepackage{multirow}
\usepackage{multicol}

\newcommand*{\tor}{ \xrightarrow{\mathrm{or}}}

\usepackage[round]{natbib}

\newtheorem{definition}{Definition}
\newtheorem{theorem}{Theorem}
\newtheorem{corollary}{Corollary}
\newtheorem{lemma}{Lemma}

\newtheorem{proposition}{Proposition}

\newtheorem{example}{Example}

\title{On the Representation of Pairwise Causal Background Knowledge and Its Applications in Causal Inference}

\author{Zhuangyan Fang\textsuperscript{1,2} \;~
	Ruiqi Zhao\textsuperscript{1,3} \;~
	Yue Liu\textsuperscript{4} \;~
	Yangbo He\textsuperscript{1}\thanks{Correspondence to: heyb@pku.edu.cn .} \\
	~\\
	\textsuperscript{1}Peking University \quad \textsuperscript{2}Xiaomi Corporation \\
	\textsuperscript{3}Inspur Industrial Innovation (Shandong) Intelligent Manufacturing Co., Ltd.\\ \textsuperscript{4}Renmin University of China}

\begin{document}

\maketitle

\begin{abstract}
Pairwise causal background knowledge about the existence or  absence of causal edges and paths is frequently encountered in observational studies. Such  constraints   allow the shared directed  and undirected edges in the constrained subclass of Markov equivalent DAGs to be represented as a causal maximally partially directed acyclic graph (MPDAG). In this paper, we first provide a sound and complete graphical characterization of causal MPDAGs and introduce a minimal representation of a causal MPDAG. Then, we give a unified representation for three types of pairwise causal background knowledge, including direct, ancestral and non-ancestral causal knowledge,  by introducing a novel concept called direct causal clause (DCC). {Using DCCs, we study the consistency and equivalence of pairwise causal background knowledge and show that any pairwise causal background knowledge set can be uniquely and equivalently decomposed into the causal MPDAG representing the refined Markov equivalence class and a minimal residual set of DCCs.} Polynomial-time algorithms are also provided for checking consistency and equivalence, as well as for finding the decomposed MPDAG and the residual DCCs. Finally, with pairwise causal background knowledge, we prove a sufficient and necessary condition to identify causal effects and surprisingly find that the identifiability of causal effects only depends on the decomposed MPDAG. We also develop a local IDA-type algorithm to estimate the possible values of an unidentifiable effect. Simulations suggest that pairwise causal background knowledge can significantly improve the identifiability of causal effects.
\end{abstract}

\section{Introduction}
Causal background knowledge refers to the understanding or consensus of causal and non-causal relations in a system. Such information, as a supplement to data, may be obtained from domain knowledge or experts' judgments (such as smoking causes lung cancer and eating betel nuts causes oral cancer), from common sense (such as a subsequent event is not a cause of a prior event), or even from previous experimental studies (such as double-blind experiments or A/B tests). Under the framework of causal graphical models,  representing and exploiting causal background knowledge  may improve the identifiability of causal structures or causal effects in a study of causal discovery or causal inference  \citep{meek1995causal,Perkovic2020mpdag}.
For example, as shown in Figure~\ref{fig:intro}, consider a simple causal chain with three binary variables: smoking, bronchitis and dyspnea. With observational data only, it is possible to consistently estimate a completed partially directed acyclic graph (CPDAG) shown in Figure~\ref{fig:intro1}, representing a set of statistically equivalent DAGs called Markov equivalent. In this case, neither the causal structure among the three vertices nor the causal effect of smoking on bronchitis is identifiable, since there are three Markov equivalent DAGs and the causal effects estimated based on each of them are not identical. However, if we have already known that smoking can cause dyspnea, then there is only one DAG in the Markov equivalence class satisfying this causal constraint and thus
the causal effect of smoking on bronchitis is identifiable.

This paper  focuses on representing    \emph{pairwise} causal background knowledge and   incorporating  this knowledge into causal inference assuming no hidden variables or selection biases. 
We consider three types of pairwise causal background knowledge, including direct, ancestral and non-ancestral causal knowledge. A piece of direct causal knowledge is defined as the presence of a directed edge in a DAG.  Direct causal knowledge  is natural  and   has been studied extensively \citep{dor1992simple, meek1995causal,perkovic2017interpreting, 2019arXiv190702435H, Perkovic2020mpdag,Witte2020efficient,Guo2020minimal}. A piece of ancestral (non-ancestral) causal knowledge is defined as the presence (absence) of a directed path in a causal DAG. One can learn ancestral (non-ancestral) causal relations  from observational data \citep{fang2021local}, or from interventional data   since when a variable is intervened, its descendant variables could be changed  while other variables usually keep unchanged \citep{he2008active}. Each non-ancestral relation essentially implies a causal topological order between two variables.  Thus,  a causal topological order among variables, which is a common type of causal background knowledge in literature  \citep{park2017bayesian,wang2019directed},  can   be translated to  pairwise non-ancestral relations equivalently:  each   variable is not an ancestral variable of its preceding variables  in the order.

Existing works on pairwise causal background knowledge mainly focus on direct  causal knowledge~\citep{dor1992simple, meek1995causal,perkovic2017interpreting, 2019arXiv190702435H, Perkovic2020mpdag,Witte2020efficient,Guo2020minimal}. \citet{meek1995causal} proved that the set of DAGs in a Markov equivalence class satisfying given direct causal knowledge is non-empty if and only if it can be represented by a causal maximally partially directed acyclic graph (MPDAG), which contains both directed and undirected edges. Benefiting from the compact graphical representation, many researchers discussed the identifiability and efficient estimation of a causal effect, or  the estimation of all possible causal effects of a treatment on a response  with direct  causal knowledge~\citep{perkovic2017interpreting, 2019arXiv190702435H, Perkovic2020mpdag,Witte2020efficient,Guo2020minimal}. Recently, \citet{fang2020bgida} further  studied non-ancestral causal knowledge and proved that non-ancestral causal knowledge can also be represented exactly by causal MPDAGs. However, causal MPDAGs may fail to represent ancestral causal knowledge. The DAGs in a Markov equivalence class satisfying given ancestral causal knowledge may satisfy some structural constraints that cannot be posed by any causal MPDAG. An example is provided in Example \ref{rep:failrep}.


{Instead, the representation of ancestral causal knowledge remains under-explored. In the existing studies, ancestral causal knowledge is generally regarded as a constraint on the existence of directed paths~\citep[see, for example,][]{Borboudakis2012Incorporating}. Such a constraint is \emph{global} in the sense that it imposes complex restrictions on the direction of all edges along the paths connecting two nodes. As a result, it becomes challenging to answer some basic queries, such as whether two pieces of ancestral causal knowledge contradict each other, or whether one piece of ancestral causal knowledge can be inferred from others. Existing literature often addresses these queries, as well as other issues related to ancestral causal knowledge, by explicitly or implicitly enumerating all equivalent DAGs within a given Markov equivalence class and examining the paths in each DAG~\citep{Borboudakis2012Incorporating}. However, this enumeration-based approach is infeasible
	in high dimensional settings, where the number of Markov equivalent DAGs grows exponentially with the number of variables. This limitation also restricts the practical application of pairwise causal background knowledge in causal inference.}


\begin{figure}[!t]
	\centering
	\subfloat[A CPDAG over three variables\label{fig:intro1}]{
		\begin{minipage}[t]{0.35\linewidth}
			\centering
			\includegraphics[width=1\linewidth]{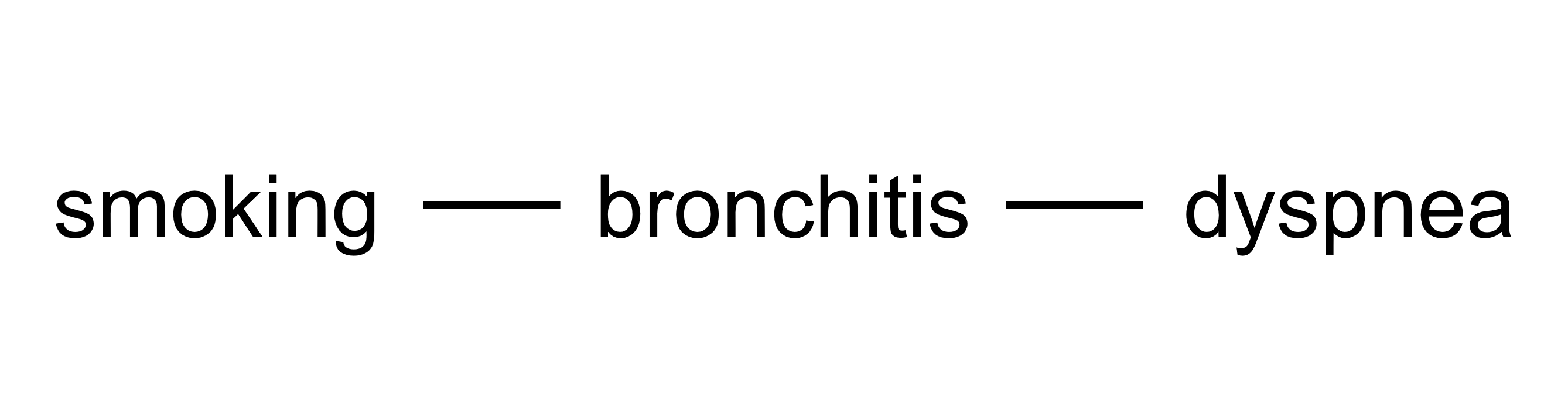}
		\end{minipage}%
	}%
	\hspace{0.15\linewidth}
	\subfloat[Three Markov equivalent DAGs\label{fig:intro2}]{
		\begin{minipage}[t]{0.35\linewidth}
			\centering
			\includegraphics[width=1\linewidth]{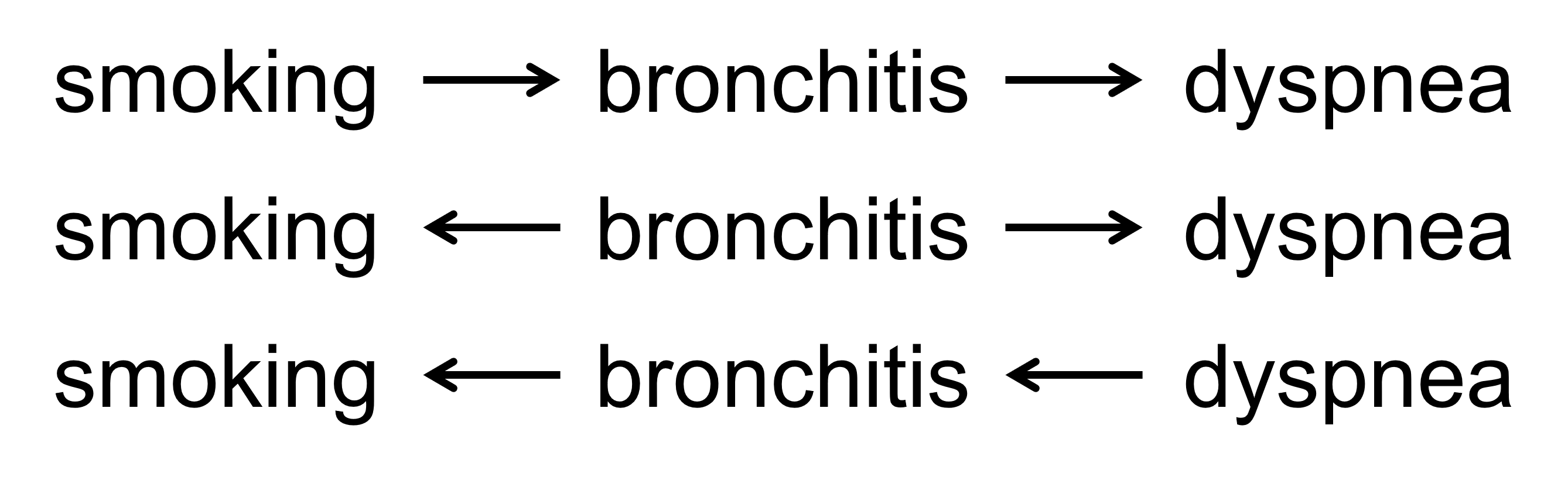}
		\end{minipage}%
	}%

	\caption{A CPDAG over three variables including smoking, bronchitis and dyspnea is given in Figure~\ref{fig:intro1}, which represents the Markov equivalent class shown in Figures~\ref{fig:intro2}.}
	\label{fig:intro}
\end{figure}

In this paper, we first provide a sound and complete graphical characterization of causal MPDAGs, together with their minimal representation. We establish sufficient and necessary conditions under which a partially directed graph qualifies as a causal MPDAG. Next, we introduce a novel representation of pairwise causal background knowledge, called direct causal clauses, which provide a unified way to represent direct, ancestral, and non-ancestral knowledge within a single framework. We further analyze the consistency and equivalence of pairwise causal background knowledge represented by direct causal clauses, and show that any set of pairwise causal background knowledge can be uniquely and equivalently decomposed into (i) the causal MPDAG representing the subset of Markov equivalent DAGs consistent with the given knowledge, and (ii) a minimal residual set of direct causal clauses. We also provide sufficient and necessary conditions under which an MPDAG can exactly represent pairwise causal background knowledge.

Leveraging direct causal clauses, we then propose polynomial-time algorithms for checking the consistency and equivalence of pairwise causal background knowledge, as well as for constructing the corresponding decomposed causal MPDAG and residual direct causal clauses. To the best of our knowledge, these are the first polynomial-time algorithms that address these tasks in the presence of all three types of pairwise causal knowledge.

Finally, we investigate the identifiability of causal effects when pairwise causal background knowledge is available. We find, perhaps surprisingly, that the identifiability depends only on the decomposed MPDAG derived from the background knowledge. When a causal effect is not identifiable, its possible values are determined jointly by the MPDAG and the residual set of direct causal clauses. For such cases, we develop IDA-type algorithms to estimate the possible effects locally or semi-locally, building on new local orientation rules for CPDAGs with direct causal clauses.

The following three subsections devote to some preliminaries. Unless otherwise stated, we use capital letters such as $X$ to denote variables or vertices or nodes, and use boldface letters like $\mathbf{X}$ to denote variable sets or vectors. An instantiation of a variable or vector is denoted by a lowercase letter, such as $x$ and $\mathbf{x}$. We use $\mathbf{X}\subseteq\mathbf{Y}$, $\mathbf{X}\subsetneq\mathbf{Y}$ and $\mathbf{X}\nsubseteq\mathbf{Y}$ to denote that $\mathbf{X}$ is a subset, proper subset and not a subset of $\mathbf{Y}$, respectively.


\subsection{Causal Graphical Models}\label{sec:sec:graphical}

In this paper, we use $\textbf{V}(\mathcal{G})$, $\textbf{E}(\mathcal{G})$, $\textbf{E}_d(\mathcal{G})$, and $\textbf{E}_u(\mathcal{G})$ to denote the vertex set (or node set), edge set, set of directed edges, and set of undirected edges of a given graph $\mathcal{G}$, respectively. Here, $\mathcal{G}$ can be a directed, undirected or partially directed graph.
The {skeleton} of $\cal G$ is the undirected graph obtained by  removing all arrowheads from $\cal G$. For any $\textbf{V}'\subseteq \textbf{V}$, the {induced subgraph} of $\mathcal{G}$ over $\textbf{V}'$, denoted by $\mathcal{G}(\textbf{V}')$, is the graph with vertex set $\textbf{V}'$ and edge set $\textbf{E}'\subseteq \textbf{E}$ containing all and only edges between vertices in $\textbf{V}'$. The {undirected subgraph} and {directed subgraph} of ${\cal G}$ are denoted by $\mathcal{G}_u$ and $\mathcal{G}_d$, respectively. The former is defined as the undirected graph obtained by removing all directed edges, while the latter is the directed graph obtained by removing all undirected edges from ${\cal G}$. An undirected (or directed) induced subgraph of ${\cal G}$ over $\textbf{V}'\subseteq \textbf{V}$ is the induced subgraph of $\mathcal{G}_u$ (or $\mathcal{G}_d$) over $\textbf{V}'$.

In a graph $\mathcal{G}$, $X_i$ is a {parent} of $X_j$ and $X_j$ is a {child} of $X_i$ if $X_i\rightarrow X_j$, and $X_i$ is a {sibling} of $X_j$ if $X_i - X_j$. Two vertices $X_i$ and $X_j$ are {adjacent} and called {neighbors} of each other if they are connected by an edge. We use $pa(X_i, \cal G)$, $ch(X_i, \cal G)$, $sib(X_i, \cal G)$, and  $adj(X_i, \cal G)$ to denote the sets of parents, children, siblings, and adjacent vertices of $X_i$ in $\mathcal{G}$, respectively. A graph is called {complete} if every two distinct vertices are adjacent. A vertex is {simplicial} if its neighbors induce a complete subgraph.

A {path} is a sequence of distinct vertices $(X_{1},\cdots,X_{n})$ such that any two consecutive vertices are adjacent. $X_{1}$ and $X_{n}$ are endpoints and the others are intermediate nodes. A path connecting $X\in\mathbf{X}$ and $Y\in\mathbf{Y}$ is {proper} if the intermediate nodes on the path are not in $\mathbf{X}\cup \mathbf{Y}$. If every two distinct vertices in a graph are connected by a path, then the graph is {connected}. A path from $X_1$ to $X_n$ is partially directed if $X_i \leftarrow X_{i+1}$ does not occur in $\cal G$ for any $i=1,...,n-1$ and $X_i \to X_{i+1}$ for some $i=1,...,n-1$.
Moreover, a path from $X_{1}$ to $X_{n}$ is {possibly causal} if $X_{i} \leftarrow X_{j}$ does not occur in $\cal {G}$ for any $i,j=1, \ldots, n$ and $i<j$, and is non-causal otherwise~\citep{perkovic2017interpreting}. A path from $X_1$ to $X_n$ is {directed} if $X_i \to X_{i+1}$ for every $i=1,...,n-1$, and is undirected if $X_i - X_{i+1}$ for every $i=1,...,n-1$. A partially directed (directed, or undirected) cycle is a partially directed (directed, or undirected) path from $X_{1}$ to $X_{n}$ together with a directed or an undirected edge (a directed edge, or an undirected edge) from $X_{n}$ to $X_{1}$. A directed graph is {acyclic} (DAG) if there are no directed cycles. A partially directed acyclic graph (PDAG) is a partially directed graph without directed cycles.
A chain graph is a partially directed graph with no partially directed cycles~\citep{lauritzen2002chain}. The length of a path (cycle) is the number of edges on the path (cycle). A vertex $X_i$ is an {ancestor} of $X_j$ and $X_j$ is a {descendant} of $X_i$  if there is a directed path from $X_i$ to $X_j$ or $X_i = X_j$; the sets of ancestors and descendants of $X_i$ in $\cal G$ are denoted by $an(X_i, \cal G)$ and $de(X_i, \cal G)$, respectively. A vertex $X_j$ is a possible descendant of $X_i$ if there is a possibly causal path from $X_i$ to $X_j$. A {chord} of a path (cycle) is an edge joining two nonconsecutive vertices on the path (cycle). An undirected graph is {chordal} if it has no chordless cycle with length greater than three.

Let $\pi = (X_1, \cdots , X_n)$ be a path in $\mathcal{G}$. $X_i$ ($i\neq1,n$) is a collider on $\pi$ if $X_{i-1} \rightarrow X_i \leftarrow X_{i+1}$, and is a {definite non-collider} on $\pi$ if $X_{i-1} \leftarrow X_i$, or $X_i \rightarrow X_{i+1}$, or $X_{i-1} - X_i - X_{i+1}$ but $X_{i-1}$ is not adjacent to $X_{i+1}$.  Moreover, $X_i$ is of definite status on $\pi$ if it is a collider, or a definite non-collider, or an endpoint on $\pi$~\citep{Guo2020minimal}. A path $\pi$ is of definite status if its nodes are of definite status. For distinct vertices $X_i, X_j$ and $X_k$, if $X_i\rightarrow X_j\leftarrow X_k$ and $X_i$ is not adjacent to $X_k$ in $\cal G$, the triple $(X_i, X_j, X_k)$ is called a {v-structure} collided on $X_j$. A definite status path $\pi$ from $X$ to $Y$ is d-separated (blocked) by $\bf Z$ ($X,Y\notin{\bf Z}$) if $\pi$ has a definite non-collider in $\bf Z$ or $\pi$ has no collider who has a descendant in $\bf Z$, and is d-connected given $\bf Z$ otherwise.


Two DAGs are Markov {equivalent} if they induce the same d-separation relations.~\citet{pearl1989conditional} proved that two DAGs are equivalent if and only if they have the same skeleton and the same v-structures. A  {Markov equivalence class} contains all DAGs equivalent to each other. A Markov equivalence class can be uniquely represented by a completed PDAG, or {essential graph}, defined as follows:
\begin{definition}[Completed PDAG,~\citealt{andersson1997characterization}]\label{cpdag}
	Given a DAG  $\cal G$, the completed PDAG (CPDAG) of $\cal G$, denoted by ${\cal G}^*$, is a PDAG that has the same skeleton as $\cal G$, and  a directed edge occurs in $\mathcal{G}^*$ if and only if it appears in all equivalent DAGs of $\cal G$.
\end{definition}
We assume that the CPDAG $\mathcal{G}^*$ of the Markov equivalence class containing the underlying DAG $\cal G$ is provided, and use $[\mathcal{G}]$ or $[\mathcal{G}^*]$ to represent the Markov equivalence class.~\citet{andersson1997characterization} proved that a CPDAG is a chain graph, and its undirected subgraph is the union of disjoint connected chordal graphs, which are called {chain components}. A causal DAG model consists of a DAG $\cal G$ and a distribution $f$ over the same set $\textbf{V}$ such that
$f(x_1, . . . ,x_{n}) =\prod_{i=1}^{n} f(x_i |pa(x_i, \cal G))$.


\subsection{Pairwise Causal Background Knowledge}\label{sec:sec:interpretbg}

In this paper, we mainly consider \emph{pairwise} causal background knowledge, which can be formally defined in terms of constraints as follows.
	
	\begin{definition}[Pairwise Causal  Constraints]\label{interpretbg:causaldef} A direct causal constraint denoted by%
		\linebreak $X\to Y$ is a proposition saying that $X$ is a parent of $Y$, that is, $X$ is a direct cause of $Y$. An {ancestral causal constraint} denoted by $X\dashrightarrow Y$ is a proposition saying that $X$ is an ancestor of $Y$, that is, $X$ is a cause of $Y$. A {non-ancestral causal constraint} denoted by $X\longarrownot\dashrightarrow Y$ is a proposition saying that $X$ is not an ancestor of $Y$, that is, $X$ is not a cause of $Y$. Moreover, $X$ is called the tail and $Y$ is called the head in the above notions.
	\end{definition}
	
	
	
	Non-pairwise causal background knowledge will be briefly discussed in Section \ref{sec:sec:non-pair}. A pairwise causal constraint set is also called a (causal) background knowledge set for short.  A pairwise causal constraint set over $\bf V$ consists of some of the constraints with  heads and tails in $\bf V$. Given a DAG ${\cal G}$, a pairwise causal constraint set ${\cal B}$ over ${\bf V}({\cal G})$ is said to hold for ${\cal G}$, or equivalently, ${\cal G}$ is said to satisfy ${\cal B}$, if every proposition in ${\cal B}$ is true for ${\cal G}$. We define the restricted Markov equivalence class induced by $\mathcal{G}^*$ and    $\cal B$ as follows.
	
	
	\begin{definition}[Restricted Markov Equivalence Class]\label{interpretbg:reseqclassdef}
		The restricted Markov equivalence class induced by a CPDAG $\mathcal{G}^*$ and  a pairwise causal constraint set  $\cal B$ over ${\bf V}(\mathcal{G}^*)$, denoted by  $[\mathcal{G}^*, {\cal B}]$, consists of all equivalent DAGs in $[{\cal G^*}]$ that satisfy  $\cal B$.
	\end{definition}

	If ${\cal B}$ is empty, then $[\mathcal{G}^*, {\cal B}]=[\mathcal{G}^*]$. A restricted Markov equivalence class  $[\mathcal{G}^*, {\cal B}]$ is empty if none of the DAGs   in $[{\cal G^*}]$  satisfies  $\cal B$. For example, if two exclusive  constraints, say both $X\dashrightarrow Y$ and $X\longarrownot\dashrightarrow Y$, appear in $\cal B$, we have that $[{\mathcal G}^*, {\cal B}]=\varnothing$.  Conversely, we say that $\cal B$ is {consistent} with $\mathcal{G}^*$ if $[\mathcal{G}^*, {\cal B}]\neq\varnothing$.

	A PDAG is maximal (MPDAG) if it is closed under the four Meek's rules shown in Figure~\ref{fig:meek}~\citep{meek1995causal}. Given a  CPDAG  ${\cal G}^*$ and a pairwise causal constraint set $\cal B$ consistent with ${\cal G}^*$, the causal MPDAG of $[\mathcal{G}^*, {\cal B}]$ is defined as follows.
	\begin{figure}[!t]
		\centering
		\vspace{2em}
		\begin{minipage}[b]{1\linewidth}
			\centering
			\includegraphics[width=0.95\linewidth]{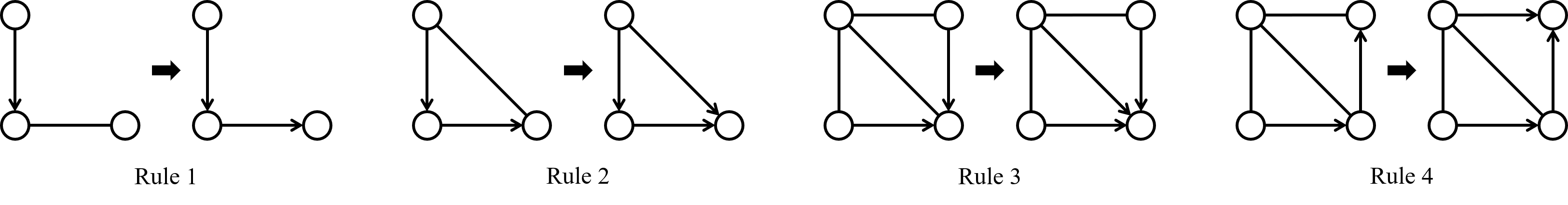}
		\end{minipage}
		\caption{A visualization of Meek's rules. If the graph on the left-hand side of a rule is an induced
			subgraph of a PDAG, then orient the undirected edge such that the resulting subgraph is the one on the right-hand side of the rule.}
		\label{fig:meek}
	\end{figure}

\begin{definition}[Causal MPDAG]\label{interpretbg:mpdagresclass}   The MPDAG $\cal H$ of a non-empty restricted Markov equi-%
	\linebreak valence class $[\mathcal{G}^*, {\cal B}]$ induced by a CPDAG $\mathcal{G}^*$ and a pairwise causal constraint set ${\cal B}$ is a PDAG such that (1)  $\cal H$ has the same skeleton and v-structures as $\mathcal{G}^*$, and (2) an edge is directed in $\cal H$ if and only if it appears in all DAGs in $[{\mathcal G}^*, {\cal B}]$. An MPDAG $\cal H$ is a causal MPDAG if there exists a CPDAG ${\cal G}^*$ and a pairwise causal constraint set   $\cal B$ (possibly empty)  consistent with ${\cal G}^*$ such that $\cal H$ is the MPDAG of $[\mathcal{G}^*, {\cal B}]$.
\end{definition}

It is easy to verify that Definition~\ref{interpretbg:mpdagresclass} indeed defines an MPDAG.  Clearly,  the MPDAG $\cal H$ of $[{\mathcal G}^*, {\cal B}]$ contains the common direct causal relations of all restricted Markov equivalent DAGs in  $[{\mathcal G}^*, {\cal B}]$.  Let $[\cal H]$ be the set of DAGs which contain all directed edges of $\cal H$ and have the same skeleton and v-structures as $\cal H$. Following Definition~\ref{interpretbg:mpdagresclass}, if $\cal H$ is the MPDAG of $[{\mathcal G}^*, {\cal B}]$, then every DAG in $[{\mathcal G}^*, {\cal B}]$ belongs to $[\cal H]$, that is,  $[\mathcal{G}^*, {\cal B}]\subseteq [\cal H]$. An example illustrating that $[{\cal G}^*,{\cal B}]$ may be a proper subset of $[\cal H]$  is shown by Example \ref{rep:failrep}.

\begin{example}
	\label{rep:failrep}
	Figure~\ref{fig:intro2-1} shows a CPDAG ${\cal G}^*$,   and ${\cal G}_1$ to ${\cal G}_4$ shown in Figures~\ref{fig:intro2-2} to~\ref{fig:intro2-5} are DAGs  in $[{\cal G}^*]$ satisfying the ancestral causal constraint ${\cal B} = \{ X\dashrightarrow Y\}$. That is, the restricted Markov equivalence class  $[{\cal G}^*,{\cal B}]=\{{\cal G}_1,{\cal G}_2,{\cal G}_3,{\cal G}_4\}$. The causal MPDAG $\cal H$ of $[{\cal G}^*,\cal B]$ is shown in Figure~\ref{fig:intro2-6}, which has two directed edges $A\to Y$ and $B\to Y$ as they are both in ${\cal G}_1$ to ${\cal G}_4$. On the other hand, $[\cal H]$ consists of ${\cal G}_1$ to ${\cal G}_6$, meaning that $[{\cal G}^*,{\cal B}]\subsetneq[\cal H]$. In summary, $\cal B$ implies two  direct causal relations,  $A\to Y$ and $B\to Y$, as well as a  constraint  that $X$ is a direct cause of either $A$ or $B$. The MPDAG $\cal H$ does not imply the latter constraint.


\end{example}

\begin{figure}[!t]
	\centering
	\subfloat[CPDAG ${\cal G}^*$ \label{fig:intro2-1}]{
		\begin{minipage}[t]{0.15\linewidth}
			\centering
			\includegraphics[width=0.8\linewidth]{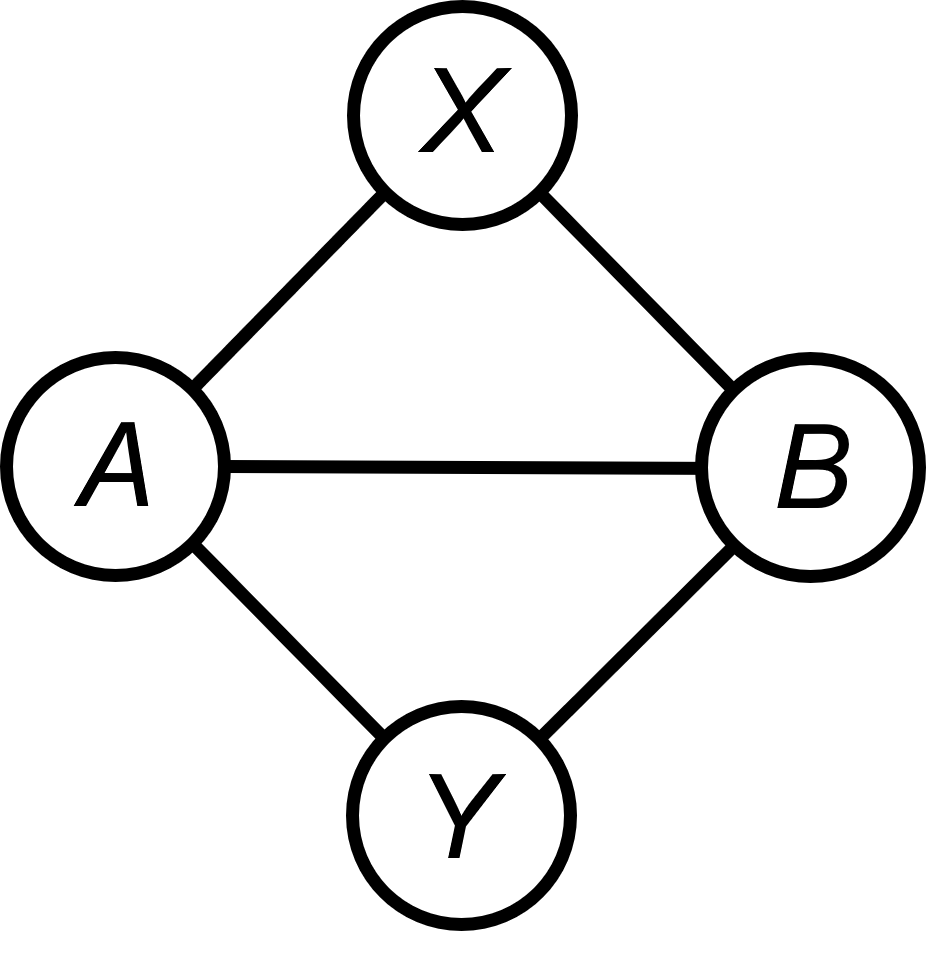}
		\end{minipage}%
	}%
	\hspace{0.05\linewidth}
	\subfloat[${\cal G}_1$ \label{fig:intro2-2}]{
		\begin{minipage}[t]{0.15\linewidth}
			\centering
			\includegraphics[width=0.8\linewidth]{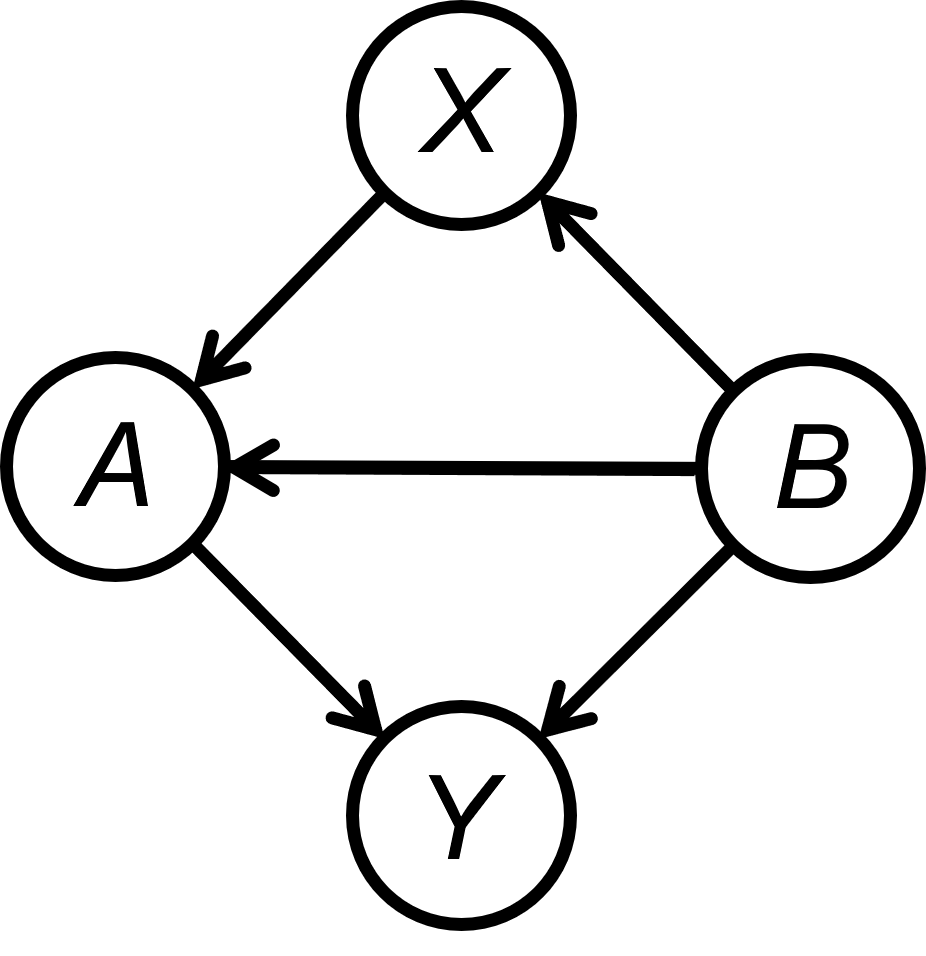}
		\end{minipage}%
	}%
	\hspace{0.05\linewidth}
	\subfloat[${\cal G}_2$ \label{fig:intro2-3}]{
		\begin{minipage}[t]{0.15\linewidth}
			\centering
			\includegraphics[width=0.8\linewidth]{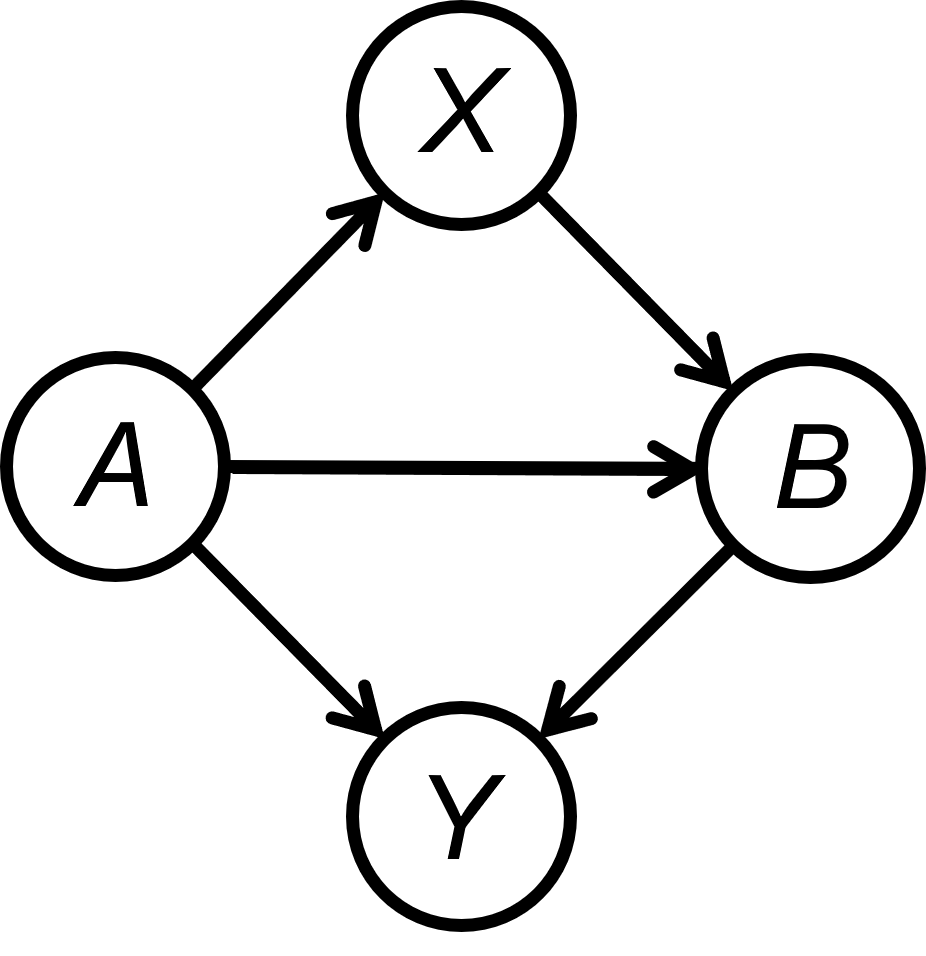}
		\end{minipage}%
	}%
	\hspace{0.05\linewidth}
	\subfloat[${\cal G}_3$ \label{fig:intro2-4}]{
		\begin{minipage}[t]{0.15\linewidth}
			\centering
			\includegraphics[width=0.8 \linewidth]{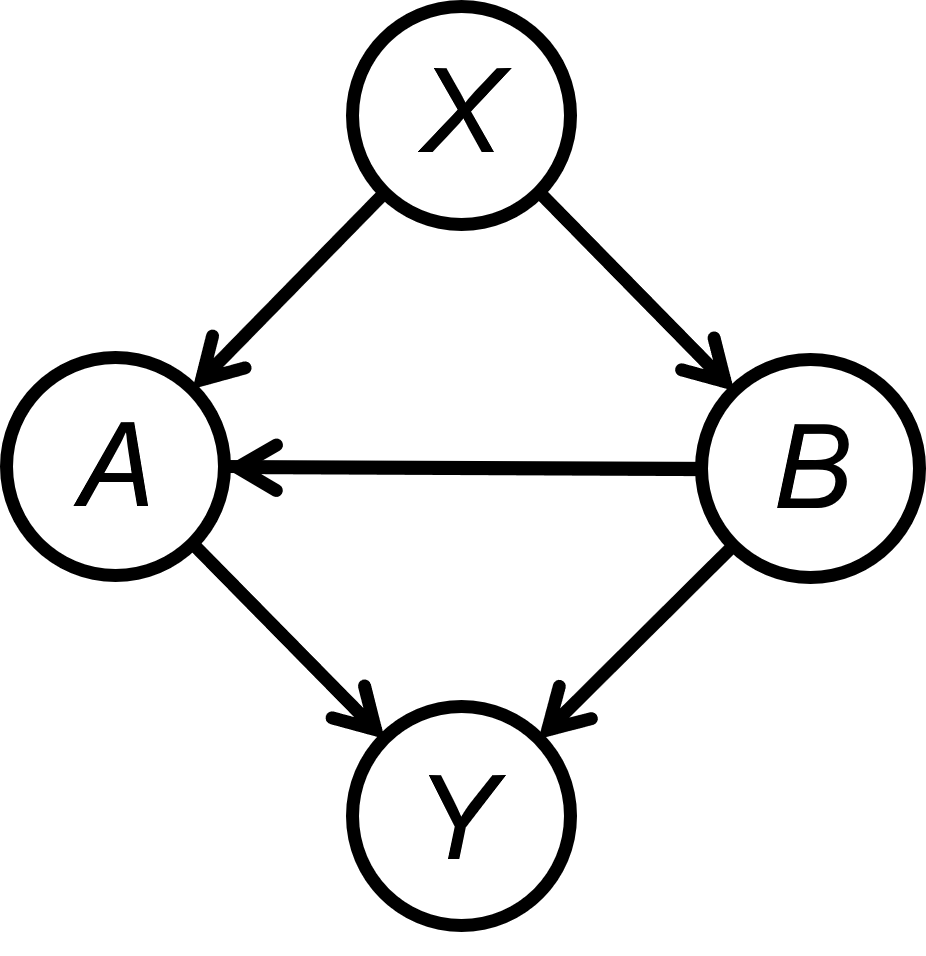}
		\end{minipage}%
	}%

	\subfloat[${\cal G}_4$ \label{fig:intro2-5}]{
		\begin{minipage}[t]{0.15\linewidth}
			\centering
			\includegraphics[width=0.8 \linewidth]{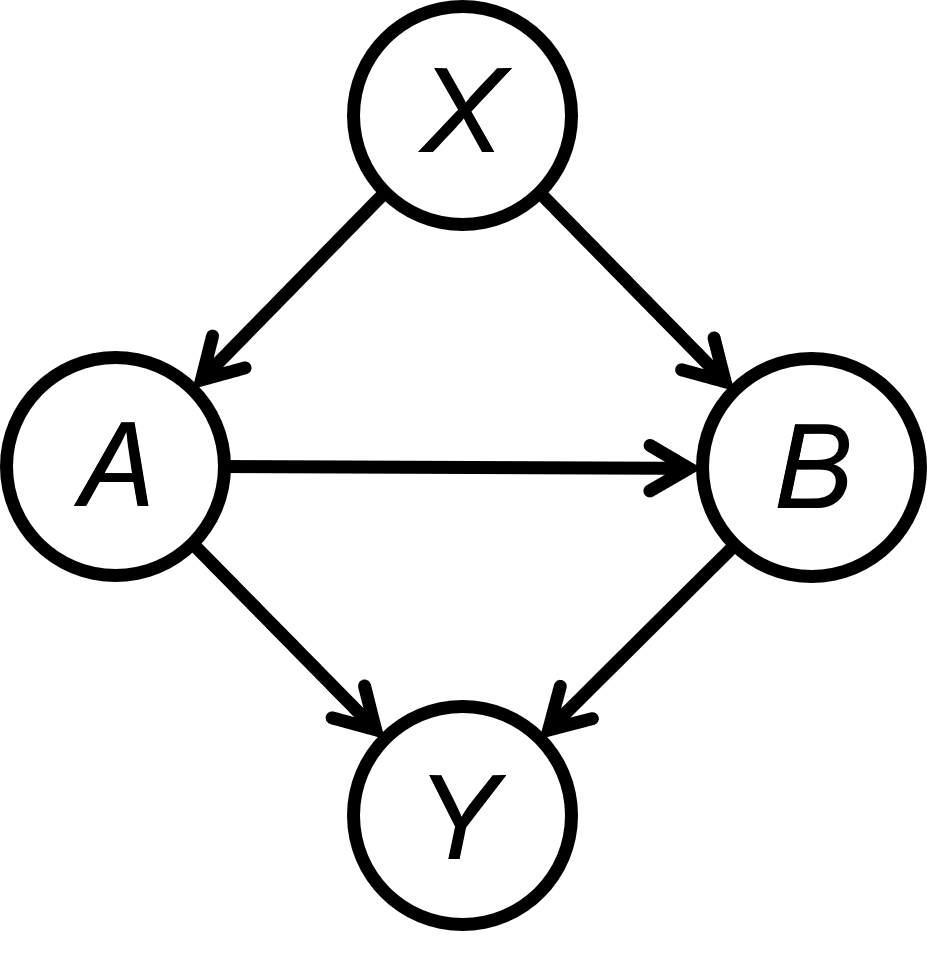}
		\end{minipage}%
	}%
	\hspace{0.05\linewidth}
	\subfloat[MPDAG ${\cal H}$ \label{fig:intro2-6}]{
		\begin{minipage}[t]{0.15\linewidth}
			\centering
			\includegraphics[width=0.8 \linewidth]{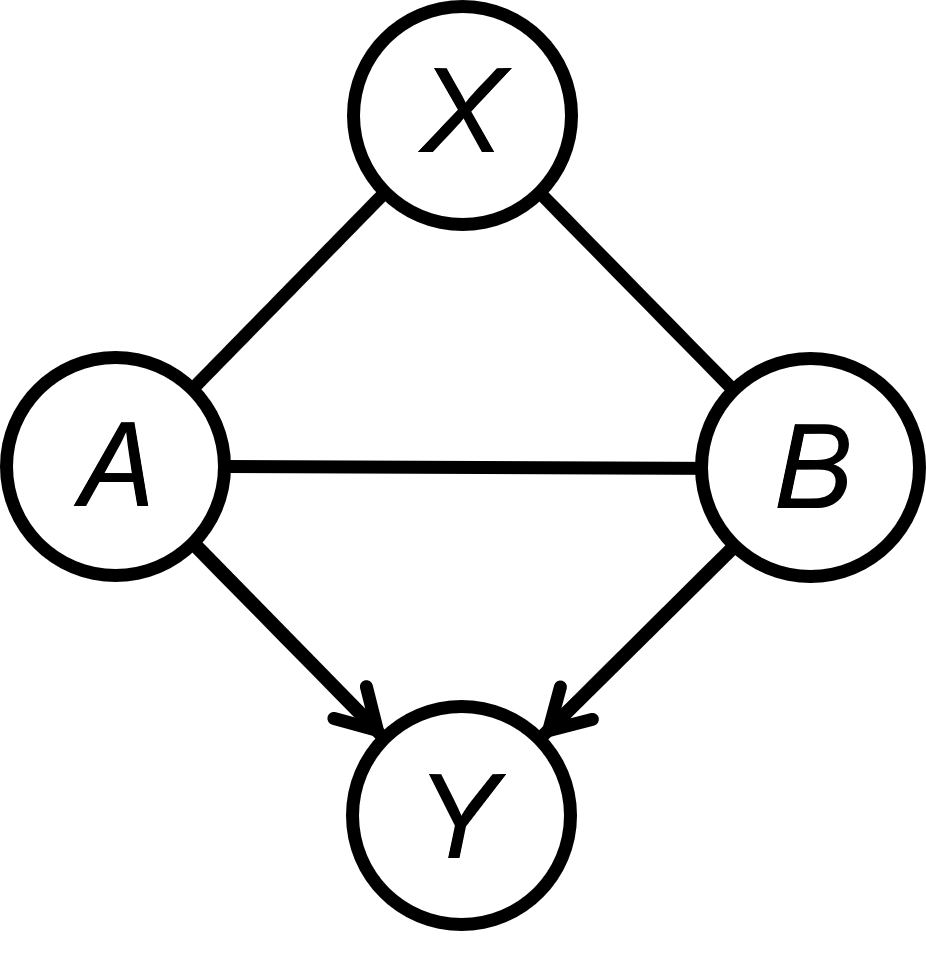}
		\end{minipage}%
	}%
	\hspace{0.05\linewidth}
	\subfloat[${\cal G}_5$ \label{fig:intro2-7}]{
		\begin{minipage}[t]{0.15\linewidth}
			\centering
			\includegraphics[width=0.8 \linewidth]{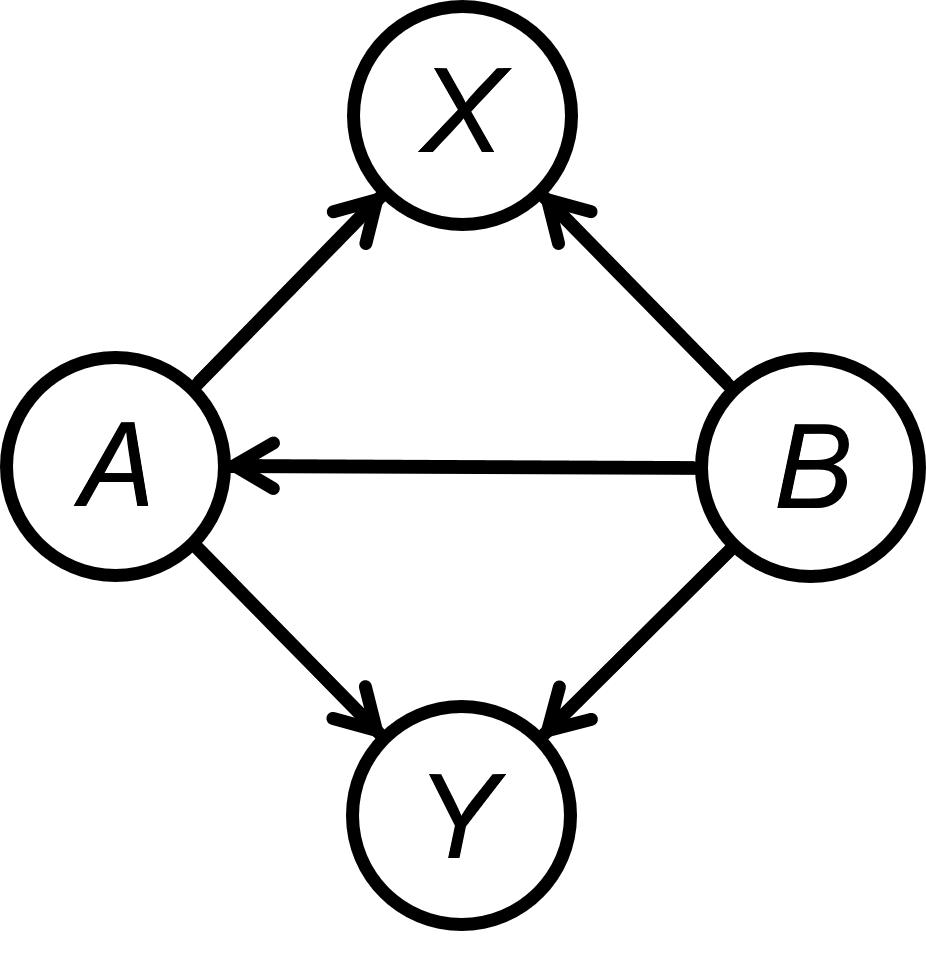}
		\end{minipage}%
	}%
	\hspace{0.05\linewidth}
	\subfloat[${\cal G}_6$ \label{fig:intro2-8}]{
		\begin{minipage}[t]{0.15\linewidth}
			\centering
			\includegraphics[width=0.8 \linewidth]{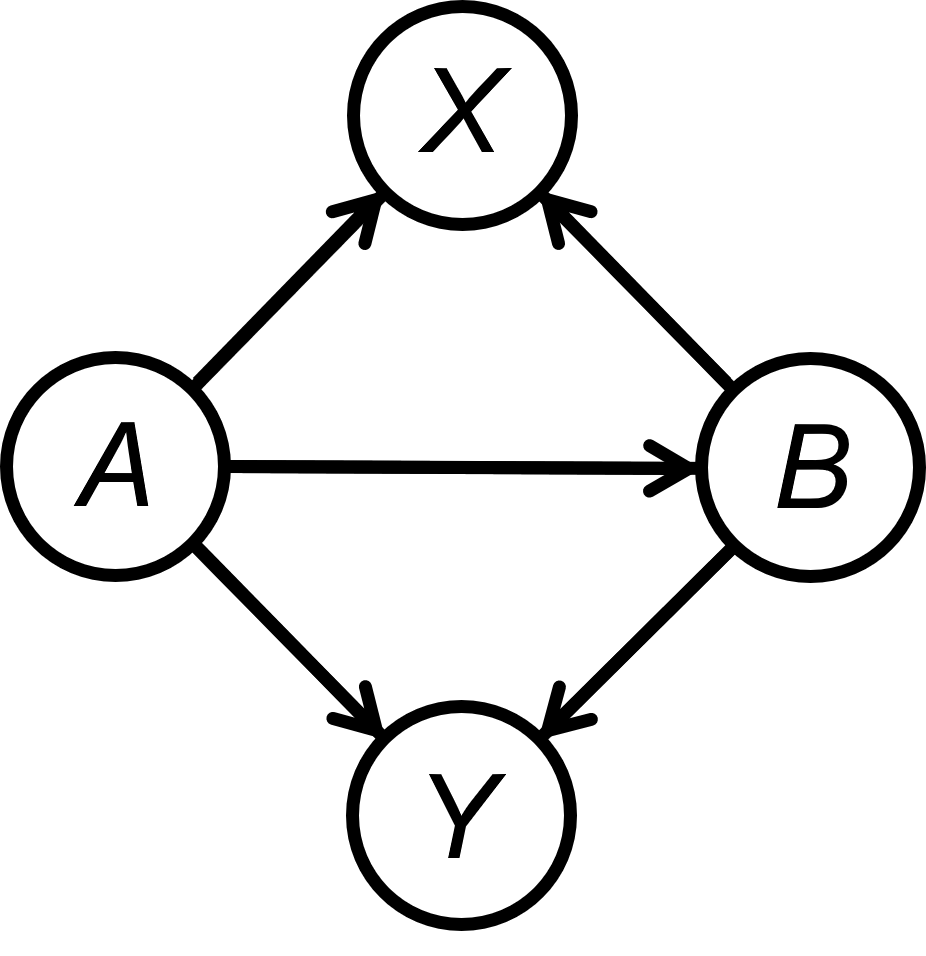}
		\end{minipage}%
	}%
	
	\caption{An example of a causal MPDAG. Let ${\cal B}=\{ X\dashrightarrow Y\}$ and  $\cal H$ be the  MPDAG of $[{\cal G}^*,\cal B]$.  $[{\cal G}^*,\cal B]$ consists of the DAGs from ${\cal G}_1$ to ${\cal G}_4$, while  $[{\cal H}]$ consists of the DAGs from ${\cal G}_1$ to ${\cal G}_6$, which indicates that $[{\cal G}^*,{\cal B}]\subsetneq[\cal H]$.}
	\label{fig:intro_ancestral_fail}
\end{figure}

\begin{definition}[Fully Informative MPDAG]\label{def:fullyinfo}
	A causal MPDAG   $\cal H$ is fully informative with respect to a restricted Markov equivalence class if the set $[\cal H]$ is identical to the restricted Markov equivalence class.
\end{definition}

When  $\cal H$ is   fully informative   with respect to $[\mathcal{G}^*, {\cal B}]$,
$\cal H$  can represent $[\mathcal{G}^*, {\cal B}]$ exactly.   When $\cal B$ only contains direct causal constraints,~\citet{meek1995causal} proved that the MPDAG of $[\mathcal{G}^*, {\cal B}]$ is fully informative, and the MPDAG can be constructed in polynomial time using Meek's rules.~\citet{fang2020bgida} further showed that  any non-ancestral causal  constraint  can be represented equivalently by some direct causal constraints and the corresponding fully informative $\cal H$ can also be constructed from $\mathcal{G}^*$ and ${\cal B}$  in polynomial time.

\subsection{Intervention Calculus}

In order to obtain the effect of an intervention on a response variable,~\citet{pearl2009causality} employed the notion of \emph{do-operator} to formulate the post-intervention distribution as follows: given a DAG $\cal G$ over the vertex set $\mathbf{V}=\{X_1, . . . ,X_{n}\}$ and $\mathbf{X}\subseteq \mathbf{V}$,
\begin{eqnarray}\label{leq2}
	f( \mathbf{v} \,|\, do(\mathbf{X}=\mathbf{x}))
	=\left\{
	\begin{array}{cr}
		\prod\limits_{X_i\in\mathbf{V}\setminus\mathbf{X}} f(x_i \,|\, pa(x_i, {\cal G}))|_{\mathbf{X}=\mathbf{x}},& \mathrm{if}\ \mathbf{v}|_{\mathbf{X}}=\mathbf{x}, \\
		0,& \mathrm{otherwise}.
	\end{array}
	\right.
\end{eqnarray}
Here, $f( \mathbf{v} \,|\, do(\mathbf{X}=\mathbf{x}))$ (or $f( \mathbf{v} \,|\, do(\mathbf{x}))$ for short) is the post-intervention distribution over $\mathbf{V}$ after intervening on $\mathbf{X}$, by forcing $\mathbf{X}$ to equal $\mathbf{x}$; $\mathbf{v}$ is an instantiation of $\mathbf{V}$; $\mathbf{v}|_{\mathbf{X}}=\mathbf{x}$ means the value of $\mathbf{X}$ in the instantiation $\mathbf{v}$ equals $\mathbf{x}$. The post-intervention distribution $f(\mathbf{y} \,|\, do(\mathbf{x}))$ is defined by integrating out all variables other than $\mathbf{Y}$ in $f( \mathbf{v} \,|\, do(\mathbf{x}))$. Given a treatment set $\mathbf{X}$ and a response set $\mathbf{Y}$, if there exists an $\mathbf{x}\neq \mathbf{x}'$ such that $f(\mathbf{y} \,|\, do(\mathbf{x}))\neq f(\mathbf{y} \,|\, do(\mathbf{x}'))$, then $\mathbf{X}$ has a causal effect on $\mathbf{Y}$ \citep{pearl2009causality}. Following the notion of~\citet{pearl2009causality}, we simply use  $f( \mathbf{y} \,|\, do(\mathbf{x}))$ to represent the causal effect of $\mathbf{X}$ on $\mathbf{Y}$.

Given  the underlying  causal DAG, the post-intervention distribution can be calculated from observational distribution by using a number of criteria. For example, the post-intervention distribution of a single response $Y\notin pa(X, {\cal G})$ after intervening on a single treatment $X$ can be calculated by
\begin{eqnarray}\label{eq:backdoor}
	\begin{aligned}
		f(y \,|\, do(x)) = \int f(y \,|\, X=x, pa(X, {\cal G})=u ) f(u) du.
	\end{aligned}
\end{eqnarray}
If $Y\in pa(X, \cal G)$, then $f(y \,|\, do(x))=f(y \,|\, do(x'))$ for any two instantiations $x,x'$ of $X$. Equation (\ref{eq:backdoor}) is a special case of the {backdoor adjustment}~\citep{pearl1995causal, pearl2009causality}, and $pa(X, \cal G)$ is a backdoor adjustment set.
However, if the underlying DAG is not fully known, $f( \mathbf{v} \,|\, do(\mathbf{x}))$ may not be \emph{identifiable}~\citep{pearl2009causality}. Recently, the identifiability of a causal effect given an MPDAG has been studied \citep{perkovic2017interpreting,Perkovic2020mpdag}.~\citet{Perkovic2020mpdag} proved that $f( \mathbf{v} \,|\, do(\mathbf{x}))$ is identifiable if and only if every proper possibly causal path from $\mathbf{X}$ to $\mathbf{Y}$ starts with a directed edge in the given MPDAG.

If a causal effect is not identifiable, we can use the IDA framework to estimate all possible causal effects. The original IDA enumerates all possible causal effects of a single treatment $X$ on a single response $Y$  given a CPDAG  by listing all possible parental sets of $X$ and adjusting for each of them. To decide whether a set of variables is possible to be the parents of $X$,~\citet[Lemma~3.1]{maathuis2009estimating} provided a locally valid orientation rule.
Recently,~\citet[Theorem~1]{fang2020bgida} extended the locally valid orientation rule to MPDAGs and proposed a fully local extension of IDA to deal with direct causal and non-ancestral causal constraints. For multiple interventions,~\citet{nandy2017estimating} proposed the joint-IDA. Compared with IDA, this extension is semi-local, which uses Meek's rules to check the validity of each candidate parental set. However, Meek's rules are global in the sense that they require an entire PDAG as input.~\citet{perkovic2017interpreting} further extended the joint-IDA to MPDAGs, and the algorithm is called the semi-local IDA. The recent work on efficient adjustment~\citep[see, for example,][]{2019arXiv190702435H} also motivates other extensions of IDA, such as the works of~\citet{Witte2020efficient, liu2020cida, Guo2020minimal}.



In Section~\ref{sec:sec:cha}, we study the graphical characterization of causal MPDAGs and their minimal representation. Section~\ref{sec:rep} then introduces direct causal clauses and demonstrates how to use them to represent pairwise causal constraints. Algorithms for checking the consistency and equivalence of pairwise causal constraints, as well as for constructing the decomposed causal MPDAG and residual direct causal clauses, are presented in Section~\ref{sec:algorithms}. In Section~\ref{sec:causal}, we focus on the identifiability of causal effects and the methods for locally or semi-locally estimating all possible causal effects, with simulations also provided. Additional algorithms and detailed proofs are given in the appendices.

\section{A Graphical Characterization of Causal MPDAGs}\label{sec:sec:cha}


In this section, we study the necessary and sufficient conditions for a partially directed graph to be a causal MPDAG, as well as the minimal representation of a causal MPDAG. MPDAGs serve as graphical tools for representing direct and non-ancestral causal constraints, and they play a crucial role in identifying causal effects under pairwise causal constraints, as demonstrated in Section~\ref{sec:sec:id}. Before presenting the main results in Theorems \ref{the:MPDAG} and \ref{thm:MSPrep}, we note that the results in this section are not directly used in the subsequent sections. Readers primarily interested in the general representation may choose to skip this section and proceed directly to Section~\ref{sec:rep}.

We first introduce   two  concepts related to  partially directed graphs.



\begin{definition}[B-component]\label{Bcomp}
	Given a partially directed graph $ \cal G $, a B-component $ {\cal C}^b $ of $ \cal G $ is an induced subgraph of $ \cal G $ over the vertices which are connected by at least one undirected path in $\cal G$.
\end{definition}

{The letter ``B" in ``B-component" comes from ``{\bf B}ucket". We note that, the definition of a B-component is related to \emph{Bucket} defined in~\citet{Perkovic2020mpdag}. A bucket is a maximal undirected connected subset of the node set of an MPDAG~\citep{Perkovic2020mpdag}. In an MPDAG, the vertex set of a B-component corresponds to a bucket and the induced subgraph over a bucket is a B-component. However, in contrast to buckets, which are defined solely on MPDAGs, B-components are defined for general partially directed graphs.}
We also remark that the B-component generalizes the concept of a chain component in a chain graph. In fact, for a chain graph, the definition of a B-component degenerates to that of a chain component. However, unlike chain components, a B-component may contain both directed and undirected edges. For example, the MPDAG $ \cal H $ shown in Figure \ref{MPDAGexample} has three B-components: the induced subgraphs of $ \cal H $ over $ \{ A \}, \{ E \} $, and $ \{ B, C, D \} $. The directed edges in a partially directed graph can be divided  into two parts: the edges between two B-components (such as $ A \rightarrow B, A \rightarrow E $ in Figure \ref{MPDAGexample}), and the edges within a B-component (such as $ C \rightarrow B $ in Figure \ref{MPDAGexample}).

\begin{figure}[!t]
	\centering
	\subfloat[ A causal MPDAG $ \cal H $ \label{MPDAGexample}]{
		\begin{minipage}[t]{0.3\linewidth}
			\centering
			\includegraphics[width=0.45\linewidth]{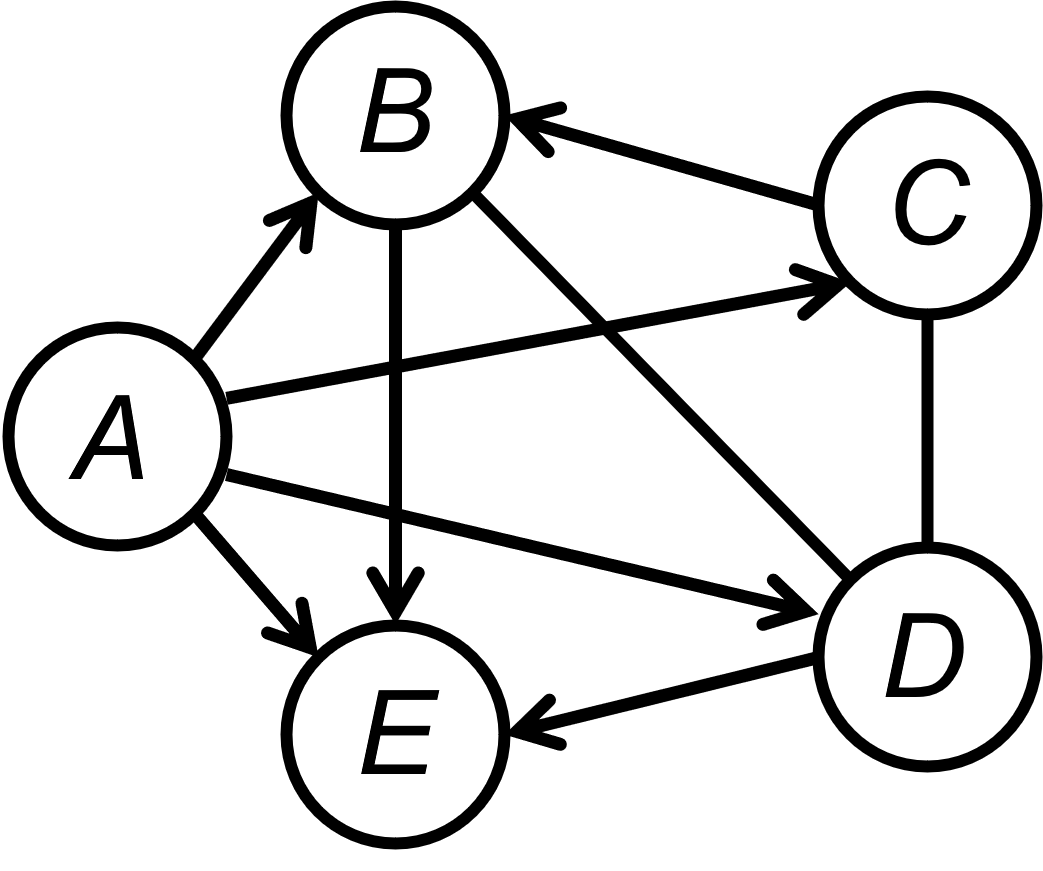}
		\end{minipage}%
	}%
	\hspace{0.02\linewidth}
	\subfloat[ The chain skeleton $ {\cal H}_c $ of $ \cal H $ \label{convertedMPDAG}]{
		\begin{minipage}[t]{0.3\linewidth}
			\centering
			\includegraphics[width=0.45\linewidth]{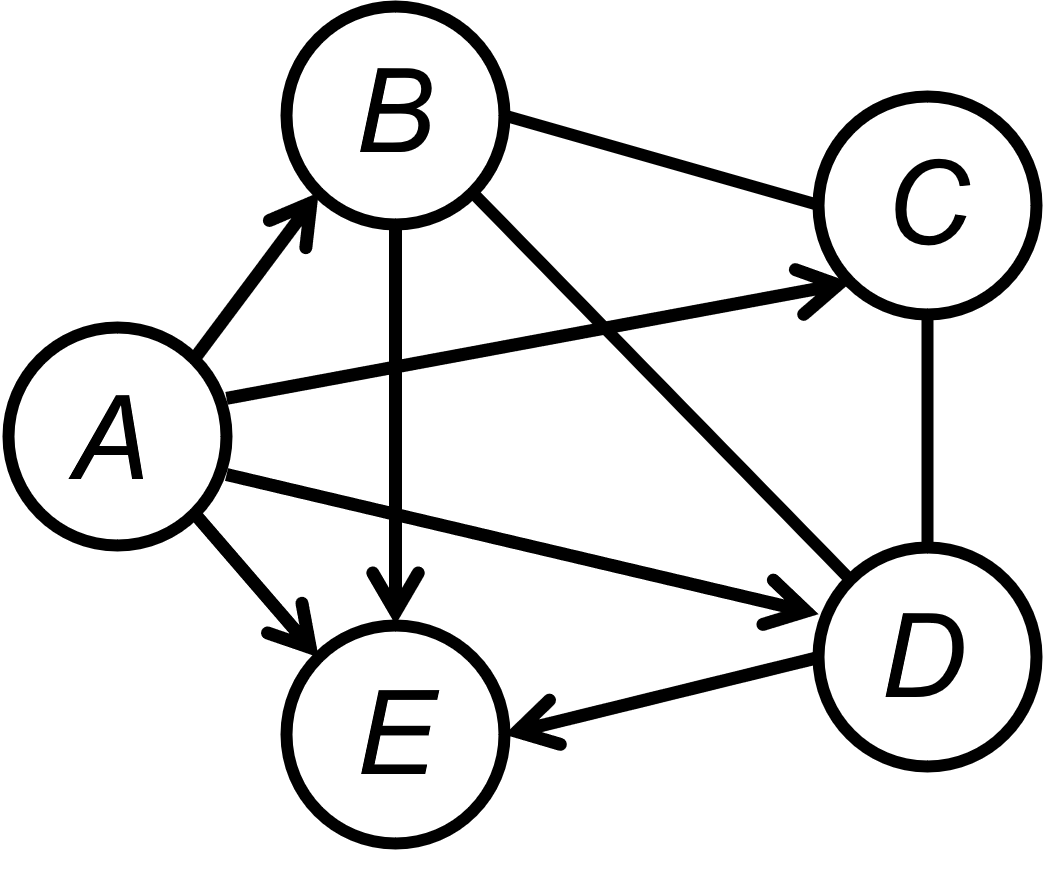}
		\end{minipage}%
	}%
	\hspace{0.02\linewidth}
	\subfloat[A DAG represented by $\cal H$ \label{DAGinMPDAG}]{
		\begin{minipage}[t]{0.3\linewidth}
			\centering
			\includegraphics[width=0.45\linewidth]{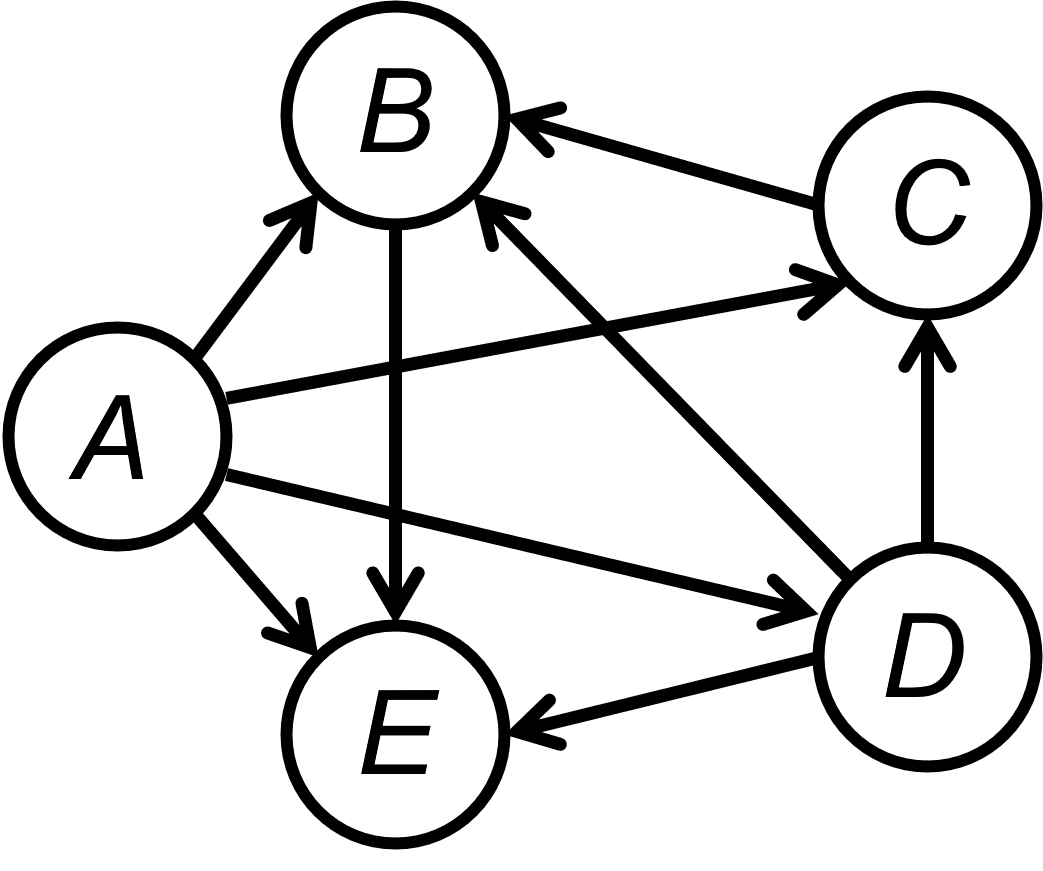}
		\end{minipage}%
	}%
	\caption{Examples to illustrate the graphical characterization of causal MPDAGs.}
	\label{fig:MPDAG}
\end{figure}

\begin{definition}[Chain Skeleton]\label{chainskel}
	Given a partially directed graph $ \cal G $, the chain skeleton of $\cal G$, denoted by ${\cal G}_c$,  is the   graph  obtained from $\cal G$ by  removing  arrowheads of all directed edges in every B-component of $\cal G$.
\end{definition}

According to the definition of a B-component, all undirected edges of $\cal G$ appear in  B-components  of $\cal G$, so the undirected  subgraph of ${\cal G}_c$ is a union of the skeletons of  the B-components  of $\cal G$.
Figure \ref{convertedMPDAG} displays the chain skeleton of $ \cal H $ illustrated in Figure~\ref{MPDAGexample}. In  Figure \ref{convertedMPDAG}, the induced subgraph of the chain skeleton ${\cal H}_c$ over $\{B,C,D\}$ is  undirected.

Theorem \ref{the:MPDAG}
provides sufficient and necessary conditions for a partially directed graph $ \mathcal{H} = (\mathbf{V}, \mathbf{E}) $ to be a causal MPDAG.

\begin{theorem}\label{the:MPDAG}
	A partially directed graph $ \mathcal{H} = (\mathbf{V}, \mathbf{E}) $ is a causal MPDAG if and only if $\mathcal{H}$ satisfies the following conditions.
	\begin{enumerate}
		\item[(i)] The chain skeleton $ {\cal H}_c $ of $\mathcal{H}$ is a chain graph.
		\item[(ii)] The skeleton of each B-component of $ {\cal H}$ is chordal.
		\item[(iii)] The vertices in the same B-component have the same parents in $ {\cal H}_c $.
		\item[(iv)] For any directed edge $ X \rightarrow Y $ in any B-component ${\cal C}^b$ of $ \cal H $, $ pa(X, {\cal H}) \subseteq pa(Y, {\cal H}) \setminus \{X\} $ and $ adj(Y, {\cal C}^b) \setminus \{X\} \subseteq adj(X, {\cal C}^b) $.
	\end{enumerate}
\end{theorem}

Comparing to the graphical characterizations of  essential graphs \citep{andersson1997characterization} and   intervention essential graphs \citep{hauser2012characterization},  the conditions in Theorem  \ref{the:MPDAG} are weaker since these two types of graphs are also  causal MPDAGs. That is, these conditions are necessary but not sufficient for a graph  to be an essential graph  or an intervention essential graph.

In Theorem \ref{the:MPDAG}, condition (i) states a global
characteristic of the partially directed graph $\cal H$, that is, there are no partially directed circles in the chain skeleton ${\cal H}_c$ of  $\cal H$. The last three conditions   characterize the graphical structure related to B-components of $\cal H$. Condition (ii) states that the undirected induced subgraphs of ${\cal H}_c$ are chordal, and condition (iii) shows that a vertex out of a B-component is either a parent of all vertices  in the B-component, or not a parent of any vertex  in the B-component. Condition (iv) indicates that the neighbor and parental sets of the two endpoints of a directed edge in a B-component must satisfy some inclusion relations.



In Section \ref{sec:sec:interpretbg},  a causal MPDAG is defined as an MPDAG that can represent a  restricted Markov equivalence class induced by a CPDAG $\mathcal{G}^*$ and a pairwise causal constraint set ${\cal B}$.  Below, we first show that such a CPDAG $\mathcal{G}^*$ is unique.

\begin{proposition}\label{nonemptyMPDAG}
	Given a causal MPDAG $\cal H$, any restricted Markov equivalence class that can be represented by $\cal H$ is induced by the same CPDAG ${\cal G}^*$ and some pairwise causal constraint set. Moreover, there exists a direct causal constraint set ${\cal B}_d$ such that $[{\cal H}]=[{\cal G}^*, {\cal B}_d]$.
\end{proposition}




Following Proposition~\ref{nonemptyMPDAG}, we introduce the notions of a generator and a minimal generator.

\begin{definition}[Generator and Minimal Generator]
	Let $ \cal H $ be a causal MPDAG and $ {\cal G}^* $ be the unique CPDAG of all restricted Markov equivalence classes that can be represented by $ \cal H $. A direct causal constraint set (or equivalently, a set of directed edges) $ \cal A $ is called a generator of $ \cal H $  if $[{\cal H }]=[{\cal G}^*, {\cal A }]$. A generator $ \cal A $  is called minimal if the number of direct causal constraints in $ \cal A $ is less than or equal to that in any other generator of $\cal H$.
\end{definition}


Proposition \ref{nonemptyMPDAG}  shows that every causal MPDAG has a generator. Below we will show that the minimal generator is unique. A new concept called M-strongly protected is required.



\begin{definition}[M-Strongly Protected]\label{def:strong}
	Let $ \cal H $ be a causal MPDAG. A directed edge $ X \rightarrow Y $ in $ \cal H $ is M-strongly protected  if $ X \rightarrow Y $ occurs in at least one of  the  five configurations in Figure \ref{fig:MSPfive}  as an induced graph of $ \cal H $.
\end{definition}
\begin{figure}[!h]
	\centering
	\subfloat[ \label{MSP-1}]{
		\begin{minipage}[t]{0.1\linewidth}
			\centering
			\includegraphics[width=0.9\linewidth]{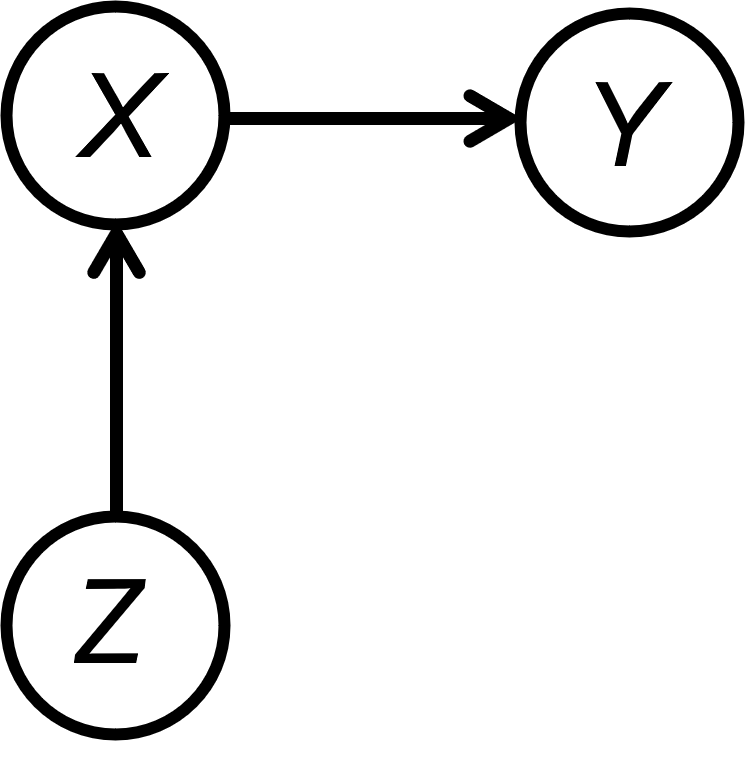}
		\end{minipage}%
	}%
	\hspace{0.08\linewidth}
	\subfloat[  \label{MSP-2}]{
		\begin{minipage}[t]{0.1\linewidth}
			\centering
			\includegraphics[width=0.9\linewidth]{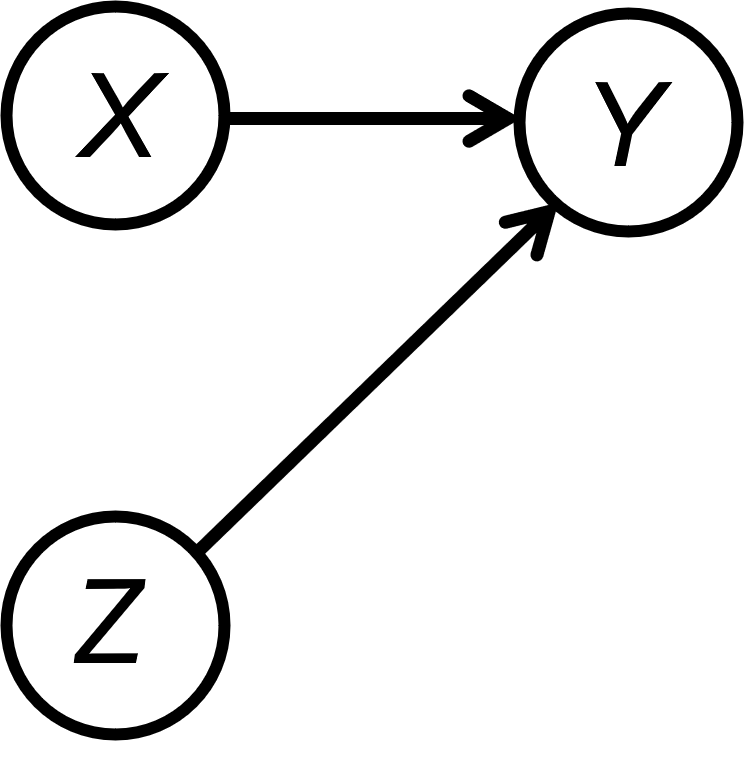}
		\end{minipage}%
	}%
	\hspace{0.08\linewidth}
	\subfloat[  \label{MSP-3}]{
		\begin{minipage}[t]{0.1\linewidth}
			\centering
			\includegraphics[width=0.9\linewidth]{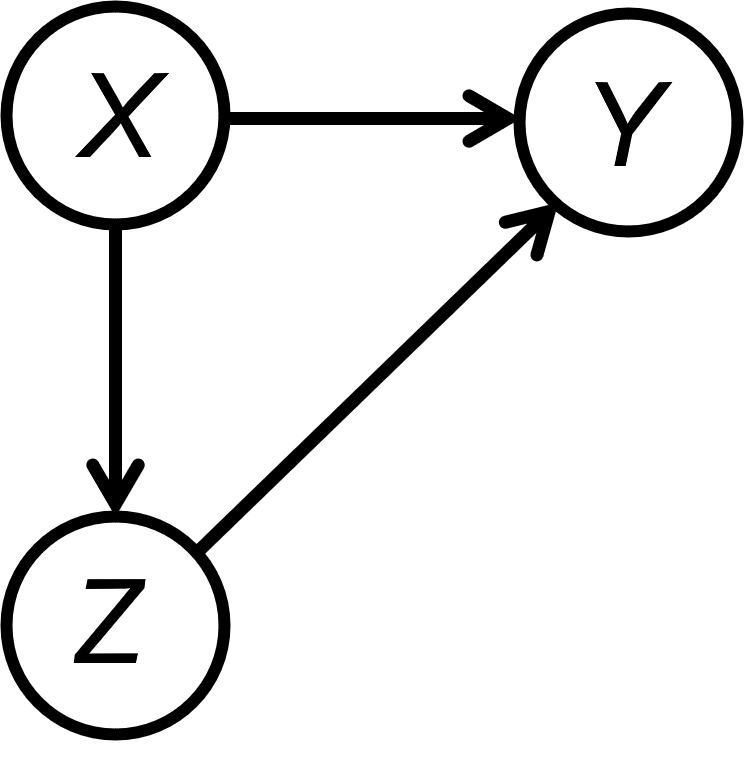}
		\end{minipage}%
	}%
	\hspace{0.08\linewidth}
	\subfloat[ \label{MSP-4}]{
		\begin{minipage}[t]{0.1\linewidth}
			\centering
			\includegraphics[width=0.9\linewidth]{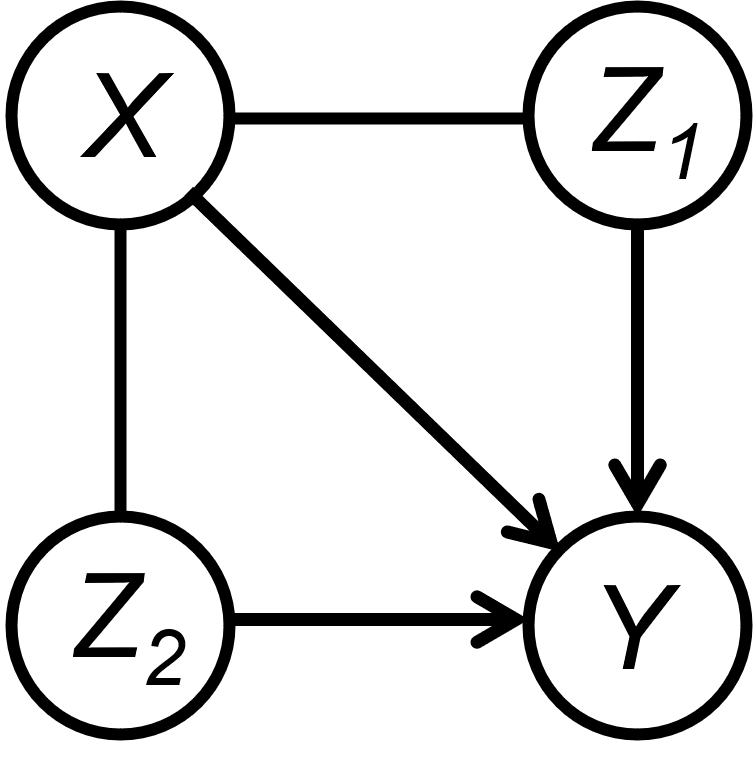}
		\end{minipage}%
	}%
	\hspace{0.08\linewidth}
	\subfloat[ \label{MSP-5}]{
		\begin{minipage}[t]{0.1\linewidth}
			\centering
			\includegraphics[width=0.9\linewidth]{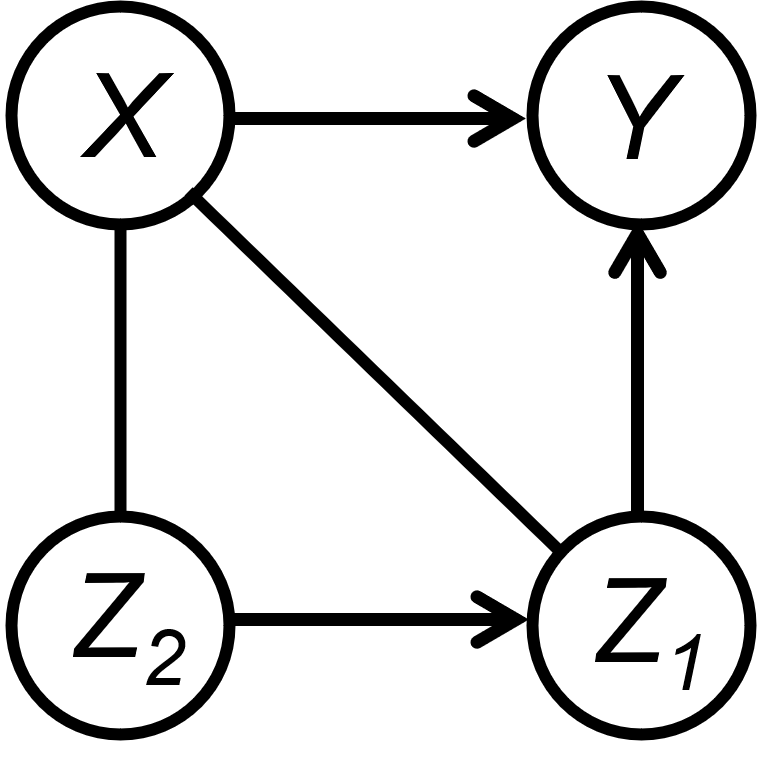}
		\end{minipage}%
	}%
	\caption{The five configurations of M-strongly protected edges. }\label{fig:MSPfive}
	\label{fig:MSP}
\end{figure}



The following proposition establishes the uniqueness of the minimal generator.

\begin{proposition}\label{thm:MSP}
	Given a causal MPDAG $ \cal H $, a set of directed edges $ \cal A $ is a minimal generator of $ \cal H $ if and only if $ \cal A $ is the set of directed edges which are not M-strongly protected in $ \cal H $. Moreover, the minimal generator of $ \cal H $ is unique.
\end{proposition}

Proposition  \ref{thm:MSP} also provides a method to find the minimal generator of a given causal MPDAG.  For example, consider the causal MPDAG $\cal H$ shown in Figure~\ref{MPDAGexample}, the four directed edges $A\to B, A\to E, B\to E$ and $D\to E$ are M-strongly protected and the other three directed edges $A\to C, A\to D$ and $C\to B$   are not, so  ${\cal A}=\{A\to C, A\to D, C\to B\}$ is the unique minimal generator of $\cal H$. In summary, we have the following theorem.

\begin{theorem}\label{thm:MSPrep}
	Let $\cal H$ be a causal MPDAG. Then, there exists a unique CPDAG ${\cal G}^*$ and a unique minimal generator ${\cal B}_m$  such that $[{\cal H}]=[{\cal G}^*, {\cal B}_m]$.
\end{theorem}

The first four configurations of M-strongly protected edges in Definition \ref{def:strong} are exactly the  configurations of strongly protected defined for essential graphs.~\citet{andersson1997characterization} proved that every directed edge in an essential graph is strongly protected. Therefore, strongly protected can be used to figure out the directed edges that can be learned from observational data. In contrast, from Proposition \ref{thm:MSP} and Theorem \ref{thm:MSPrep}, in a  causal MPDAG, a directed edge can be learned from observational data, or can be  inferred from the minimal generator if the directed edge is M-strongly protected.

\section{A Unified Representation of Pairwise Causal Constraints}\label{sec:rep}


Let  ${\mathcal G}^*$ be a CPDAG and  $\cal B$ be a set of pairwise causal constraints  consistent  with ${\mathcal G}^*$.  As discussed in Section \ref{sec:sec:interpretbg}, a causal MPDAG $\cal H$ of $[{\mathcal G}^*,\cal B]$ may not be fully informative when  $\cal B$ contains ancestral causal constraints. That is, $[{\mathcal G}^*,\cal B]$ cannot be represented exactly by any causal MPDAG. In this section, we first introduce a new representation, called direct causal clauses, which can be used to represent all types of pairwise causal constraints, and discuss the consistency and equivalence of direct causal clauses. Then, in Section \ref{sec:sec:info}, we show that any pairwise causal background knowledge set can be equivalently decomposed into a causal MPDAG plus a minimal residual set of direct causal clauses, and prove sufficient and necessary conditions for a causal MPDAG $\cal H$ of $[{\mathcal G}^*,\cal B]$ to be fully informative. Finally, Section~\ref{sec:sec:non-pair} briefly discusses the representation of non-pairwise causal background knowledge.


\subsection{Direct Causal Clauses}\label{sec:sec:clause}

In this section, we develop a non-graphical tool called direct causal clauses, which can uniformly represent direct, ancestral, and non-ancestral causal constraints.


\begin{definition}[Direct Causal Clause]\label{def:dcc}
	A   direct causal clause (DCC for short) $\kappa$, also denoted by $\kappa_t\tor\kappa_h$, over a variable set $\bf V$ is a proposition saying that $\kappa_t$  is a direct cause of   at least one variable  in   $\kappa_h$, where ${\kappa_t}\in {\bf V}$ is called  the tail of $\kappa$,  and ${\kappa_h}\subseteq{\bf V}$ satisfying $\kappa_t\notin \kappa_h$ is the head set  whose elements are called the heads of $\kappa$.
	
\end{definition}

When  the head set $\kappa_h$ of a DCC $\kappa$  is a singleton set, say $\kappa_h=\{D\}$, $\kappa_t\tor\kappa_h$ is equivalent to the proposition that $\kappa_t$ is a direct cause of $D$, denoted by $\kappa_t\to D$.
For ease of presentation, we will use  $\kappa_t\tor D$ or  $ \kappa_t \to {D}$ as a shorthand for $\kappa_t\tor \{D\}$, and we will use them interchangeably throughout this paper.    Given a DAG $\mathcal{G}$ and a DCC $\kappa$  over the same variable set $\mathbf{V}$,  we say that $\kappa$  holds  for $\cal G$ and $\cal G$ satisfies $\kappa$ if $\kappa_h \cap  ch({\kappa}_t, {\cal G})\neq \varnothing$.


\begin{proposition}\label{fact1}
	
	For any DAG $\cal G$  over  $\mathbf V$, we have that {(i)}
	a DCC $\kappa$ over $\mathbf V$ with $\kappa_h=\varnothing$ never holds for $\cal G$, and {(ii)}
	for any DCC $\kappa$ over $\mathbf V$,  $\kappa \iff   \bigvee_{D\in \kappa_h}  (\kappa_t \to D)$ for $\cal G$.
\end{proposition}


The first statement of Proposition \ref{fact1}  naturally holds since for any $\kappa$ with $\kappa_h=\varnothing$, $\kappa_h \cap  ch({\kappa}_t, {\cal G})= \varnothing$, no matter whether $ch({\kappa}_t, {\cal G})=\varnothing$ or not. Even if for a singleton graph ${\cal G}=(\kappa_t, \varnothing)$ containing $\kappa_t$ only, $\kappa_t\tor \varnothing$ does not hold for ${\cal G}$. Proposition \ref{fact1} also shows  that a DCC $\kappa$ holds for $\cal G$  if and only if there exists at least one variable $D\in \kappa_h$ such that $\kappa_t\to D$  holds for $\cal G$. It implies that a DCC is a disjunction of  direct causal constraints.

To demonstrate that DCCs can represent pairwise causal constraints, we introduce the concept of a critical set as follows.

\begin{definition}[Critical Set]\label{critical}
	Let $\mathcal{G}^*$ be  a causal MPDAG, and let $X$ and $Y$ be two distinct vertices in $\mathcal{G}^*$. The critical set of $X$ with respect to $Y$ in $\mathcal{G}^*$, denoted by $\mathbf{C}_{XY}(\mathcal{G}^*)$, consists of all neighbors of $X$ lying on at least one chordless partially directed path or chordless undirected path from $X$ to $Y$.
\end{definition}

This concept was first introduced for CPDAGs by~\citet{fang2020bgida}.  As an example, consider the CPDAG shown in Figure~\ref{fig:1-1}. The critical set of $A$ with respect to $Y$ consists of $B$ and $C$, as $A-B\to Y$ and $A-C\to Y$ are two chordless partially directed path from $A$ to $Y$. On the other hand, variable $X$ is not in $\mathbf{C}_{AY}(\mathcal{G}^*)$, since $A-X-C\to Y$ has a chord $A-C$ and $A-X-B\to Y$ has a chord $A-B$. Similarly, the critical set of $D$ with respect to $Y$ is $\{A,X\}$, for $D-A-B\to Y$ and $D-X-C\to Y$ are chordless.


\begin{theorem}
	\label{thm:nbr_set_const}	
	Let $\mathcal{G}^*$ be a CPDAG, $X, Y \in \mathbf{V}(\mathcal{G}^*)$, and $\mathcal{G}\in [\mathcal{G}^*]$. Denote by $\mathbf{C}_{XY}(\mathcal{G}^*)$ the critical set of $X$ with respect to $Y$ in $\mathcal{G}^*$. Then, we have:
	\begin{enumerate}
		\item[(i)]  $X$ is a direct cause of $Y$ in $\mathcal{G}$ if and only if $X\tor Y$ holds for $\mathcal{G}$.
		\item[(ii)]   $X$ is a cause of $Y$ in $\mathcal{G}$ if and only if $X \tor \mathbf{C}_{XY}(\mathcal{G}^*)$ holds for $\mathcal{G}$.
		\item[(iii)]   $X$ is not a cause of $Y$ in $\mathcal{G}$ if and only if $C\tor X$ holds for $\mathcal{G}$ for every $C\in\mathbf{C}_{XY}(\mathcal{G}^*)$.
	\end{enumerate}
\end{theorem}

Theorem~\ref{thm:nbr_set_const} extends Lemma~2 in~\citet{fang2020bgida} to all pairwise causal constraints.
Analogue to Definition~\ref{interpretbg:reseqclassdef}, we can define the restricted Markov equivalence class $[\mathcal{G}^*, \mathcal{K}]$ induced by a CPDAG $\mathcal{G}^*$ and a set $\mathcal{K}$ of DCCs over ${\bf V}(\mathcal{G}^*)$, as the subset of $[{\cal G^*}]$ in which every DAG satisfies all DCCs in $\mathcal{K}$. Likewise, the MPDAG of a non-empty $[\mathcal{G}^*, \mathcal{K}]$ can also be defined analogously to Definition~\ref{interpretbg:mpdagresclass}.
Given a CPDAG and a set of pairwise causal constraints, Theorem~\ref{thm:nbr_set_const} proves that there is a set of DCCs which induces the same restricted Markov equivalence class as the original constraints. 

From  Theorem~\ref{thm:nbr_set_const}, the global path constraints  are transformed to the local ones that only put constraints on the edges between $X$ and its neighbors. A polynomial-time algorithm proposed by  \citet[Algorithm~2]{fang2021local} can be used to   find critical sets,  so   we can efficiently obtain the equivalent DCCs from a given CPDAG and pairwise causal constraints.

{We remark that any pairwise causal constraint can be equivalently transformed into a set of DCCs, but not vice versa. That is, not every set of DCCs can be translated back into pairwise causal constraints. An illustrative example is provided in Figure~\ref{fig:not}, where the DCC $B \tor \{A, C\}$ is not equivalent to any combination of pairwise causal constraints. This example demonstrates that DCCs can encode more information than pairwise constraints. Assuming that a set of DCCs can be translated back into pairwise causal constraints, Appendix~\ref{sec:sec:from} presents a method for obtaining an equivalent set of pairwise causal constraints.}

\begin{figure}[!t]
	\centering	
	\includegraphics[width=0.25\linewidth]{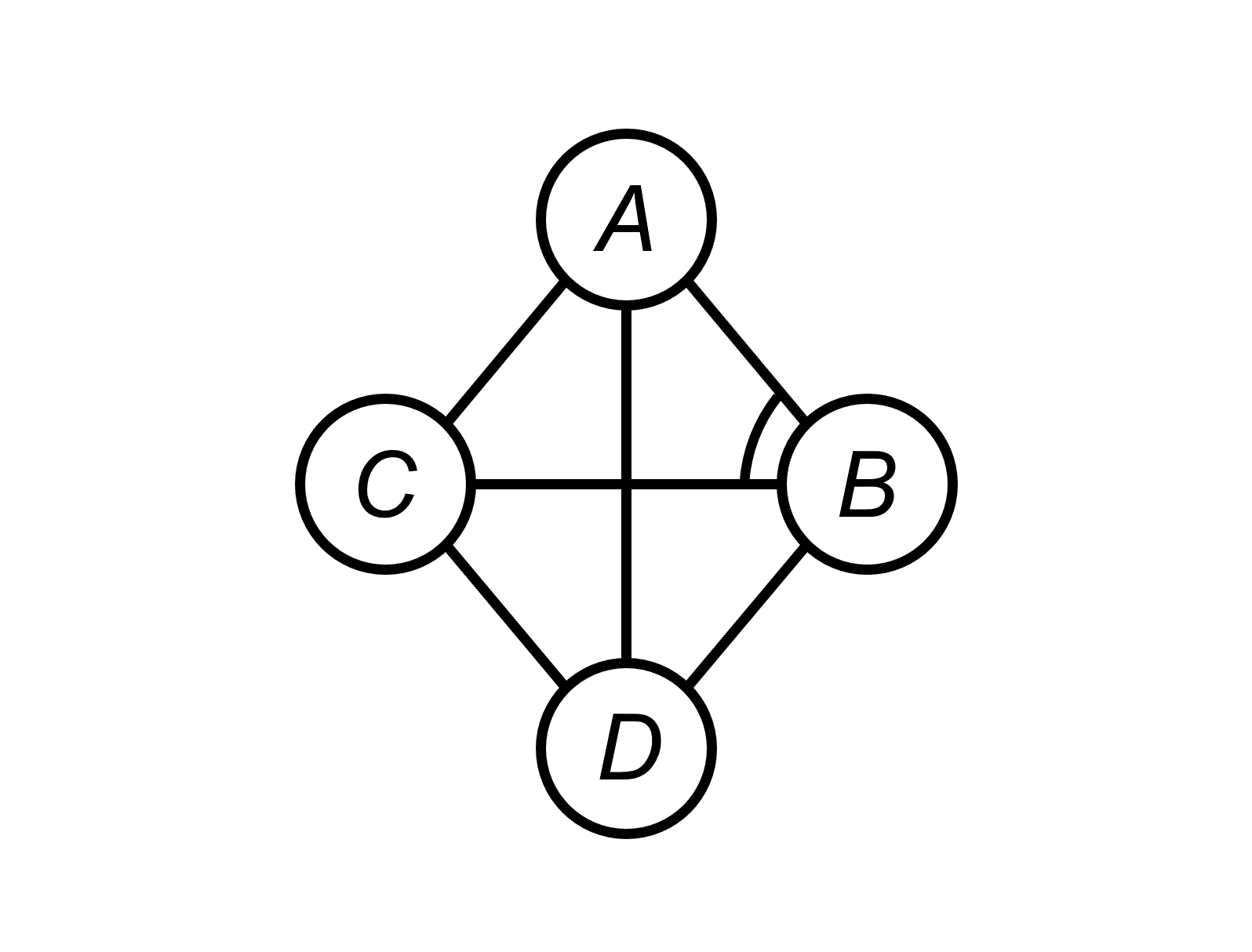}
	\caption{A CPDAG over $\{A,B,C,D\}$ and a DCC $B \tor \{A, C\}$, which is visualized by an arc. This example shows that not every DCC can be translated back into pairwise causal constraints. }
	\label{fig:not}
\end{figure}

\begin{example}
	\label{ex:cbk_vs_nbr}
	Consider the CPDAG ${\cal G}^*$ shown in Figure~\ref{fig:1-1}. Recall that $\mathbf{C}_{AY}(\mathcal{G}^*)=\{B, C\}$ and $\mathbf{C}_{DY}(\mathcal{G}^*)=\{A, X\}$. Suppose that $A\longarrownot\dashrightarrow Y$ holds for the underlying DAG, then by Theorem~\ref{thm:nbr_set_const}, $B\tor A$ and $C\tor A$ holds. Since $B\tor A \iff B\to A$ and $C\tor A \iff  C\to A$, we can represent $A\longarrownot\dashrightarrow Y$ graphically by orienting $B \to A$ and $C\to A$ in ${\cal G}^*$. Furthermore, if $D \dashrightarrow Y$ holds, then by Theorem~\ref{thm:nbr_set_const} we equivalently have $D\tor \{A, X\}$, meaning that $D$ is a direct cause of $A$ or $X$ in the underlying DAG. Figures \ref{fig:1-4} to \ref{fig:1-6} enumerate three possible orientations of the edges $D-A$ and $D-X$ in ${\cal G}^*$ with the constraint $D\tor \{A, X\}$. For any DAG in $[{\cal G}^*, D\dashrightarrow Y]$,  the edge orientations of $D-A$ and $D-X$ must be one of the three possibilities shown in Figures \ref{fig:1-4} to \ref{fig:1-6}. Conversely, every DAG in $[{\cal G}^*]$ whose edges between $D$ and $\{A, X\}$ are identical to one of the configurations shown in Figures \ref{fig:1-4} to \ref{fig:1-6} must satisfy the constraint $D\dashrightarrow Y$.
\end{example}

\begin{figure}[!t]
	\centering
	\subfloat[CPDAG ${\cal G}^*$ \label{fig:1-1}]{
		\begin{minipage}[t]{0.23\linewidth}
			\centering
			\includegraphics[width=1\linewidth]{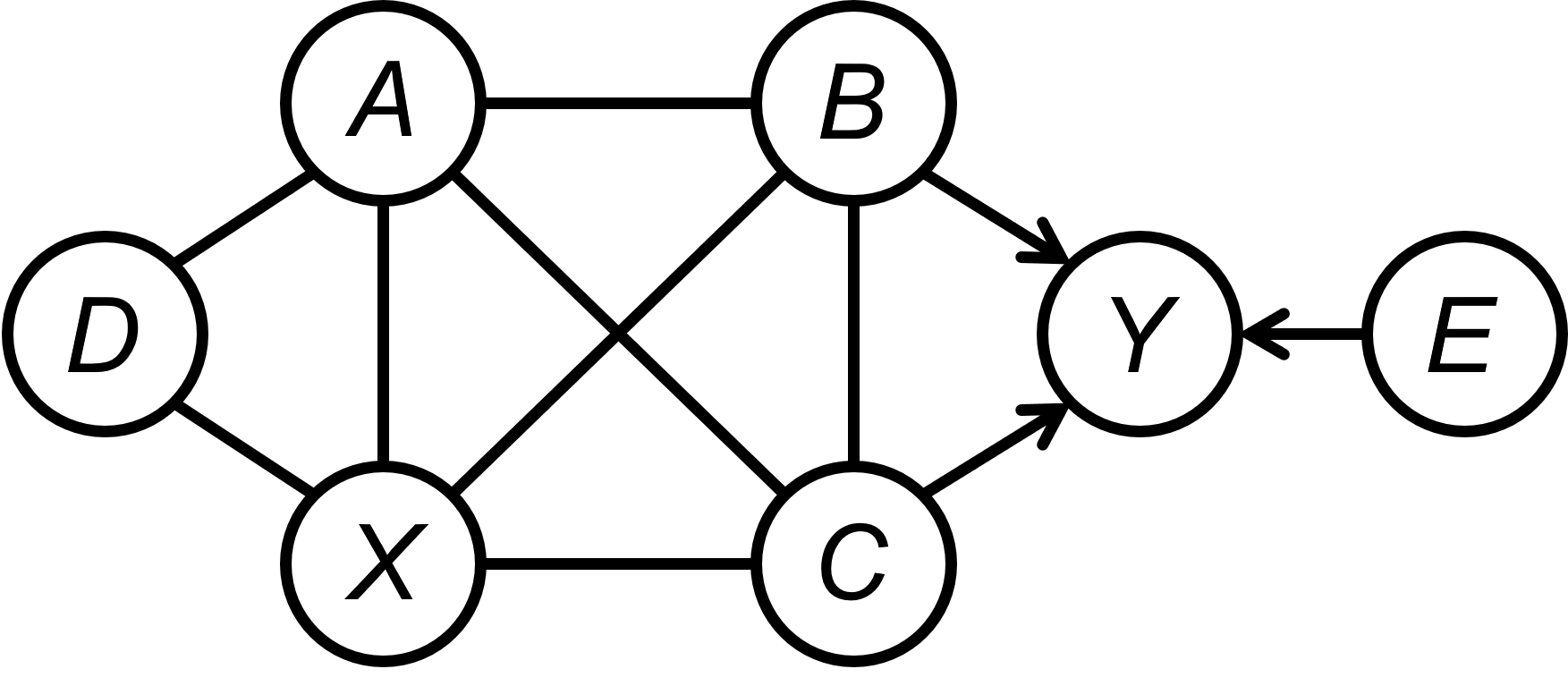}
		\end{minipage}%
	}%
	\hspace{0.1\linewidth}
	\subfloat[$B\tor A$ and $C\tor A$ \label{fig:1-2}]{
		\begin{minipage}[t]{0.23\linewidth}
			\centering
			\includegraphics[width=1\linewidth]{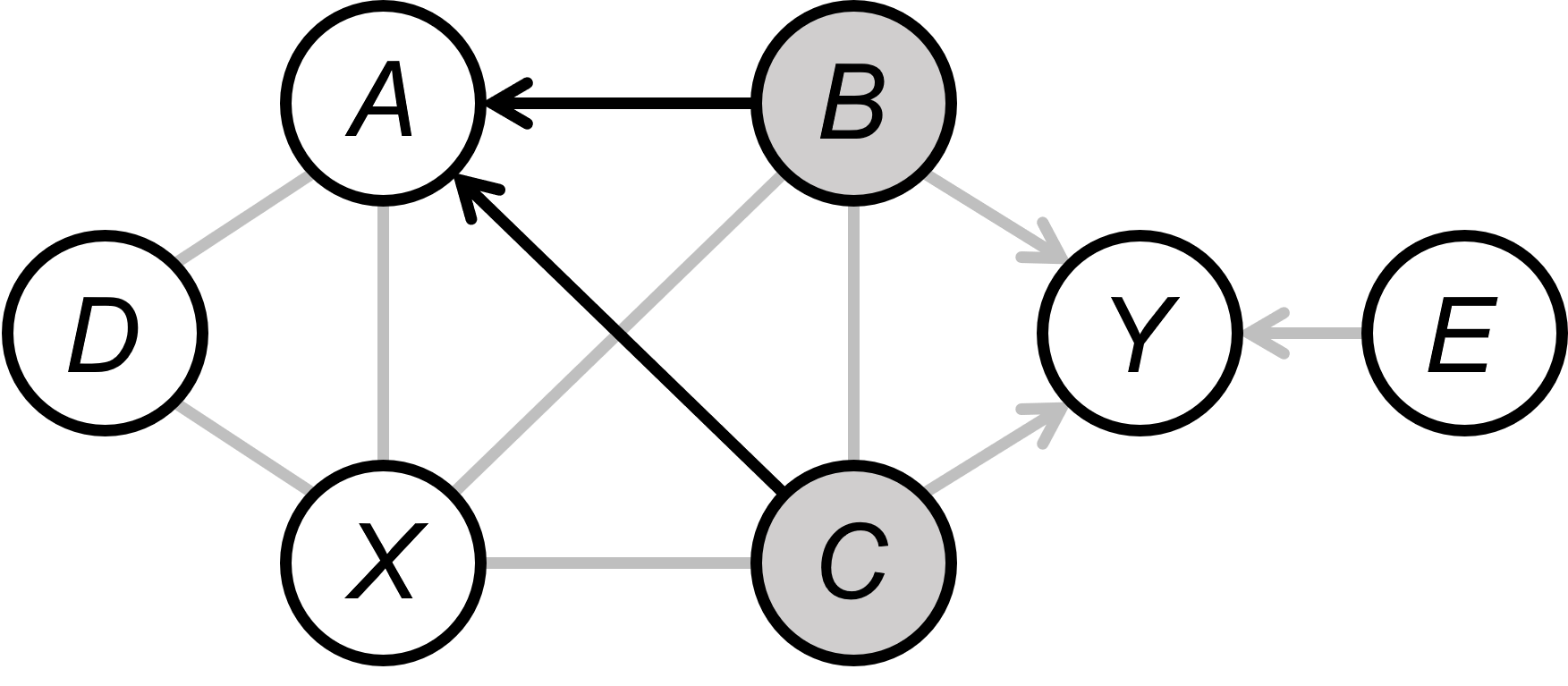}
		\end{minipage}%
	}%
	\hspace{0.1\linewidth}
	\subfloat[$D\tor \{A, X\}$ \label{fig:1-3}]{
		\begin{minipage}[t]{0.23\linewidth}
			\centering
			\includegraphics[width=1\linewidth]{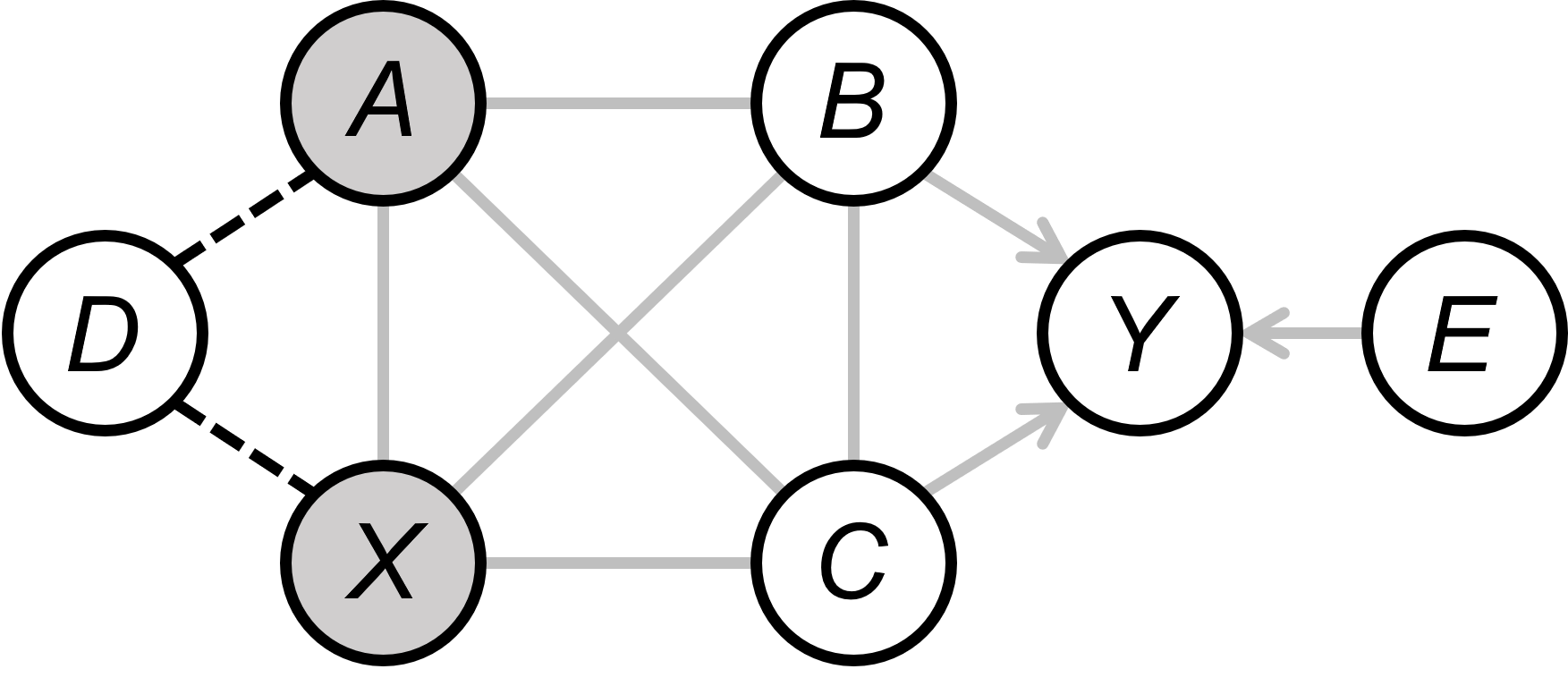}
		\end{minipage}%
	}%
	
	\subfloat[$D\to A$ and $D\to X$ \label{fig:1-4}]{
		\begin{minipage}[t]{0.23\linewidth}
			\centering
			\includegraphics[width=1\linewidth]{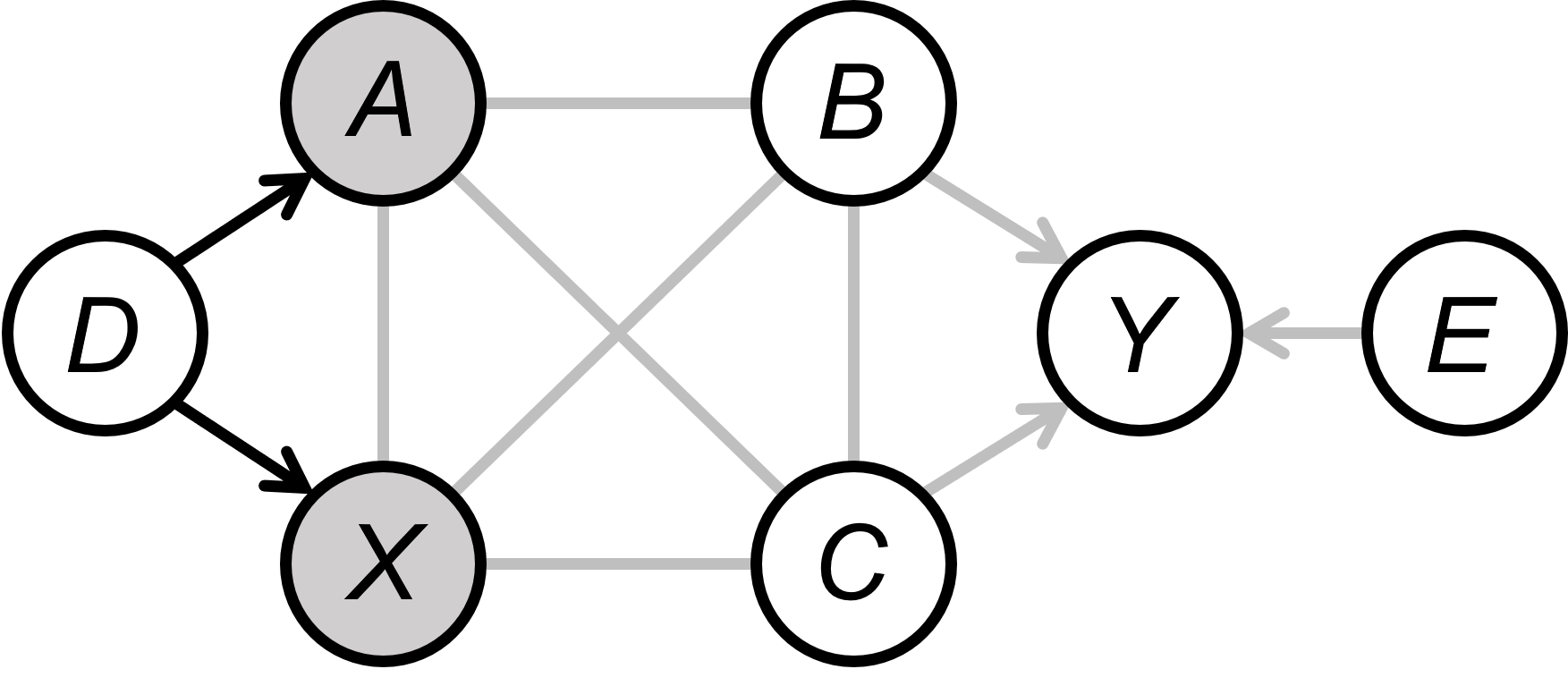}
		\end{minipage}%
	}%
	\hspace{0.1\linewidth}
	\subfloat[$D\to A$ and $X\to D$ \label{fig:1-5}]{
		\begin{minipage}[t]{0.23\linewidth}
			\centering
			\includegraphics[width=1\linewidth]{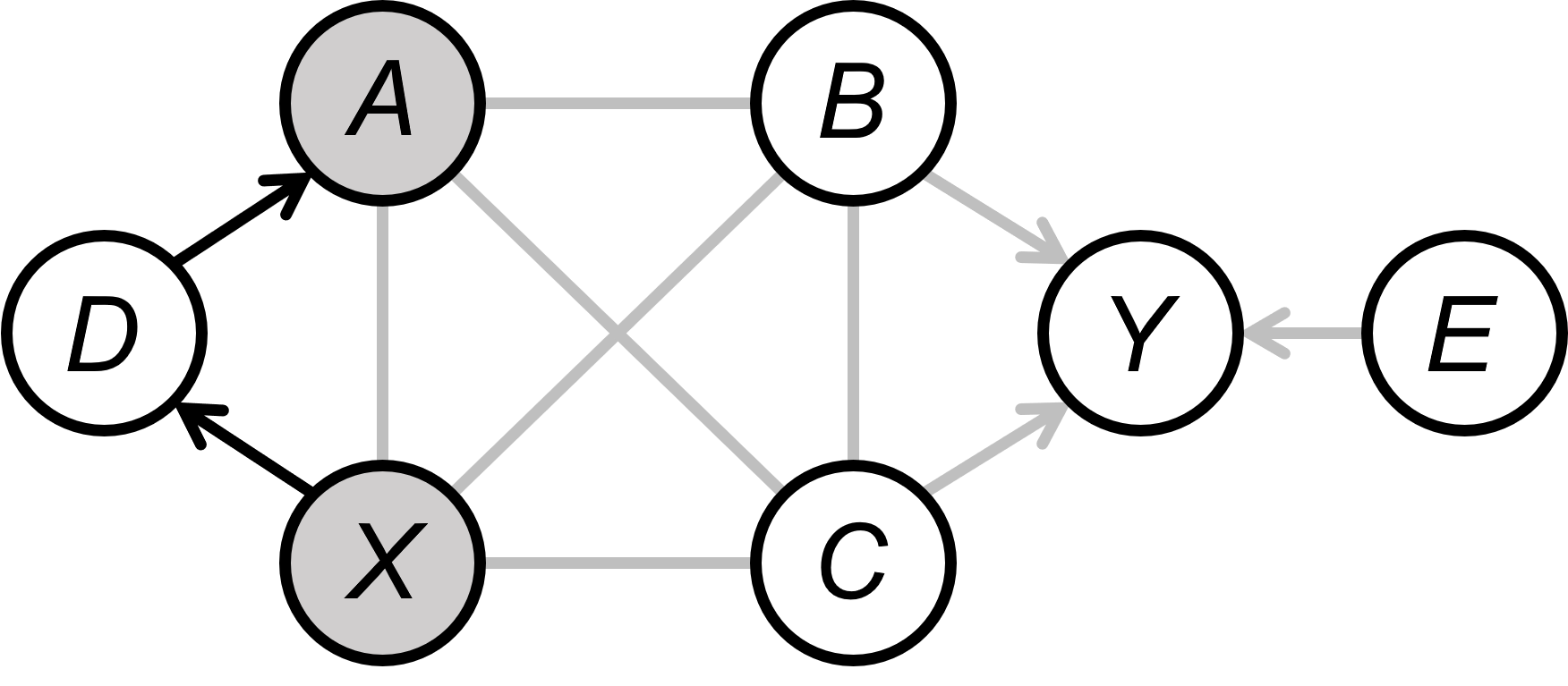}
		\end{minipage}%
	}%
	\hspace{0.1\linewidth}
	\subfloat[$A\to D$ and $D\to X$ \label{fig:1-6}]{
		\begin{minipage}[t]{0.23\linewidth}
			\centering
			\includegraphics[width=1\linewidth]{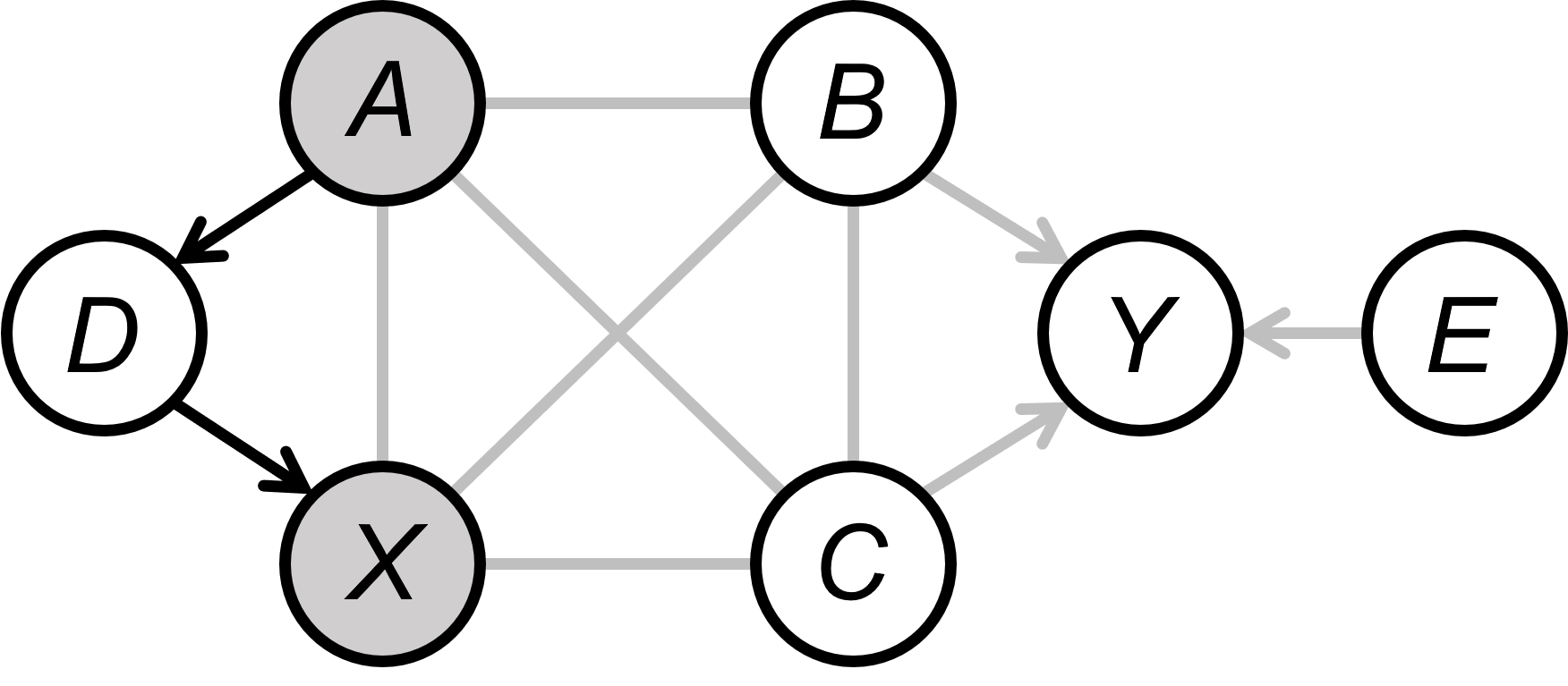}
		\end{minipage}%
	}%
	\caption{An illustration of a critical set and  Theorem~\ref{thm:nbr_set_const}. Figure~\ref{fig:1-1} shows a CPDAG. Figures~\ref{fig:1-2} and \ref{fig:1-3} show the equivalent DCCs to the pairwise causal background knowledge $A\longarrownot\dashrightarrow Y$ and $D\dashrightarrow Y$, respectively. Figures \ref{fig:1-4} to \ref{fig:1-6} enumerate three possible orientations of the edges $D-A$ and $D-X$ in ${\cal G}^*$ with the constraint $D\tor \{A,X\}$.}
	\label{fig:thm1}
\end{figure}

Below, we discuss the consistency of a DCC set and the equivalence of two DCC sets with respect to a given CPDAG. We first introduce the definition of consistency as follows.

\begin{definition}[Consistency] \label{def:consistency}
	Given a CPDAG $\mathcal{G}^*$ and a set $\mathcal{K}$ of DCCs over ${\bf V}(\mathcal{G}^*)$, $\mathcal{K}$ is consistent with $\mathcal{G}^*$ if $[\mathcal{G}^*, \mathcal{K}]\neq \varnothing$.
\end{definition}

Following Definition \ref{def:consistency}, a  set $\mathcal{K}$ of DCCs over ${\bf V}(\mathcal{G}^*)$ is consistent with $\mathcal{G}^*$ if and only if there exists at least one equivalent DAG ${\cal G}\in [{\cal G}^*]$ which satisfies all clauses in
$\cal K$.  Below, we give a definition of equivalence for two sets of DCCs, which allows us to discuss the consistency of a DCC set via its equivalent reduced form.




\begin{definition}[Equivalence] \label{def:equiv}
	Given a CPDAG $\mathcal{G}^*$,  two sets $\mathcal{K}_1$ and $\mathcal{K}_2$ of DCCs  over ${\bf V}(\mathcal{G}^*)$  are equivalent with respect to $\mathcal{G}^*$ if $[\mathcal{G}^*, \mathcal{K}_1]=[\mathcal{G}^*, \mathcal{K}_2]$.
\end{definition}

Let $\kappa:= \kappa_t\tor \kappa_h$ be a DCC. Denoting  $[\mathcal{G}^*, \{\kappa\}]$   by $[\mathcal{G}^*, \kappa]$    for convenience and assuming that $\kappa$ is over ${\bf V}(\mathcal{G}^*)$, the following proposition gives several  equivalent reduced forms of $\kappa$.

\begin{proposition}\label{fact3} For any CPDAG $\mathcal{G}^*$ and any direct causal clause $\kappa:= \kappa_t\tor \kappa_h$, we have that {(i)}  $[\mathcal{G}^*, \kappa]=[\mathcal{G}^*]$ if $ch(\kappa_t,\mathcal{G}^*)\cap \kappa_h\neq \varnothing $,
	{(ii)} $[\mathcal{G}^*, \kappa]=[\mathcal{G}^*,\kappa_t\tor (\kappa_h\cap adj(\kappa_t,\mathcal{G}^*))]$,
	and {(iii)} $[\mathcal{G}^*, \kappa]=[\mathcal{G}^*,\kappa_t\tor (\kappa_h\setminus pa(\kappa_t,\mathcal{G}^*)) ]$.
\end{proposition}

In Proposition \ref{fact3}, the first result holds since every ${\cal G}\in [\mathcal{G}^*]$ satisfies $ \kappa$ when $ch(\kappa_t,\mathcal{G}^*) \cap \kappa_h\neq \varnothing $. Therefore, a DCC with $ch(\kappa_t,\mathcal{G}^*)\cap \kappa_h\neq \varnothing $ is redundant for ${\cal G}^*$.
For any ${\cal G}\in [\mathcal{G}^*]$, $\cal G$ satisfies $\kappa$ if and only if there exists at least a variable $D\in \kappa_h$ such that $\kappa_t\to D$ appears in $\cal G$. If such a variable $D$ exists, it must be adjacent to $\kappa_t$, so  the second equation  holds. Similarly, the third equation  holds since $D$ must not a direct cause of $\kappa_t$. 
Consequently, we give a reduce form of a set of DCCs $\cal K$ in Definition \ref{def:reducedf}.



\begin{definition}[Reduced Form]\label{def:reducedf}
	Given a CPDAG  ${\cal G}^*$ and a set  of DCCs $\cal K$ over ${\bf V}(\mathcal{G}^*)$, a reduced form of $\cal K$ with respect to ${\cal G}^*$, denoted by ${\cal K}(\mathcal{G}^*)$, is defined as follows.
	\begin{equation}\label{eq:n_u}
		{\cal K}(\mathcal{G}^*)\coloneqq\{\kappa_t\tor \left(\kappa_h \cap sib(\kappa_t, \mathcal{G}^*)\right)  \mid  \kappa \in {\cal K} \;\text{and}\; \kappa_h\cap ch(\kappa_t, \mathcal{G}^*)=\varnothing\}.
	\end{equation}
	Specifically, ${\cal K}(\mathcal{G}^*)\coloneqq\varnothing$ if ${\cal K} = \varnothing$.
\end{definition}



\begin{proposition}[Equivalent Reduced Form]\label{pro:reducedf}
	Given a CPDAG $\mathcal{G}^*$ and a set of DCCs $\cal K$ over ${\bf V}(\mathcal{G}^*)$, we have that $\cal K$ is equivalent to ${\cal K}(\mathcal{G}^*)$ defined in Equation (\ref{eq:n_u})  with respect to $\mathcal{G}^*$.
\end{proposition}

Proposition \ref{pro:reducedf} shows that the reduced form of  $\cal K$  is equivalent to $\cal K$ with respect to ${\cal G}^*$.  Below, we define a subset  of $\cal K$ restricted on an undirected induced subgraph of ${\cal G}^*$.

\begin{definition}[Restriction Subset]\label{def:restrict}
	Suppose that $\mathcal{U}$ is an undirected induced subgraph of a CPDAG $\mathcal{G}^*$ over $\mathbf{V}(\mathcal{U})\subseteq {\bf V}(\mathcal{G}^*)$, and $\cal K$ is a set of DCCs over ${\bf V}(\mathcal{G}^*)$. The restriction subset of $\cal K$   on $\mathcal{U}$ is defined by
	\begin{equation}\label{eq:n_uu}
		{\cal K}(\mathcal{U})\coloneqq\{\kappa\in {\cal K}(\mathcal{G}^*) \mid \{\kappa_t\}\cup\kappa_h\subseteq \mathbf{V}(\mathcal{U})\}.
	\end{equation}
\end{definition}

It can be seen that ${\cal K}(\mathcal{G}^*)={\cal K}(\mathcal{G}^*_u)$. Basically, ${\cal K}(\mathcal{U})$ consists of all clauses in ${\cal K}(\mathcal{G}^*)$ whose tail and heads are all in $\mathcal{U}$. The following concept of a potential leaf node is of key importance in checking consistency of a set of DCCs.

\begin{definition}[Potential Leaf Node]\label{def:pln}
	Let $\mathcal{G}^*$ be a CPDAG and $\cal K$ be a set of DCCs over ${\bf V}(\mathcal{G}^*)$. Given an undirected induced subgraph $\mathcal{U}$ of $\mathcal{G}^*$ and a vertex $X$ in $\mathcal{U}$, $X$ is called a potential leaf node in $\mathcal{U}$ with respect to $\cal K$ and $\mathcal{G}^*$, if $X$ is a simplicial vertex in $\mathcal{U}$ and $X$ is not the tail of any clause in ${\cal K}(\mathcal{U})$.
\end{definition}

We note that if $\mathcal{U}=(\{X\}, \varnothing)$ only contains a singleton $X$, then $X$ is trivially a potential leaf node in  $\mathcal{U}$ with respect to any $\cal K$ that does not contain any DCC of the form $X\tor\varnothing$. A leaf node of a directed graph is a vertex that has no child. Analogously, a potential leaf node defined above is a vertex that may have no child in at least one ${\cal G}$ in $[\mathcal{G}^*, {\cal K}]$ (see Lemma~\ref{lem:app-pln} in Appendix for more details). Now, we present the sufficient and necessary condition for a set of DCCs $\cal K$ to be consistent with ${\cal G}^*$.

\begin{theorem}\label{thm:consistency}
	Let $\mathcal{G}^*$ be a CPDAG and $\cal K$ be a set of DCCs over ${\bf V}(\mathcal{G}^*)$. Then, the following two statements are equivalent.
	\begin{enumerate}
		\item[(i)]  $\cal K$ is consistent with $\mathcal{G}^*$.
		\item[(ii)]  {\rm (Potential-leaf-node condition)} Any connected undirected induced subgraph of $\mathcal{G}^*$ has a potential leaf node with respect to ${\cal K}$ and $\mathcal{G}^*$.
	\end{enumerate}
\end{theorem}

The proof of Theorem~\ref{thm:consistency} also motivates a polynomial-time algorithm for checking consistency of a set of DCCs. The details are given in Section~\ref{sec:sec:algoconsistent}. The potential-leaf-node condition given in Theorem~\ref{thm:consistency} is similar to the fact that any induced subgraph of a DAG has a leaf node. Below we give an example to demonstrate this result.

\begin{example} 
	Recall that in Example \ref{ex:cbk_vs_nbr} we show that, with respect to the CPDAG ${\cal G}^*$ (Figure \ref{fig:1-1}), $A\longarrownot\dashrightarrow Y$ is equivalent to $C \tor A$ and $B \tor A$, and $D\dashrightarrow Y$ is equivalent to $D\tor \{A, X\}$. Suppose that we have both $A\longarrownot\dashrightarrow Y$ and $D\dashrightarrow Y$, then the equivalent DCCs $\cal K$ consists of $C \tor A$, $B \tor A$ and $D \tor \{A, X\}$. However, since $\{D, A, X, B\}$ induces an undirected subgraph where none of the vertices is a potential leaf node with respect to $\cal K$ and ${\cal G}^*$, $\cal K$ is inconsistent by Theorem \ref{thm:consistency}.
\end{example}

Next, we give    sufficient and necessary conditions under which two DCC sets, ${\cal K}_1$ and ${\cal K}_2$, are equivalent with respect to a CPDAG $\mathcal{G}^*$. We will show that these conditions can be expressed both in terms of consistency and redundancy, where the latter is defined as follows.


\begin{definition}[Redundancy] \label{def:eqbg}
	Given a restricted Markov equivalence class $[{\cal G}^*,{\cal K}]$ induced by a CPDAG ${\cal G}^*$ and a set ${\cal K}$ of DCCs  over ${\bf V}(\mathcal{G}^*)$, a DCC $\kappa$  over ${\bf V}(\mathcal{G}^*)$ is   redundant with respect to $[{\cal G}^*,{\cal K}]$ if  $[{\cal G}^*,{\cal K}]= [{\cal G}^*,{\cal K}\cup \{\kappa\}]$. {A set ${\cal K}$ of DCCs is redundant with respect to a CPDAG ${\cal G}^*$ if there exists at least one $\kappa\in{\cal K}$ that is redundant with respect to $[{\cal G}^*,{\cal K}\setminus\{\kappa\}]$. Otherwise, the set ${\cal K}$ is non-redundant.}
\end{definition}

{
	In the remainder of the paper, when the context is clear, we may occasionally omit ${\cal G}^*$ and simply say that “a DCC $\kappa$ is redundant with respect to ${\cal K}$” for brevity. If $\kappa_t\to s$ appears in ${\cal G}^*$ for some $s\in\kappa_h$, then $\kappa$ is redundant with respect to any DCC set. If $\cal K$ is inconsistent with ${\cal G}^*$, then any DCC $\kappa$ is redundant with respect to $[{\cal G}^*,{\cal K}]$, as $[{\cal G}^*,{\cal K}]= [{\cal G}^*,{\cal K}\cup \{\kappa\}]=\varnothing$. However, an inconsistent DCC set ${\cal K}$ itself may not be redundant.
}

According to Definition \ref{def:eqbg}, a DCC $\kappa$ is  redundant with respect to  $[{\cal G}^*,{\cal K}]$ if and only if  $\kappa$ holds for all DAGs in  $[{\cal G}^*,{\cal K}]$. With this concept, the following Theorem~\ref{thm:equi_dcc} discusses the equivalence of two sets of DCCs.

\begin{theorem}\label{thm:equi_dcc}
	Given a CPDAG $\mathcal{G}^*$ and two sets of DCCs ${\cal K}_1$ and ${\cal K}_2$ over ${\bf V}(\mathcal{G}^*)$, the following statements are equivalent.
	\begin{enumerate}
		\item[(i)]   ${\cal K}_1$ and ${\cal K}_2$ are equivalent given $\mathcal{G}^*$.
		\item[(ii)]   Every DCC  in ${\cal K}_1$, if exists, is redundant with respect to $[\mathcal{G}^*, {\cal K}_2]$, and every DCC in ${\cal K}_2$, if exists, is redundant with respect to $[\mathcal{G}^*, {\cal K}_1]$.
		\item[(iii)]  For every $\kappa\in {\cal K}_1$,  $\cup_{D\in {\kappa_h}}\{D\to \kappa_t\}\cup {\cal K}_2$ is not consistent with ${\cal G}^*$, and for every $\kappa\in {\cal K}_2$,  $\cup_{D\in {\kappa_h}}\{D\to \kappa_t\}\cup {\cal K}_1$ is not consistent with ${\cal G}^*$.
	\end{enumerate}
\end{theorem}

Theorem~\ref{thm:equi_dcc} builds the relations among equivalence, redundancy and consistency. Moreover, using the third statement of Theorem~\ref{thm:equi_dcc}, we can determine whether two sets of DCCs are equivalent by checking the consistency of a series of DCC sets.

{
	
	We end this subsection by showing that two equivalent non-redundant consistent DCC sets must have the same number of DCCs. The following definition serves as a prerequisite.
	
	\begin{definition}[Minimal Redundancy]\label{def:minimal_dcc}
		Given a restricted Markov equivalence class %
		\linebreak $[{\cal G}^*,{\cal K}]$ induced by a CPDAG ${\cal G}^*$ and a consistent set ${\cal K}$ of DCCs over ${\bf V}(\mathcal{G}^*)$, a DCC $\kappa$  over ${\bf V}(\mathcal{G}^*)$ is minimally redundant with respect to $[{\cal G}^*,{\cal K}]$ if $\kappa$ is redundant with respect to $[{\cal G}^*,{\cal K}]$, but either $|\kappa_h| = 1$, or for any proper subset ${\bf s} \subsetneq \kappa_h$, the DCC $\kappa_t \tor {\bf s}$ is not redundant with respect to $[{\cal G}^*, {\cal K}]$.
	\end{definition}
	
	With this definition, we give the following theorem that summarizes key properties of   equivalent non-redundant consistent DCC sets over a CPDAG.


	\begin{theorem}\label{thm:equal_card}
		Suppose that ${\cal G}^*$ is a CPDAG and ${\cal K}$ is a non-redundant consistent DCC set over ${\bf V}(\mathcal{G}^*)$. Then:
		\begin{enumerate}
			\item[(i)] There exists a unique non-redundant DCC set ${\cal K}'$ over ${\bf V}(\mathcal{G}^*)$ such that ${\cal K}$ is equivalent to ${\cal K}'$ and $\kappa'$ is minimally redundant with respect to $[{\cal G}^*,{\cal K}]$ for any $\kappa'\in {\cal K}'$. Moreover, for every $\kappa\in{\cal K}$, there exists a unique $\kappa'\in{\cal K}'$ such that $\kappa'_t=\kappa_t$ and $\kappa'_h\subseteq\kappa_h$; likewise, for every $\kappa'\in{\cal K}'$, there exists a unique $\kappa\in{\cal K}$ such that $\kappa'_t=\kappa_t$ and $\kappa'_h\subseteq\kappa_h$.
			\item[(ii)] For every non-redundant DCC set ${\cal K}'$ over ${\bf V}(\mathcal{G}^*)$ which is equivalent to ${\cal K}$, it holds that $|{\cal K}|=|{\cal K}'|$ and $\{\kappa_t\mid \kappa\in{\cal K}\}=\{\kappa'_t\mid \kappa'\in{\cal K}'\}$.
		\end{enumerate}
	\end{theorem}

	
	
	
	

	By definition, a non-redundant DCC set is subset-minimal, meaning that none of its subsets can be inferred from the others. Furthermore, Theorem~\ref{thm:equal_card} establishes that a non-redundant DCC set is also cardinality-minimal, implying that the cardinality of the DCC set is the smallest among all equivalent DCC sets. Finally, the first statement of Theorem~\ref{thm:equal_card} guarantees the existence and uniqueness of a non-redundant and equivalent DCC set that is element-wise head-minimal, meaning that every DCC in the set has the smallest possible head.
}

\subsection{Equivalent Decomposition of Pairwise Causal Constraints}\label{sec:sec:info}
{Let $\mathcal{G}^*$ be a CPDAG, $\cal B$ be a set of consistent pairwise causal constraints, Theorem~\ref{thm:nbr_set_const} proves that $\cal B$ can be equivalently represented by a set of DCCs $\cal K$. Let $\cal H$ be the MPDAG of $[\mathcal{G}^*, {\cal B}]$. By Definition~\ref{interpretbg:mpdagresclass} of MPDAGs, $\cal H$ represents all common direct causal relations shared by all DAGs in  $[\mathcal{G}^*,{\cal B}]$. Due to the equivalence, $[\mathcal{G}^*, {\cal B}]=[\mathcal{G}^*, {\cal K}]=[{\cal H}, {\cal K}]$, where $[{\cal H}, {\cal K}]$ denotes the subset of $[{\cal H}]$ consisting of DAGs satisfying all clauses in ${\cal K}$. However, given ${\cal H}$, some of the DCCs in ${\cal K}$ may become redundant. The following result proves that any set of consistent pairwise causal constraints can be equivalently decomposed into an MPDAG and a non-redundant residual set of DCCs.}



{
	\begin{theorem}[Equivalent Decomposition]\label{prop:eqdecomp}
		Let $\mathcal{G}^*$ be a CPDAG, $\cal B$ be a set of consistent pairwise causal constraints, and $\cal H$ be the MPDAG of $[\mathcal{G}^*, {\cal B}]$. Then:	
		\begin{enumerate}
			\item[(i)] There exists a DCC set ${\cal R}$ such that $[\mathcal{G}^*,\cal B]=[{\cal H}, {\cal R}]$, where none of the DCC $\kappa\in{\cal R}$ is redundant with respect to ${\bf E}_d({\cal H})\cup\left({\cal R}\setminus\{\kappa\}\right)$.
			
			\item[(ii)] Every DCC set ${\cal R}$ such that $[\mathcal{G}^*, \cal B] = [{\cal H}, {\cal R}]$ and none of the DCC $\kappa\in{\cal R}$ is redundant with respect to ${\bf E}_d({\cal H})\cup\left({\cal R}\setminus\{\kappa\}\right)$ contains the same number of DCCs.
			
			\item[(iii)] There exists a unique DCC set ${\cal R}^*$ such that (1) $[\mathcal{G}^*, \cal B] = [{\cal H}, {\cal R}^*]$, (2) none of the DCC $\kappa\in{\cal R}^*$ is redundant with respect to ${\bf E}_d({\cal H})\cup\left({\cal R}^*\setminus\{\kappa\}\right)$, and (3) every $\kappa\in {\cal R}^*$, if ${\cal R}^*\neq\varnothing$, is minimally redundant with respect to $[{\cal G}^*,{\cal B}]$.
			
			\item[(iv)] For any DCC set ${\cal R}$ such that $[\mathcal{G}^*,\cal B]=[{\cal H}, {\cal R}]$ and none of the DCC $\kappa\in{\cal R}$ is redundant with respect to ${\bf E}_d({\cal H})\cup\left({\cal R}\setminus\{\kappa\}\right)$, ${\cal R}\neq\varnothing$ if and only if ${\cal R}^*\neq\varnothing$, and when ${\cal R}\neq\varnothing$, it holds that, (1) for every $\kappa\in{\cal R}$, there exists a unique $\kappa^*\in{\cal R}^*$ such that $\kappa^*_t=\kappa_t$ and $\kappa^*_h\subseteq\kappa_h$, and (2) for every $\kappa^*\in{\cal R}^*$, there exists a unique $\kappa\in{\cal R}$ such that $\kappa^*_t=\kappa_t$ and $\kappa^*_h\subseteq\kappa_h$. Here, ${\cal R}^*$ is the unique DCC set defined in (iii).
		\end{enumerate}
		
	\end{theorem}
}

	
	
	{Given a CPDAG, Theorem~\ref{prop:eqdecomp} indicates that any set of pairwise causal constraints can be equivalently decomposed into the MPDAG of the induced restricted Markov equivalence class (which is unique), and a residual set of DCCs in which every DCC is non-redundant with respect to the other DCCs and the directed edges in the MPDAG. Although the decomposed residual sets of DCCs are not unique, they all contain the same number of DCCs, which means they are all minimal in terms of cardinality. Moreover, the third statement of Theorem~\ref{prop:eqdecomp} establishes the existence and uniqueness of an element-wise head-minimal residual set of DCCs where all DCCs are minimally redundant, meaning that each DCC in the set contains the minimum possible number of heads, thereby generalizing Theorem~\ref{thm:MSPrep} in Section~\ref{sec:sec:cha}. Finally, the last statement characterizes the relationship between the element-wise head-minimal residual set and the other residual sets.}
	
	Note that, since the MPDAG $\cal H$ contains all direct causal edges that appear in all DAGs in $[{\cal G}^*,{\cal B}]$, every directed edge in the MPDAG of $[\mathcal{G}^*, {\cal R}]$ is also in $\cal H$. That is, ${\cal R}$ cannot bring more directed causal edges other than those in $\cal H$.

	

	When the residual set of DCCs is empty, the MPDAG of the induced restricted Markov equivalence class is fully informative. As mentioned in Section~\ref{sec:sec:interpretbg}, a sufficient condition that guarantees the emptiness of a residual set is when $\cal B$ only contains direct and non-ancestral causal constraints. Yet, this condition is not necessary, as shown in the following example.
	
	\begin{figure}[!t]
		\centering
		\subfloat[A CPDAG ${\cal G}^*$ \label{fig:fi1}]{
			\begin{minipage}[t]{0.28\linewidth}
				\centering
				\includegraphics[width=0.63\linewidth]{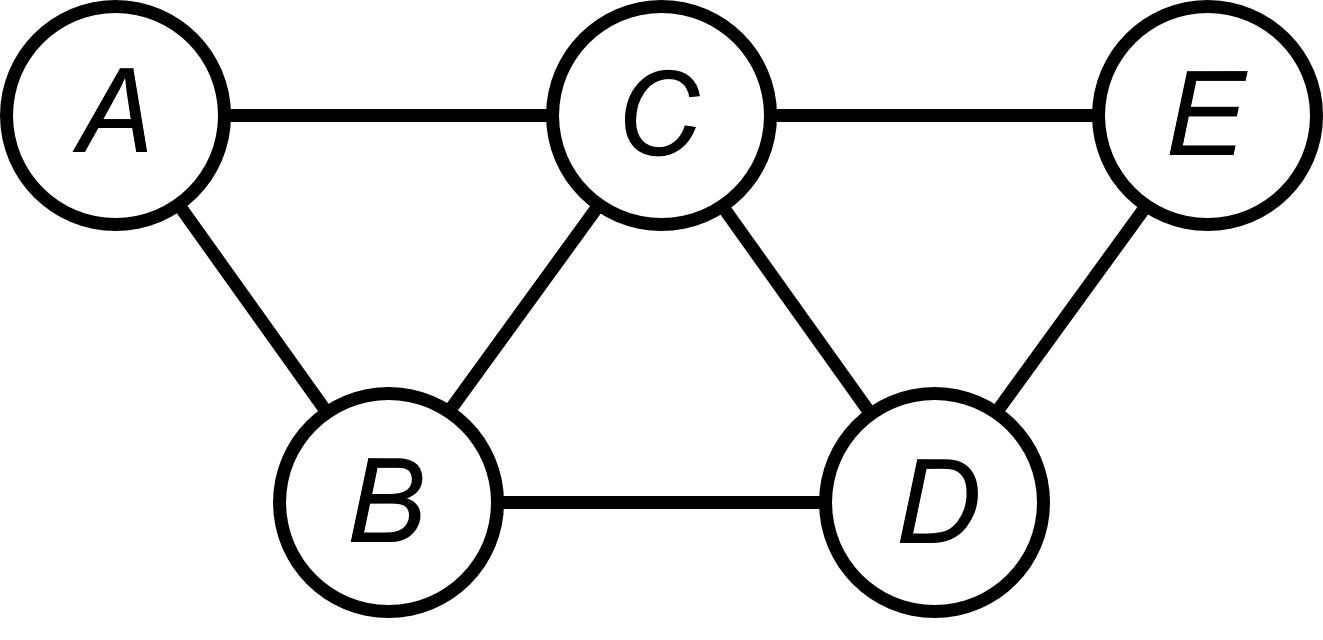}
			\end{minipage}%
		}%
		\hspace{0.01\linewidth}
		\subfloat[${\cal G}^*$ and $\cal K$ \label{fig:fi2}]{
			\begin{minipage}[t]{0.28\linewidth}
				\centering
				\includegraphics[width=0.63\linewidth]{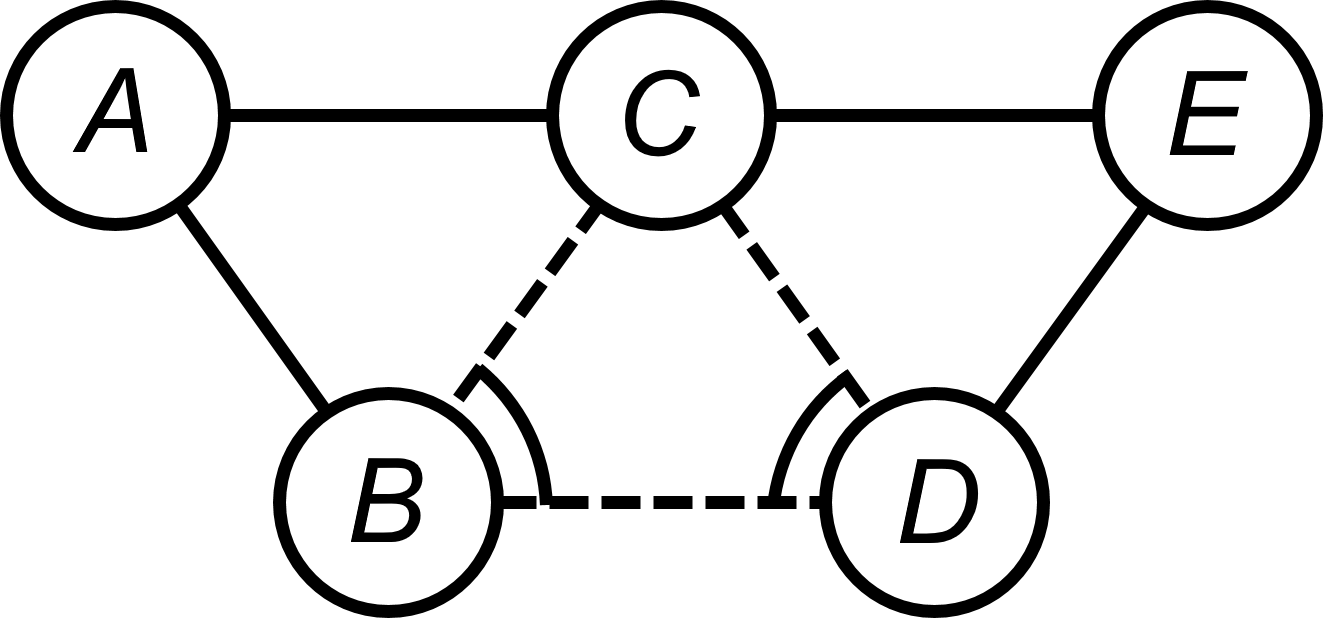}
			\end{minipage}%
		}%
		\hspace{0.01\linewidth}
		\subfloat[The MPDAG $\cal H$  \label{fig:fi3}]{
			\begin{minipage}[t]{0.28\linewidth}
				\centering
				\includegraphics[width=0.63\linewidth]{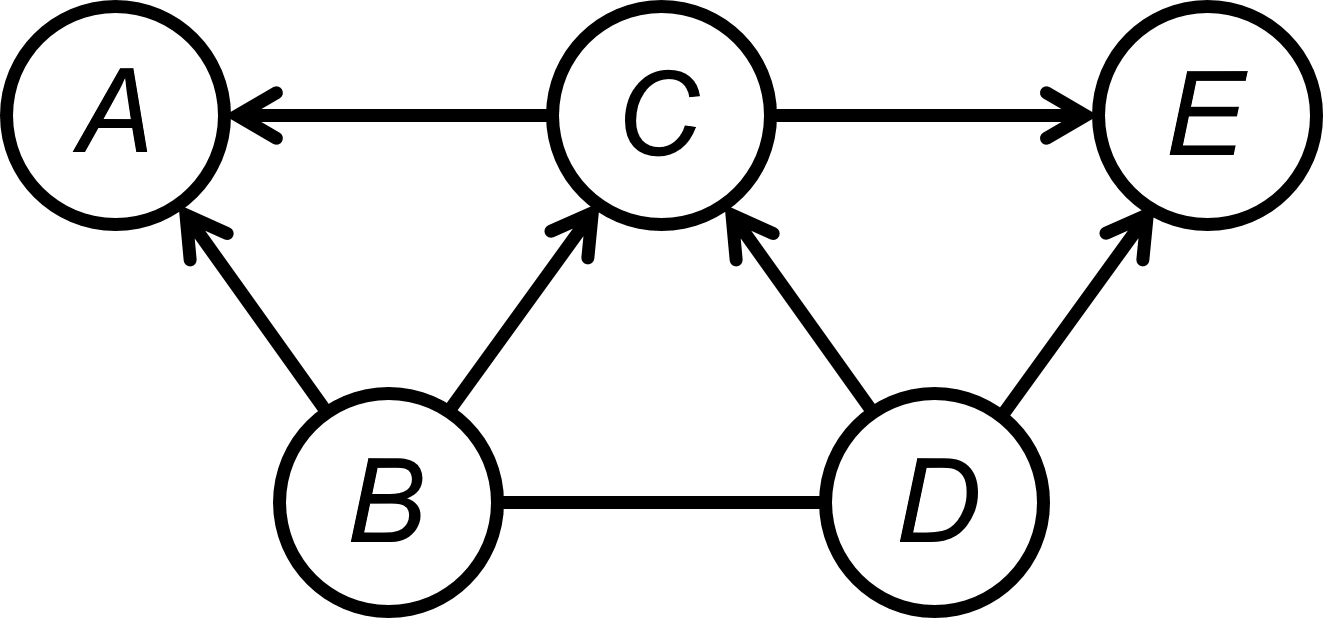}
			\end{minipage}%
		}%
		
		\caption{An example of a fully informative MPDAG.}
		\label{fig:fully_informative}
	\end{figure}

	\begin{example}
		Figure~\ref{fig:fi1} shows a CPDAG ${\cal G}^*$. Consider two ancestral causal constraints $B\dashrightarrow E$ and $D\dashrightarrow A$. By Theorems~\ref{thm:nbr_set_const} and \ref{prop:eqdecomp},  $B\dashrightarrow E$ and $D\dashrightarrow A$ are equivalent to $B\tor\{C,D\}$ and $D\tor\{B,C\}$, respectively, which are visualized by arcs in Figure~\ref{fig:fi2}. Let ${\cal K}=\{B\tor\{C,D\}, D\tor\{B,C\}\}$.  We first find the MPDAG $\cal H$ of $[{\cal G}^*, {\cal K}]$. If there is a DAG ${\cal G}\in[{\cal G}^*, {\cal K}]$ where $C\to B$ is in $\cal G$, then by the constraint $B\tor\{C,D\}$, $B \to D$ is in $\cal G$. Likewise, by the constraint $D\tor\{B,C\}$, $D\to C$ is in $\cal G$. However, $C\to B\to D\to C$ is a directed cycle. Therefore, every DAG in $[{\cal G}^*, {\cal K}]$ should have $B\to C$, implying that $B\to C$ is in $\cal H$. Similarly, $D\to C$  is in $\cal H$. Finally, by Meek's rules we obtain the MPDAG shown in Figure~\ref{fig:fi3}. Notice that $[\cal H]$ only contains two DAGs: one has $B\to D$ and the other has $D\to B$. Since both DAGs satisfy the constraints $B\dashrightarrow E$ and $D\dashrightarrow A$, $[\cal H] \subseteq [\mathcal{G}^*, {\cal B}]$. On the other hand, as mentioned in Section~\ref{sec:sec:interpretbg} that $[\mathcal{G}^*, {\cal B}] \subseteq [\cal H]$, it holds that $[\cal H] = [\mathcal{G}^*, {\cal B}]$. Therefore, $\cal H$ is fully informative.
	\end{example}

	

	Below, we present the necessary and sufficient conditions for an MPDAG to be fully informative.

	\begin{theorem}\label{thm:h_equals_n}
		Suppose that ${\cal G}^*$ is a CPDAG, ${\cal K}$ is a set of consistent DCCs, and $\cal H$ is the MPDAG of $[{\cal G}^*, {\cal K}]$. Then, the following statements are equivalent.
		\begin{enumerate}
			\item[(i)]  $\cal H$ is fully informative with respect to ${\cal K}$ and ${\cal G}^*$.
			\item[(ii)] {$[{\cal H}, \varnothing]=[{\cal G}^*, {\cal B}]$.}
			\item[(iii)]  For any $ \kappa \in {\cal K}$, either $\kappa_h \cap sib(\kappa_t, {\cal H})$ induces an incomplete subgraph of ${\cal H}$, or $\kappa_h \cap ch(\kappa_t, {\cal H})\neq\varnothing$.
		\end{enumerate}
	\end{theorem}
	
	We note that, in Section~\ref{sec:sec:algo_mpdag}, we give a polynomial-time algorithm to find the MPDAG given a CPDAG ${\cal G}^*$ and a set $\cal K$ of DCCs. With the found MPDAG, the condition in statement (iii) can also be verified in polynomial-time.

	\begin{example}
		\label{ex:not_fi}
		Figure~\ref{fig:nfi-1} shows a CPDAG ${\cal G}^*$ and a set of DCCs $\cal K$ consisting of $D\tor \{E, B\}$, $E\tor \{A, C\}$, $E\tor \{B, F\}$ and $G\tor \{B, H\}$, and Figure~\ref{fig:nfi-2} shows the MPDAG of $[{\cal G}^*, \cal K]$. (How to find this MPDAG is left to Example~\ref{ex:moc}, Section~\ref{sec:sec:algo_mpdag}). Since $E \to A$, $E \to F$ and $G\to H$ are in the MPDAG, $E\tor \{A, C\}$, $E\tor \{B, F\}$ and $G\tor \{B, H\}$ are redundant given $\cal H$. However, $D\tor \{B, E\}$ is not redundant, as $\{B, E\}$ induces a complete subgraph of $\cal H$. Therefore, in this example, the MPDAG $\cal H$ is not fully informative.
	\end{example}
	
	\begin{figure}[!h]
		\centering
		\subfloat[${\cal G}^*$ and $\cal K$ \label{fig:nfi-1}]{
			\begin{minipage}[t]{0.31\linewidth}
				\centering
				\includegraphics[width=0.7\linewidth]{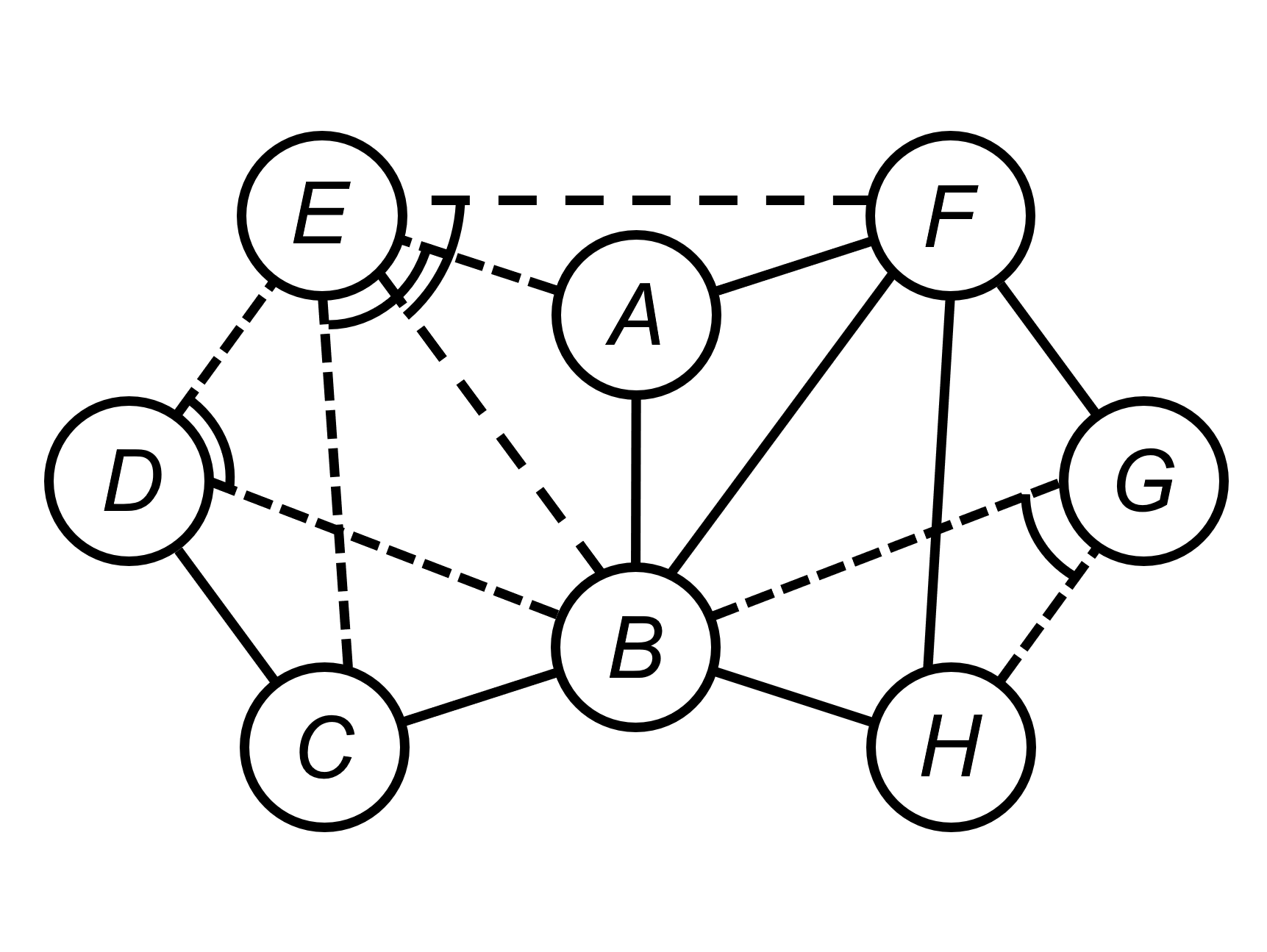}
			\end{minipage}%
		}%
		\hspace{0.01\linewidth}
		\subfloat[The MPDAG $\cal H$  \label{fig:nfi-2}]{
			\begin{minipage}[t]{0.31\linewidth}
				\centering
				\includegraphics[width=0.7\linewidth]{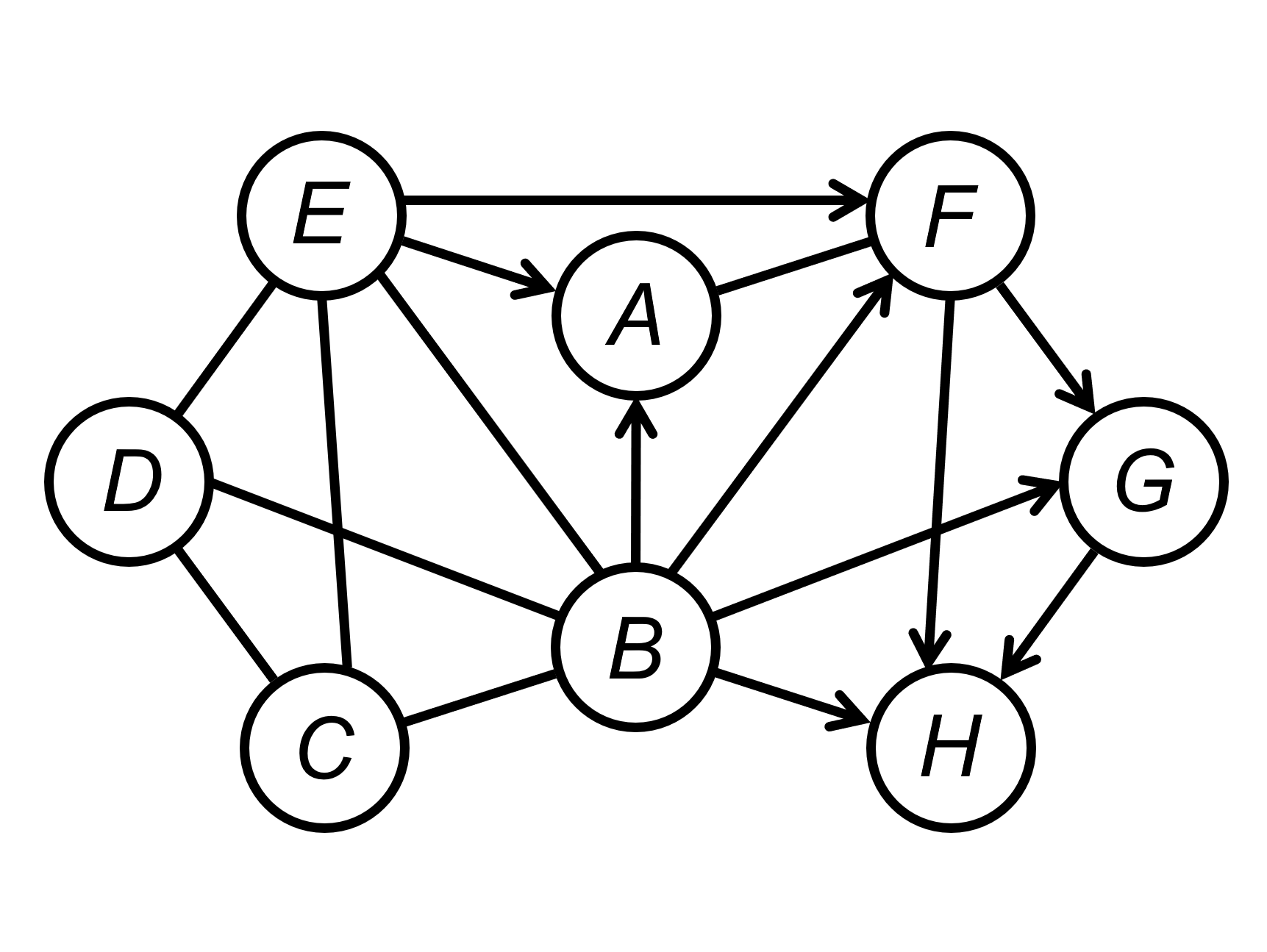}
			\end{minipage}%
		}%
		\hspace{0.01\linewidth}
		\subfloat[$D\tor \{B, E\}$ is not redundant  \label{fig:nfi-3}]{
			\begin{minipage}[t]{0.31\linewidth}
				\centering
				\includegraphics[width=0.7\linewidth]{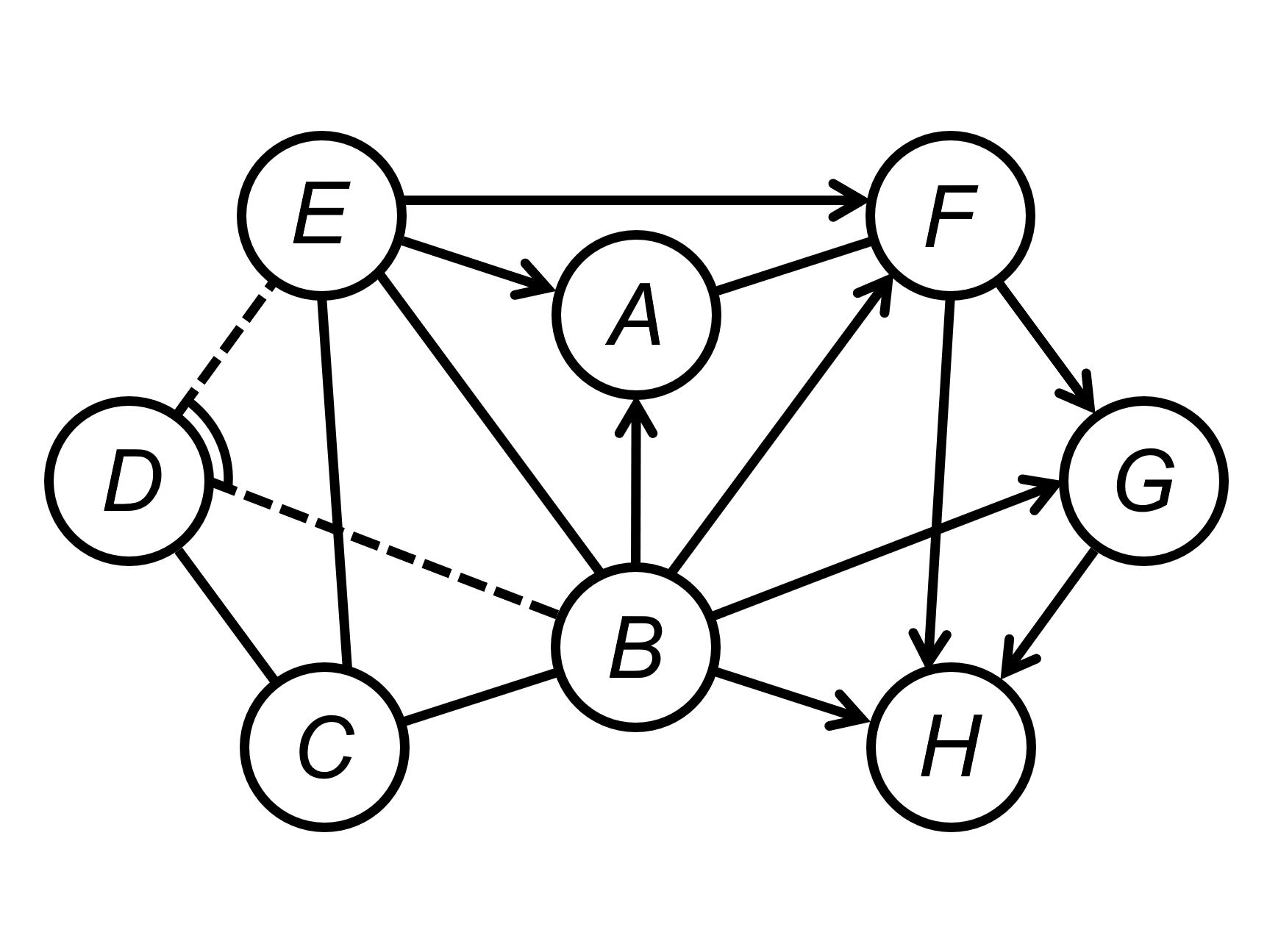}
			\end{minipage}%
		}%
		
		\caption{An MPDAG which is not fully informative.}
		\label{fig:not_fi}
	\end{figure}

	
	Notice that, given a DCC $\kappa$, if $\kappa_h\cap sib(\kappa_t, {\cal H})$ induces an incomplete subgraph of ${\cal H}$, then $\kappa_h\cap sib(\kappa_t, {\cal G}^*)$ induces an incomplete subgraph of ${\cal G}^*$. Conversely, if $\kappa_h\cap sib(\kappa_t, {\cal G}^*)$ induces an incomplete subgraph of ${\cal G}^*$, then by Rule 1 of Meek's rules, $\kappa $ holds for all DAGs in $[{\cal G}^*]$, and thus $\kappa $ is redundant. For a given DCC set $\cal K$, removing the above redundant DCCs from the equivalent reduced form ${\cal K}({{\cal G}^*})$, we define a subset of ${\cal K}({{\cal G}^*})$ as follows.
	\begin{equation}\label{eq:n_u^c}
		{\cal K}^c(\mathcal{G}^*_u)\coloneqq\{ \kappa \mid  \kappa\in {\cal K}(\mathcal{G}^*_u) \text{ and the induced  subgraph }   \mathcal{G}^*({\kappa_h})   \text{ is complete}\}.
	\end{equation}
	Based on the above argument, we have the following corollary.
	
	\begin{corollary}\label{coro:fullyinfo}
		Let $\mathcal{H}$ be an MPDAG representing the restricted Markov equivalence class induced by a CPDAG ${\cal G}^*$ and a set $\cal K$ of consistent DCCs, then $[\cal H]=[{\cal G}^*, {\cal K}]$ if and only if for any $\kappa \in {\cal K}^c(\mathcal{G}^*_u)$,  $\kappa_h \cap ch(\kappa_t, {\cal H})\neq\varnothing$ holds.
	\end{corollary}


	
	
	\subsection{Non-Pairwise Causal Background Knowledge}\label{sec:sec:non-pair}
	
	To end this section, we present some discussions on non-pairwise causal background knowledge. Non-pairwise causal background knowledge is also common in practice. For example, \emph{tiered background knowledge} is non-pairwise. Suppose that $\mathbf{T}\coloneqq\{\mathbf{V}_1, \mathbf{V}_2,...,\mathbf{V}_n\}$ defines a disjoint partition of the variable set $\mathbf{V}$. The tiered background knowledge is a proposition saying that for all $X_i\in \mathbf{V}_i$ and $X_j\in \mathbf{V}_j$ such that $1\leq i\leq j\leq n$, either $X_i\dashrightarrow X_j$ or $X_i$ is not adjacent to $X_j$~\citep{Andrews2020tiered}. Nonetheless, most non-pairwise causal background knowledge can be viewed as Boolean combinations of pairwise causal background knowledge. For instance, tiered background knowledge can be interpreted and denoted by
	\[\mathrm{tbk}^{\mathbf{T}}\coloneqq\bigwedge_{\substack{X_i\in \mathbf{V}_i,\, X_j\in \mathbf{V}_j, \\ 1\leq i\leq j\leq n}} X_i\dashrightarrow X_j \vee \left(X_i\nrightarrow X_j \wedge X_j\nrightarrow X_i\right).\]
	Moreover, given a CPDAG ${\cal G}^*$, as $X_i\nrightarrow X_j \wedge X_j\nrightarrow X_i$ holds for all DAGs in $[{\cal G}^*]$ if and only if $X_i\notin adj(X_j, {\cal G}^*)$, we have
	\[\mathrm{tbk}^{\mathbf{T}} \iff \bigwedge_{\substack{X_i\in \mathbf{V}_i,\, X_j\in \mathbf{V}_j, \\ X_i \in adj(X_j, {\cal G}^*),\, 1\leq i\leq j\leq n}} X_i\to X_j.\]
	Thus, the results in the main text can be extended to tiered background knowledge.
	
	Unfortunately, directly extending our results to an arbitrary Boolean combination of pairwise causal background knowledge is difficult. Nevertheless, a possible solution exists. Let ${\cal B}$ be a set of Boolean combinations of pairwise causal background knowledge. It is clear that a DAG satisfies all constraints in  ${\cal B}$ if and only if $\wedge_{b\,\in\,{\cal B}}\,b$ holds for the DAG. Since for any $b\in{\cal B}$, $b$ is a Boolean combination of pairwise causal background knowledge, $\wedge_{b\,\in\,{\cal B}}\,b$ can be reformulated into its disjunctive normal form,
	\begin{equation}\label{eq:logic}
		\bigwedge_{b\,\in\,{\cal B}}b = \bigvee_{i\,\in\, I} \left(\bigwedge_{j\,\in\, J} c_{ij}\right),
	\end{equation}
	where $I,J$ are finite indicator sets, and each $c_{i,j}$ is a  direct, non-ancestral or ancestral causal constraint. For each $i\in I$, let ${\cal B}_i=\{c_{ij}\mid j\in J\}$,  we can then extend our results to each ${\cal B}_i$ and combine them together. For example, to check the consistency of $\wedge_{b\,\in\,{\cal B}}\,b$, it suffices to show that there exists a consistent ${\cal B}_i$; to construct the MPDAG representing $[{\cal G}^*, {\cal B}]$, one needs first construct the MPDAG for each ${\cal B}_i$, then subtract all common direct causal relations from those MPDAGs.
	
	We remark that the bottleneck of the above solution is the computation of Equation~\eqref{eq:logic}. Generally, the disjunctive normal form is not unique, and the reformulation may be exponential in complexity. How to efficiently represent Boolean combinations of pairwise causal background knowledge is regarded as future work.
	
	\section{Polynomial-Time Algorithms}\label{sec:algorithms}
	
	In this section, we present algorithms for checking the consistency and equivalence of DCCs, and for finding the MPDAG and the minimal residual sets of DCCs given a CPDAG and a set of consistent pairwise causal constraints. All proposed algorithms run in polynomial time.
	
	\subsection{Algorithms for Checking Consistency and Equivalence}\label{sec:sec:algoconsistent}
	
	Recall that Theorem~\ref{thm:consistency} provides a sufficient and necessary condition for consistency. The proof of Theorem~\ref{thm:consistency} also motivates an algorithm for checking consistency. Algorithm~\ref{algo:consistency-peo} shows the schema. The inputs of Algorithm~\ref{algo:consistency-peo} are a CPDAG $\mathcal{G}^*$ and a set $\cal K$ of DCCs  over   ${\bf V}(\mathcal{G}^*)$. It first initializes ${\cal U}$ by $\mathcal{G}^*_u$, and then sequentially removes potential leaf nodes from ${\cal U}$ in lines $2$-$6$, until no more potential leaf node can be found.
	
	\begin{algorithm}[!t]
		\caption{Checking the consistency of a DCC set.}
		\label{algo:consistency-peo}
		\begin{algorithmic}[1]
			\REQUIRE
			A CPDAG $\mathcal{G}^*$ and a set $\cal K$ of DCCs  over ${\bf V}(\mathcal{G}^*)$.
			\ENSURE
			A Boolean value indicating whether  $\cal K$ is consistent with $\mathcal{G}^*$.	
			\STATE {Set ${\cal U}=\mathcal{G}^*_u$ and compute ${\cal K}({\cal U})$ according to Equation~\eqref{eq:n_u} and Equation~\eqref{eq:n_uu}.}
			\WHILE {${\cal U}$ has a potential leaf node with respect to $\cal K$ and $\mathcal{G}^*$,}
			\STATE {Find a potential leaf node in ${\cal U}$ with respect to $\cal K$ and $\mathcal{G}^*$ and denote it by $Y$.}
			\STATE {Update ${\cal U}$ by removing $Y$ and the edges connected to $Y$.}
			\STATE {Update ${\cal K}({\cal U})$ by removing all clauses whose heads contain $Y$.}
			\ENDWHILE
			\IF {${\cal U}$ is an empty graph,}
			\RETURN {$\textrm{True}$.}
			\ENDIF				
			\RETURN {$\textrm{False}$.}
		\end{algorithmic}
	\end{algorithm}
	
	{The complexity of computing ${\cal K}({\cal U})$ is bounded by $O(|{\cal K}|\cdot|{\bf V}(\mathcal{G}^*)|^2)$. Moreover, algorithm~\ref{algo:consistency-peo} runs the while-loop at most $|{\bf V}(\mathcal{G}^*)|$ times. Every time it runs the while-loop, it checks for each vertex in ${\cal U}$ whether the vertex  is simplicial and not the tail of any clause in ${\cal K}(\mathcal{U})$. The complexity of the former is bounded by $O(|{\bf V}(\mathcal{G}^*)|^3)$, and the complexity of the later is bounded by $O(|{\cal K}|\cdot|{\bf V}(\mathcal{G}^*)|)$. The complexity of removing $Y$, as well as the clauses whose heads contain $Y$, is bounded by $O(|{\bf V}(\mathcal{G}^*)|)$ and $O(|{\cal K}|\cdot|{\bf V}(\mathcal{G}^*)|)$, respectively. Therefore, the complexity of Algorithm~\ref{algo:consistency-peo} is upper bounded by $O(|{\bf V}(\mathcal{G}^*)|^4+|{\cal K}|\cdot|{\bf V}(\mathcal{G}^*)|^2)$, which is polynomial to both $|{\bf V}(\mathcal{G}^*)|$ and $|{\cal K}|$.}
	
	Algorithm~\ref{algo:consistency-peo} is related to the perfect elimination ordering (PEO) of a chordal graph~\citep{blair1993chordal, maathuis2009estimating}. In fact, when $\cal K$ is consistent, the ordering of the removal potential leaf nodes is a PEO of the chordal graph $\mathcal{G}^*_u$. PEOs are important to construct Markov equivalent DAGs. We refer the interested readers to Appendix~\ref{app:proof:pre} for more details. We also remark that, when a pairwise causal background knowledge set contains only direct causal constraints, Algorithm~\ref{algo:consistency-peo} degenerates to the algorithm proposed by~\citet{dor1992simple}.

	\begin{example}[continued]\label{ex:consistency}
		We next use Algorithm \ref{algo:consistency-peo} to check the consistency of ${\cal K}=\{C \tor A, B \tor A, D \tor \{A, X\}\}$ with respect to ${\cal G}^*$ illustrated in Figure \ref{fig:1-1}. We first set ${\cal U}=\mathcal{G}^*_u$, which is the induced subgraph of ${\cal G}^*$ over $\{A,B,C,D,X\}$. With respect to ${\cal K}$, however, ${\cal U}$ has no potential leaf node, meaning that the while-loop (lines $2$-$6$) is not triggered. As ${\cal U}$ is not empty, Algorithm \ref{algo:consistency-peo} returns {\rm False}.
	\end{example}
	
	
	In some circumstances, pairwise causal background knowledge is not obtained all at once. Therefore, we also need an approach to sequentially check consistency. That is, given a CPDAG ${\cal G}^*$ and a consistent pairwise causal constraint set ${\cal B}$, we want to determine whether a newly obtained pairwise causal constraint set is consistent with ${\cal G}^*$ together with ${\cal B}$. This issue will be investigated in
	Appendix~\ref{sec:sec:seq}.
	
	Finally, with Algorithm~\ref{algo:consistency-peo} and Theorem~\ref{thm:equi_dcc}, we can check the equivalence of two DCC sets, as shown in Algorithm~\ref{algo:eq}. Note that, to accelerate the procedure, we first check the consistency of ${\cal K}_1$ and ${\cal K}_2$ separately. If neither ${\cal K}_1$ nor ${\cal K}_2$ is consistent with $\mathcal{G}^*$ then they are equivalent; if one of them is consistent with $\mathcal{G}^*$ but the other is not, then they are not equivalent.
	
	\begin{algorithm}[!t]
		\caption{Checking the equivalence of two DCC sets.}
		\label{algo:eq}
		\begin{algorithmic}[1]
			\REQUIRE
			A CPDAG $\mathcal{G}^*$ and two sets ${\cal K}_1$ and ${\cal K}_2$ of DCCs  over ${\bf V}(\mathcal{G}^*)$.
			\ENSURE
			A Boolean value indicating whether ${\cal K}_1$ is equivalent to ${\cal K}_2$ given $\mathcal{G}^*$.
			\IF {both ${\cal K}_1$ and ${\cal K}_2$ are consistent with $\mathcal{G}^*$, }
			\FOR {$\kappa \in {\cal K}_1$,}
			\IF { $\cup_{D\in {\kappa_h}}\{D\to \kappa_t\}\cup {\cal K}_2$ is consistent with ${\cal G}^*$,}
			\RETURN {$\textrm{False}$.}
			\ENDIF
			\ENDFOR
			\FOR {$\kappa \in {\cal K}_2$,}
			\IF { $\cup_{D\in {\kappa_h}}\{D\to \kappa_t\}\cup {\cal K}_1$ is consistent with ${\cal G}^*$,}
			\RETURN {$\textrm{False}$.}
			\ENDIF
			\ENDFOR
			\RETURN {$\textrm{True}$.}
			\ELSIF	{neither ${\cal K}_1$ nor ${\cal K}_2$ is consistent with $\mathcal{G}^*$,}
			\RETURN {$\textrm{True}$.}
			\ELSE
			\RETURN {$\textrm{False}$.}
			\ENDIF
		\end{algorithmic}
	\end{algorithm}
	
	
	\subsection{ Algorithms for Finding  MPDAGs and   Minimal DCC Sets}\label{sec:sec:algo_mpdag}
	We first discuss how to find the MPDAG of $[{\cal G}^*, {\cal K}]$ induced by a CPDAG ${\cal G}^*$ and a set ${\cal K}$ of  DCCs consistent with ${\cal G}^*$. By the definition of an MPDAG, it suffices to find all common direct causal relations shared by all DAGs in $[{\cal G}^*, {\cal K}]$. A new concept is needed before proceeding.
	
	\begin{definition}[Orientation Component]
		Given a CPDAG $\mathcal{G}^*$ and a DCC set $\cal K$ consistent with $\mathcal{G}^*$. With respect to $\cal K$ and $\mathcal{G}^*$, a connected undirected induced subgraph $\mathcal{U}$ of $\mathcal{G}^*$ is called an orientation component for a vertex $X$ if $X$ is the only potential leaf node in $\mathcal{U}$.
	\end{definition}

	\begin{example}
		\label{ex:demo_concepts}
		An orientation component is illustrated by this example.
		Figure~\ref{fig:2-1} shows a CPDAG ${\cal G}^*$. Consider ${\cal K}=\{A\tor \{X, B, D\}, B\tor \{X, A\}, B\tor \{X, C, Y\}, X\tor \{B, C\} \}$. Since $\{X, B, D\}\cap sib(A, {\cal G}^*)=\{X, B\}$ and $ \{X, C, Y\}\cap ch(B,{\cal G}^*)\neq\varnothing$, by Equation~(\ref{eq:n_u}), ${\cal K}(\mathcal{G}^*_u)$ consists of $A\tor \{X, B\}$, $B\tor \{X, A\}$ and $X\tor \{B, C\}$, which are visualized by arcs in Figure~\ref{fig:2-2}. First, consider the undirected induced subgraph ${\cal U}$ over $\{A, B, X\}$. By definition, ${\cal K}(\mathcal{U})$ consists of $A\tor \{X, B\}$ and $B\tor \{X, A\}$. Therefore, there is only one potential leaf node in $\mathcal{U}$, namely $X$, and thus $\mathcal{U}$ is an orientation component for $X$. Next, consider the undirected induced subgraph $\cal U$ over $\{X, B, C\}$. ${\cal K}(\mathcal{U})$ has only one constraint $X\tor \{B, C\}$, hence $\mathcal{U}$ has two potential leaf nodes $B$ and $C$. Finally, consider the entire undirected subgraph ${\cal G}^*_u$. ${\cal G}^*_u$ has only one potential leaf node $C$, thus ${\cal G}^*_u$ is an orientation component for $C$.
	\end{example}
	
	\begin{figure}[!h]
		\centering
		\subfloat[CPDAG ${\cal G}^*$ \label{fig:2-1}]{
			\begin{minipage}[t]{0.44\linewidth}
				\centering
				\includegraphics[width=0.52\linewidth]{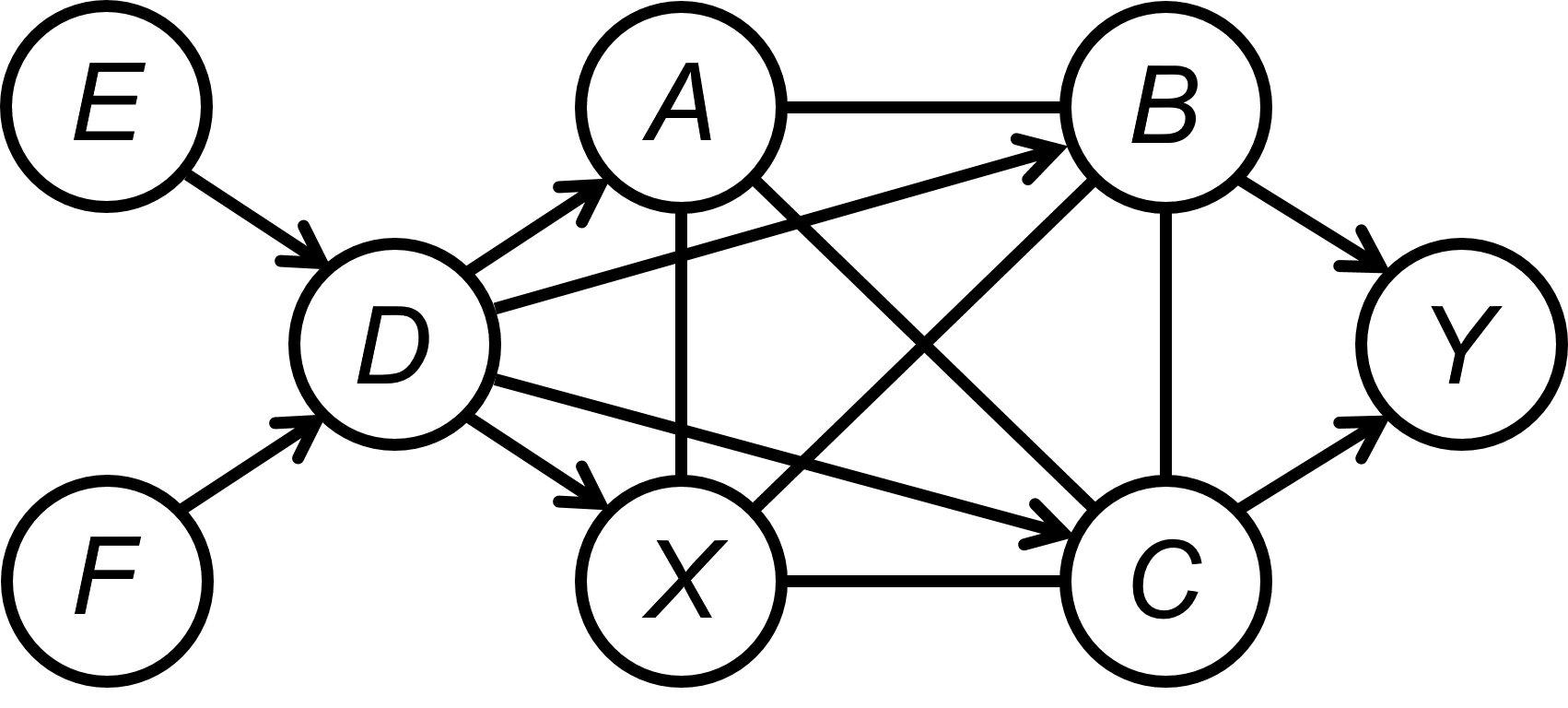}
			\end{minipage}%
		}%
		\hspace{0.01\linewidth}
		\subfloat[${\cal G}^*$ with DCCs \label{fig:2-2}]{
			\begin{minipage}[t]{0.44\linewidth}
				\centering
				\includegraphics[width=0.52\linewidth]{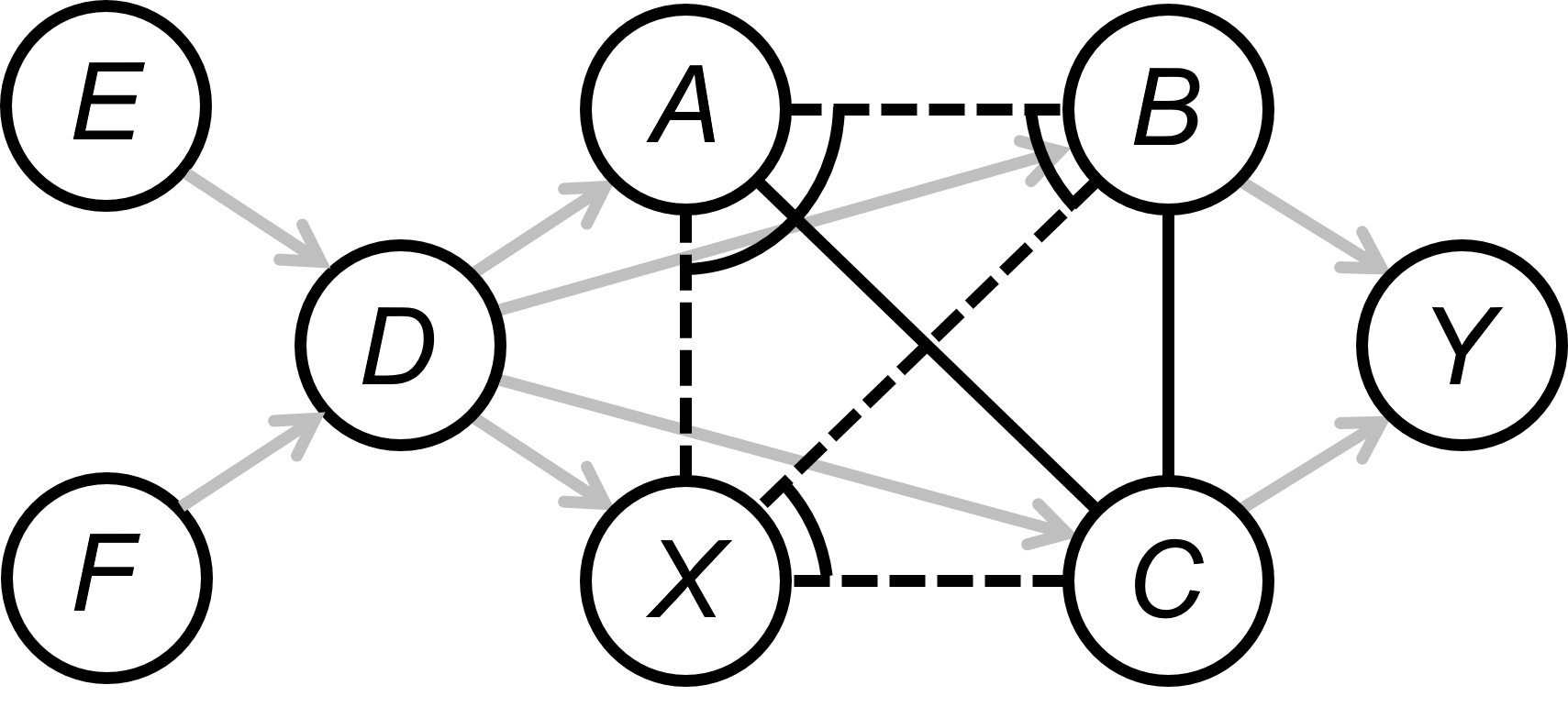}
			\end{minipage}%
		}%
		
		\hspace{0.02\linewidth}
		\subfloat[$X\to B$ results a directed cycle \label{fig:2-3}]{
			\begin{minipage}[t]{0.22\linewidth}
				\centering
				\includegraphics[width=0.45\linewidth]{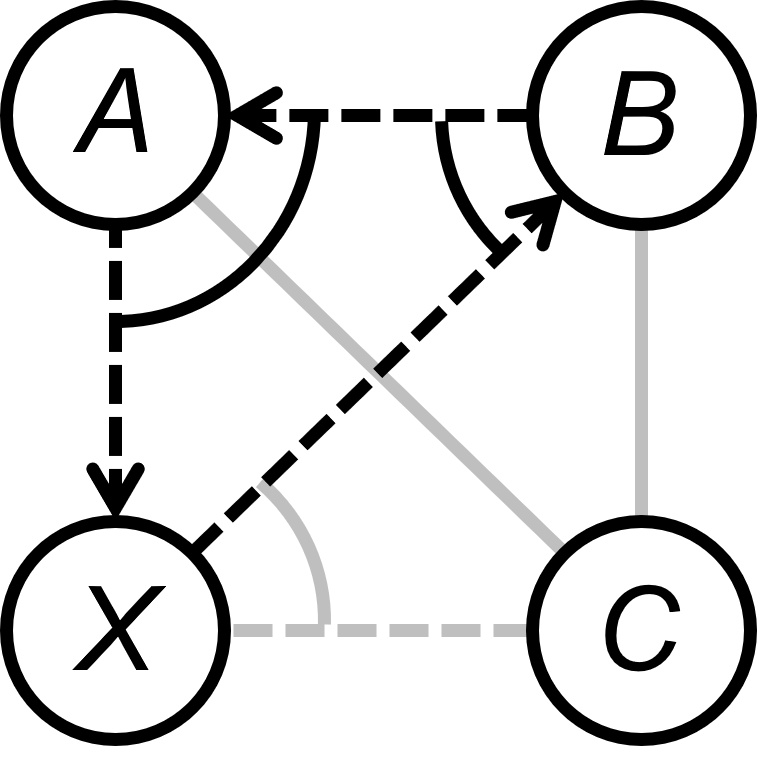}
			\end{minipage}%
		}%
		\hspace{0.01\linewidth}
		\subfloat[$X\to A$ results a directed cycle \label{fig:2-4}]{
			\begin{minipage}[t]{0.22\linewidth}
				\centering
				\includegraphics[width=0.45\linewidth]{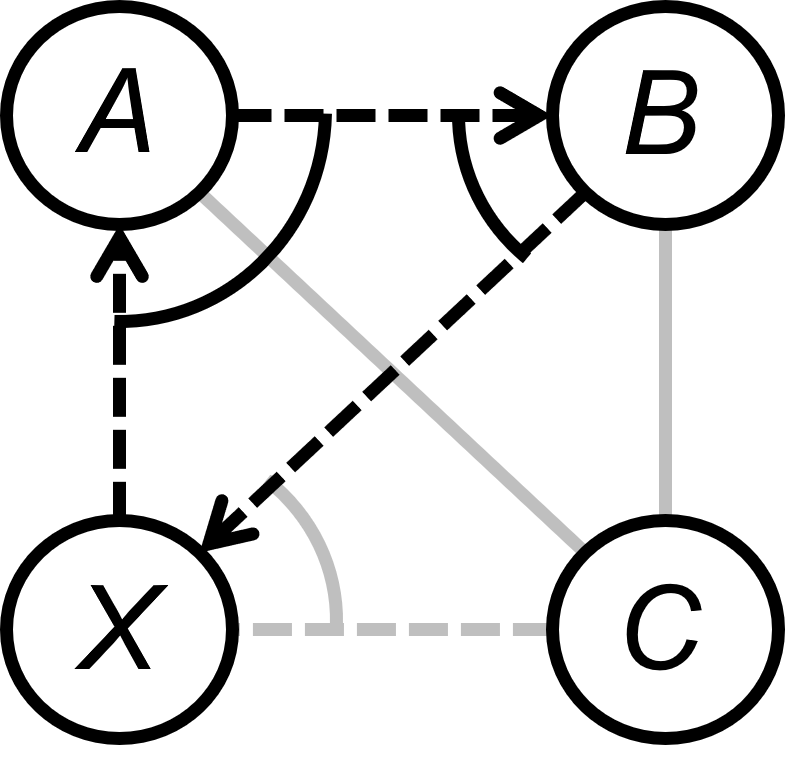}
			\end{minipage}%
		}%
		\hspace{0.01\linewidth}
		\subfloat[$X \to C$ is in the MPDAG  \label{fig:2-5}]{
			\begin{minipage}[t]{0.22\linewidth}
				\centering
				\includegraphics[width=0.45\linewidth]{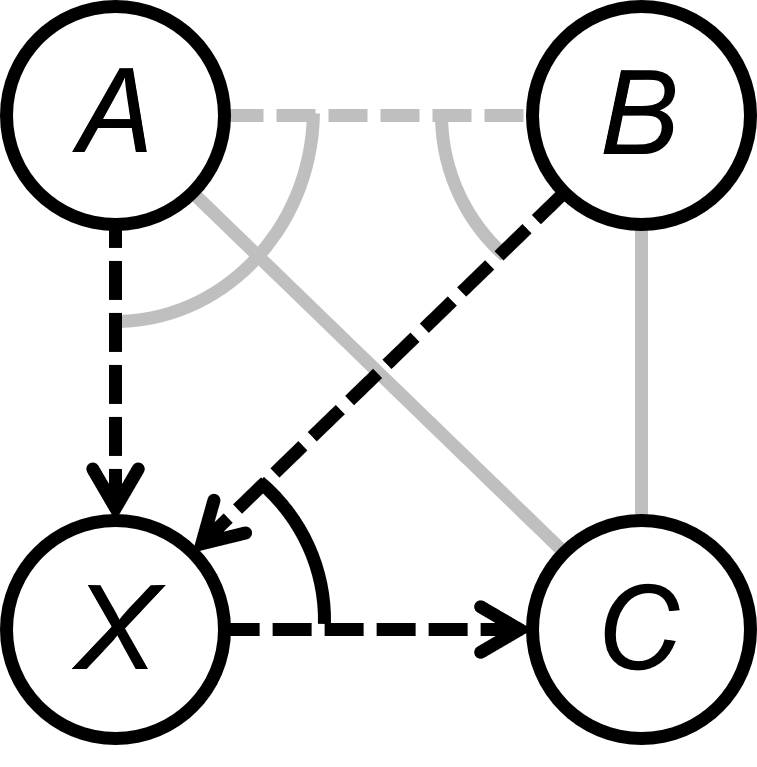}
			\end{minipage}%
		}%
		\hspace{0.01\linewidth}
		\subfloat[The founded directed edges in the MPDAG \label{fig:2-6}]{
			\begin{minipage}[t]{0.22\linewidth}
				\centering
				\includegraphics[width=0.45\linewidth]{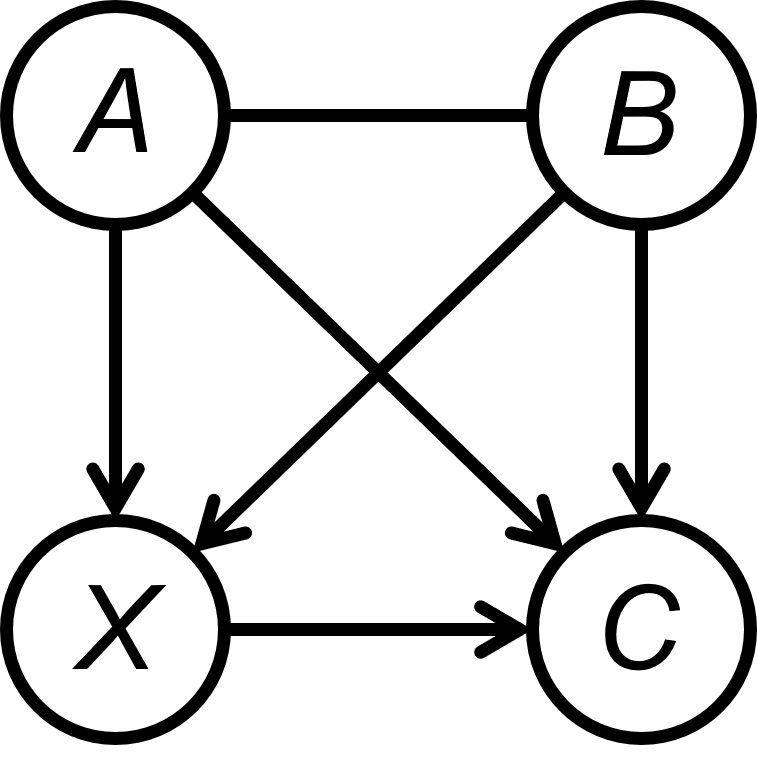}
			\end{minipage}%
		}%
		
		\caption{An illustration of an orientation component is given by this example.}
		\label{fig:inv_heads}
	\end{figure}
	
	The following proposition shows that an orientation component  can be used to  identify some direct causal relations.
	
	\begin{proposition}
		\label{prop:suff_oc}
		Let $\mathcal{G}^*$ be a CPDAG and $\cal K$ be a set of DCCs consistent with $\mathcal{G}^*$. For any orientation component $\mathcal{U}$ for $X$ with respect to $\cal K$ and $\mathcal{G}^*$ with $adj(X,\mathcal{U})\neq\varnothing$, all variables in  $adj(X,\mathcal{U})$ are direct causes of  $X$  in every DAG  in  $[\mathcal{G}^*, {\cal K}]$.
	\end{proposition}
	
	Together with the directed edges in the CPDAG, the orientation components can completely characterize the common direct causal relations given a CPDAG and a set of  DCCs.
	
	
	\begin{theorem}
		\label{thm:mpdag}
		Let $\mathcal{G}^*$ be a CPDAG and $\cal K$ be a set of DCCs consistent with $\mathcal{G}^*$, then $X\rightarrow Y$ is in every DAG in $[\mathcal{G}^*, {\cal K}]$ if and only if  $X\rightarrow Y$ appears in $\mathcal{G}^*$ or there exists an orientation component for $Y$ containing $X$ with respect to $\cal K$ and $\mathcal{G}^*$.
	\end{theorem}
	
	Graphically, Theorem~\ref{thm:mpdag} shows all possible sources of the directed edges in an MPDAG. Consider Figure~\ref{fig:2-2} as an example. As discussed in Example \ref{ex:demo_concepts}, the undirected induced subgraph over $\{A,B,X\}$ and $\{A,B,X,C\}$ are orientation components for $X$ and $C$ respectively. If $X \to B$ is in a DAG ${\cal G}\in [\mathcal{G}^*, {\cal K}]$ (Figure \ref{fig:2-3}), then due to the constraint $B\tor \{X, A\}$ we must have that $B \to A$ is in ${\cal G}$. Again, due to the constraint $A\tor \{X, B\}$, $A\to X$ is in ${\cal G}$. However, $X\to B \to A \to X$ is a directed cycle, meaning that $X \to B$ is not in any DAG ${\cal G}\in [\mathcal{G}^*, {\cal K}]$, and thus $B\to X$ should be in the MPDAG. Similarly, $A\to X$ is in the MPDAG (Figure \ref{fig:2-4}). Since $B\to X$ and $X\tor \{B, C\}$, $X \to C$ is also in the MPDAG (Figure \ref{fig:2-5}). Finally, applying Meek's rules results $A \to C$ and $B \to C$ (Figure \ref{fig:2-6}).
	
	
	Theorem~\ref{thm:mpdag} provides an intuitive method to  find the MPDAG by finding all orientation components for all variables, but this method is time-consuming as the number of orientation components for a variable may be very large. The following proposition  gives a clue to improve  this method by finding the exact orientation component of interest.
	
	\begin{proposition}\label{prop:union_of_oc}
		Let $\mathcal{G}^*$ be a CPDAG and $\cal K$ be a set of DCCs consistent with $\mathcal{G}^*$. Suppose that with respect to $\cal K$ and $\mathcal{G}^*$, $\mathcal{U}_1$ and $\mathcal{U}_2$ are two orientation components for the same vertex $X$, then the undirected induced subgraph of $\mathcal{G}^*$ over $\mathbf{V}(\mathcal{U}_1)\cup\mathbf{V}(\mathcal{U}_2)$ is also an orientation component for $X$.
	\end{proposition}
	
	In some literature, the undirected induced subgraph of $\mathcal{G}^*$ over $\mathbf{V}(\mathcal{U}_1)\cup\mathbf{V}(\mathcal{U}_2)$ is called the union of $\mathcal{U}_1$ and $\mathcal{U}_2$. Thus, Proposition \ref{prop:union_of_oc} indicates that the union of two orientation components for $X$ is still an orientation component for $X$. This result motivates the definition of a maximal orientation component.
	
	\begin{definition}[Maximal Orientation Component, MOC]
		\label{def:moc}
		Given a CPDAG $\mathcal{G}^*$ and a set $\cal K$ of DCCs consistent with $\mathcal{G}^*$, with respect to $\cal K$ and $\mathcal{G}^*$, an orientation component for a variable $X$ is called maximal if every orientation component for $X$ is its induced subgraph.
	\end{definition}
	
	Based on Definition~\ref{def:moc} and Theorem~\ref{thm:mpdag}, we have the following corollary.

	\begin{corollary}\label{coro:maximaloc}
		Let $\mathcal{G}^*$ be a CPDAG and $\cal K$ be a set of DCCs consistent with $\mathcal{G}^*$, then $X\rightarrow Y$ is in every DAG in $[\mathcal{G}^*, {\cal K}]$ if and only if $X\rightarrow Y$ is in $\mathcal{G}^*$ or the maximal orientation component for $Y$ with respect to $\cal K$ and $\mathcal{G}^*$ contains $X$.
	\end{corollary}
	
	
	
	\begin{algorithm}[!t]
		\caption{Finding the maximal orientation component for a vertex.}
		\label{algo:moc}
		\begin{algorithmic}[1]
			\REQUIRE
			A CPDAG $\mathcal{G}^*$, a set $\cal K$ of DCCs consistent with $\mathcal{G}^*$, and a vertex $X$ in $\mathcal{G}^*$.
			\ENSURE
			The maximal orientation component for $X$ with respect to $\cal K$ and $\mathcal{G}^*$.	
			\STATE {Set ${\cal U}_m$ to be the chain component containing $X$, and compute ${\cal K}_m={\cal K}({\cal U}_m)$.}
			\WHILE {${\cal U}_m$ is not an orientation component for $X$,}
			\STATE {Find a potential leaf node in ${\cal U}_m$ and denote it by $Y$.}
			\STATE {Update ${\cal U}_m$ by removing $Y$ and the edges connected to $Y$.}
			\STATE {Update ${\cal K}_m$ by removing all clauses whose heads contain $Y$.}
			\ENDWHILE				
			\RETURN {${\cal U}_m$.}
		\end{algorithmic}
	\end{algorithm}
	
	Corollary~\ref{coro:maximaloc} yields a method to find the MPDAG of a restricted Markov equivalence class. The key step is to find the maximal orientation component for a variable, whose procedure is given in Algorithm \ref{algo:moc}. Algorithm \ref{algo:moc} is a generalization of Algorithm~\ref{algo:consistency-peo}. The first step of Algorithm \ref{algo:moc} is to set ${\cal U}_m$ to be the chain component containing $X$, and compute ${\cal K}_m={\cal K}({\cal U}_m)$. This is because the maximal orientation component for $X$ is an induced subgraph of the chain component to $X$ belongs.
	If  ${\cal U}_m$ is an orientation component for $X$, then it is definitely maximal. Otherwise,
	the while-loop begins. Since $X$ is not the only potential leaf node in ${\cal U}_m$, based on Theorem \ref{thm:consistency}, there must be another potential leaf node $Y$ in ${\cal U}_m$. We then remove $Y$ and the edges connected to $Y$. The resulting graph is the induced subgraph of ${\cal U}_m$ over $\mathbf{V}({\cal U}_m)\setminus\{Y\}$, and we still denote it by ${\cal U}_m$. Finally, we remove from ${\cal K}_m$ those clauses whose heads include $Y$. This is equivalent to setting ${\cal K}_m={\cal K}_m({\cal U}_m)$. The while-loop ends when ${\cal U}_m$ is an orientation component for $X$. Note that, since in each loop we remove one vertex, ${\cal U}_m$ will eventually become an orientation component for $X$ in the finite number of loops. The correctness of Algorithm \ref{algo:moc} is guaranteed by the following theorem.

	\begin{theorem}
		\label{thm:moc}
		The outputted undirected graph of Algorithm~\ref{algo:moc} is identical to the maximal orientation component for $X$ with respect to $\cal K$ and $\mathcal{G}^*$.
	\end{theorem}
	
	{Similar to Algorithm~\ref{algo:consistency-peo}, the complexity of Algorithm \ref{algo:moc} is upper bounded by $O(|{\bf V}(\mathcal{G}^*)|^4+|{\cal K}|\cdot|{\bf V}(\mathcal{G}^*)|^2)$.} Applying Algorithm \ref{algo:moc} to each vertex separately, we can find the MPDAG of a restricted Markov equivalence class based on Corollary~\ref{coro:maximaloc}, as shown by the following example.
	

	\begin{figure}[!t]
		\vspace{-0.5em}
		\centering
		\subfloat[${\cal G}^*$ and $\cal K$ \label{fig:3-1}]{
			\begin{minipage}[t]{0.29\linewidth}
				\centering
				\includegraphics[width=0.7\linewidth]{fig/fig3-1.png}
			\end{minipage}%
		}%
		\hspace{0.03\linewidth}
		\subfloat[After removing $C$  \label{fig:3-2}]{
			\begin{minipage}[t]{0.29\linewidth}
				\centering
				\includegraphics[width=0.7\linewidth]{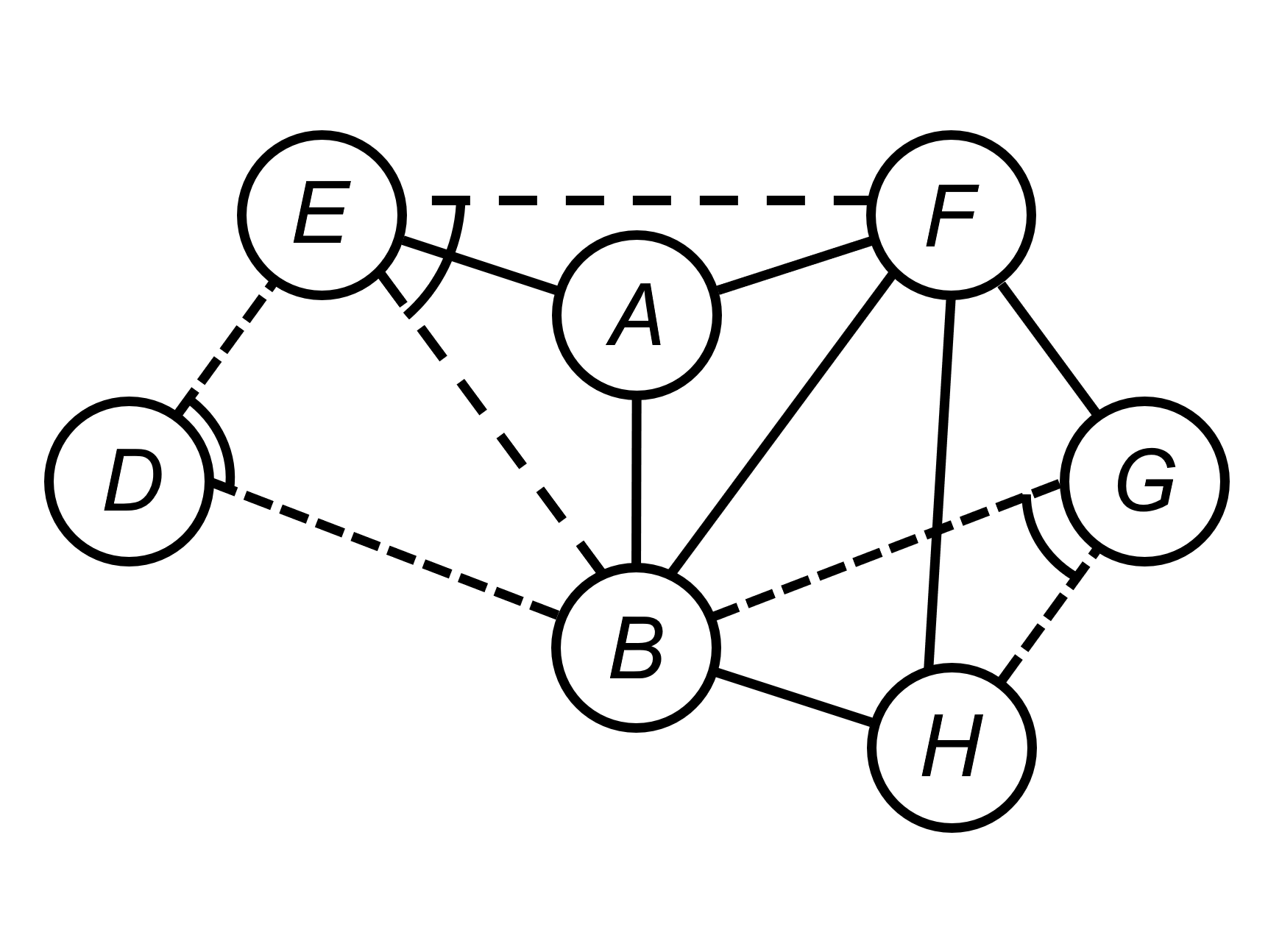}
			\end{minipage}%
		}%
		\hspace{0.03\linewidth}
		\subfloat[After removing $H$   \label{fig:3-3}]{
			\begin{minipage}[t]{0.29\linewidth}
				\centering
				\includegraphics[width=0.7\linewidth]{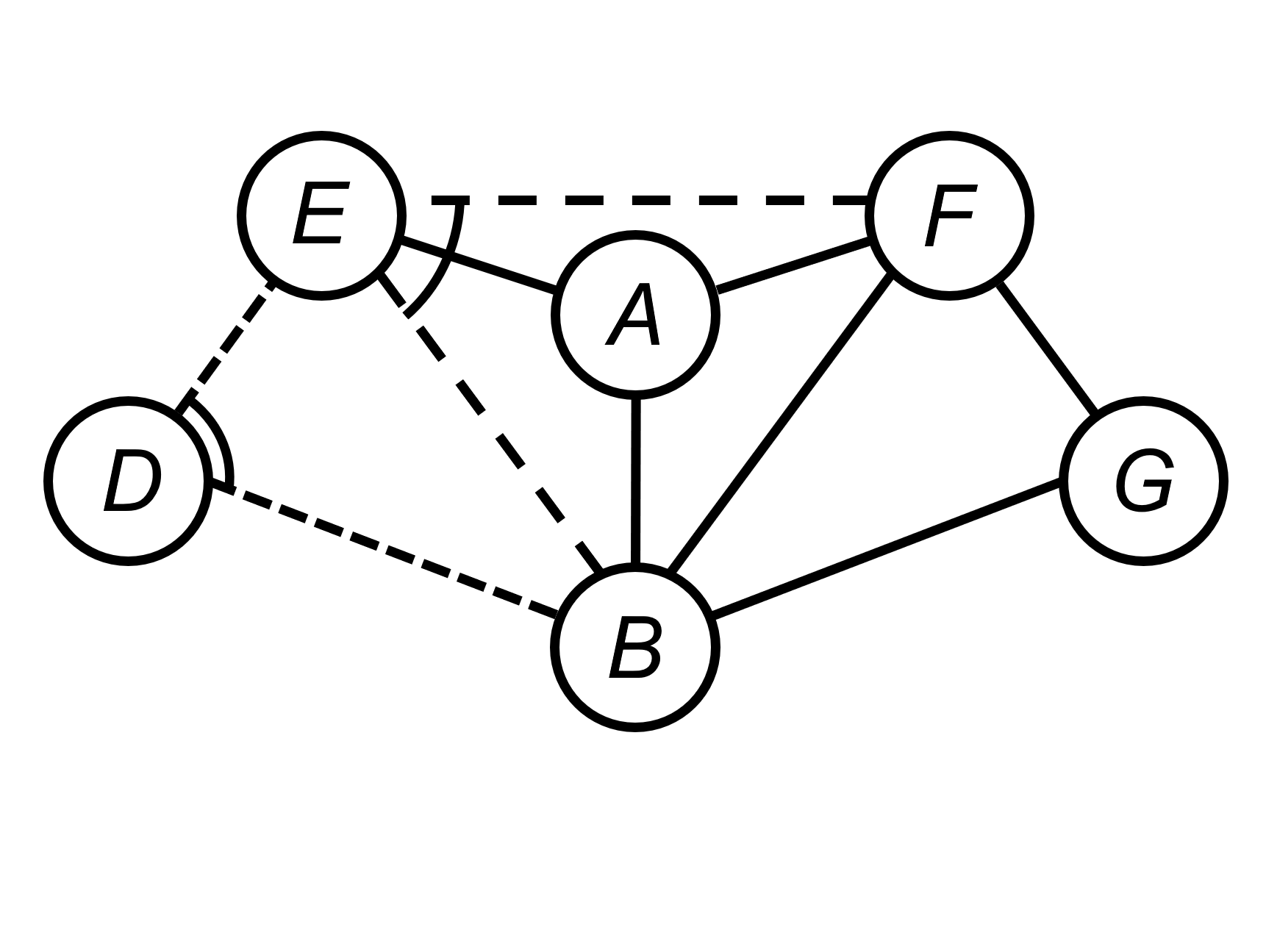}
			\end{minipage}%
		}%
		
		\subfloat[After removing $G$  \label{fig:3-4}]{
			\begin{minipage}[t]{0.29\linewidth}
				\centering
				\includegraphics[width=0.7\linewidth]{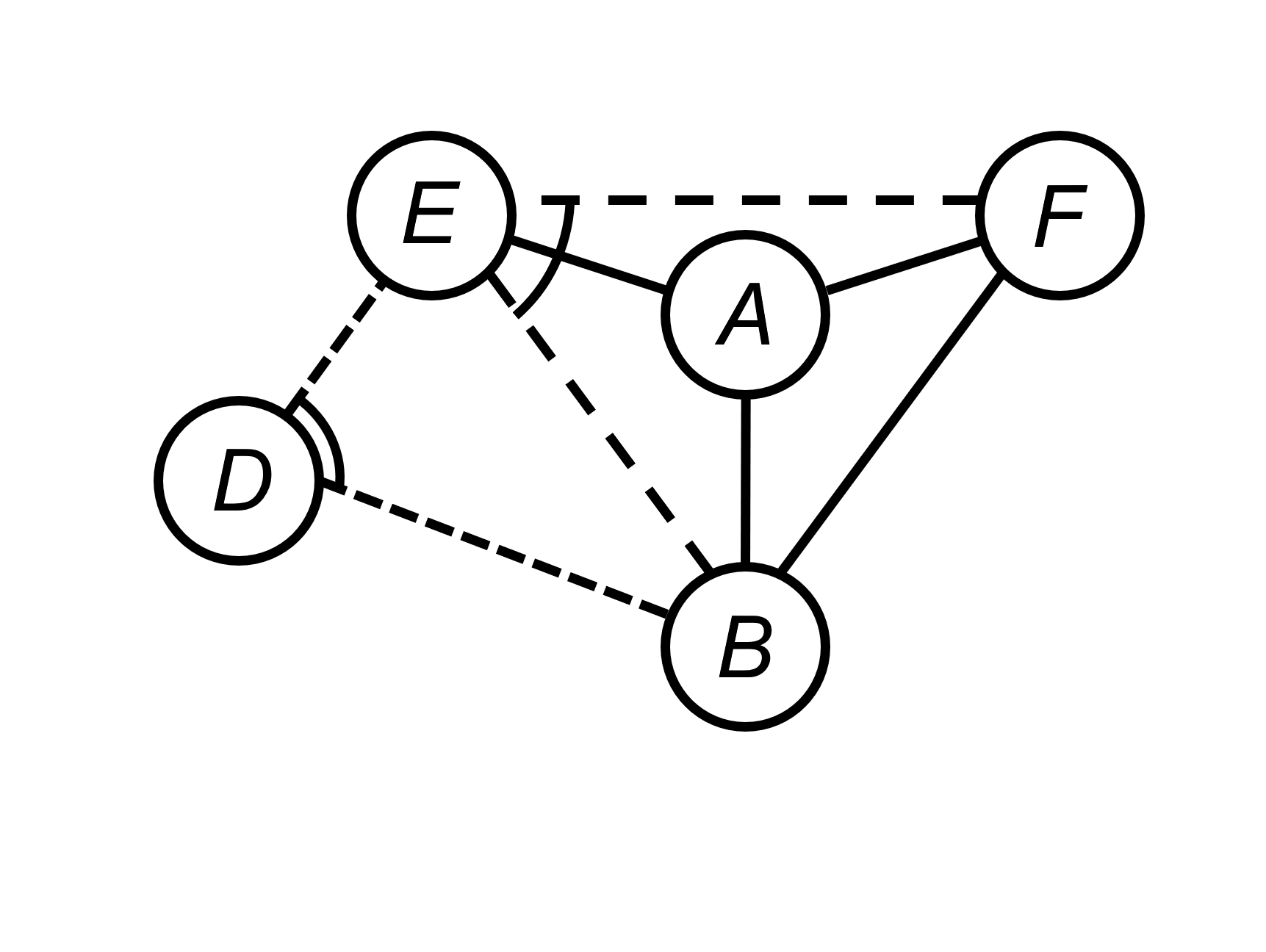}
			\end{minipage}%
		}%
		\hspace{0.03\linewidth}
		\subfloat[The MOC for $A$  \label{fig:3-5}]{
			\begin{minipage}[t]{0.29\linewidth}
				\centering
				\includegraphics[width=0.7\linewidth]{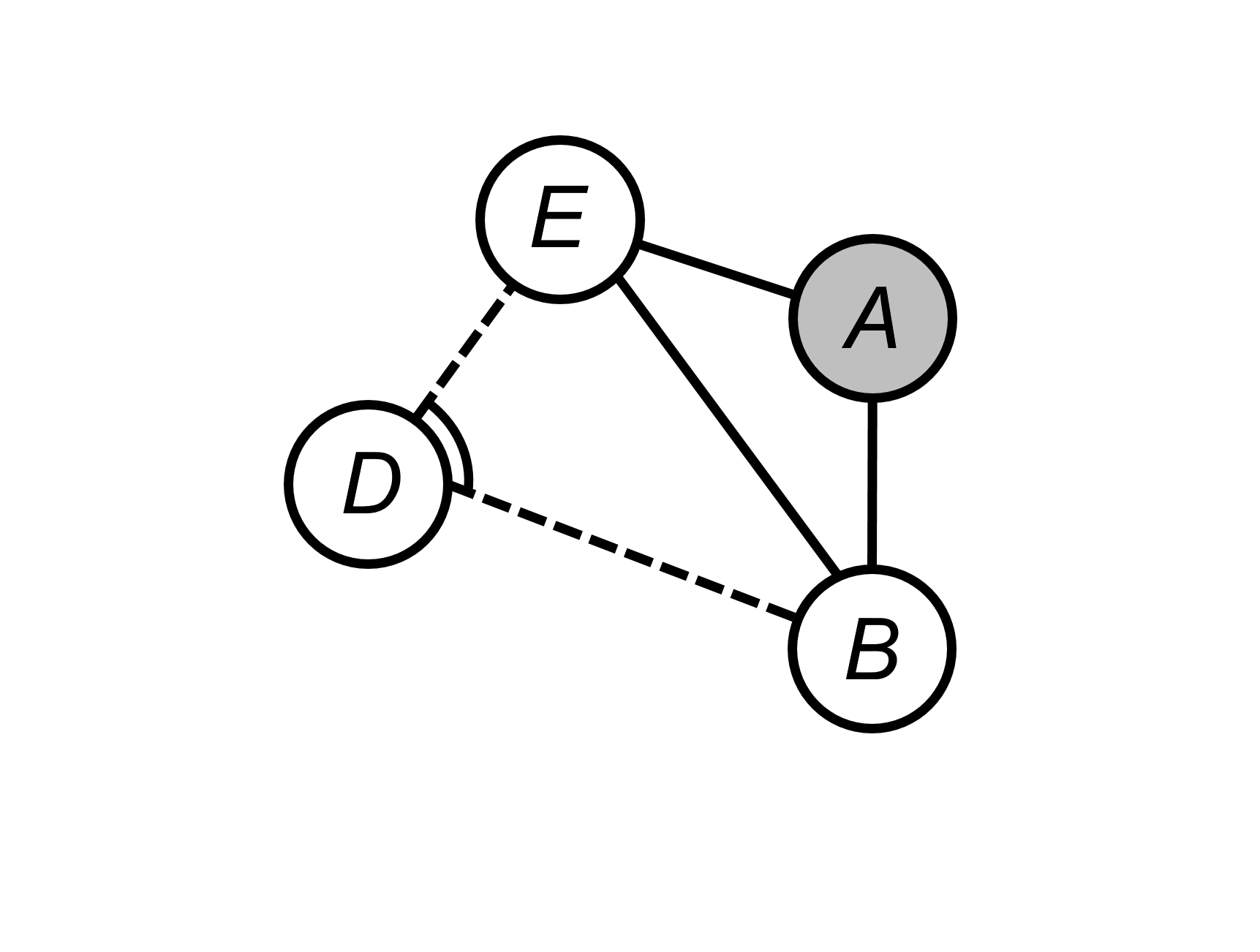}
			\end{minipage}%
		}%
		\hspace{0.03\linewidth}
		\subfloat[The MOC for $F$  \label{fig:4-1}]{
			\begin{minipage}[t]{0.29\linewidth}
				\centering
				\includegraphics[width=0.7\linewidth]{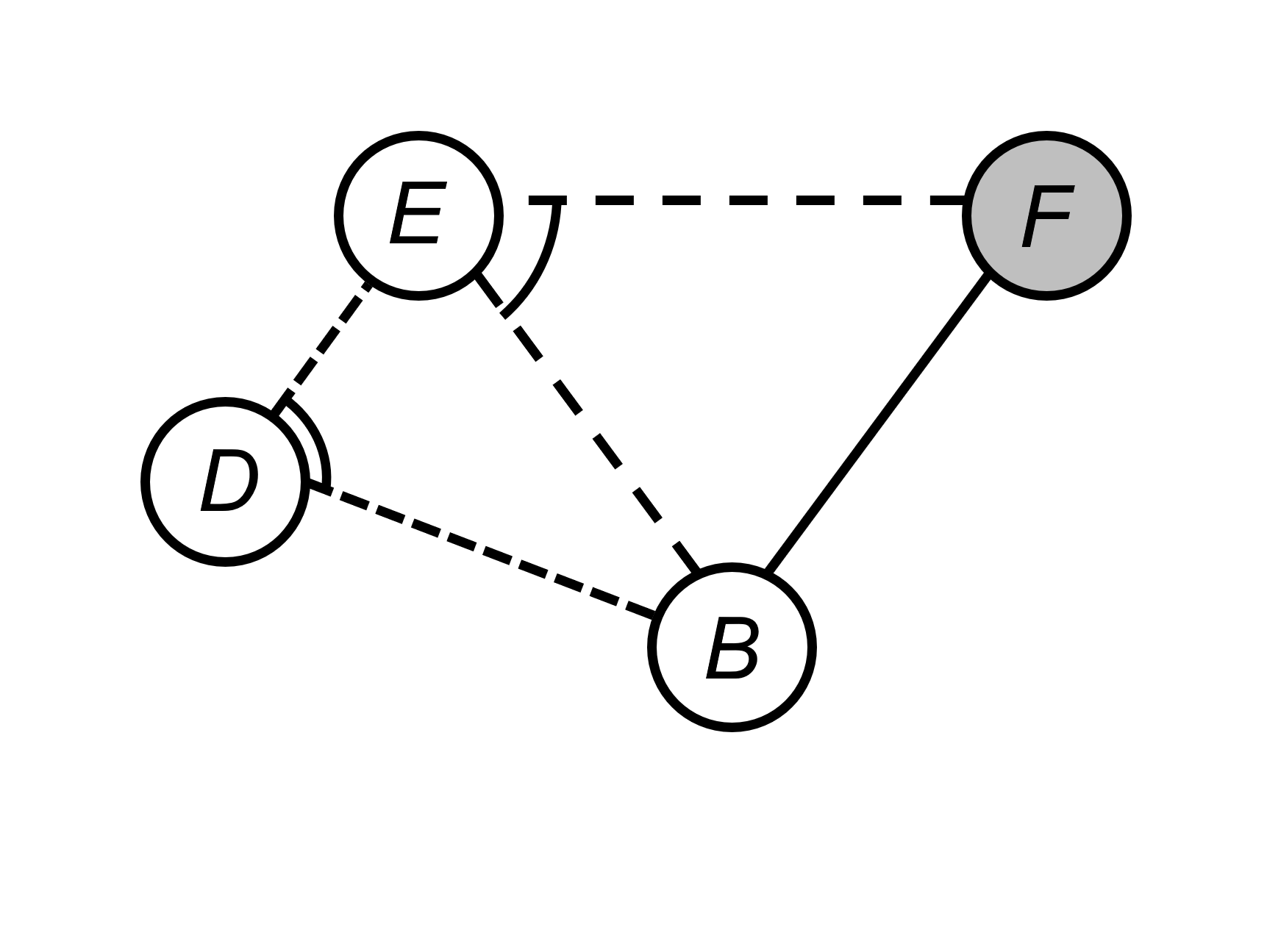}
			\end{minipage}%
		}%

		\subfloat[The MOC for $G$    \label{fig:4-2}]{
			\begin{minipage}[t]{0.29\linewidth}
				\centering
				\includegraphics[width=0.7\linewidth]{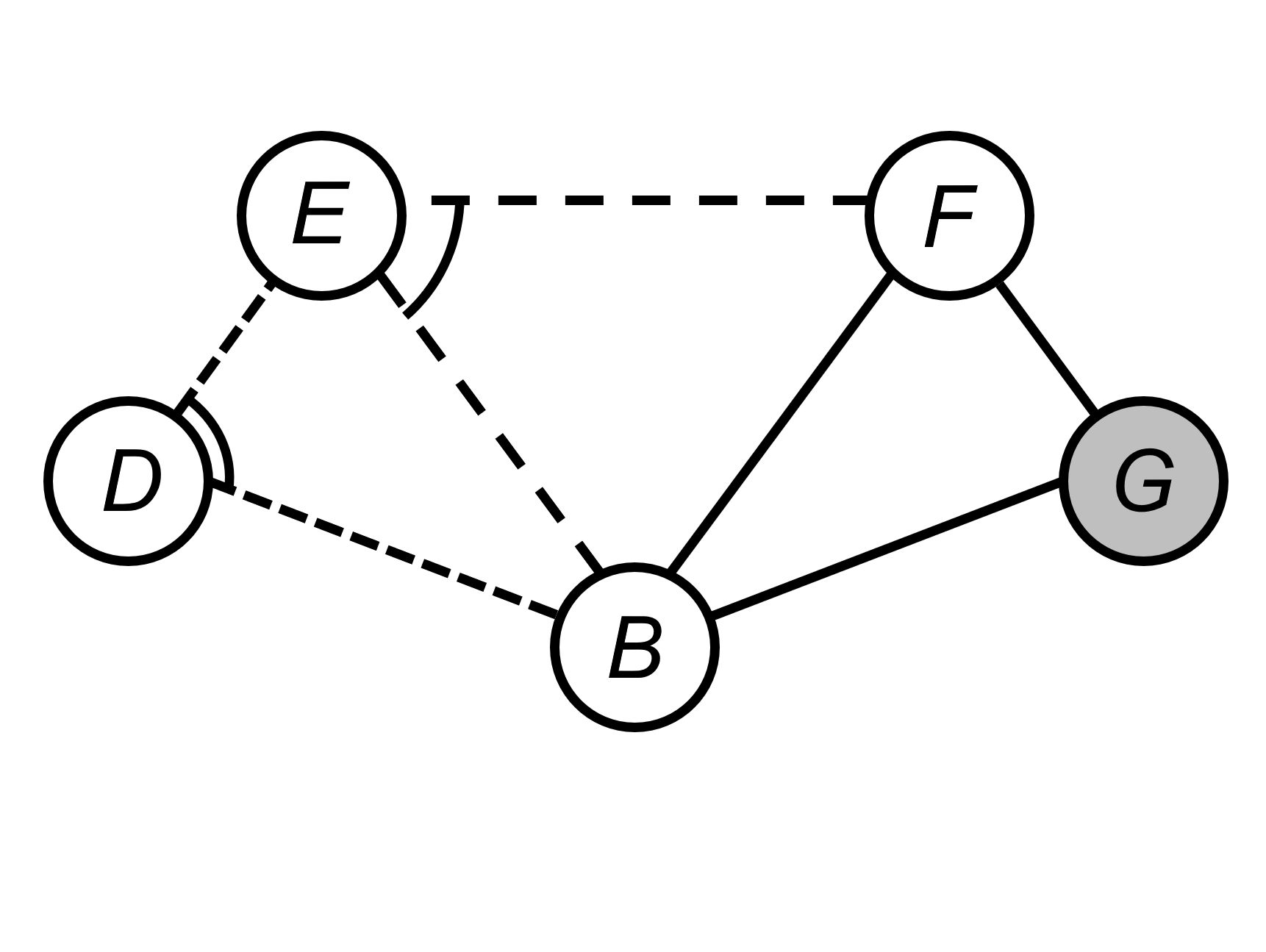}
			\end{minipage}%
		}%
		\hspace{0.03\linewidth}
		\subfloat[The MOC for $H$   \label{fig:4-3}]{
			\begin{minipage}[t]{0.29\linewidth}
				\centering
				\includegraphics[width=0.7\linewidth]{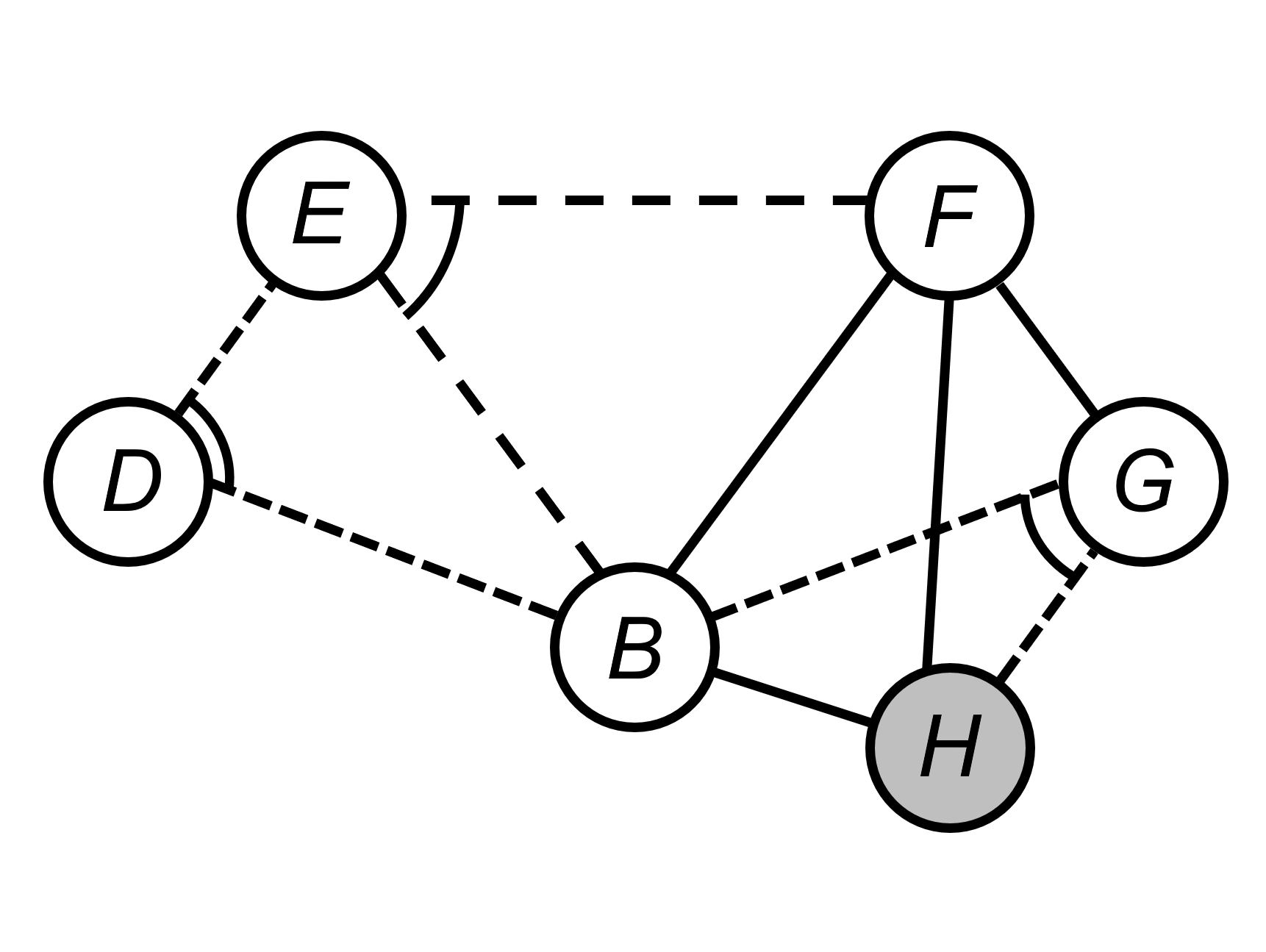}
			\end{minipage}%
		}%
		\hspace{0.03\linewidth}
		\subfloat[The MPDAG $\cal H$  \label{fig:4-4}]{
			\begin{minipage}[t]{0.29\linewidth}
				\centering
				\includegraphics[width=0.7\linewidth]{fig/fig4-5.png}
			\end{minipage}%
		}%
		
		\caption{An illustrative example to show how to use Algorithm~\ref{algo:moc} to find the MPDAG.}
		\label{fig:find_moc}
	\end{figure}

	\begin{example}
		\label{ex:moc}
		Figure~\ref{fig:3-1} shows a CPDAG ${\cal G}^*$ and a set of DCCs $\cal K$ consisting of $D\tor \{E, B\}$, $E\tor \{A, C\}$, $E\tor \{B, F\}$ and $G\tor \{B, H\}$. We first show how to use Algorithm \ref{algo:moc} to find the maximal orientation component for $A$. Note that, since ${\cal G}^*$ is undirected, ${\cal G}^*_u={\cal G}^*$ and  ${\cal K}(\mathcal{G}^*_u)={\cal K}$. As $C,A,H$ are potential leaf nodes in $\mathcal{G}^*_u$, $\mathcal{G}^*_u$ is not an orientation component for $A$, which triggers the while-loop. In the first loop, we remove $C$ and the edges connected to $C$ first, resulting the undirected graph ${\cal U}_m$ shown in Figure \ref{fig:3-2}, and then remove $E\tor \{A, C\}$ from ${\cal K}$ as it has a head not in ${\cal U}_m$. The remaining clauses are $D\tor \{E, B\}$, $E\tor \{B, F\}$ and $G\tor\{B, H\}$, which are visualized by arcs in Figure \ref{fig:3-2}. Since $H$ is still a potential leaf node in ${\cal U}_m$, we remove $H$ and the edges connected to $H$ from the current graph, and remove $G\tor \{B, H\}$ from the current set of clauses. The result is shown in Figure \ref{fig:3-3}. In the next two loops, we sequentially remove $G$ (Figure \ref{fig:3-4}) and $F$ (Figure \ref{fig:3-5}). Finally, we have the undirected graph shown in Figure \ref{fig:3-5}, which is the maximal orientation component for $A$. Similarly, using Algorithm \ref{algo:moc} we can find the maximal orientation components for $F$, $G$ and $H$ separately, which are shown in Figures \ref{fig:4-1} to \ref{fig:4-3}, respectively. Note that, The maximal orientation components for the remaining variables are all singleton graphs. Therefore, by Corollary \ref{coro:maximaloc}, $E\to A$, $B \to A$, $E\to F$, $B\to F$, $F\to G$, $B\to G$, $B\to H$, $F\to H$, and $G\to H$ are all and only directed edges in the MPDAG. The resulting MPDAG is given in Figure~\ref{fig:4-4}.
	\end{example}

	{
		\begin{algorithm}[!h]
			\caption{ Finding the MPDAG and the minimal residual set (Decomposing pairwise causal constraints).}
			\label{algo:decomp}
			\begin{algorithmic}[1]
				\REQUIRE
				A CPDAG $\mathcal{G}^*$, a consistent causal constraint set $\cal B$.
				\ENSURE
				The MPDAG of $[\mathcal{G}^*, {\cal B}]$, a cardinality-minimal residual set of DCCs $\cal R$, and the unique element-wise head-minimal residual set of DCCs ${\cal R}^*$.
				\STATE {Construct the equivalent DCCs ${\cal K}$ based on Theorem \ref{thm:nbr_set_const}.}
				\STATE {Let ${\cal H}=\mathcal{G}^*$.}
				\FOR {$X\in \mathbf{V}(\mathcal{G}^*)$,}
				\STATE {Find the maximal orientation component for $X$ according to Algorithm~\ref{algo:moc} and denote it by $\cal U$.}
				\FOR {$Y\in adj(X, {\cal U})$,}
				\STATE {Replace $Y-X$ by $Y\to X$ in ${\cal H}$.}
				\ENDFOR
				\ENDFOR		
				\STATE{Set ${\cal R} = {\cal K}$.}
				\WHILE {${\cal R}\neq\varnothing$,}
				\IF  {there exists a $\kappa\in {\cal R}$ such that $\cup_{D\in \kappa_h}\{D\to \kappa_t\}\cup{\bf E}_d({\cal H}) \cup \left({\cal R}\setminus\{\kappa\}\right)$ is inconsistent with ${\cal G}^*$,}
				\STATE {Set ${\cal R}={\cal R}\setminus\{\kappa\}$.}
				\ELSE
				\STATE {\bf break}
				\ENDIF
				\ENDWHILE
				\STATE{Set ${\cal R}^* = {\cal R}$.}
				\FOR {$\kappa\in {\cal R}^*$}
				\WHILE {there exists a $D\in \kappa_h$ such that $\cup_{D'\in \kappa_h\setminus\{D\}}\{D'\to \kappa_t\}\cup {\cal R}^*$ is inconsistent with ${\cal G}^*$,}
				\STATE {Set $\kappa_h=\kappa_h\setminus\{D\}$.}
				\ENDWHILE
				\ENDFOR 		
				\RETURN {${\cal H}$, ${\cal R}$ and ${\cal R}^*$.}
			\end{algorithmic}
		\end{algorithm}
	}

	{Algorithm~\ref{algo:decomp} summarizes the procedure for decomposing pairwise causal constraints. Lines 2 to 8 describe how to use Algorithm~\ref{algo:moc} to construct the MPDAG. To obtain a cardinality-minimal residual set of DCCs $\cal R$, we iteratively remove redundant DCCs $\kappa$ from $\cal K$, as illustrated in lines 9 to 16. According to Theorem~\ref{thm:equi_dcc}, $\kappa$ is redundant with respect to ${\bf E}_d({\cal H})\cup\left({\cal R}\setminus\{\kappa\}\right)$ if and only if $\cup_{D\in \kappa_h}\{D\to \kappa_t\}\cup{\bf E}_d({\cal H}) \cup \left({\cal R}\setminus\{\kappa\}\right)$ is inconsistent with ${\cal G}^*$. Algorithm~\ref{algo:consistency-peo} can be applied to check such consistency. Lines 17 to 22 detail the procedure for obtaining the unique element-wise head-minimal residual DCC set. Starting from the cardinality-minimal set $\cal R$, we prune the heads of each DCC. Specifically, for each $\kappa \in \cal R$, we check whether there exists a $D \in \kappa_h$ such that $\cup_{D' \in \kappa_h \setminus \{D\}} \{D' \to \kappa_t\} \cup {\cal R}$ is inconsistent with ${\cal G}^{*}$. If so, $\kappa_t \tor \kappa_h \setminus \{D\}$ can be inferred from $\cal R$. Since $\kappa_t \tor \kappa_h \setminus \{D\}$ implies $\kappa_t \tor \kappa_h$, replacing $\kappa$ in $\cal R$ with $\kappa_t \tor \kappa_h \setminus \{D\}$ yields a DCC set equivalent to $\cal R$. We repeat this pruning procedure until no further reduction is possible. After pruning all DCCs in $\cal R$, we obtain an element-wise head-minimal DCC set. The correctness of this procedure is guaranteed by Definition~\ref{def:minimal_dcc} and statements (iii) and (iv) of Theorem~\ref{prop:eqdecomp}.}
	
	According to~\cite{fang2021local}, the complexity of line 1 is at most $O(|{\cal B}|\cdot|{\bf V}(\mathcal{G}^*)|^3)$. Since Algorithm \ref{algo:moc} is applied to each vertex separately to find the MPDAG, the complexity of lines 3 to 8 is upper bounded by $O(|{\bf V}(\mathcal{G}^*)|^5+|{\cal K}|\cdot|{\bf V}(\mathcal{G}^*)|^3)$. Next, the while-loop in lines 10 to 16 runs at most $|{\cal K}|$ times, with each iteration taking at most $O(|{\bf V}(\mathcal{G}^*)|^4+(|{\cal K}|+|{\bf V}(\mathcal{G}^*)|)\cdot|{\bf V}(\mathcal{G}^*)|^2)$ time. Thus, the total complexity of this while-loop is bounded by $O(|{\cal K}|\cdot|{\bf V}(\mathcal{G}^*)|^4+(|{\cal K}|^2+|{\cal K}|\cdot|{\bf V}(\mathcal{G}^*)|)\cdot|{\bf V}(\mathcal{G}^*)|^2)$. Similarly, the for-loop in lines 18 to 22 has complexity bounded by $O(|{\cal K}|\cdot|{\bf V}(\mathcal{G}^*)|^5+(|{\cal K}|^2\cdot|{\bf V}(\mathcal{G}^*)|+|{\cal K}|\cdot|{\bf V}(\mathcal{G}^*)|^2)\cdot|{\bf V}(\mathcal{G}^*)|^2)$. In summary, the overall time complexity of Algorithm~\ref{algo:decomp} is polynomial in $|{\bf V}(\mathcal{G}^*)|$, $|{\cal K}|$ and $|{\cal B}|$. Since the sizes of $\cal B$ and $\cal K$ are both at most quadratic in the number of vertices, the worst-case time complexity of Algorithm~\ref{algo:decomp} is bounded by $O(|{\bf V}(\mathcal{G}^*)|^7)$.

	\section{Causal Inference with  Background Knowledge}\label{sec:causal}
	In this section, we study the causal inference problem when pairwise causal background knowledge is available. It is well-known that the causal effect of a treatment (or multiple treatments) on a response (or multiple responses) may not be identifiable given a CPDAG, while
	additional information in background knowledge could make the causal effect  identifiable or less uncertain.
	In Section~\ref{sec:sec:id}, we study the identification condition of a causal effect under restricted Markov equivalence represented by a CPDAG and a set of DCCs, and in Section~\ref{sec:sec:ida}, when a causal effect is unidentifiable, we further extend the IDA framework to estimate all possible causal effects.
	\subsection{Identifiability}\label{sec:sec:id}
	
	We first present the definition of causal effect identifiability. This definition is a generalization of~ \citet[Definition~3.1]{Perkovic2020mpdag} to the restricted Markov equivalence class induced by a CPDAG and background knowledge.
	
	\begin{definition}[Causal Effect Identifiability]
		Let $\mathcal{G}^*$ be a CPDAG over the vertex set $\mathbf{V}$ and $\cal B$ be a consistent pairwise causal constraint set (or a consistent DCC set). Suppose that $\mathbf{X},\mathbf{Y}\subseteq\mathbf{V}$ are  two disjoint vertex sets. The causal effect of $\mathbf{X}$ on $\mathbf{Y}$ is identifiable from $\mathcal{G}^*$ and $\cal B$ (or in $[\mathcal{G}^*, \cal B]$), if and only if  $f(\mathbf{y} \,|\, do(\mathbf{x}))$ is uniquely computable from any observational distribution Markovian to any ${\cal G} \in [\mathcal{G}^*, {\cal B}]$.
	\end{definition}
	
	
	Our main result of identifiability is given below.
	
	
	\begin{theorem}\label{thm:id}
		Let $\mathcal{G}^*$ be a CPDAG and $\cal B$ be a consistent pairwise causal constraint set (or a consistent DCC set). Denote by $\cal H$ the MPDAG of $[\mathcal{G}^*, {\cal B}]$. For any two disjoint vertex sets $\mathbf{X},\mathbf{Y}\subseteq\mathbf{V}(\mathcal{G}^*)$, the causal effect of $\mathbf{X}$ on $\mathbf{Y}$ is identifiable in $[\mathcal{G}^*, \cal B]$ if and only if it is identifiable in $[\mathcal{H}]$.
	\end{theorem}
	
	{ Recall that $[\mathcal{G}^*, {\cal B}] \subseteq [\mathcal{H}]$, Theorem~\ref{thm:id} shows that although $\mathcal{H}$ carries less knowledge than ${\cal B}$, it is enough to identify all identifiable causal effects. In other words, the information that cannot be represented by the MPDAG (that is, the residual set of DCCs defined in Section~\ref{sec:sec:info}) contributes nothing to the identifiability.}
	
	By \citet{Perkovic2020mpdag}, the causal effect of $\mathbf{X}$ on $\mathbf{Y}$ is identifiable in $[\mathcal{H}]$ if and only if every proper possibly causal path from $\mathbf{X}$ to $\mathbf{Y}$ in $\cal H$ starts with a directed edge. Therefore, one can graphically identify causal effects by constructing the MPDAG from $\mathcal{G}^*$ and $\cal B$ using Algorithm~\ref{algo:decomp} first.
	
	
	\begin{example}\label{ex:dcc_unid}
		Consider the MPDAG $\cal H$ with two DCCs $A\tor\{B,C\}$ and $C\tor\{A, X\}$ shown in Figure~\ref{fig:bgk-ida-mpdag}. By Theorem~\ref{thm:id}, the causal effect of $X$ on $Y$ is not identifiable, as $X-A\to Y$ and $X-C\to Y$ are possibly causal paths on which the first edges are undirected. To verify this result, we enumerate all possible parental sets of $X$ in Figures~\ref{fig:bgk-ida-1} to \ref{fig:bgk-ida-4}. It can be seen that the causal effect of $X$ on $Y$ is definitely zero in Figure~\ref{fig:bgk-ida-1}, while the causal effects of $X$ on $Y$ are possibly non-zero in Figures~\ref{fig:bgk-ida-2} to \ref{fig:bgk-ida-4}.
		
		\begin{figure}[!h]
			\centering
			\subfloat[\label{fig:bgk-ida-mpdag}]{
				\begin{minipage}[t]{0.17\linewidth}
					\centering
					\includegraphics[width=0.8\linewidth]{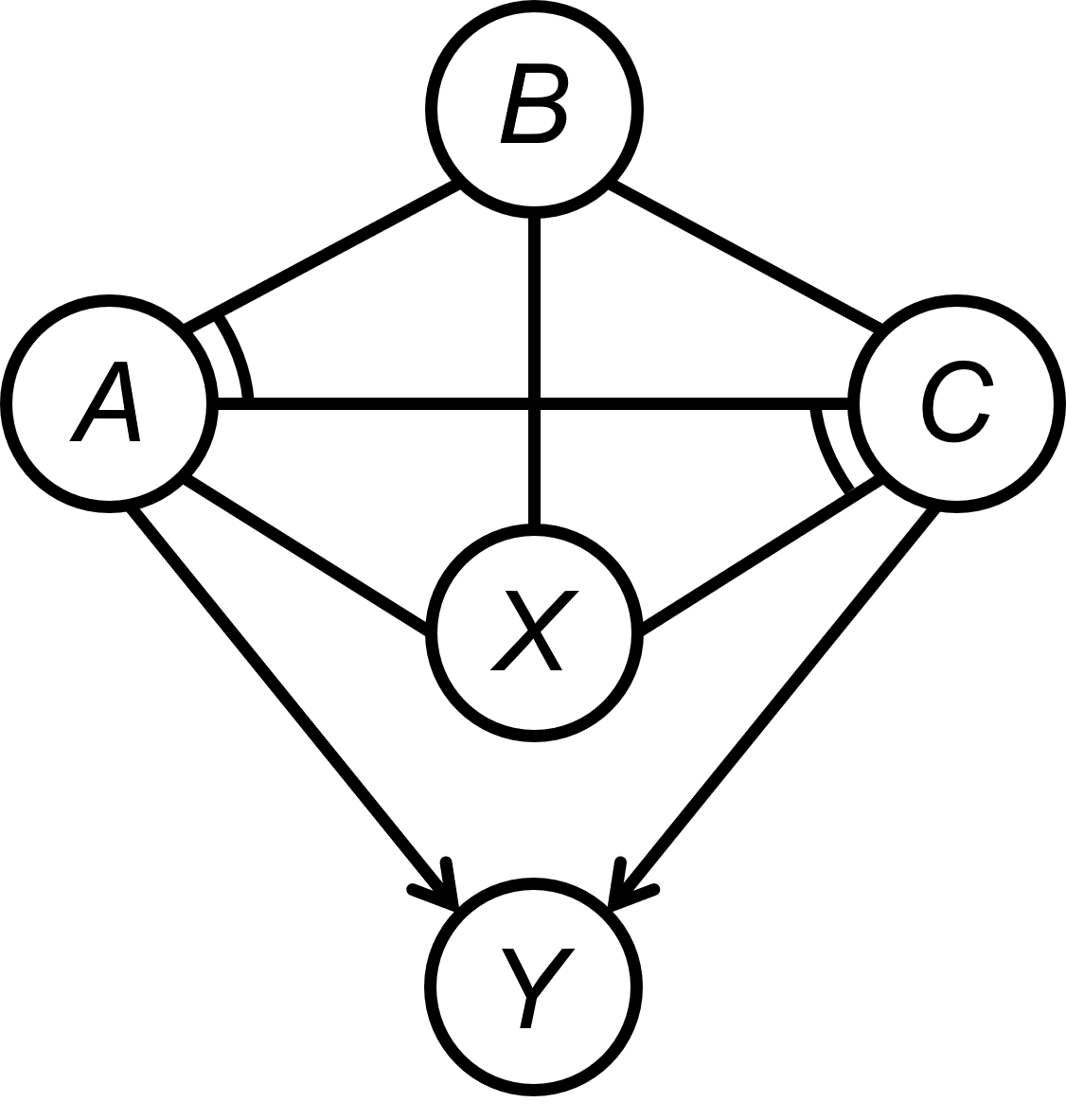}
				\end{minipage}%
			}%
			\hspace{0.01\linewidth}
			\subfloat[\label{fig:bgk-ida-1}]{
				\begin{minipage}[t]{0.17\linewidth}
					\centering
					\includegraphics[width=0.8\linewidth]{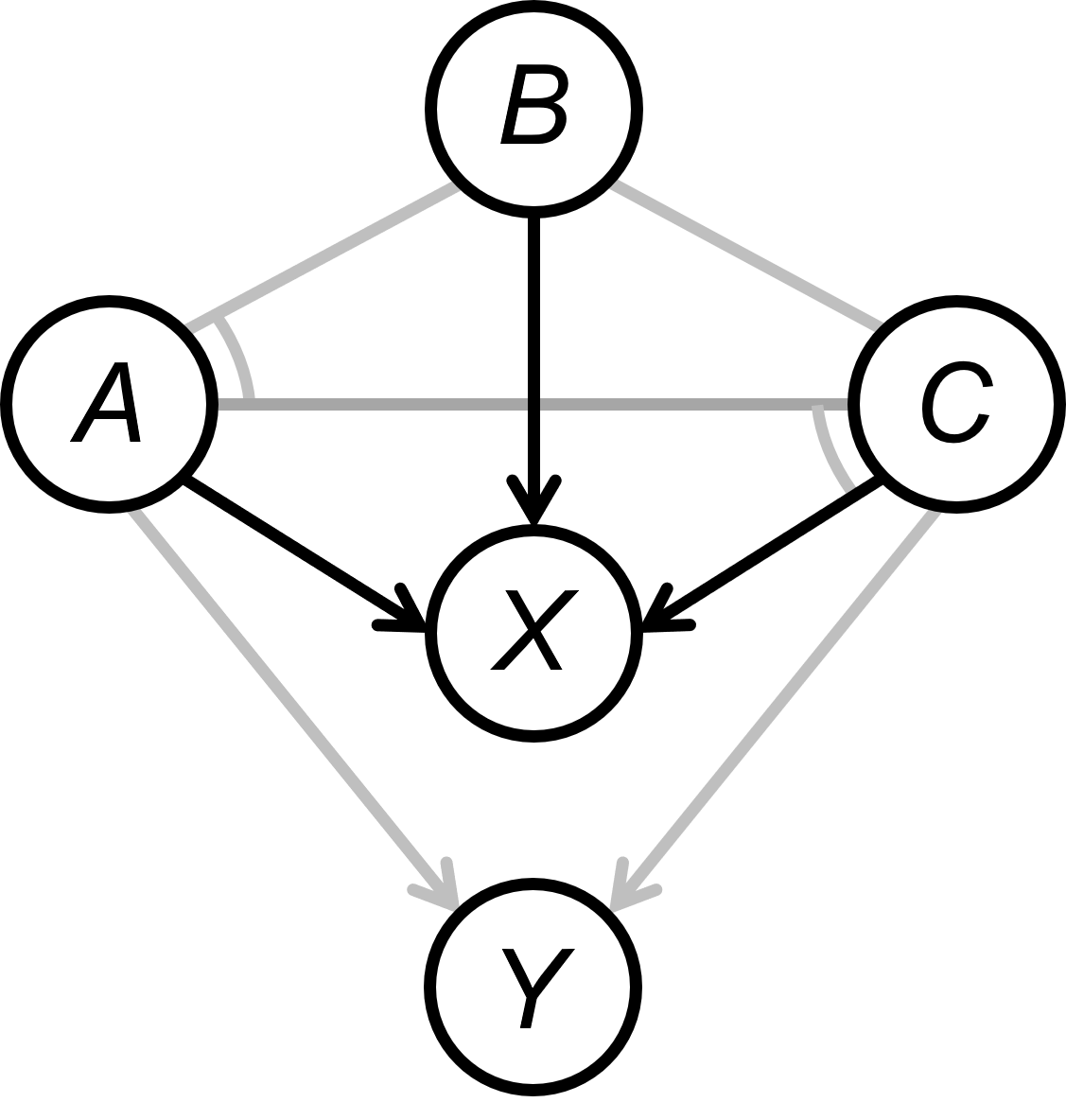}
				\end{minipage}%
			}%
			\hspace{0.01\linewidth}
			\subfloat[\label{fig:bgk-ida-2}]{
				\begin{minipage}[t]{0.17\linewidth}
					\centering
					\includegraphics[width=0.8\linewidth]{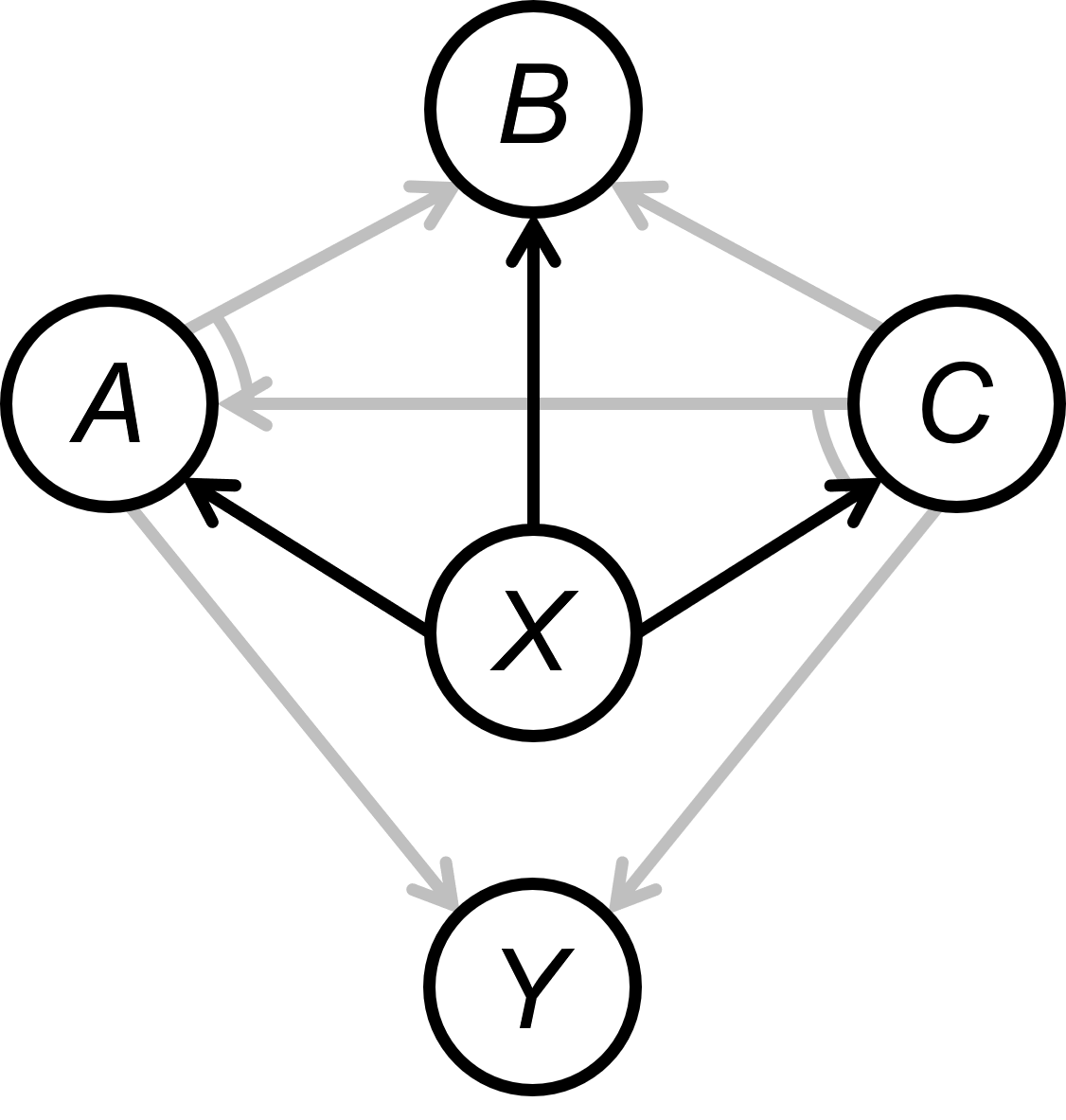}
				\end{minipage}%
			}%
			\hspace{0.01\linewidth}
			\subfloat[\label{fig:bgk-ida-3}]{
				\begin{minipage}[t]{0.17\linewidth}
					\centering
					\includegraphics[width=0.8\linewidth]{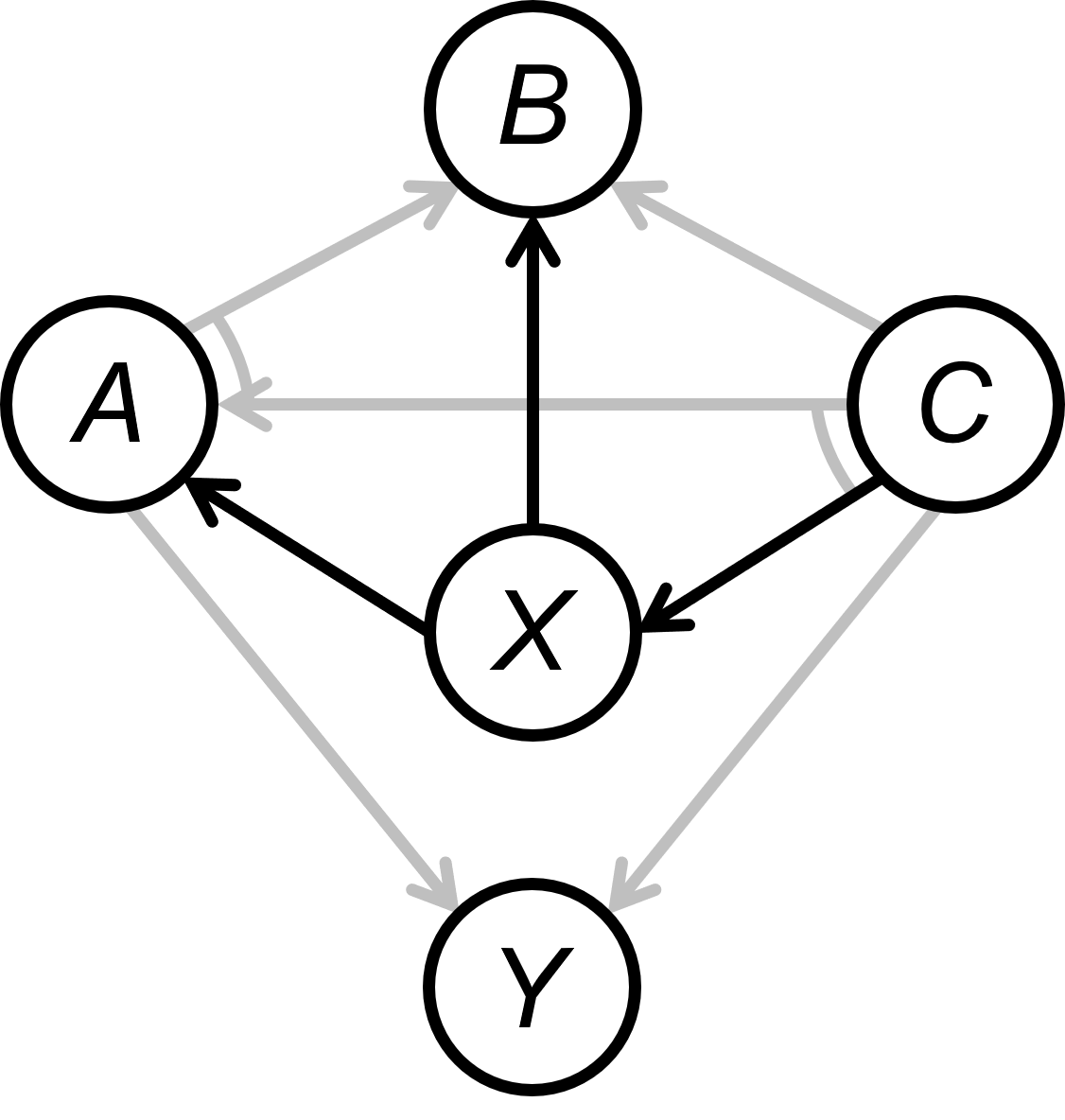}
				\end{minipage}%
			}%
			\hspace{0.01\linewidth}
			\subfloat[\label{fig:bgk-ida-4}]{
				\begin{minipage}[t]{0.17\linewidth}
					\centering
					\includegraphics[width=0.8\linewidth]{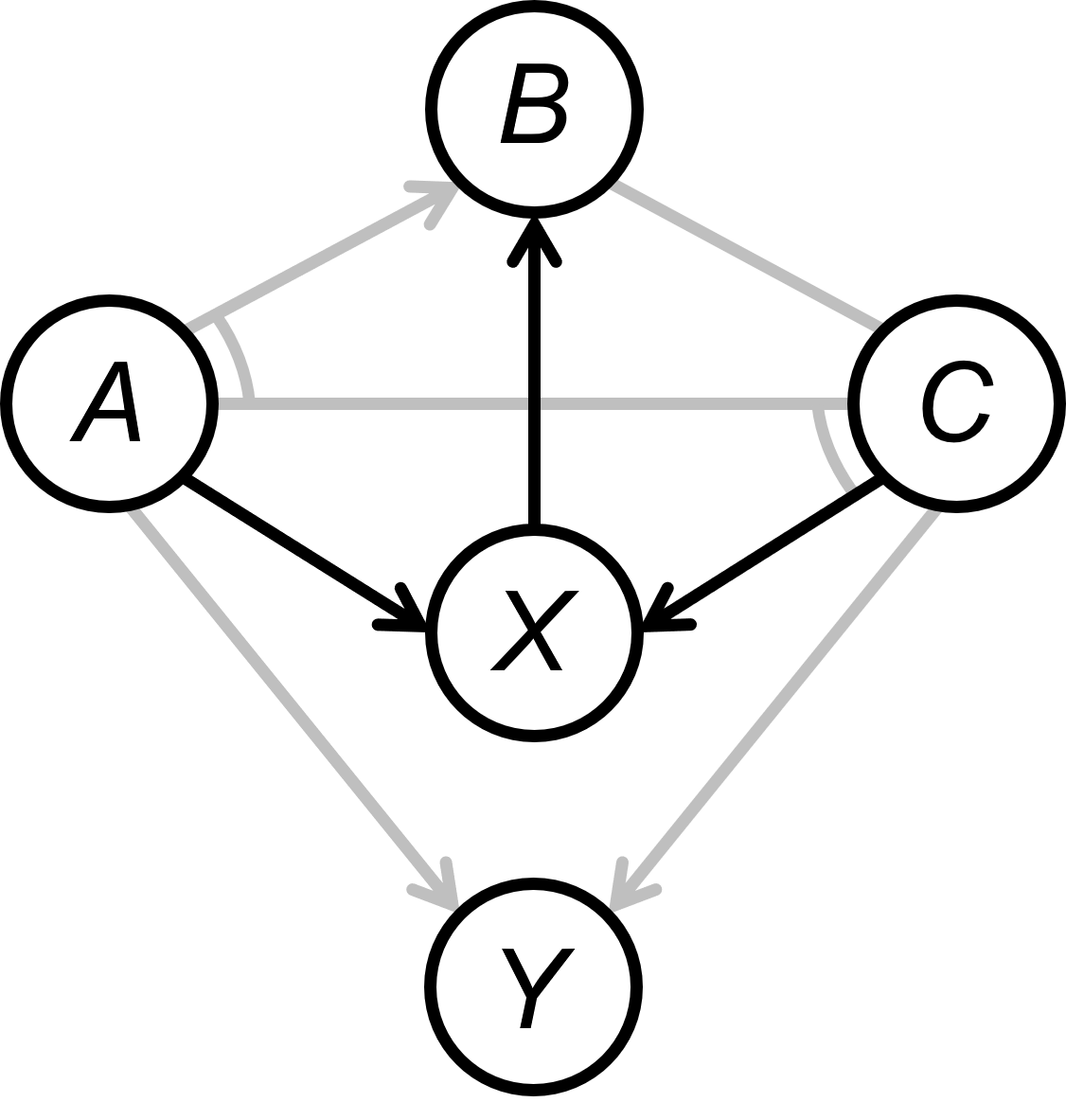}
				\end{minipage}%
			}%

			\caption{An MPDAG $\cal H$ with two DCCs (depicted by arcs) is shown in Figure~\ref{fig:bgk-ida-mpdag}, and the four possible parental sets of $X$ are shown in Figures~\ref{fig:bgk-ida-1} to \ref{fig:bgk-ida-4}.}
			\label{fig:possible}
		\end{figure}
	\end{example}
	
	{\citet{Perkovic2020mpdag} proved the following formula to calculate identifiable causal effects in an MPDAG. This formula can be directly applied in our setting as Theorem~\ref{thm:id} shows that the causal effect identification condition for $\mathcal{G}^*$ and $\cal B$ is the same as that for the MPDAG $\cal H$ represents $[\mathcal{G}^*, {\cal B}]$. Suppose that the causal effect of $\bf X$ on $\bf Y$ is identifiable in the MPDAG $\cal H$, then for any distribution $f$ Markovian to $\cal G$, it holds that
		\[f({\bf y}\mid do({\bf x}))=\int \prod_{i=1}^k f({\bf b}_i\mid pa({\bf b}_i, {\cal H})) d{\bf b},\]
		for values $pa({\bf b}_i, {\cal H})$ that are in agreement with $\bf x$, where $({\bf B}_1, ..., {\bf B}_k)=PCO(an({\bf Y}, {\cal H}({\bf V}\setminus {\bf X}))$ and ${\bf B}=an({\bf Y}, {\cal H}({\bf V}\setminus {\bf X}))\setminus {\bf Y}$. The definition of the operator $PCO$ can be found in~\citet[][Algorithm~1]{Perkovic2020mpdag}. This formula generalizes the g-formula of~\citet{robins1986}, the truncated factorization formula of~\citet{pearl2009causality}, or the manipulated density formula of~\citet{spirtes2000causation} to MPDAGs.}

	
	The following corollary of  Theorem~\ref{thm:id} provides conditions under which a pairwise causal background knowledge set can definitely increase the number of identifiable effects.
	
	
	\begin{corollary}\label{coro:id}
		Suppose that $\mathcal{G}^*$ is a CPDAG over the vertex set $\mathbf{V}$ and $\cal K$ is a set of DCCs consistent with $\mathcal{G}^*$, then the following two statements are equivalent.
		\begin{enumerate}
			\item[(i)]  At least one unidentifiable effect in $[\mathcal{G}^*]$ becomes identifiable in $[\mathcal{G}^*, {\cal K}]$.
			\item[(ii)]  The DAGs in $[\mathcal{G}^*, {\cal K}]$ have at least one common direct causal relation that is not encoded by a directed edge in $\mathcal{G}^*$.
		\end{enumerate}
		In particular, if $\cal K$ is derived from a consistent pairwise causal background knowledge set $\cal B$ and there is a direct or non-ancestral causal constraint in $\cal B$ which does not hold for all DAGs in $[\mathcal{G}^*]$, or there is an ancestral causal constraint $X\dashrightarrow Y$ in $\cal B$ such that $Y\longarrownot\dashrightarrow X$ does not hold for all DAGs in $\mathcal{G}^*$, then at least one unidentifiable effect in $[\mathcal{G}^*]$ becomes identifiable in $[\mathcal{G}^*, {\cal K}]$.
	\end{corollary}
	
	

	
	Another approach to calculate causal effect is through adjustment. The following theorem gives a sound and complete adjustment criterion.
	
	
	\begin{theorem}\label{thm:adjustment}
		Let $\mathcal{G}^*$ be a CPDAG over the vertex set $\mathbf{V}$ and $\cal B$ be a consistent pairwise causal constraint set (or a consistent DCC set). Denote by $\cal H$ the MPDAG of $[\mathcal{G}^*, {\cal B}]$. For any pairwise disjoint vertex sets $\mathbf{X},\mathbf{Y}, \mathbf{Z}\subseteq\mathbf{V}$, $\mathbf{Z}$ is an adjustment set for $(\mathbf{X},\mathbf{Y})$ with respect to $\mathcal{G}^*$ and $\cal B$ if and only if $\mathbf{Z}$ satisfies the b-adjustment criterion relative to $(\mathbf{X},\mathbf{Y})$ in $\cal H$.
	\end{theorem}
	
	The definitions of the adjustment set and the b-adjustment criterion in Theorem \ref{thm:adjustment} are given below.
	
	\begin{definition}[Adjustment Set]\label{def:adjustment}
		Let $\mathcal{G}^*$ be a CPDAG over the vertex set $\mathbf{V}$ and $\cal B$ be a consistent pairwise  causal constraint set (or a consistent DCC set). Suppose that $\mathbf{X},\mathbf{Y},\mathbf{Z}\subseteq\mathbf{V}$ are pairwise disjoint vertex sets. Then, $\mathbf{Z}$ is called an adjustment set for $(\mathbf{X},\mathbf{Y})$ with respect to $\mathcal{G}^*$ and $\cal B$ if for any DAG ${\cal G} \in [\mathcal{G}^*, {\cal B}]$ and observational distribution $f$ Markovian to ${\cal G}$, the interventional distribution $f(\mathbf{y} \,|\, do(\mathbf{x}))$ can be calculated by
		\begin{eqnarray}\label{leq3}
			f( \mathbf{y} \,|\, do(\mathbf{x}))
			=\left\{
			\begin{array}{cr}
				\int f(\mathbf{y} \,|\, \mathbf{x}, \mathbf{z})f(\mathbf{z})d\mathbf{z}, & \mathrm{if}\ \mathbf{Z}\neq\varnothing, \\
				f( \mathbf{y} \,|\, \mathbf{x}),& \mathrm{otherwise}.
			\end{array}
			\right.
		\end{eqnarray}
	\end{definition}
	
	\begin{definition}[b-Adjustment Criterion,~\citealt{perkovic2017interpreting}]\label{def:b-adjustment}
		Let $\mathbf{X},\mathbf{Y}, \mathbf{Z}\subseteq\mathbf{V}$ be pairwise disjoint vertex sets in an MPDAG $\cal H$, and $\mathrm{Forb}(\mathbf{X},\mathbf{Y}, {\cal H})$ be the set of variables which are possible descendants of some $W\notin\mathbf{X}$ lying on a proper possibly causal path from $\mathbf{X}$ to $\mathbf{Y}$ in $\cal H$. Then, $\mathbf{Z}$ satisfies the b-adjustment criterion relative to $(\mathbf{X},\mathbf{Y})$ in $\cal H$ if: (i) all proper possibly causal paths from $\mathbf{X}$ to $\mathbf{Y}$ start with a directed edge in $\cal H$, (ii) $\mathbf{Z}\cap \mathrm{Forb}(\mathbf{X},\mathbf{Y}, {\cal H})=\varnothing$, and (iii) all proper definite status non-causal paths from $\mathbf{X}$ to $\mathbf{Y}$ are blocked by $\mathbf{Z}$ in ${\cal H}$.
	\end{definition}
	
	Notice that the adjustment criterion in Theorem~\ref{thm:adjustment} is the same as that in  \citet[Theorem~4.4]{perkovic2017interpreting} for $[\cal H]$ induced by an MPDAG $\cal H$. { As discussed by~\cite{Perkovic2020mpdag}, when $\bf X$ or $\bf Y$ are non-singleton sets, there exist examples showing that some causal effects of $\bf X$ on $\bf Y$ cannot be identified using adjustment, meaning that not all causal effects can be identified through adjustment. However, under the condition that both $\bf X$ and $\bf Y$ are singleton sets, the adjustment is sufficient and necessary for identifying causal effects. }
	
	Using Theorem~\ref{thm:adjustment}, the results on optimal adjustment sets for MPDAGs can be naturally extended to general pairwise causal background knowledge.  The details can be found in   \citet{2019arXiv190702435H, Andrea2019}.

	\subsection{Estimating Possible Causal Effects}\label{sec:sec:ida}

	When a causal effect is not identifiable, we can estimate its bounds by enumerating all possible causal effects. Based on the proposed Algorithm~\ref{algo:consistency-peo} for checking consistency, it is straightforward to extend the semi-local IDA~\citep[Algorithm~2]{perkovic2017interpreting} to estimate all possible causal effects of multiple treatments on multiple responses. Moreover, by Theorem~\ref{thm:adjustment}, the optimal IDA~\citep{Witte2020efficient} and the minimal IDA~\citep{Guo2020minimal} can be similarly extended. Thus, in the following, we mainly focus on extending the IDA framework to fully-locally estimate all possible causal effects of a single treatment on a single target. The key result is the following local orientation rules for CPDAGs with DCCs.
	
	\begin{theorem}\label{thm:local}
		Let $\cal K$ be a DCC set consistent with a CPDAG ${\cal G}^*$, and $\mathcal{H}$ be the MPDAG of $[{\cal G}^*, {\cal K}]$. For any vertex $X$ and $\textbf{S}\subseteq sib(X, \mathcal{H})$, the following statements are equivalent.
		\begin{enumerate}
			\item[(i)]  There is a DAG $\mathcal{G}$ in $[{\cal G}^*, {\cal K}]$ such that $pa(X, \mathcal{G}) = \textbf{S}\cup pa(X, \mathcal{H})$ and $ch(X, \mathcal{G}) = sib(X, \mathcal{H})\cup ch(X, \mathcal{H})\setminus \textbf{S}$.
			\item[(ii)]  The restriction subset of ${\cal K}\cup D_X$ on ${\cal G}^*(\{X\}\cup sib(X, {\cal G}^*))$ is consistent with ${\cal G}^*(\{X\}\cup sib(X, {\cal G}^*))$, where $D_X\coloneqq\{u\to X \mid u\in pa(X, \mathcal{H}) \cup \textbf{S}\} \cup \{X\to v \mid v\in sib(X, \mathcal{H})\cup ch(X, \mathcal{H})\setminus \textbf{S}\}$.
		\end{enumerate}
	\end{theorem}

	We remark that ${\cal G}^*(\{X\}\cup sib(X, {\cal G}^*))$ is a chordal graph, which can be viewed as a CPDAG~\citep{andersson1997characterization}, and thus the consistency in statement (ii) is well-defined. Theorem~\ref{thm:local}
	provides a method to locally enumerate all possible parental sets of a given $X$, which are then used to estimate all possible causal effects. Algorithm~\ref{algo:ida} shows the procedure.
	
	
	\begin{algorithm}[!h]
		\caption{The bgk-IDA algorithm.}
		\label{algo:ida}
		\begin{algorithmic}[1]
			\REQUIRE
			A CPDAG $\mathcal{G}^*$, a consistent pairwise causal constraint set $\cal B$, a treatment  $X$, and a response $Y$.
			\ENSURE
			$\Theta_X$ which stores all possible causal effects of $X$ on $Y$.
			
			\STATE{Derive the DCC set $\cal K$ from $\mathcal{G}^*$ and $\cal B$ based on Theorem~\ref{thm:nbr_set_const}.}
			\STATE{Construct the MPDAG $\cal H$ of $[\mathcal{G}^*, {\cal K}]$.}
			\FOR{each $\mathbf{S}\subseteq sib(X, {\cal H})$ such that the restriction subset of ${\cal K}\cup D_X$ on ${\cal G}^*(\{X\}\cup sib(X, {\cal G}^*))$ is consistent with ${\cal G}^*(\{X\}\cup sib(X, {\cal G}^*))$, where $D_X$ is defined in Theorem~\ref{thm:local},}
			\STATE Estimate the causal effect of $X$ on $Y$ by adjusting for $\mathbf{S}\cup pa(X, {\cal H})$, and add the causal effect to $\Theta_X$.
			\ENDFOR
			
			\STATE{\textbf{return} $\Theta_X$.}
		\end{algorithmic}
	\end{algorithm}

	To illustrate Theorem~\ref{thm:local}, Figure~\ref{fig:impossible} shows the four impossible parental sets of $X$ in the MPDAG shown in Figure~\ref{fig:bgk-ida-mpdag}. Recall that $X$ has $3$ siblings, meaning that there are totally $8$ candidate parental sets of $X$. For example, if we let $B\to X$ and $X \to \{A, C\}$ (Figure~\ref{fig:impossible-1}) then $C\tor \{X, A\}$ indicates that $C\to A$, and further causes $A\to B$ by the constraint $A\tor \{B, C\}$. Hence, ${\cal K}\cup \{B\to X, X \to A, X\to C\}$ on ${\cal G}^*(\{X\}\cup sib(X, {\cal G}^*))$ is inconsistent with ${\cal G}^*(\{X\}\cup sib(X, {\cal G}^*))$, as $X\to A\to B\to X$ is a directed cycle.

	\begin{figure}[!h]
		\centering
		
		\subfloat[\label{fig:impossible-1}]{
			\begin{minipage}[t]{0.18\linewidth}
				\centering
				\includegraphics[width=0.8\linewidth]{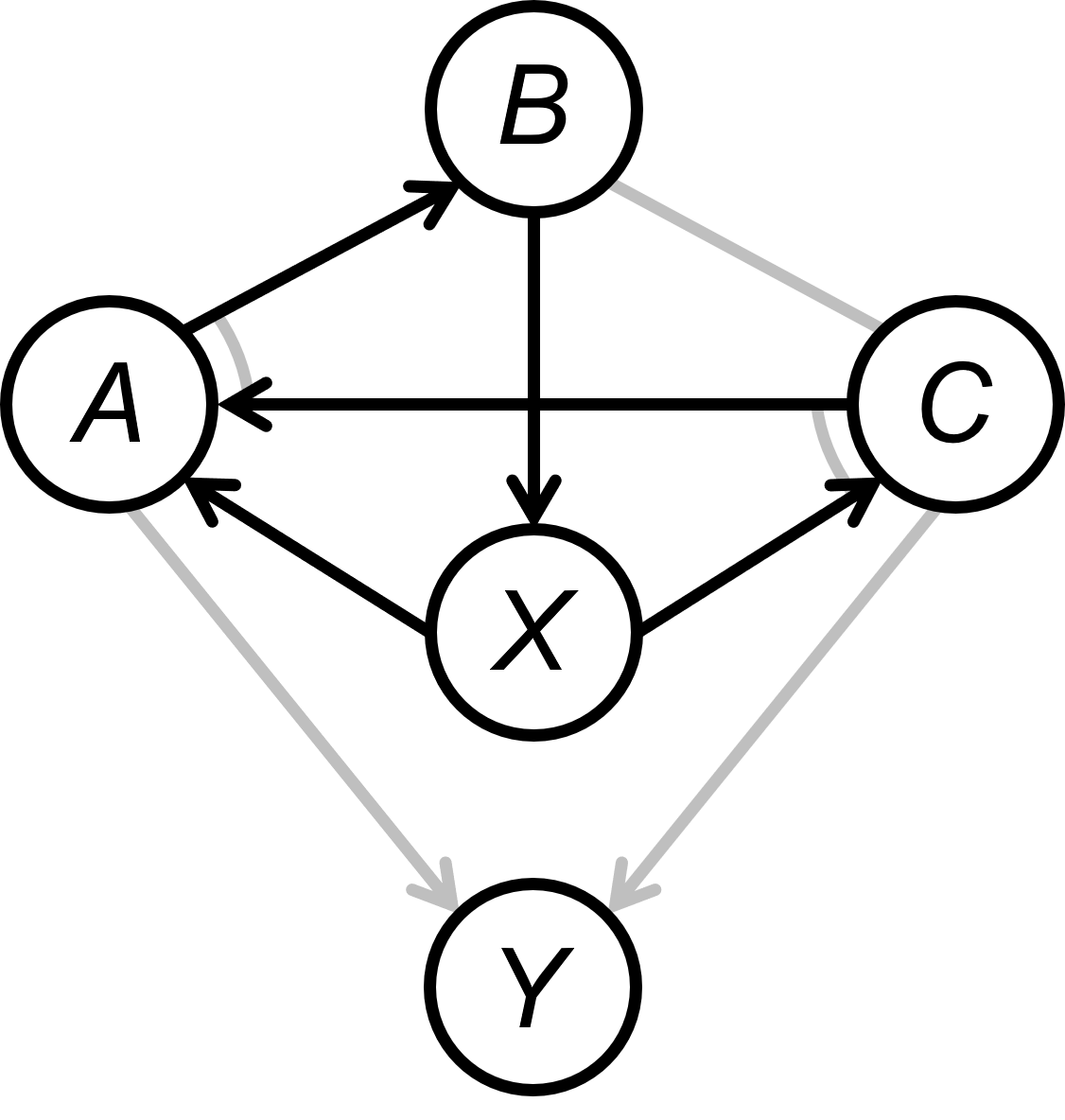}
			\end{minipage}%
		}%
		\hspace{0.04\linewidth}
		\subfloat[\label{fig:impossible-2}]{
			\begin{minipage}[t]{0.18\linewidth}
				\centering
				\includegraphics[width=0.8\linewidth]{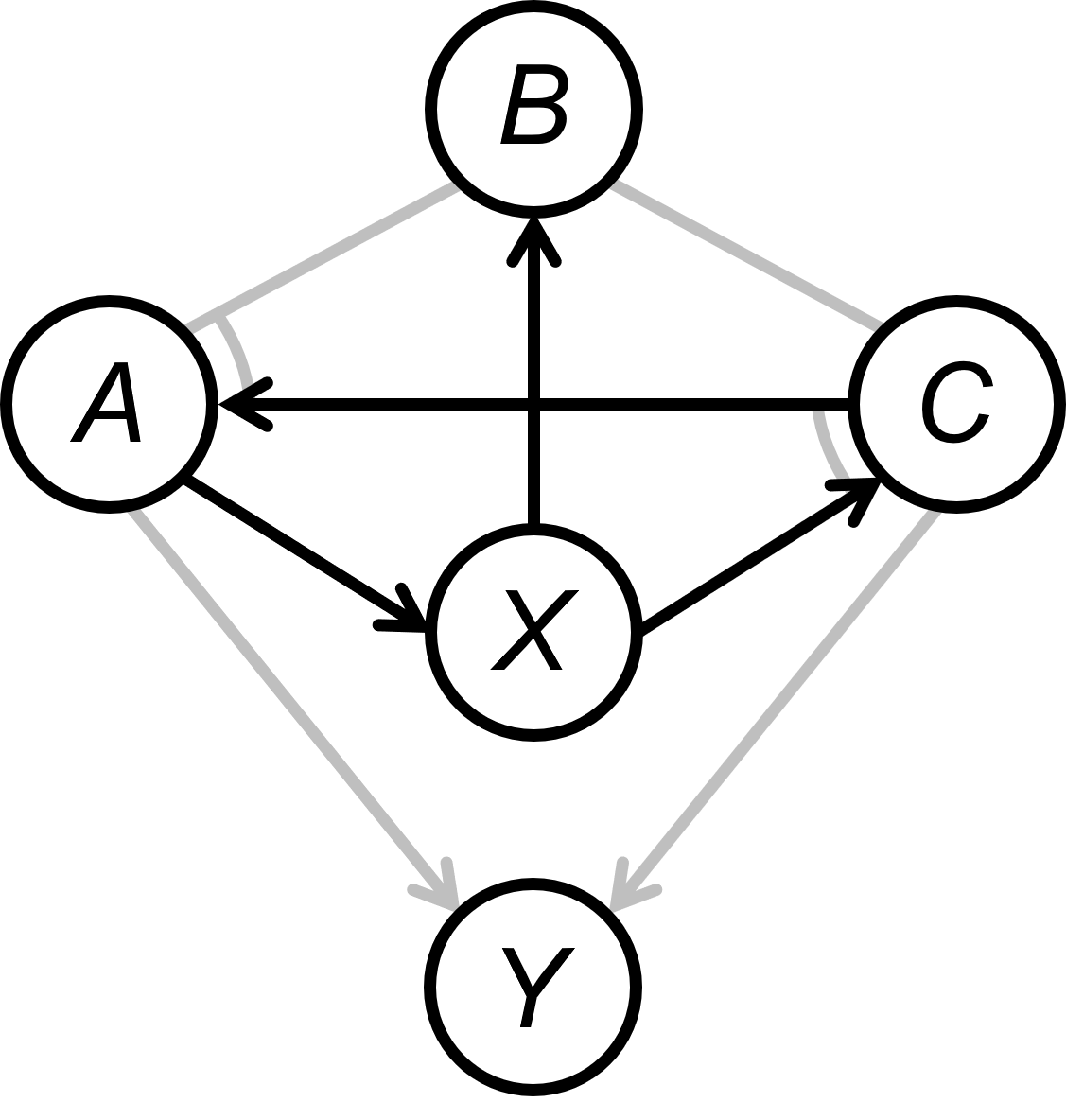}
			\end{minipage}%
		}%
		\hspace{0.04\linewidth}
		\subfloat[\label{fig:impossible-3}]{
			\begin{minipage}[t]{0.18\linewidth}
				\centering
				\includegraphics[width=0.8\linewidth]{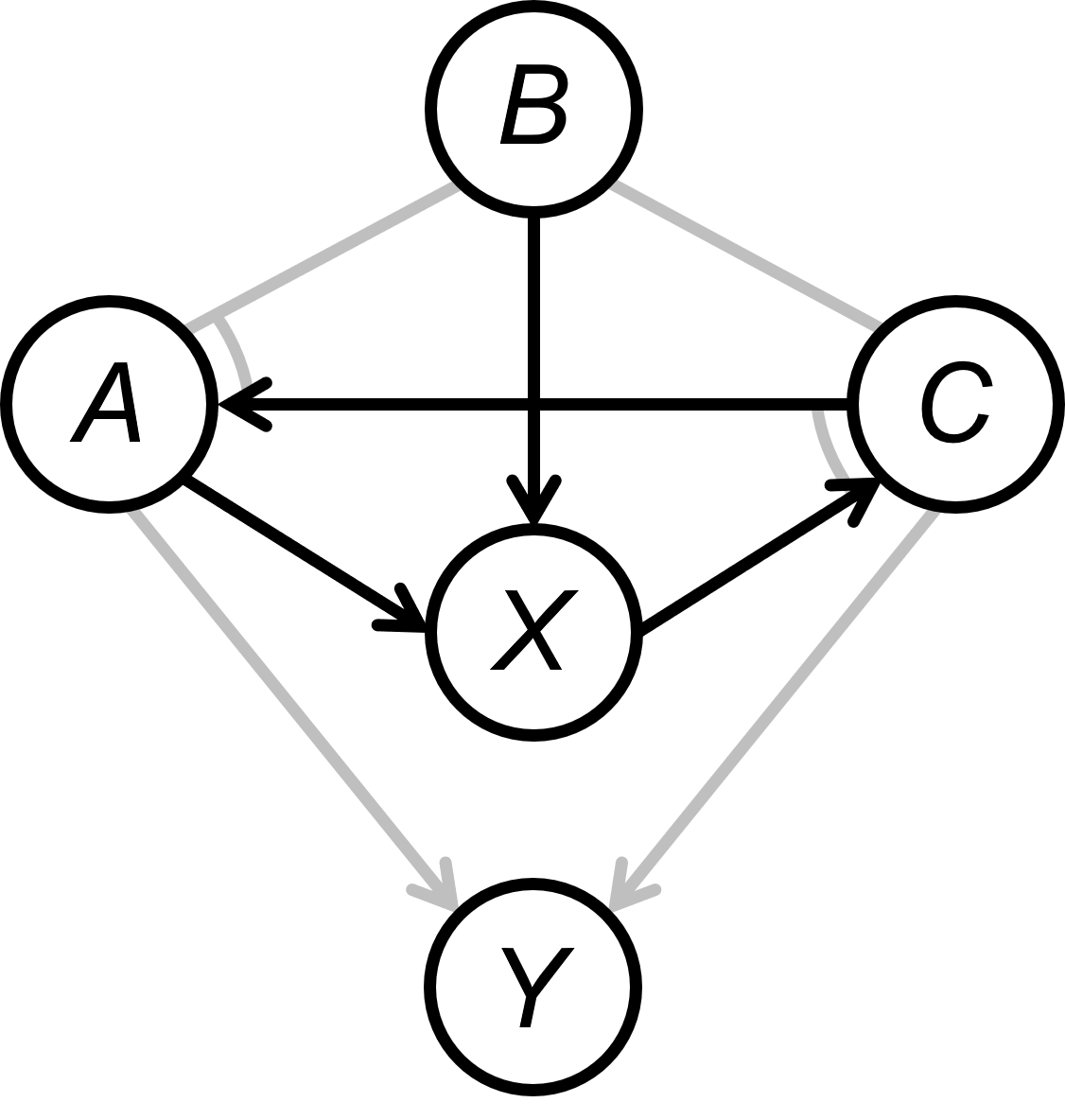}
			\end{minipage}%
		}%
		\hspace{0.04\linewidth}
		\subfloat[\label{fig:impossible-4}]{
			\begin{minipage}[t]{0.18\linewidth}
				\centering
				\includegraphics[width=0.8\linewidth]{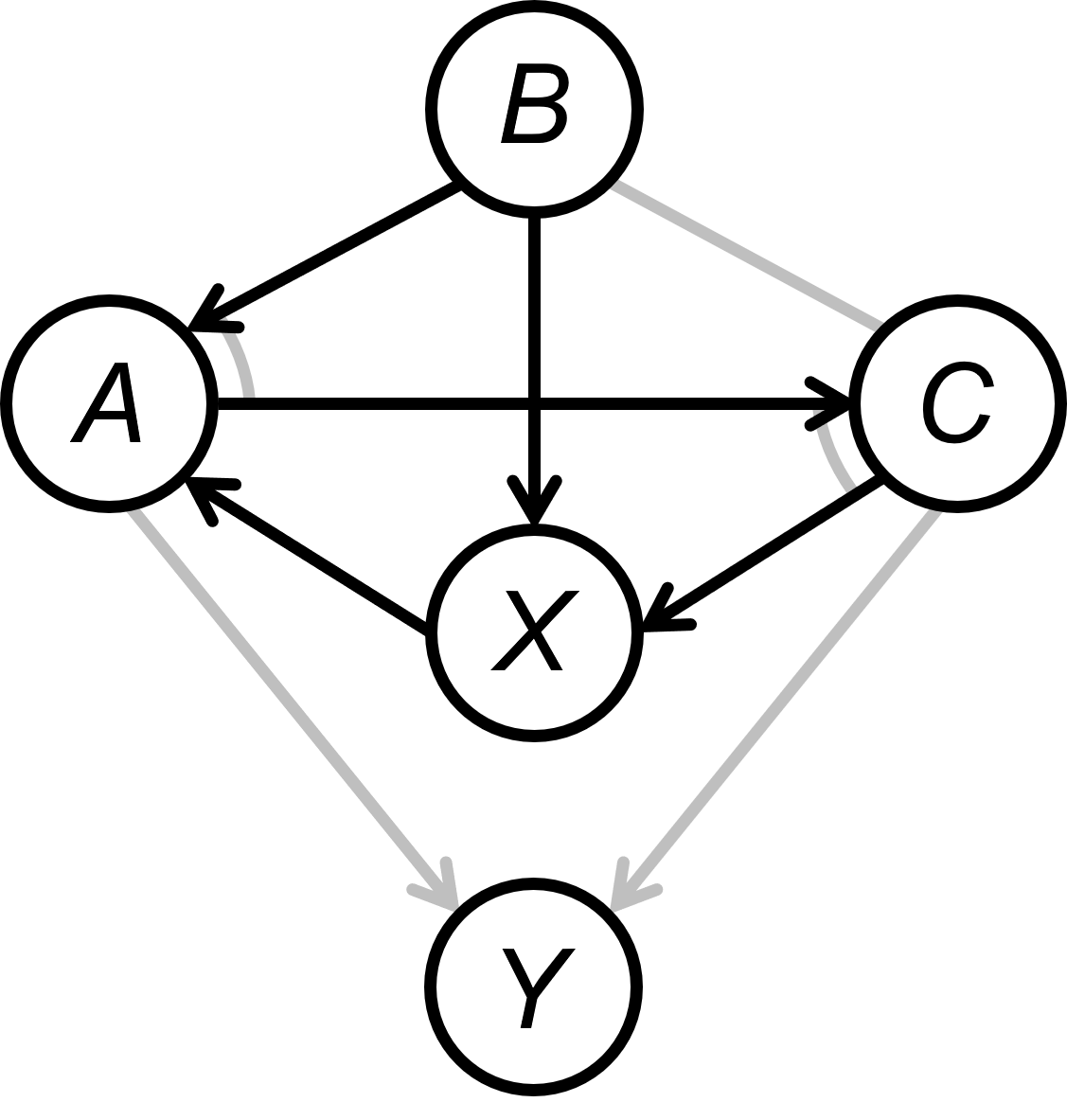}
			\end{minipage}%
		}%

		\caption{The four impossible parental sets of $X$ in the MPDAG shown in Figure~\ref{fig:bgk-ida-mpdag}.}
		\label{fig:impossible}
	\end{figure}

	\subsection{Simulations}\label{sec:sec:sim}

	We empirically study how pairwise causal background knowledge improves the identifiability of a causal effect in this subsection.
	We use randomly sampled chordal graphs instead of CPDAGs in our simulations since the impact of the clauses can be separately considered for each chain component of the CPDAG after transforming pairwise causal background knowledge into DCCs (please refer to Lemma~\ref{lem:app-equi} in Appendix for details). {The chordal graphs were sampled in a reject-sampling manner and the detailed algorithm is in Appendix~\ref{app:app:sim}. }
	
	For each combination of $n\in\{10, 30\}$ and $e$, where $e\in\{10, 15, 20, 25\}$ if $n=10$ and $e\in\{30, 45, 60, 75\}$ if $n=30$, we first sampled $500$ chordal graphs with $n$ vertices
	and $e$ edges. Next, for each chordal graph ${\cal G}^*$, we treated it as a CPDAG and sampled a DAG $\cal G$ from $[{\cal G}^*]$, and assigned each edge in $\cal G$ with a weight sampled from $\mathrm{Uniform}(0.5, 2)$. $\cal G$ was regarded as the underlying true DAG. Then, we sampled a treatment $X$ as well as a response $Y$. Finally, we enumerated all pairs of direct, ancestral, or non-ancestral causal relations in $\cal G$ and randomly sampled $b$  direct, ancestral, or non-ancestral causal constraints from them as background knowledge, and used the bgk-IDA to compute the possible causal effects of $X$ on $Y$, where $b\in\{0 ,1 ,2, 3, 4, 5\}$ for $n=10$ and $b\in\{0 ,3 ,6, 9, 12, 15\}$ for $n=30$. We assumed the data generating mechanism is linear-Gaussian with equal variances, and the causal effects were estimated from the true covariance matrix computed from $\cal G$.
	
	To measure the discrepancy between an estimated multi-set of possible effects and the true effect computed with $\cal G$, we examined four metrics including the causal mean squared error (CMSE) introduced by~\citep{Tsirlis2018scoring, liu2020local}, the number of possible causal effects, the length of the interval determined by the minimum and maximum values of a set of possible effects, and the Int-MSE introduced by~\citet{malinsky2017estimating}. Denote by $\hat{\Theta}_{XY}$ and $\theta_{XY}$ the estimated multi-set of possible causal effects and the true causal effect, respectively. The CMSE is defined as
	\[ \mathrm{CMSE}(\hat{\Theta}_{XY}, \theta_{XY})=\frac{1}{m}\sum_{i=1}^{m}
	{\left[(\hat{\theta}_{i}-\theta_{XY})^2\right]},\]
	where $m=|\hat{\Theta}_{XY}|$ and $\hat{\theta}_{i}\in \hat{\Theta}_{XY}$ is an estimated possible causal effect. {The CMSE is calculated as a weighted average of the squared distances between the true effect and each distinct estimated effect in the multi-set, with the weights corresponding to the frequency of each distinct estimate's occurrence in the multi-set. In practice, the CMSE may be inflated if the number of possible effects in the multi-set is large, hence we examined the number of possible	causal effects as well as the length of the interval determined by the minimum and maximum possible effects. The Int-MSE, introduced by~\citet{malinsky2017estimating}, represents the mean absolute distance between the true effect, $\theta$, and the interval $[\hat{\theta}_{\min}, \hat{\theta}_{\max}]$. The distance $d$ is defined as $d=0$ if the interval covers the true effect and $d=\min(|\theta-\hat{\theta}_{\min}|, |\theta-\hat{\theta}_{\max}|)$ otherwise. The Int-MSE thus evaluates whether an estimated interval covers the true effect. For all four metrics discussed above, lower values indicate higher identifiability of the causal effect.}

	
	Since the causal effects were estimated from the true covariance matrix derived from $\cal G$, and given the correctness of IDA-type algorithms, the range between the minimum and maximum possible causal effects should contain the true effect. Therefore, the Int-MSE should be (nearly) zero. In fact, in our simulations, all Int-MSE values are below $10^{-8}$, and non-zero Int-MSE values are due to numerical errors. This result also suggests the correctness of Algorithm~\ref{algo:ida}.
	
	Except for Int-MSE, the other metrics show similar results. We focus on CMSE in what follows, and the results for the other metrics are provided in Appendix~\ref{app:app:exp}. For each sampled DAG and variable pair $(X, Y)$, we computed three sequences of CMSEs corresponding to three types of background knowledge. Each sequence contains six elements corresponding to the six possible values of $b$. For each sequence, all CMSEs were normalized by the CMSE at $b=0$, yielding the so-called rescaled CMSE.

	Figure~\ref{fig:cmse} shows the sequences of rescaled CMSEs under different types of causal constraints. The rescaled CMSE decreases rapidly as the number of constraints increases. Moreover, given the same number of constraints, the rescaled CMSE under ancestral causal constraints is much lower than that under direct causal constraints, which in turn is lower than that under non-ancestral constraints. This phenomenon occurs because ancestral causal constraints are more informative than non-ancestral ones: if $X$ is a cause of $Y$, then $Y$ cannot be a cause of $X$, but not vice versa.

	\begin{figure}[!t]
		\centering
		\subfloat[$n=10$, $e=10$]{
			\begin{minipage}[t]{0.22\linewidth}
				\centering
				\includegraphics[width=1\linewidth]{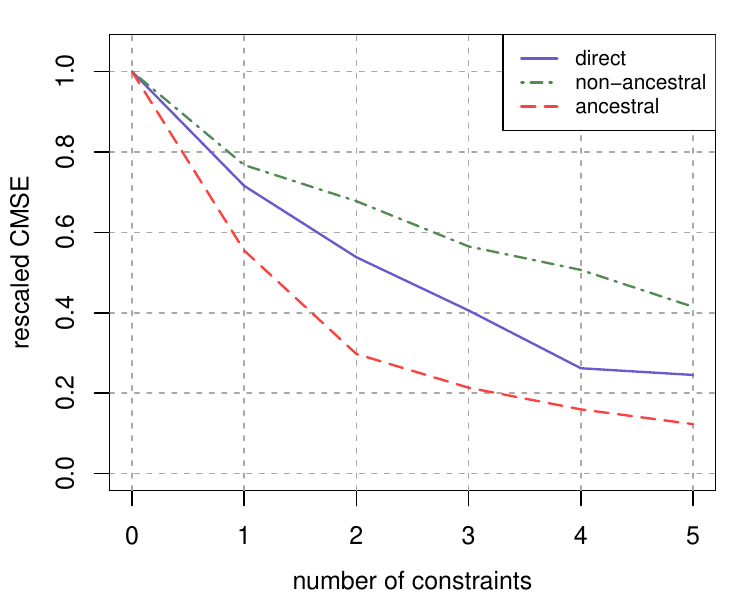}
			\end{minipage}%
		}%
		\hspace{0.01\linewidth}
		\subfloat[$n=10$, $e=15$]{
			\begin{minipage}[t]{0.22\linewidth}
				\centering
				\includegraphics[width=1\linewidth]{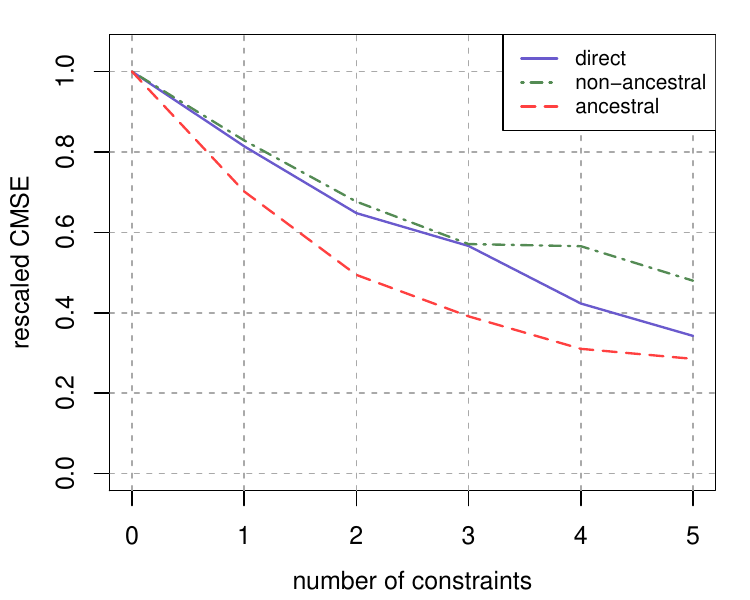}
			\end{minipage}%
		}%
		\hspace{0.01\linewidth}
		\subfloat[$n=10$, $e=20$]{
			\begin{minipage}[t]{0.22\linewidth}
				\centering
				\includegraphics[width=1\linewidth]{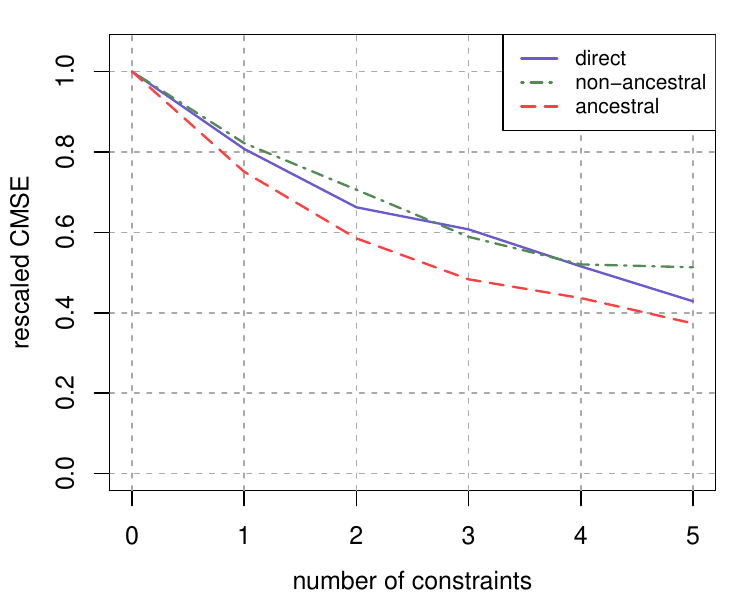}
			\end{minipage}%
		}%
		\hspace{0.01\linewidth}
		\subfloat[$n=10$, $e=25$]{
			\begin{minipage}[t]{0.22\linewidth}
				\centering
				\includegraphics[width=1\linewidth]{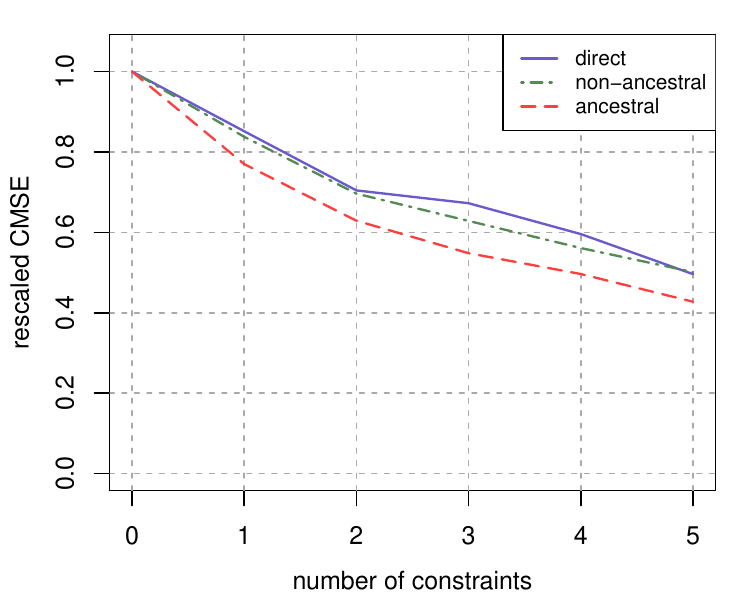}
			\end{minipage}%
		}%
		
		\subfloat[$n=30$, $e=30$]{
			\begin{minipage}[t]{0.22\linewidth}
				\centering
				\includegraphics[width=1\linewidth]{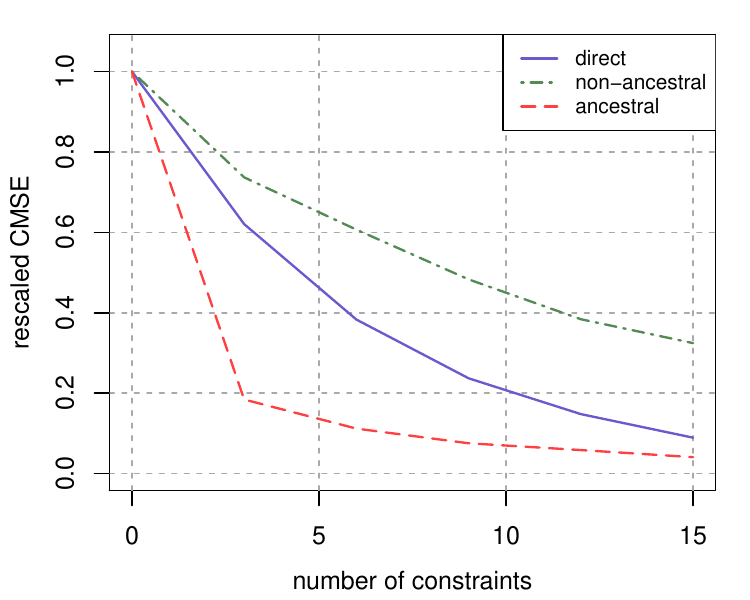}
			\end{minipage}%
		}%
		\hspace{0.01\linewidth}
		\subfloat[$n=30$, $e=45$]{
			\begin{minipage}[t]{0.22\linewidth}
				\centering
				\includegraphics[width=1\linewidth]{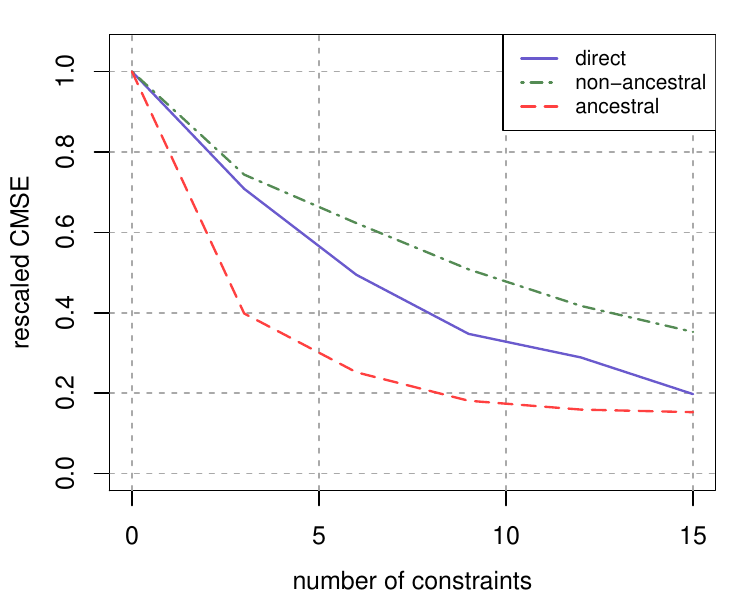}
			\end{minipage}%
		}%
		\hspace{0.01\linewidth}
		\subfloat[$n=30$, $e=60$]{
			\begin{minipage}[t]{0.22\linewidth}
				\centering
				\includegraphics[width=1\linewidth]{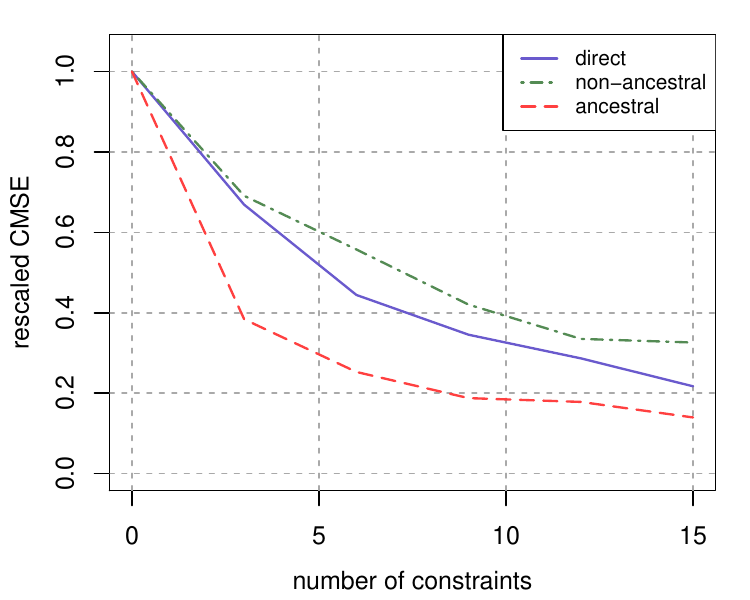}
			\end{minipage}%
		}%
		\hspace{0.01\linewidth}
		\subfloat[$n=30$, $e=70$]{
			\begin{minipage}[t]{0.22\linewidth}
				\centering
				\includegraphics[width=1\linewidth]{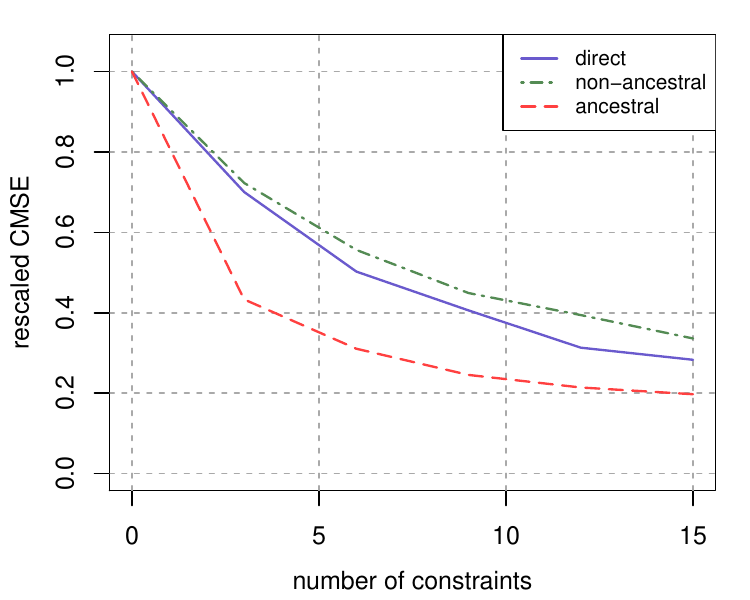}
			\end{minipage}%
		}%

		\caption{The rescaled CMSEs.}
		\label{fig:cmse}
	\end{figure}

	\section{Discussion}\label{sec:conclusion}
	
	
	Pairwise causal background knowledge is frequently encountered in real-world problems. Assuming both pairwise causal background knowledge and a sufficiently large observational data set are available, this paper systematically studies the representation of pairwise causal background knowledge, and demonstrates the potential of exploiting pairwise causal background knowledge in causal inference. The main contribution of the paper is three-fold.
	
	Firstly, we investigate the graphical characterization of causal MPDAGs.
	{We present sufficient and necessary graphical conditions for a partially directed graph to be a causal MPDAG. MPDAGs are important in representing common direct causal relations in a restricted Markov equivalent class. Our graphical characterization generalizes the existing results for  essential graphs~\citep{andersson1997characterization} and intervention essential graphs~\citep{hauser2012characterization}.}

	Despite the wide use of causal MPDAGs,
	they may fail to represent ancestral causal knowledge exactly. Therefore, we develop direct causal clauses to represent all types of pairwise causal background knowledge in a unified form. Because of the local nature of the direct causal clauses, our new representation brings a lot of convenience. As a result, we can now check the consistency of pairwise causal background knowledge, or the equivalence of two pairwise causal background knowledge sets, in polynomial time. Moreover, we prove that any pairwise causal background knowledge set can be decomposed into a causal MPDAG and a minimal residual set of direct causal clauses, and the decomposition can be achieved in polynomial time, too.
	
	The decomposition of pairwise causal background knowledge plays an important role in causal inference. The third contribution of our work is that, we show that the decomposed causal MPDAG entirely determines the identifiability of a causal effect, and the residual direct causal clauses alone contribute nothing to the identifiability, but may reduce the possible values of an unidentifiable effect. We also develop IDA-type algorithms to locally or semi-locally estimate possible causal effects. Using the proven sufficient and necessary identification condition, the adjustment criterion, and the IDA-type algorithms, we can identify causal effects as well as estimate their values or bounds.
	
	{There are also many topics worthy of further investigation. First, although the proposed algorithms run in polynomial time, they may not be the most time-efficient. Determining the optimal time complexity for consistency checking and MPDAG construction, as well as developing corresponding optimal algorithms, remains an important direction for future research. Moreover, efficiently enumerating Markov equivalent DAGs that satisfy given pairwise causal background knowledge remains an interesting open problem. The local properties of the DCC representation may help prune the search space and facilitate the design of more efficient enumeration algorithms. Furthermore, our approach assumes that the pairwise causal background knowledge is consistent with the true CPDAG (or the learned CPDAG). Extending our framework to accommodate inconsistent background knowledge is another important direction. Finally, studying the representation of pairwise causal background knowledge in the presence of hidden variables and selection bias is also an interesting direction for future work.}

	\section*{Acknowledgements}
	This work was supported by the National Key Research and Development Program of China (Grant No. 2022ZD0160300). Liu’s work was partially supported by the National Natural Science Foundation of China (Grant No. 12201629). We are grateful to the editor and the reviewers for their valuable comments and suggestions, which have significantly improved the quality of the paper.
	
	\newpage
	
	\appendix

	\section{Additional Remarks}\label{sec:remark}
	
	Some additional topics are discussed in this section.
	
	{
		\subsection{From DCCs to Pairwise Causal Constraints}\label{sec:sec:from}
		Assuming that $[{\cal H}, {\cal R}]$ is equivalent to some set of pairwise causal constraints, we can obtain the maximal equivalent set of such constraints by using Theorems~\ref{thm:nbr_set_const} and \ref{thm:equi_dcc} and Algorithm~\ref{algo:consistency-peo}.

		We begin by enumerating all pairs of nodes $X$ and $Y$, and checking whether $X$ is an ancestor of $Y$  in all DAGs within $[{\cal H}, {\cal R}]$. This is done by transforming the proposition into the corresponding DCC, $X \tor {\bf C}_{XY}({\cal G}^*)$, using Theorem~\ref{thm:nbr_set_const}, and then determining whether $X \tor {\bf C}_{XY}({\cal G}^*)$ is redundant given ${\cal H}$ and ${\cal R}$, based on Theorem~\ref{thm:equi_dcc} and the consistency-checking algorithm (Algorithm~\ref{algo:consistency-peo}). If $X \tor {\bf C}_{XY}({\cal G}^*)$ is found to be redundant, then $X$ is an ancestor of $Y$ in all DAGs in $[{\cal H}, {\cal R}]$, indicating that the proposition “$X$ is an ancestor of $Y$” can be recovered from $[{\cal H}, {\cal R}]$. Similarly, we can check whether the proposition “$X$ is not an ancestor of $Y$” holds in all DAGs in $[{\cal H}, {\cal R}]$. Denote by $\cal B$ the set of all ancestral and non-ancestral causal constraints that hold for all DAGs in $[{\cal H}, {\cal R}]$. This set $\cal B$ constitutes the maximal set of pairwise causal constraints that can be translated back from $[{\cal H}, {\cal R}]$.
		
		Moreover, one can iteratively prune the redundant constraints in $\cal B$ to further obtain a minimal equivalent set of constraints. More formally, for each $b \in {\cal B}$, we transform it into its corresponding DCCs, denoted by $\kappa(b)$. We then apply Theorem 27 and the consistency-checking algorithm (Algorithm 1) to determine whether $\kappa(b)$ is redundant given the DCCs derived from ${\cal B} \setminus {b}$. The resulting subset of $\cal B$ constitutes a minimal set of pairwise causal constraints equivalent to $[{\cal H}, {\cal R}]$.
		
	}
	
	\subsection{A Sequential Method for Checking Consistency}\label{sec:sec:seq}
	In this section, we present an approach to sequentially check consistency. Suppose that ${\cal G}^*$ is a CPDAG and $\cal K$ is a set of consistent DCCs, and $\cal H$ is the MPDAG representing $[{\cal G}^*, \cal K]$. A set of DCCs ${\cal K}'$ is called consistent with ${\cal G}^*$ given $\cal K$, if ${\cal K}'\cup {\cal K}$ is consistent with ${\cal G}^*$. For any DCC $(X\tor \mathbf{D}) \in {\cal K}'$, it holds that,
	\begin{align*}
		{\cal K}\cup\{X\tor \mathbf{D}\} \text{~is~inconsistent} \iff & X\tor \mathbf{D} \text{~does~not~hold~for~any~DAG~in~} [{\cal G}^*, {\cal K}]\\
		\iff & \forall\, {\cal G}\in [{\cal G}^*, {\cal K}],\,ch(X, {\cal G})\cap \mathbf{D}=\varnothing\\
		\iff & \forall\, {\cal G}\in [{\cal G}^*, {\cal K}],\,\mathbf{D}\subseteq pa(X, {\cal G})\\
		\iff & \forall\, {\cal G}\in [{\cal G}^*, {\cal K}],\,\mathbf{D}\to X\\
		\iff & \mathbf{D}\subseteq pa(X, {\cal H}).
	\end{align*}
	Therefore, the consistency of ${\cal K}'$ with ${\cal G}^*$ given $\cal K$ can be checked sequentially: picking one DCC from the current ${\cal K}'$  at a time; if the clause is consistent with ${\cal G}^*$ given $\cal K$, then adding it to the current ${\cal K}$ and updating the current MPDAG based on ${\cal K}$.
	
			
			

	\begin{figure}[!h]
		\centering
		\subfloat{
			\begin{minipage}[t]{0.23\linewidth}
				\centering
				\includegraphics[width=1\linewidth]{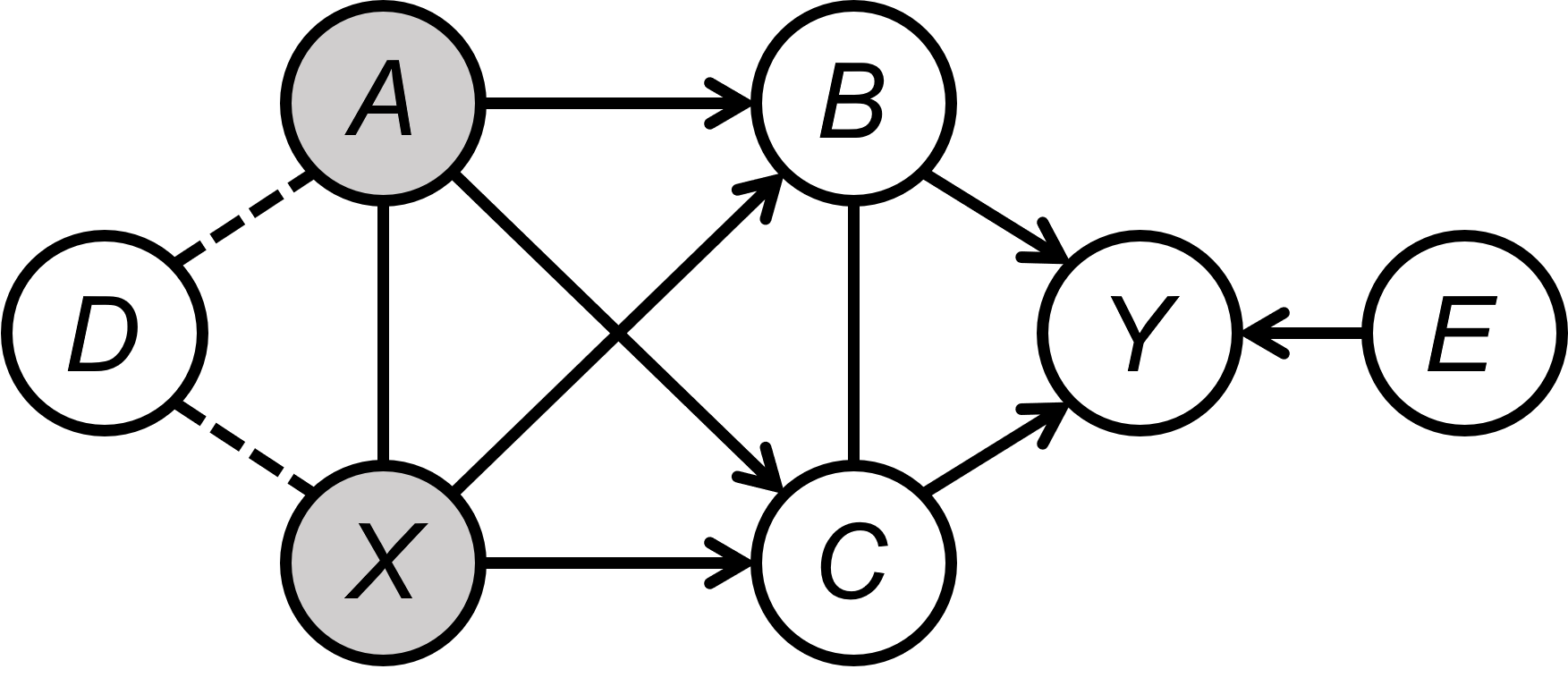}
			\end{minipage}%
		}%
		\caption{An MPDAG given in Example~\ref{ex:app}.}
		\label{fig:1-10}
	\end{figure}
	
	\begin{example}\label{ex:app}
		Following Example~\ref{ex:consistency}, we now use the sequential method to check the consistency of ${\cal K}=\{C \tor A, B \tor A, D \tor \{A, X\}\}$ with respect to ${\cal G}^*$ shown in Figure~\ref{fig:1-1}. If the first DCC chosen from $\cal K$ is $D\tor \{A, X\}$, then the updated MPDAG, denoted by $\cal H$, is the one shown in Figure \ref{fig:1-10}. Since $A \to B$ and $A\to C$ are both in $\cal H$, no matter what clause will be chosen next, the procedure returns {\rm False}. Similarly, if we choose $B\to A$ and $C \to A$ in the first two loops, then based on Rule 4 of Meek's rules, the resulting MPDAG must have edges $A \to D$ and $X \to D$. Hence, $\cal K$ is inconsistent once we consider the remaining clause $D\tor \{A, X\}$.
	\end{example}

	\section{Proofs}\label{app:proof}
	
	Before presenting all detailed proofs of the theorems, propositions and corollaries of the main text, we introduce some helpful concepts and results in Appendix~\ref{app:proof:pre}. In the following paper, for a variable $X$ and a non-empty variable set $\bf Y$, we will use the notion ${\bf Y}\to X$ to represent $Y\to X$ for all $Y\in{\bf Y}$, and use the notion $X\to{\bf Y}$ to represent $X\to Y$ for all $Y\in{\bf Y}$.
	
	\subsection{Preliminaries}
	\label{app:proof:pre}
	We briefly review some graphical properties of chordal graphs and CPDAGs. In a graph, a cycle with length three is called a triangle. A path is called unshielded if none of its three consecutive vertices form a triangle. For a graph ${\cal G}$ over a vertex set $\mathbf{V}$, $\mathbf{M}\subseteq\mathbf{V}$ is called a {clique} if $\mathbf{M}$ induces a complete subgraph. A clique is called {maximal} if it is not a proper subset  of any other cliques. Let ${\cal C}$ be a chordal graph over the vertex set ${\bf V}({\cal C})$. Any induced subgraph of ${\cal C}$ is chordal. It can be proved that any chordal graph has a simplicial vertex, and moreover, any non-complete chordal graph has two non-adjacent simplicial vertices~\citep{blair1993chordal}. A perfect elimination ordering (PEO) of ${\cal C}$ is a total ordering of the vertices in ${\bf V}({\cal C})$, denoted by $\beta = (V_1,V_2,\cdots,V_n)$, such that for any $V_i$, $i=1,2,\cdots,n$, $adj(V_i, {\cal C})\cap \{V_i,V_{i+1},\cdots,V_n\}$ induces a complete subgraph of ${\cal C}$. An undirected graph is chordal if and only if it has a PEO.
	
	Given a PEO $\beta = (V_1,V_2,\cdots,V_n)$ of ${\cal C}$, if we orient the edges in ${\cal C}$ such that $adj(V_i, {\cal C})\cap \{V_i,V_{i+1},\cdots,V_n\}$ are parents of $ V_i$, then the resulting directed graph is acyclic and v-structure-free. Conversely, any v-structure-free DAG who has the same skeleton as ${\cal C}$ can be oriented from ${\cal C}$ according to some PEO of ${\cal C}$.
	
	Let ${\cal G}^*$ be a CPDAG. It was pointed by~\citet{maathuis2009estimating} that (i) no orientation of the edges not oriented in ${\cal G}^*$ will create a directed cycle which includes an edge or edges that were oriented in ${\cal G}^*$, and (ii) no orientation of an edge not	directed in ${\cal G}^*$ can create a new v-structure with an edge that was oriented in ${\cal G}^*$. As any orientation of the edges in ${\cal G}^*$ which does not create directed cycles or v-structures corresponds to a DAG in $[{\cal G}^*]$, we can separately orient the undirected edges in each chain component such that every resulting directed graph is a DAG without v-structure.
	
	\subsection{Proof of Proposition \ref{nonemptyMPDAG}}\label{app:proof:nonemptyMPDAG}
	
	We first prove Proposition \ref{nonemptyMPDAG} as it is required in proving Theorem~\ref{the:MPDAG}.
	
	\begin{proof}
		The first claim holds because the definition of a causal MPDAG implies that for any restricted Markov equivalence class $[{\cal G}^*, {\cal B}]$ that can be represented by $\cal H$, ${\cal G}^*$ and $\cal H$ have the same skeleton and v-structures.
		
		To prove the second claim, let $[{\cal G}^*, {\cal B}]$ be a restricted Markov equivalence class that can be represented by $\cal H$, and let ${\cal B}_d = \mathbf{E}_d({\cal H}) \setminus \mathbf{E}_d({\cal G}^*) $. It is easy to verify that orienting undirected edges in ${\cal G}^*$ according to ${\cal B}_d$ does not introduce any v-structure or directed cycle and the resulting PDAG is closed under Meek's rules, as the resulting PDAG is exactly $\cal H$. Therefore, $[{\cal H}]=[{\cal G}^*, {\cal B}_d]$.
	\end{proof}

	\subsection{Proof of Theorem \ref{the:MPDAG}}\label{app:proof:the:MPDAG}
	
	We first introduce two properties of chordal graphs. For the completeness of the paper, the proofs of these results are also provided.
	
	\begin{lemma}\label{app:lem:chordal}
		Let $\cal C$ be a connected chordal graph, then the following claims hold for $\cal C$.
		\begin{enumerate}
			\item[(i)]  For any simplicial vertex $X$ in $\cal C$, there is a unique maximal clique containing $X$.
			\item[(ii)]  If $\cal C$ is not complete, then for any non-simplicial vertex $Y$ in $\cal C$ that is adjacent to some simplicial vertex $X$, $Y$ has a neighbor $Z\neq X$ which is not adjacent to $X$.
		\end{enumerate}
	\end{lemma}
	
	\begin{proof}[Proof of lemma \ref{app:lem:chordal}]
		For any simplicial vertex $X$ in $\cal C$, if there are two distinct maximal cliques ${\bf M}_i, {\bf M}_j$ containing $X$, then there exist $X_i\in {\bf M}_i$ and $X_j\in {\bf M}_j$ such that $X_i$ and $X_j$ are not adjacent, since otherwise ${\bf M}_i\cup {\bf M}_j$ is also a maximal clique.
		Note that both $X_i$ and $X_j$ are adjacent to $X$, hence $X$ is not simplicial, which is contrary to the assumption. This completes the proof of the first claim.
		
		
		We next prove the second claim. Let ${\bf M}_i$ be the maximal clique containing $X$. Since $X$ is simplicial, $Y$ is adjacent to $X$ implies that $Y\in {\bf M}_i$. Assume, for the sake of contradiction, that every neighbor of $Y$ other than $X$ is adjacent to $X$, then $adj(Y, {\cal C})\subseteq {\bf M}_i$, which means $adj(Y, {\cal C})$ induces a complete subgraph of $\cal C$. This contradicts the assumption that $Y$ is non-simplicial. Therefore, the second claim holds true. 	
	\end{proof}
	
	Next, we present Lemmas \ref{Bchordal}-\ref{re:arrowinclique} in the following. These lemmas contribute to the proof of the necessity of Theorem \ref{the:MPDAG}.

	\begin{lemma}[Necessity of Condition (ii)]\label{Bchordal}
		The skeleton of $ {\cal C}^b $ is a chordal graph for any B-component $ {\cal C}^b $ of a causal MPDAG $ \cal H $.
	\end{lemma}
	
	\begin{proof}[Proof of lemma \ref{Bchordal}]
		According to Proposition~\ref{nonemptyMPDAG}, the skeleton of each B-component of a causal MPDAG is an induced subgraph of a chain component of a CPDAG. The result comes from the fact that any induced subgraph of a chordal graph is still chordal.
	\end{proof}
	
	\begin{lemma}[Necessity of Condition (iii)]\label{lem:into-B-com}
		Let $ \cal H $ be a causal MPDAG and $ {\cal C}^b $ be a B-component of $ \cal H $. For any vertex $ X \notin \mathbf{V} ({\cal C}^b) $, if $ X \rightarrow Y $ for some vertex $ Y \in \mathbf{V} ({\cal C}^b) $, then $ X \rightarrow Y $ for every vertex $ Y \in \mathbf{V} ({\cal C}^b) $.
	\end{lemma}
	
	\begin{proof}[Proof of Lemma \ref{lem:into-B-com}]
		If $|\mathbf{V} ({\cal C}^b)|=1$, then Lemma~\ref{lem:into-B-com} naturally holds. Assume that $|\mathbf{V} ({\cal C}^b)|>1$, then $sib(Y, {\cal C}^b)\neq \varnothing$ by the definition of a B-component. Since every B-component is a connected graph, to prove Lemma~\ref{lem:into-B-com} it suffices to show that $X\to Z$ for every $Z\in sib(Y, {\cal C}^b)$. This is because every vertex in ${\cal C}^b$ other than $Y$ is connected to $Y$ by an undirected path. If the conclusion holds for every $Z\in sib(Y, {\cal C}^b)$, then the same argument can be successively applied to every vertex along the path.
		
		
		For any vertex $Z\in sib(Y, {\cal C}^b)$, it is clear that $ X $ and $ Z $ are adjacent, since otherwise by Rule 1 of Meek's rules, it holds that $ Y \rightarrow Z $ is in $\cal H$, which is a contradiction to $Z\in sib(Y, {\cal C}^b)$. As $ X \notin \mathbf{V} ({\cal C}^b) $, $ X $ and $ Z $ must be connected by a directed edge. If $ Z \rightarrow X $, then $ Z - Y $ can be oriented as $ Z \rightarrow Y $ by Rule 2 of Meek's rules, which also contradicts to $Z\in sib(Y, {\cal C}^b)$. Thus, we have $ X \rightarrow Z $.
	\end{proof}

	\begin{lemma}[Necessity of Condition (i)]\label{Hc-chain}
		Given a causal MPDAG $ \cal H $, the chain skeleton $ {\cal H}_c $ of $ \cal H $ is a chain graph. Furthermore, $ { \cal H }_c $ is also an MPDAG.
	\end{lemma}
	
	\begin{proof}[Proof of Lemma \ref{Hc-chain}]
		We first prove that the graph $ { \cal H }_c $ is a chain graph, which suffices to show that there are no partially directed cycles in $ {\cal H}_c $.  Assume that there is a partially directed cycle in $ {\cal H}_c $, and the cycle is of the following form: $ X_{11} \rightarrow X_{21} - \cdots - X_{2n_2} \rightarrow \cdots \rightarrow X_{k1} - \cdots - X_{kn_k} \rightarrow X_{1n_1} - \cdots - X_{11} $. Based on the definitions of $ {\cal H}_c $ and B-component, $X_{21}, \cdots, X_{2n_2}$ are in the same B-component while $X_{11}$ is not in this B-component. By Lemma~\ref{lem:into-B-com}, it holds that $X_{11}\to X_{2n_2}$. Similarly, $X_{2n_2} \rightarrow X_{3n_3}\rightarrow \cdots\rightarrow X_{kn_k}\to X_{11}$, which together with $X_{11}\to X_{2n_2}$ gives a directed cycle. Since all directed edges in $ {\cal H}_c $ are also in $ \cal H $, we have constructed a directed cycle in $ {\cal H}$, which contradicts to the definition of a causal MPDAG.
		
		We then show that $ { \cal H }_c $ is an MPDAG. Since we have already proved that there is no (partially) directed cycle in $ {\cal H}_c $, $ {\cal H}_c $ is a PDAG. What remains is to show that $ {\cal H}_c $ is closed under the four Meek's rules.
		
		\begin{figure}[!h]
			\centering
			\subfloat[ \label{rule3left}]{
				\begin{minipage}[t]{0.1\linewidth}
					\centering
					\includegraphics[width=1\linewidth]{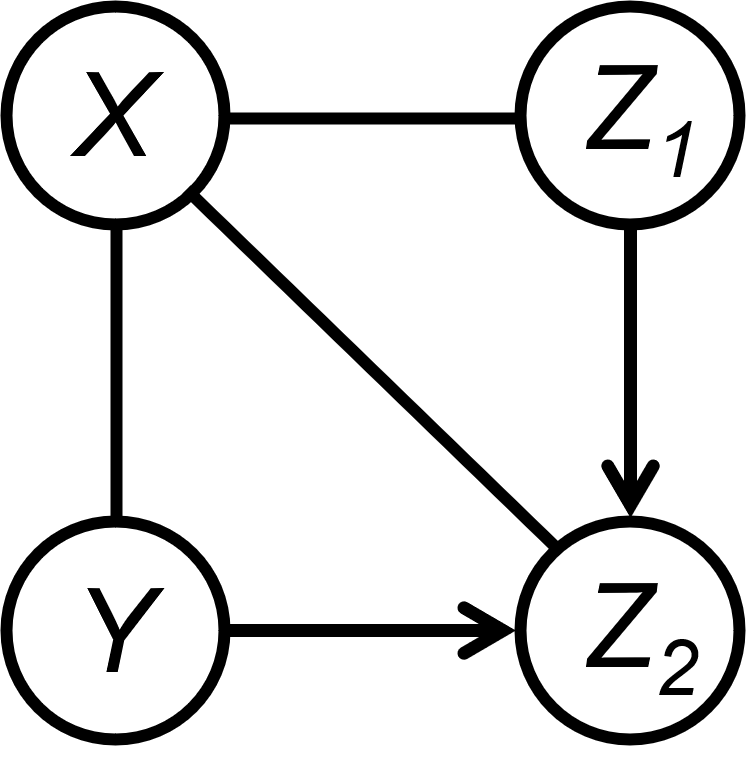}
				\end{minipage}%
			}%
			\hspace{0.1\linewidth}
			\subfloat[  \label{rule4left-1}]{
				\begin{minipage}[t]{0.1\linewidth}
					\centering
					\includegraphics[width=1\linewidth]{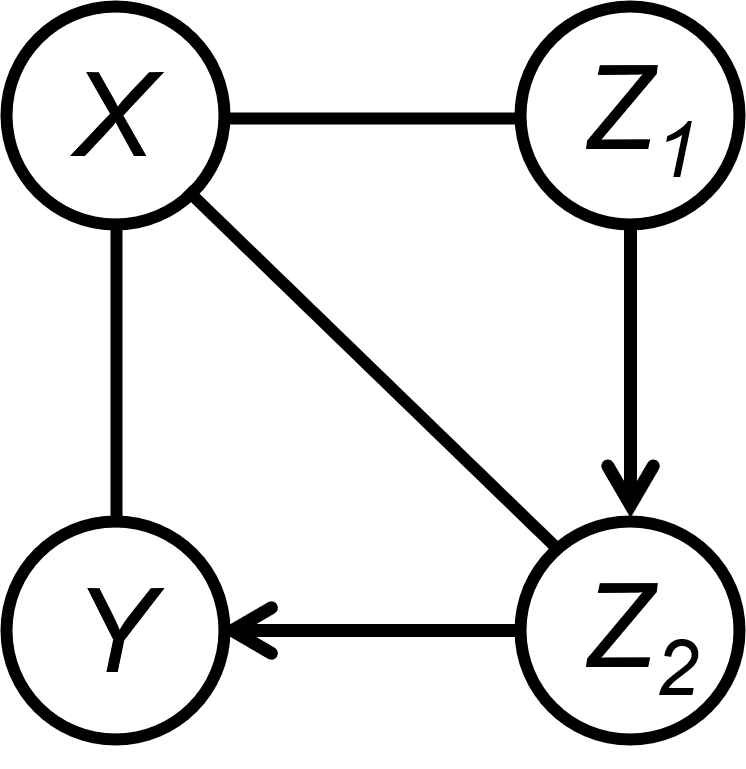}
				\end{minipage}%
			}%
			\hspace{0.1\linewidth}
			\subfloat[  \label{rule4left-2}]{
				\begin{minipage}[t]{0.1\linewidth}
					\centering
					\includegraphics[width=1\linewidth]{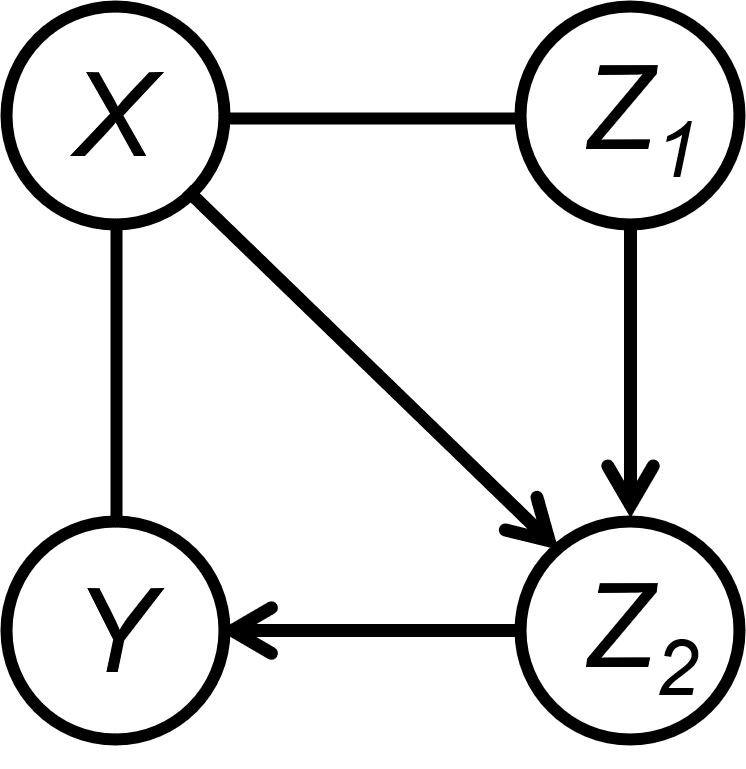}
				\end{minipage}%
			}%
			\hspace{0.1\linewidth}
			\subfloat[ \label{rule4left-3}]{
				\begin{minipage}[t]{0.1\linewidth}
					\centering
					\includegraphics[width=1\linewidth]{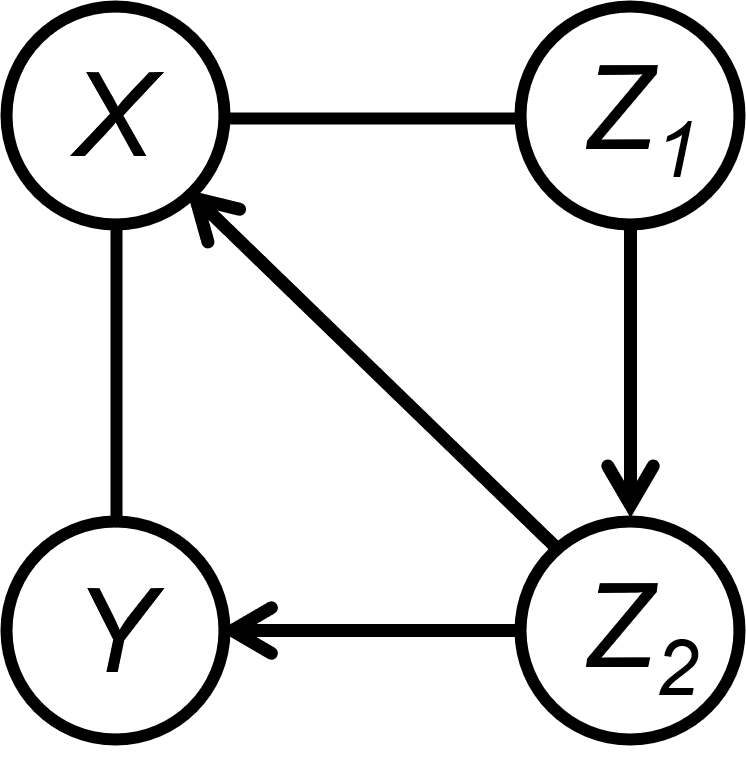}
				\end{minipage}%
			}%
			\caption{The cases discussed in the proof of Lemma \ref{Hc-chain}}
			\label{fig:HcMPDAG}
		\end{figure}

		\begin{enumerate}
			\item[(i)] If $ {\cal H}_c $ is not closed under the first Meek's rule, then $ {\cal H}_c $ has an induced subgraph $ X \rightarrow Y - Z $, in which $ X \notin adj(Z, {\cal H}_c) $. By the construction of $ {\cal H}_c $, $Y$ and $Z$ are in the same B-component of ${\cal H}$ while $X$ is not in that B-component.  According to Lemma \ref{lem:into-B-com}, $ X \rightarrow Z $ must be in $ {\cal H}_c $, which contradicts the assumption that $ X \notin adj(Z, {\cal H}_c) $.
			
			\item[(ii)] If $ {\cal H}_c $ is not closed under the second Meek's rule, then $ {\cal H}_c $ has an induced subgraph consisting of $ X \rightarrow Y \rightarrow Z $ and $ X - Z $. By the similar argument for (i), $ Y \rightarrow X $ must be in $ {\cal H}_c $, which leads to a contradiction.
			
			\item[(iii)] If  $ {\cal H}_c $ is not closed under the third Meek's rule, then $ {\cal H}_c $ has an induced subgraph with the configuration shown in Figure \ref{rule3left}. In this case, $ X, Y, Z_1, Z_2 $ are in the same B-component. Again, by the construction of $ {\cal H}_c $, $ Y \rightarrow Z_2, Z_1 \rightarrow Z_2 $ should not be in $ {\cal H}_c $. which leads to a contradiction.
			
			\item[(iv)] If $ {\cal H}_c $ is not closed under the fourth Meek's rule, then one of the configurations shown in Figures \ref{rule4left-1}-\ref{rule4left-3} must appear in $ {\cal H}_c $ as an induced subgraph. In the following, we will prove that none of the configurations is in $ {\cal H}_c $. In fact, if $ {\cal H}_c $ has an induced subgraph with the configuration shown in Figure \ref{rule4left-1}, then by the similar argument for (iii), $ Z_1 \rightarrow Z_2 \rightarrow Y $ should not be in $ {\cal H}_c $ as $X, Y, Z_1, Z_2$ are in the same B-component of $ {\cal H} $. If $ {\cal H}_c $ has an induced subgraph with the configuration shown in Figure \ref{rule4left-2}, then the induced subgraph of $ {\cal H}_c $ over $ X, Z_2, Y $ is not closed under the second Meek's rule. Similarly, in Figure \ref{rule4left-3}, the induced subgraph of $ {\cal H}_c $ over $ Z_1, Z_2, X $ is not closed under the second Meek's rule either.
		\end{enumerate}
		This completes the proof.	
	\end{proof}

	\begin{lemma}\label{rule1or4}
		Given a causal MPDAG $ \cal H $, if a directed edge in $ \cal H $ can be oriented by Rule 1 or Rule 4 of Meek's rules, then the involved vertices of the corresponding Meek's rule are not in the same B-component.
	\end{lemma}
	
	We remark that, a directed edge $Y\to Z$ \emph{can} be oriented by Rule 1 means that there is an $X\notin adj(Z, {\cal H})$ such that $X\to Y$ is in ${\cal H}$. In this case, $X, Y, Z$ are the involved vertices of Rule 1. The meaning of the expression that "a directed edge can be oriented by Rule 4" is similar. We also note that, the condition of Lemma~\ref{rule1or4} does not rule out the possibility that the edge can also be oriented by Rules 2, 3 or from the background knowledge set.
	
	\begin{proof}[Proof of Lemma \ref{rule1or4}]
		If a directed edge $ Y \rightarrow Z $ can be oriented by Rule 1 of Meek's rules, then there exists a vertex $ X \in \mathbf{V}({\cal H}) $ such that $ X \rightarrow Y $ and $ X \notin adj(Z, {\cal H}) $. We need to prove that $ X, Y, Z $ are not in the same B-component of ${\cal H}$. Assume, for the sake of contradiction, that  $ X, Y, Z $ are in the same B-component $ {\cal C}^b $, then there is an undirected path $ Y - W_1 - \cdots - W_n - Z $ connecting $ Y $ and $ Z $ where $n\geq 1$ and $ W_i \in {\cal C}^b $ for $i=1,2,\cdots,n$. Since $Y-W_1$ is undirected, we have $ X \in adj(W_1, {\cal H}) $, otherwise we would have $ Y \rightarrow W_1 $ due to Rule 1 of Meek's rules. Similarly, $ Y \in adj(W_n, {\cal H}) $ since $ Y \rightarrow Z - W_n $ is in ${\cal H}$.
		
		\begin{figure}[!h]
			\centering
			\vspace{2em}
			\begin{minipage}[b]{0.33\linewidth}
				\centering
				\includegraphics[width=1\linewidth]{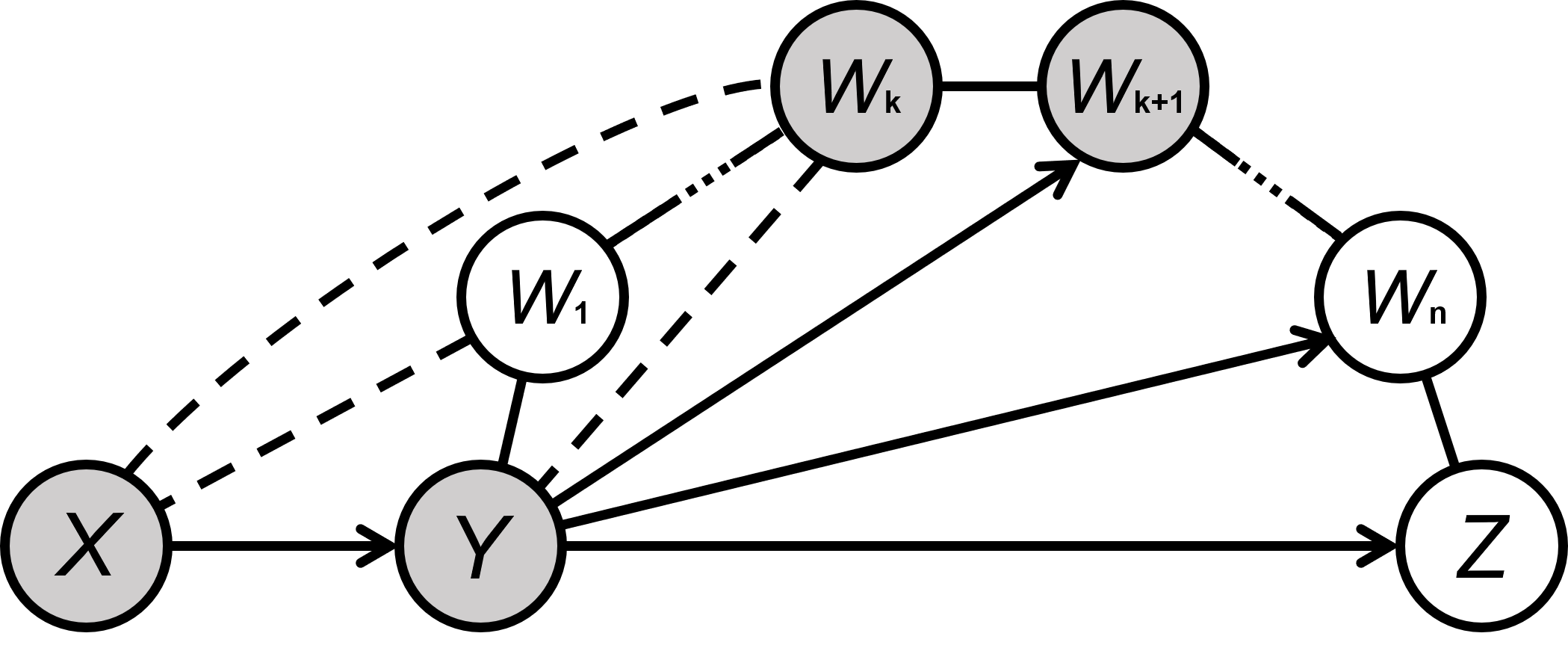}
			\end{minipage}
			\caption{An illustration of the graph structure discussed in the proof of Lemma \ref{rule1or4}. A dashed undirected edge connecting two vertices indicates they are adjacent, but the direction of the edge is not relevant to the proof.}
			\label{fig:rule1or4}
		\end{figure}

		If $n = 1$ or $X\in adj(W_n, {\cal H})$, then $W_n\to Z$ should be in ${\cal H}$ by Rule 4 of Meek's rules (Figure~\ref{fig:rule1or4}), which contradicts our assumption. Now consider the case where $n>1$ and $X\notin adj(W_n, {\cal H})$. Since $X\to Y$, by the first Meek's rule we have $Y\to W_n$. Moreover, since $W_{n-1}-W_n$, it holds that $Y\in adj(W_{n-1}, {\cal H}) $. Let $ k \coloneqq \arg\max_{1 \leq j \leq n-1} X \in adj(W_j, {\cal H}) $ denote the largest subscript of the vertex on the path which is adjacent to $ X $, then by the same argument we can show that $Y\to W_i$ for $i=k+1,\cdots,n$ and $Y\in adj(W_k, {\cal H})$. Note that, the induced subgraph of ${\cal H}$ over $X, Y, W_{k+1}$ and $W_k$ is the same as the left-hand side of the fourth Meek's rule, $W_{k}\to W_{k+1}$ is in ${\cal H}$, which is contrary to our assumption.

		If a directed edge $ X \rightarrow Y $ can be oriented by Rule 4 of Meek's rules but cannot be oriented by Rule 1 of Meek's rules, then there exist vertices $ Z_1, Z_2 \in {\cal H} $ such that $ \cal H $ has an induced subgraph shown in Figure \ref{rule4left-1}. Note that, $Z_2\to Y$ can be oriented by Rule 1 of Meek's rules since $Z_1\to Z_2$ and $Z_1$ is not adjacent to $Y$, $Z_1, Z_2$ and $Y$ are not in the same B-component by the first part of the proof. Thus, $X, Z_1, Z_2$ and $Y$ are not in the same B-component, which completes the proof.
	\end{proof}
	
	\begin{lemma}\label{onlyrule2}
		Let $ \cal H $ be a causal MPDAG and $ {\cal C}^b $ be a B-component of $ \cal H $. If a directed edge in $ {\cal C}^b $ can be oriented by Meek's rules, then it can only be oriented by Rule 2 of Meek's Rules, and the directed edges in the configuration on the left-hand side of Rule 2 are all in $ {\cal C}^b $.
	\end{lemma}
	
	\begin{proof}[Proof of Lemma \ref{onlyrule2}]
		Let $X\to Y$ be a directed edge in $ {\cal C}^b $ that can be oriented by Meek's rules. If $X\to Y$ can be oriented by the first Meek's rule, then there is a $Z\notin adj(Y, {\cal H})$ such that $Z\to X$ is in ${\cal H}$. By Lemma~\ref{lem:into-B-com}, $Z\in \mathbf{V}({\cal C}^b)$. However, this is impossible based on Lemma~\ref{rule1or4}. Therefore, $X\to Y$ cannot be oriented by the first Meek's rule.
		
		If $X\to Y$ can be oriented by the third Meek's rule, then there are $Z_1, Z_2\in adj(X, {\cal H})$ such that $Z_1\to Y \leftarrow Z_2$ is a v-structure in ${\cal H}$ while $Z_1\to X \leftarrow Z_2$ is not in ${\cal H}$. If $Z_1\notin \mathbf{V}({\cal C}^b)$, then $Z_1\to X \to Z_2$ by Lemma~\ref{lem:into-B-com} and the first Meek's rule. On the other hand, $Z_2\to Y$ while $Z_2\not\to X$ implies that $Z_2\in \mathbf{V}({\cal C}^b)$ by Lemma~\ref{lem:into-B-com}. However, this is impossible by Lemma~\ref{lem:into-B-com}, since $Z_1$ and $Z_2$ are not adjacent. Therefore, $Z_1\in \mathbf{V}({\cal C}^b)$. Similarly, $Z_2\in \mathbf{V}({\cal C}^b)$, meaning that $X, Y, Z_1$ and $Z_2$ are in the same B-component of ${\cal H}$. However, as implied by Proposition~\ref{nonemptyMPDAG}, any B-component is an induced subgraph of some chain component of a CPDAG, and thus the v-structure $Z_1\to Y \leftarrow Z_2$ is not allowed in $ {\cal C}^b $, leading to a contradiction.
		
		If $X\to Y$ can be oriented by the fourth Meek's rule but cannot be oriented by the first Meek's rule, then there are $Z_1\in sib(X, {\cal H})$ and $ Z_2\in sib(X, {\cal H})$ such that $X-Z_1\to Z_2\to Y$ and $Z_1\notin adj(Y, {\cal H})$. Since $Z_1-X$ and $Z_2-X$, we have that $Z_1, Z_2, X$ and $Y$ are in the same B-component $ {\cal C}^b $. However, $Z_2 \to Y$ can be oriented by the first Meek's rule, which is contrary to Lemma~\ref{rule1or4}. Therefore, $X\to Y$ cannot be oriented by the fourth Meek's rule.
		
		
		Finally, if $X\to Y$ can be oriented by the second Meek's rule, then there is a $Z$ such that $X\to Z\to Y$ is in ${\cal H}$. By Lemma~\ref{lem:into-B-com}, $Z$ is in $ {\cal C}^b $. Thus, $X\to Z$ and $Z\to Y$ are all in $ {\cal C}^b $.
	\end{proof}

	\begin{lemma}\label{arrowinclique}
		Given a causal MPDAG $ \cal H $ and a B-component $ {\cal C}^b $ of $ \cal H $, let $ \mathbf{M}_i $ and $ \mathbf{M}_j $ be two distinct maximal cliques of $ {\cal C}^b $ such that $ \mathbf{M}_{ij} \coloneqq \mathbf{M}_i \cap \mathbf{M}_j \neq \emptyset $. For any vertex $ X \in \mathbf{M}_i \backslash \mathbf{M}_{ij}, Y \in \mathbf{M}_{ij} $, the directed edge $ X \rightarrow Y $ does not exist in $ {\cal C}^b $.
	\end{lemma}
	
	
	\begin{proof}[Proof of Lemma~\ref{arrowinclique}]
		$ X \in \mathbf{M}_i \backslash \mathbf{M}_{ij}$ implies that there must be a  $ Z \in \mathbf{M}_j \setminus \mathbf{M}_{ij} $ such that $Z$ is not adjacent to $X$, since otherwise $X$ is adjacent to every vertex in $\mathbf{M}_j $ and consequently $X\in \mathbf{M}_j$. Assume that such a directed edge $ X \rightarrow Y $ exists in $\cal H$, then such a $Z$ must be a child of $Y$ in $\cal H$ based on the first Meek's rule. This contradicts Lemma~\ref{onlyrule2}.
	\end{proof}
	
	\begin{lemma}[Necessity of Condition (iv)]\label{re:arrowinclique}
		Suppose that $ \cal H $ is a causal MPDAG and $ {\cal C}^b $ is a B-component of $ \cal H $. For any directed edge $ X \rightarrow Y $ in $ {\cal C}^b $, it can be proved that {(i)} $ adj(Y, {\cal C}^b) \backslash \{X\} \subseteq adj(X, {\cal C}^b) $, and {(ii)} $ pa(X, {\cal H}) \subseteq pa(Y, {\cal H}) \backslash \{X\} $.
	\end{lemma}
	
	\begin{proof}[Proof of Lemma~\ref{re:arrowinclique}]
		We first prove the correctness of statement (i).
		For any directed edge $X\to Y$ in a B-component ${\cal C}^b$ of ${\cal H}$, there is a maximal clique $\mathbf{M}_i$ of ${\cal C}^b$ containing both $X$ and $Y$. If $ adj(Y, {\cal C}^b) \backslash \{X\} \nsubseteq adj(X, {\cal C}^b) $, then there exists a vertex $ Z $ satisfying $ Z \in adj(Y, {\cal C}^b) \backslash \{X\} $ but $ Z \notin adj(X, {\cal C}^b) $, implying that $ Y$ and $Z $ also belong to another maximal clique $ \mathbf{M}_j $ of $ {\cal C}^b $ which does not contain $X$. Hence, we have that $ X \in \mathbf{M}_i \setminus \mathbf{M}_{ij}$ and $ Y \in \mathbf{M}_{ij} $, which contradicts Lemma \ref{arrowinclique}.
		
		We next show that statement (ii) holds. When $ pa(X, {\cal H}) = \varnothing $, the result is trivial, so we assume that $ pa(X, {\cal H}) \neq \varnothing $. For any vertex $ Z \in pa(X, {\cal H}) $, if $ Z \notin {\cal C}^b $, then $ Z \rightarrow Y $ by Lemma \ref{lem:into-B-com}. If $ Z \in {\cal C}^b $, then we have $ Z \in adj(Y, \cal H) $, since otherwise $ Z \rightarrow X \rightarrow Y $ implies that $X\to Y$ can be oriented by the first Meek's rule, contradicted to Lemma \ref{onlyrule2}. Therefore, $ Z \rightarrow Y $ appears in ${\cal H}$ by applying Rule 2 of Meek's rules.
	\end{proof}
	
	
	Finally, we present the proof of  Theorem \ref{the:MPDAG}.
	
	\begin{proof}[{Proof of  Theorem \ref{the:MPDAG}}]
		The necessity of conditions (i)-(iv) follows from Lemma \ref{Hc-chain}, Lemma \ref{Bchordal}, Lemma \ref{lem:into-B-com}, and Lemma \ref{re:arrowinclique} , respectively. In the following, we will prove the sufficiency of conditions (i) to (iv). Let $ \mathcal{H} = (\mathbf{V}, \mathbf{E}) $ be a partially directed graph which satisfies conditions (i)-(iv). Our goal is to show that $ \cal H $ is an MPDAG and $\cal H$ is causal.
		
		We first show that $ \cal H $ is acyclic. That is, there is no directed cycle in $ \cal H $. Assume, for the sake of contradiction, that there exist directed cycles in $ \cal H $ and let $\rho=(X_1, X_2,\cdots,X_n,X_1)$ be the shortest one. By condition (i), all the vertices on $\rho$ are in the same B-component, since otherwise the corresponding cycle of $\rho$ in $ {\cal H}_c $ is a partially directed cycle, which contradicts condition (i). If $n=3$, then $X_3\in pa(X_1, {\cal H })$ while $X_3\notin pa(X_2, {\cal H })$, meaning that $pa(X_1, {\cal H })\nsubseteq pa(X_2, {\cal H })\backslash \{X_1\}$, which contradicts condition (iv). If $n>3$, however, condition (iv) implies that $X_n\to X_2$, and thus $X_n\to X_2\to X_3\to\cdots\to X_n$ is a directed cycle of length $n-1$, contrary to the assumption that $\rho$ is the shortest. Therefore, $ \cal H $ is acyclic.
		
		To prove that $ \cal H $ is an MPDAG, it suffices to show that $ \cal H $ is closed under Meek's rules. We will consider each rule separately in below.
		\begin{enumerate}
			\item[(i)] If $ \cal H $ is not closed under Rule 1 of Meek's rules, then $ \cal H $ has an induced subgraph $ X \rightarrow Y - Z $, in which $ X \notin adj(Z, \cal H) $. Since $ Y $ and $ Z $ are connected by an undirected edge, there exists a B-component $ {\cal C}^b $ of $ \cal H $ such that $ Y, Z \in {\cal C}^b $. According to condition (iii), $ X \in {\cal C}^b $ since otherwise $X \rightarrow Z$ should be in $ \cal H $. However, if $ X \in {\cal C}^b $, then $ Z \in adj(Y, {\cal C}^b) $ but $ Z \notin adj(X, {\cal C}^b) $, which contradicts condition (iv). Thus, $ \cal H $ is closed under the first Meek's rule.
			
			\item[(ii)] If $ \cal H $ is not closed under Rule 2 of Meek's rules, then $ \cal H $ has an induced subgraph consists of $ X \rightarrow Y \rightarrow Z $ as well as $ X - Z $. By condition (iii), $ X, Y, Z $ are in the same B-component. However, for the directed edge $ Y \rightarrow Z $, $ X \in pa(Y, \cal H) $ but $ X \notin pa(Z, {\cal H}) \setminus \{Y\} $, which contradicts condition (iv). Thus, $ \cal H $ is closed under the second Meek's rule.
			
			\item[(iii)] If $ \cal H $ is not closed under Rule 3 of Meek's rules, then $ \cal H $ has an induced subgraph shown in Figure \ref{rule3left}. The vertices $ X, Y, Z_1, Z_2 $ are in the same B-component $ {\cal C}^b $. However, this is impossible as {$ Y \in adj(Z_2, {\cal C}^b) \backslash \{Z_1\}$ but $ Y \notin adj(Z_1, {\cal C}^b) $}.
			
			\item[(iv)] If  $ \cal H $ is not closed under Rule 4 of Meek's rules, then $ \cal H $ has an induced subgraph shown in one of Figures \ref{rule4left-1}-\ref{rule4left-3}. By the similar argument for (ii), no matter which induced subgraph $ \cal H $ has, the vertices $ X, Y, Z_1, Z_2 $ are in the same B-component $ {\cal C}^b $. However, for the directed edge $ Z_1 \rightarrow Z_2 $, we have $ Y \in adj(Z_2, {\cal C}^b) \backslash \{Z_1\}$ but $ Y \notin adj(Z_1, {\cal C}^b) $, contrary to condition (iv). Hence, $ \cal H $ is closed under the fourth Meek's rule.
		\end{enumerate}
		
		So far, we have proved that $ \cal H $ is an MPDAG. It remains to show that $\cal H$ is causal. By the definition of a causal MPDAG, it suffices to show that there is a DAG that has the same skeleton and the same v-structures as $ \cal H $, or equivalently, there exists an orientation of all undirected edges in  $ \cal H $ that does not create a new v-structure or a directed cycle.
		
		We first claim that (1) no orientation of the undirected edges in ${\cal H}$ will create a directed cycle which includes a directed edge or edges in ${\cal H}_c$ and (2) no orientation of an undirected edge in ${\cal H}$ can create a new v-structure with an edge that was oriented in ${\cal H}_c$. In fact, the first claim holds because of condition (i), and the second claim holds because of condition (iii).

		Based on the above two claims, to prove that  $\cal H$ is causal, it suffices to show that the skeleton of each B-component of $ \cal H $ has a perfect elimination ordering whose corresponding DAG contains the existing directed edges in that B-component.
		
		Let $ {\cal C}^b $ be a B-component of $ \cal H $. If $ {\cal C}^b $ is a complete graph with $ n $ vertices, then every vertex is simplicial in $ {\cal C}^b $. We claim that, there exists a vertex $V_1$ in $ {\cal C}^b $ which has no child in $ {\cal C}^b $. In fact, if every vertex in $ {\cal C}^b $ has a child in $ {\cal C}^b $, then there will be a directed cycle in $ {\cal C}^b $, which is impossible.  Note that, the induced subgraph of $ {\cal C}^b $ over $ \mathbf{V}({\cal C}^b) \setminus \{V_1\} $ is still complete. Hence, repeat the above procedure we can find a sequence of vertices $V_1, V_2, \cdots, V_n$. It can be easily verified that the ordering of the vertices forms a PEO of $ {\cal C}^b $, and the corresponding DAG contains  all directed edges in $ {\cal C}^b $.
		
		We then consider the case where $ {\cal C}^b $ is not a complete graph. Assume that every simplicial vertex has a child in $ {\cal C}^b $. Let $X$ be a simplicial vertex and $\mathbf{M}_i$ be the (unique) maximal clique that contains $X$ (Lemma~\ref{app:lem:chordal}). Denote by $\mathbf{S}$ the set of simplicial vertices of ${\cal C}^b$ contained in $\mathbf{M}_i$. Since $\mathbf{M}_i$ induces a complete subgraph of $ {\cal C}^b $, $\mathbf{S}\subseteq \mathbf{M}_i$ also induces a complete subgraph of $ {\cal C}^b $. If every simplicial vertex in $\mathbf{S}$ has a child which is also in $\mathbf{S}$, then we can construct a directed cycle. Thus, there must be a simplicial vertex in $\mathbf{S}$ whose child is not in $\mathbf{S}$. Without loss of generality, we can assume that such a vertex is $X$. Notice that, $ {\cal C}^b $ is an incomplete connected chordal graph, $\mathbf{M}_i\setminus\mathbf{S}\neq\varnothing$, and thus $X\to Y$ is in $ {\cal C}^b $ for some $Y\in \mathbf{M}_i\setminus\mathbf{S}$. As $\mathbf{S}$ consists of all simplicial vertices contained in $\mathbf{M}_i$, $Y$ is not simplicial, and thus there is a $Z\in adj(Y, {\cal C}^b)$ that is not adjacent to $X$ (Lemma~\ref{app:lem:chordal}). This means that $adj(Y, {\cal C}^b) \setminus\{X\}\nsubseteq adj(X, {\cal C}^b)$, which violates condition (iv).
		
		Therefore, we can find a simplicial vertex, denoted by $ V_1 $, that does not have any child in $ {\cal C}^b $. Since the induced subgraph of $ {\cal C}^b $ over $ \mathbf{V}({\cal C}^b) \setminus \{V_1\} $ is still chordal, we can repeat the above procedure and find a sequence  vertices $V_1, V_2, \cdots, V_n$. Again, it can be checked that the ordering of the vertices forms a PEO of $ {\cal C}^b $, and the corresponding DAG contains  all directed edges in $ {\cal C}^b $. This completes the proof.
	\end{proof}

	\subsection{Proof of Proposition \ref{thm:MSP}}\label{app:proof:thm:MSP}

	\begin{proof}
		To prove the sufficiency, we first show that $ \cal A $ is a generator of $ \cal H $, and then prove its minimality.
		
		
		Denote by $ {\cal G}^* $ the CPDAG with the same skeleton and v-structures as $ {\cal H} $. To prove that $ \cal A $ is a generator of $ \cal H $, or equivalently, to prove that $\cal H$ is the MPDAG of  $ [{\cal G}^*, {\cal A}] $, by Corollary \ref{coro:maximaloc} it suffices to show that for every directed edge $ X \rightarrow Y $ in $ \mathbf{E}_d({\cal H}) \setminus (\mathbf{E}_d({\cal G}^*)\cup {\cal A} )$, an orientation component for $Y$ with respect to ${\cal A}$ and ${\cal G}^*$ contains $X$.
		
		By construction of ${\cal A}$, the edges in $ \mathbf{E}_d({\cal H}) \setminus (\mathbf{E}_d({\cal G}^*)\cup {\cal A} )$ are all M-strongly protected. On the other hand, as $\cal H$ can be viewed as the MPDAG of $[{\cal G}^*, \mathbf{E}_d({\cal H})]$, the maximal orientation component for $Y$ with respect to $\mathbf{E}_d({\cal H})$ and ${\cal G}^*$, denoted by ${\cal U}$, contains $X$. In the following, we will show that ${\cal U}$ is an orientation component for $Y$ with respect to ${\cal A}$ and ${\cal G}^*$.

		Suppose that, with respect to ${\cal A}$ and ${\cal G}^*$, ${\cal U}$ is not an orientation component for $Y$, then by the definition of an orientation component, there exists another potential leaf node $Z$ in ${\cal U}$ with respect to ${\cal A}$ and ${\cal G}^*$. Such a variable $Z$ must have the following properties.
		
		\begin{enumerate}
			\item[(P1)] $Z$ is simplicial in ${\cal U}$.
			\item[(P2)] If an undirected edge connected to $Z$ in ${\cal U}$ is directed in ${\cal A}$, $Z$ is not the tail of that directed edge.
			\item[(P3)] For every sibling $W$ of $Z$ in ${\cal U}$ such that $Z\to W$ is in ${\cal H}$ (such a $W$ definitely exists), $Z\to W$ is not in ${\cal A}$. (Note that, since every directed edge in ${\cal A}$ is also in ${\cal H}$, $Z\to W$ is in ${\cal H}$ implies that $W\to Z$ is not in ${\cal A}$ either).
		\end{enumerate}
		
		The third property comes from that fact that, with respect to $\mathbf{E}_d({\cal H})$ and ${\cal G}^*$, ${\cal U}$ is the maximal orientation component for $Y$.
		
		Since $Z\to W$ is not in ${\cal A}$, $Z\to W$ is M-strongly protected in ${\cal H}$. Note that, $Z$ and $W$ are adjacent in ${\cal U}$, they are in the same chain component of ${\cal G}^*$. We first prove that $Z\to W$ cannot occur in the configurations (b) and (d), and if it occurs in one of the configurations (a), (c) and (e), the involved vertices in those configurations are all in the same chain component of ${\cal G}^*$.
		
		\begin{enumerate}
			\item[(i)] If $Z\to W$ occurs in the configuration (a), then there exists a vertex $ W_2 \notin adj(W, \cal H) $ such that $ W_2 \rightarrow Z \to W $ is in $\cal H$. If $W_2$ is not in the same chain component as $Z$ and $W$, then $W_2\to Z$ implies that $W_2\to W$ is also in ${\cal H}$, which contradicts the configuration (a).
			
			\item[(ii)] If $Z\to W$ occurs in the configuration (b), then there is a vertex $W_2\notin adj(Z, {\cal H})$ such that $Z\to W \leftarrow W_2$, which is a v-structure collided on $W$, is in ${\cal H}$. This means that $W_2$ is not in the same chain component as $Z$, and thus, $W_2\to Z$ should be in ${\cal H}$, leading to a contradiction.
			
			\item[(iii)] If $Z\to W$ occurs in the configuration (c), then there is a vertex $W_2$ such that $Z\to W_2 \to W$ is in ${\cal H}$. If $W_2$ is not in the same chain component as $Z$ and $W$, then $W_2\to W$ implies that $W_2\to Z$ is also in ${\cal H}$, which contradicts the configuration (c).
			
			\item[(iv)] If $Z\to W$ occurs in the configuration (d), then there are vertices $W_2, W_3\in sib(Z, {\cal H})$ such that $W_2$ is not adjacent to $ W_3$ in ${\cal H}$ and $W_2\to W \leftarrow W_3$ is in ${\cal H}$. It is easy to see that $W_2$, $W_3$ and $\{Z, W\}$ are in three different chain components, and thus, $W_2\to Z \leftarrow W_3$ should be in ${\cal H}$. This contradicts the configuration (d).

			\item[(v)] If $Z\to W$ occurs in the configuration (e), then there are vertices $ W_2, W_3 \in sib(Z, \cal H) $ such that $W_2\to W_3\to W$ and $W_2$ is not adjacent to $W$ in $\cal H$. If $W_3$ is not in the same chain component as $Z$ and $W$, then  $W_3\to Z$ should be in $\cal H$, contradicted to the configuration. If $W_2$ not in the same chain component as $Z$, $W$ and $W_3$, then $W_2\to Z$ should also be in $\cal H$, contradicted to the configuration.		
		\end{enumerate}
		
		Below, we will consider the configurations (a), (c) and (e) separately and show that although $Z\to W$ is M-strongly protected in ${\cal H}$, $Z\to W$ cannot occur in any of these configurations.
		
		\begin{enumerate}
			\item[(i)] If $Z\to W$ occurs in the configuration (a), then there exists a vertex $ W_2 \notin adj(W, \cal H) $ such that $W_2$, $Z$ and $W$ are in the same chain component and $ W_2 \rightarrow Z $ is in $\cal H$. If $ W_2 \in \mathbf{V}(\cal U ) $, then $ Z $ is not simplicial in $ \cal U $, which contradicts (P1). Thus, $ W_2 \notin \mathbf{V}(\cal U) $. We claim that the induced subgraph of ${\cal G}^*$ over $ \mathbf{V}({\cal U}) \cup \{W_2\} $, denoted by ${\cal U}'$,  is also an orientation component for $ Y $ with respect to $\mathbf{E}_d({\cal H})$ and ${\cal G}^*$. In fact, with respect to $\mathbf{E}_d({\cal H})$ and ${\cal G}^*$, since none of the vertices except for $Y$ is a potential leaf node in ${\cal U}$, the vertices in ${\bf V}({\cal U})\setminus\{Y\}$ are definitely not potential leaf nodes in ${\cal U}'$ either. On the other hand, $W_2$ is not a potential leaf node in ${\cal U}'$, as $W_2\to Z$ is in $\mathbf{E}_d({\cal H})$, the consistency of $\mathbf{E}_d({\cal H})$ with ${\cal G}^*$ implies that the only potential leaf node in ${\cal U}'$ must be $Y$ (Theorem~\ref{thm:consistency}). Therefore, ${\cal U}'$ is also an orientation component for $ Y $ with respect to $\mathbf{E}_d({\cal H})$ and ${\cal G}^*$. This contradicts the maximality of $\cal U$.
			
			\item[(ii)] If $Z\to W$ occurs in the configuration (e), then there are vertices $ W_2, W_3 \in sib(Z, \cal H) $ such that $W_2\to W_3\to W$ and $W_2$ is not adjacent to $W$ in $\cal H$. If both $ W_2$ and $W_3$ are in $\mathbf{V}(\cal U ) $, then $ Z $ is not a simplicial vertex in $ \cal U $, which is again contradicted to (P1). On the other hand, if either $ W_2 $ or $ W_3 $ is not in $ \mathbf{V(\cal U)} $, then by the similar argument given in (i), the induced subgraph of ${\cal G}^*$ over $ \mathbf{V}({\cal U}) \cup \{W_2, W_3\} $ is also an orientation component for $ Y $ with respect to $\mathbf{E}_d({\cal H})$ and ${\cal G}^*$. This is again contradicted to the maximality of $\cal U$.

			\item[(iii)] If $Z\to W$ occurs in the configuration (c), then there is a vertex $W_2$ such that {$Z\to W_2 \to W$} is in ${\cal H}$. If $ W_2 \notin \mathbf{V}(\cal U) $, then by the similar argument given in (i), the induced subgraph of ${\cal G}^*$ over $ \mathbf{V}({\cal U}) \cup \{W_2\} $ is also an orientation component for $ Y $ with respect to $\mathbf{E}_d({\cal H})$ and ${\cal G}^*$, contradicted to the maximality of $\cal U$. Hence, $ W_2 \in \mathbf{V}(\cal U) $, and by (P3), $ Z \rightarrow W_2$ is also M-strongly protected in $\cal H$. By the same argument given in (i) and (ii), $ Z \rightarrow W_2$ can only occur in the configuration (c). Repeat the above procedure we can find a sequence of vertices $W_2,W_3,\cdots W_n, \cdots$ in $adj(Z, {\cal U})$, such that $Z\to W_i$ for $i=2,3,\cdots $ and $W\leftarrow W_2 \leftarrow W_3 \leftarrow \cdots\leftarrow  W_n \cdots\leftarrow $. Note that, since the above procedure never ends, but there are only a finite number of vertices in $adj(Z, {\cal U})$,  a subpath of $W\leftarrow W_2 \leftarrow W_3 \leftarrow \cdots$ must be a directed cycle, leading to a contradiction.
		\end{enumerate}

		Therefore, although $Z\to W$ is M-strongly protected in ${\cal H}$, $Z\to W$ cannot occur in any of the configurations (a) to (e). This is impossible. Thus, ${\cal U}$ is an orientation component for $Y$ with respect to ${\cal A}$ and ${\cal G}^*$. Consequently, $ \cal A $ is a generator of $ \cal H $.
		
		We then show that $ \cal A $ is minimal. If there is another generator $ {\cal A}^- $ of $ \cal H $ such that of $ |{\cal A}^- | < |{\cal A}|$, then there is a directed edge $ X \rightarrow Y $ in  $ {\cal A}$ which is not in $ {\cal A}^-$. Since $ {\cal A}^- $ is a generator of $\cal H$ and $ X \rightarrow Y $ is in  $\cal H$, $ X \rightarrow Y $ is either in a v-structure of the form $X\to Y\leftarrow Z$, or can be oriented by Meek's rules. This means that $ X \rightarrow Y $ occurs in at least one configurations labeled by (a) to (e) as an induced subgraph of $\cal H$. Hence, $ X \rightarrow Y $ is M-strongly protected in $\cal H$, which is contrary to the assumption that $ X \rightarrow Y $ is in $ {\cal A}$. This completes the proof of sufficiency.
		
		Next, we prove the necessity. Let ${\cal A}$ be a minimal generator of $\cal H$. We first show that ${\cal A}$ contains all directed edges in $\cal H$ that are not M-strongly protected. Suppose that there is a directed edge $ X \rightarrow Y \notin {\cal A} $ which is in $ \cal H $ but not M-strongly protected. Then, since ${\cal A}$ generates $\cal H$, $ X \rightarrow Y $ is either in a v-structure of the form $X\to Y\leftarrow Z$, or can be oriented by Meek's rules. This means that $ X \rightarrow Y $ occurs in at least one configurations labeled by (a) to (e) as an induced subgraph of $\cal H$. That is, $ X \rightarrow Y $ is M-strongly protected. This contradicts our assumption. We next show that  ${\cal A}$ only contains the directed edges in $\cal H$ that are not M-strongly protected. In fact, if this is not the case, then the proper subset of ${\cal A}$ which consisting of the directed edges that are not M-strongly protected in $\cal H$ is a generator of $\cal H$, meaning that ${\cal A}$ is not minimal, which is a contradiction.

		Finally, the uniqueness follows from the fact that the set of  directed edges in $\cal H$ that are not M-strongly protected in $\cal H$ is unique.
	\end{proof}

	\subsection{Proof of Theorem \ref{thm:MSPrep}}\label{app:proof:thm:MSPrep}
	\begin{proof}
		The conclusion follows directly from Propositions~\ref{nonemptyMPDAG} and \ref{thm:MSP}.
	\end{proof}
	
	\subsection{Proof of Proposition \ref{fact1}}\label{app:proof:fact1}
	\begin{proof}
		Both claims hold by the definition of a DCC, and thus we omit the proof.
	\end{proof}
	
	\subsection{Proof of Theorem \ref{thm:nbr_set_const}}\label{app:proof:nbr_set_const}
	\begin{proof}
		The first statement is clearly true and the third statement can be derived from \citet[Lemma~2]{fang2020bgida}. Since the second statement is the inverse of the third statement, the proof is completed.
	\end{proof}
	
	\subsection{Proof of Proposition \ref{fact3}}\label{app:proof:fact3}
	\begin{proof}
		The proof follows directly from the definition of a DCC.
	\end{proof}
	
	\subsection{Proof of Proposition \ref{pro:reducedf}}\label{app:proof:reducedf}
	\begin{proof}
		The proof follows from Proposition \ref{fact3}.
	\end{proof}
	
	\subsection{Proof of Theorem \ref{thm:consistency}}\label{app:proof:consistency}
	
	We first prove two lemmas.
	
	\begin{lemma}\label{lem:app-pln}
		Let $\mathcal{G}^*$ be a CPDAG, $\cal K$ be a set of consistent DCCs, and $\cal U$ be a connected undirected induced subgraph of $\mathcal{G}^*$. For any DAG ${\cal G} \in [{\cal G}^*, {\cal K}]$, every leaf node in the induced subgraph of $\cal G$ over $\mathbf{V}(\mathcal{U})$ is a potential leaf node in $\mathcal{U}$ with respect to ${\cal K}$ and ${\cal G}^*$.
	\end{lemma}
	
	\begin{proof}
		Denote by ${\cal G}_{\mathrm{sub}}$ the induced subgraph of $\cal G$ over $\mathbf{V}(\mathcal{U})$. By the definition of an undirected induced subgraph and the fact that $\mathcal{G}^*$ has no partially directed cycle~\citep{andersson1997characterization}, $\mathcal{U}$ is  the induced subgraph of $\mathcal{G}^*$ over $\mathbf{V}(\mathcal{U})$. Thus, two vertices are adjacent in $\mathcal{U}$ if and only if they are adjacent in ${\cal G}_{\mathrm{sub}}$. Let $V_{\mathrm{leaf}}$ be a leaf node in ${\cal G}_{\mathrm{sub}}$. By the definition of a potential leaf node, the conclusion holds if $adj(V_{\mathrm{leaf}}, \mathcal{U})=\varnothing$. Therefore, to prove the lemma it suffices to show that, (1) $adj(V_{\mathrm{leaf}}, \mathcal{U})$ induces a complete subgraph of $\mathcal{U}$, and (2) $V_{\mathrm{leaf}}$ is not the tail of any DCC in ${\cal K}(\mathcal{U})$.
		
		As $V_{\mathrm{leaf}}$ is a leaf node in ${\cal G}_{\mathrm{sub}}$, $adj(V_{\mathrm{leaf}}, {\cal G}_{\mathrm{sub}}) \to V_{\mathrm{leaf}}$ are in ${\cal G}_{\mathrm{sub}}$. Since $adj(V_{\mathrm{leaf}}, {\cal G}_{\mathrm{sub}}) \subseteq sib(V_{\mathrm{leaf}}, {\cal G}^*)$, the configuration $adj(V_{\mathrm{leaf}}, {\cal G}_{\mathrm{sub}}) \to V_{\mathrm{leaf}}$ contains no v-structure collided on $V_{\mathrm{leaf}}$. Thus, $adj(V_{\mathrm{leaf}}, {\cal G}_{\mathrm{sub}})$ induces a complete subgraph of ${\cal G}_{\mathrm{sub}}$, meaning that $adj(V_{\mathrm{leaf}}, \mathcal{U})$ induces a complete subgraph of $\mathcal{U}$. This completes the proof of statement (1). On the other hand, if there is a $V_{\mathrm{leaf}}\tor \kappa_h$ in ${\cal K}(\mathcal{U})$, then by Equation~(\ref{eq:n_uu}), $\kappa_h\subseteq adj(V_{\mathrm{leaf}}, \mathcal{U})=adj(V_{\mathrm{leaf}}, {\cal G}_{\mathrm{sub}})$ and $V_{\mathrm{leaf}}\tor \kappa_h$ is in ${\cal K}({\cal G}^*)$. Notice that $\cal K$ is equivalent to ${\cal K}({\cal G}^*)$, $V_{\mathrm{leaf}}\tor \kappa_h$ must hold for ${\cal G}$.  Consequently, $V_{\mathrm{leaf}}\tor \kappa_h$ holds for ${\cal G}_{\mathrm{sub}}$. However, $V_{\mathrm{leaf}}$ is a leaf node in ${\cal G}_{\mathrm{sub}}$, which leads to a contradiction.
	\end{proof}
	
	Lemma~\ref{lem:app-pln} suggests that if a vertex is not a potential leaf node in some connected undirected induced subgraph $\mathcal{U}$, then it cannot be a leaf node in the induced subgraph over $\mathbf{V}(\mathcal{U})$ of any restricted Markov equivalent DAG.
	
	\begin{lemma}\label{lem:app-equi}
		Let $\mathcal{G}^*$ be a CPDAG and $\cal K$ be a set of DCCs. Then, ${\cal K}$ is consistent with ${\cal G}^*$ if and only if ${\cal K}({\cal C})$ is consistent with ${\cal C}$ for any chain component ${\cal C}$.
	\end{lemma}
	
	\begin{proof}
		Suppose that $[{\cal G}^*, {\cal K}({\cal G}^*_u)]\neq\varnothing$ and let ${\cal G} \in [{\cal G}^*, {\cal K}({\cal G}^*_u)]$. For any chain component ${\cal C}$, let ${\cal G}_{\mathrm{sub}}$ denote the induced subgraph of ${\cal G}$ over $\mathbf{V}({\cal C})$. It is easy to verify that ${\cal G}_{\mathrm{sub}}\in [{\cal C}, {\cal K}({\cal C})]$. Conversely, if $[{\cal C}, {\cal K}({\cal C})]\neq\varnothing$ for any chain component ${\cal C}$, then choose ${\cal G}_{\cal C}\in [{\cal C}, {\cal K}({\cal C})]$ arbitrarily for each chain component and orient undirected edges in ${\cal G}^*$ according to $\{{\cal G}_{\cal C}\}$. That is, orient $X-Y$ in ${\cal G}^*$ as $X\to Y$ if $X$ is a parent of $Y$ in the DAG ${\cal G}_{\cal C}$, where $\cal C$ is the chain component containing $X$ and $Y$.  Notice that ${\cal G}^*_u$ is a union of (disjoint) chain components, we have ${\cal K}({\cal G}^*_u)=\bigcup_{\cal C} {\cal K}({\cal C})$, where ${\cal C}$ is a chain component. It is straightforward to show that the resulting DAG with the orientations defined above is in $[{\cal G}^*, {\cal K}({\cal G}^*_u)]$.
	\end{proof}
	
	Finally, we present the proof of Theorem~\ref{thm:consistency}.
	
	\begin{proof}[Proof of Theorem~\ref{thm:consistency}]
		We first prove the necessity. If $\cal K$ is consistent, then there is a DAG $\cal G$ in $[\mathcal{G}^*, {\cal K}]$. Let $\mathcal{U}$ be an arbitrary connected undirected induced subgraph of $\mathcal{G}^*$, and denote the induced subgraph of $\cal G$ over $\mathbf{V}(\mathcal{U})$ by ${\cal G}_{\mathrm{sub}}$. Since any induced subgraph of a DAG is still a DAG, ${\cal G}_{\mathrm{sub}}$ is a DAG, and thus it must have a leaf node $V_{\mathrm{leaf}}$. By Lemma~\ref{lem:app-pln}, we can conclude that $V_{\mathrm{leaf}}$ is a potential leaf node in $\mathcal{U}$ with respect to ${\cal K}$ and ${\cal G}^*$.
		
		We next prove the sufficiency. By Lemma~\ref{lem:app-equi}, ${\cal K}$ is consistent with ${\cal G}^*$ if and only if ${\cal K}({\cal C})$ is consistent with ${\cal C}$ for any chain component ${\cal C}$. Therefore, given a chain component ${\cal C}$, our goal is to prove that ${\cal K}({\cal C})$ is consistent with ${\cal C}$, providing that the potential-leaf-node condition holds. Based on the analysis in Appendix \ref{app:proof:pre}, we need to construct a PEO of ${\cal C}$ such that the corresponding DAG satisfies all DCCs in ${\cal K}({\cal C})$. By assumption, ${\cal C}$ has a potential leaf node with respect to ${\cal K}$ and ${\cal G}^*$, denoted by $V_1$. By the definition of a potential leaf node, $V_1$ is simplicial in ${\cal C}$ . Next, consider the induced subgraph of ${\cal C}$ over $\mathbf{V}({\cal C})\setminus\{V_1\}$, denoted by ${\cal C}_2$. ${\cal C}_2$ is clearly connected. Hence, by assumption ${\cal C}_2$ has a potential leaf node, denoted by $V_2$. Following the above procedure, we have a sequence of undirected graphs $({\cal C} = {\cal C}_1, {\cal C}_2, \cdots, {\cal C}_m)$ and a sequence of vertices $(V_1, V_2,\cdots,V_m)$, where $m=|\mathbf{V}({\cal C})|$. By the construction, the ordering of the vertices in this sequence forms a PEO of ${\cal C}$. Denote by ${\cal G}_{\cal C}$ the corresponding DAG of this PEO. If there is a $V_i\tor \kappa_h\in {\cal K}({\cal C})$ which does not hold for ${\cal G}_{\cal C}$, then $\kappa_h\to V_i$ are in ${\cal G}_{\cal C}$. This means $\kappa_h\subseteq \{V_{i+1}, V_{i+2}, \cdots, V_m\}$, which contradicts the construction of the vertex sequence as $V_i$ is definitely not a potential leaf node in the induced subgraph ${\cal C}_i$. Therefore, ${\cal G}_{\cal C}\in [{\cal C}, {\cal K}({\cal C})]$. This completes the proof of Theorem~\ref{thm:consistency}.
	\end{proof}

	\subsection{Proof of Theorem \ref{thm:equi_dcc}}
	\begin{proof}
		The equivalence of statements (i) and (ii) follows from the definition of equivalence and redundancy. Observed that when ${\cal K}$ is consistent with ${\cal G}^*$,
		\begin{equation}\label{eq:induction}
			\begin{split}
				\kappa_t\tor \kappa_h \text{~is~redundant~with~respect~to~} [{\cal G}^*, {\cal K}]
				\iff & \forall\, {\cal G}\in [{\cal G}^*, {\cal K}],\, \kappa_t\tor \kappa_h \text{~holds~for~} {\cal G} \\
				\iff & \forall\, {\cal G}\in [{\cal G}^*, {\cal K}],\, ch(\kappa_t, {\cal G})\cap \kappa_h\neq \varnothing \\
				\iff & \forall\, {\cal G}\in [{\cal G}^*, {\cal K}],\, \kappa_h\nsubseteq pa(\kappa_t, {\cal G})\\
				\iff & \{\kappa_h\to \kappa_t\}\cup {\cal K} \text{~is~inconsistent~with~} {\cal G}^*.
			\end{split}
		\end{equation}
		Therefore, statement (ii) is equivalent to statement (iii) if both  ${\cal K}_1$ and ${\cal K}_2$ are consistent with ${\cal G}^*$. If neither ${\cal K}_1$ nor ${\cal K}_2$ is consistent with ${\cal G}^*$, then statements (i) and (iii) hold simultaneously. Finally, if ${\cal K}_1$ is consistent with ${\cal G}^*$ but ${\cal K}_2$ is not, then statements (i) does not hold. Thus, we need only to show that statements (iii) does not hold either. In fact, by Equation~\eqref{eq:induction}, if $\cup_{D\in {\kappa_h}}\{D\to \kappa_t\}\cup {\cal K}_1$ is not consistent with ${\cal G}^*$ for every $\kappa \in {\cal K}_2$, then every $\kappa \in {\cal K}_2$ is redundant with respect to $[{\cal G}^*, {\cal K}_1]$. Consequently, ${\cal K}_2 \cup {\cal K}_1$ is equivalent to ${\cal K}_1$ given ${\cal G}^*$ and thus is consistent. However, this is impossible, as the union of an inconsistent DCC and a consistent DCC is definitely inconsistent.
	\end{proof}

	{
		\subsection{Proof of Theorem~\ref{thm:equal_card}}
		Some technical lemmas are required before we present the proof of Theorem~\ref{thm:equal_card}.
		
		\begin{lemma}\label{lem:moc_of_dcc}
			Let $\mathcal{G}^*$ be a CPDAG and $\cal K$ be a set of DCCs consistent with $\mathcal{G}^*$, then a DCC $\kappa$ holds for every DAG in $[\mathcal{G}^*, {\cal K}]$ if and only if either $\kappa_t\rightarrow s$ appears in $\mathcal{G}^*$ for some $s\in\kappa_h$ or there exists a connected undirected induced subgraph $\cal U$ of $\mathcal{G}^*$ containing the vertices in $\kappa_h\cap sib(\kappa_t, \mathcal{G}^*)$ and $\kappa_t$ such that with respect to $\cal K$ and $\mathcal{G}^*$, every potential leaf node in $\cal U$ is in $\kappa_h\cap sib(\kappa_t, \mathcal{G}^*)$.
		\end{lemma}
		\begin{proof}
			We first prove the sufficiency. It is clear that $\kappa_t\rightarrow s$ appears in $\mathcal{G}^*$ for some $s\in\kappa_h$ implies that $\kappa$ holds for every DAG in $[\mathcal{G}^*, {\cal K}]$. On the other hand, if there exists a connected undirected induced subgraph $\cal U$ of $\mathcal{G}^*$ containing the vertices in $\kappa_h\cap sib(\kappa_t, \mathcal{G}^*)$ and $\kappa_t$ such that with respect to $\cal K$ and $\mathcal{G}^*$, every potential leaf node in $\cal U$ is in $\kappa_h\cap sib(\kappa_t, \mathcal{G}^*)$, by Lemma~\ref{lem:app-pln}, for every DAG ${\cal G}\in[\mathcal{G}^*, {\cal K}]$, there exists an $s\in\kappa_h$ such that $s$ is a leaf node in the induced subgraph of $\cal G$ over ${\bf V}({\cal U})$. Therefore, $\kappa_t\to s$ is in $\cal G$, indicating that $\kappa$ holds for $\cal G$.
			
			We next prove the necessity. By Theorem~\ref{thm:equi_dcc}, $\cup_{s\in {\kappa_h}}\{s\to \kappa_t\}\cup {\cal K}$ is not consistent with ${\cal G}^*$. Therefore, by Theorem~\ref{thm:consistency}, there is a connected undirected induced subgraph ${\cal U}_0$ of $\mathcal{G}^*$ which has no potential leaf node with respect to $\cup_{s\in {\kappa_h}}\{s\to \kappa_t\}\cup {\cal K}$ and ${\cal G}^*$. If $\kappa_t\rightarrow s$ appears in $\mathcal{G}^*$ for some $s\in\kappa_h$, then the proof is completed. Thus, we assume that $\kappa_t\nrightarrow s$ in $\mathcal{G}^*$ for all $s\in\kappa_h$. Let $P=\{s\in\kappa_h\mid s\to\kappa_t\; \text{is in}\; {\cal G}^{*}\}$ and $Q=\kappa_h\cap sib(\kappa_t, \mathcal{G}^*)=\{s\in\kappa_h\mid s-\kappa_t\; \text{is in}\; {\cal G}^{*}\}$. Since ${\cal K}$ is consistent with ${\cal G}^*$, $Q\neq\varnothing$ and ${\cal U}_0$ must contain $\kappa_t$ and at least one $s\in Q$. Now consider the connected undirected induced subgraph $\cal U$ of $\mathcal{G}^*$ over ${\bf V}({\cal U}_0)\cup Q$. With respect to $\cup_{s\in {\kappa_h}}\{s\to \kappa_t\}\cup {\cal K}$ and ${\cal G}^*$, none of the vertices in $Q$ is a potential leaf node as $s\to\kappa_t$ for $s\in Q$, and none of the vertices in ${\bf V}({\cal U}_0)$ is a potential leaf node as ${\cal U}_0$ has no potential leaf node with respect to $\cup_{s\in {\kappa_h}}\{s\to \kappa_t\}\cup {\cal K}$ and ${\cal G}^*$. As a result, $\cal U$ has no potential leaf node with respect to $\cup_{s\in {\kappa_h}}\{s\to \kappa_t\}\cup {\cal K}$ and ${\cal G}^*$.
			
			However, $\cal U$ has potential leaf nodes with respect to ${\cal K}$ and ${\cal G}^*$. If a vertex $p\in {\bf V}({\cal U}_0)\setminus Q$ is a potential leaf node with respect to ${\cal K}$ and ${\cal G}^*$, then $p$ is also a potential leaf node with respect to $\cup_{s\in {\kappa_h}}\{s\to \kappa_t\}\cup {\cal K}$ and ${\cal G}^*$. Therefore, the potential leaf nodes in $\cal U$ with respect to ${\cal K}$ and ${\cal G}^*$ are all in $Q$.
		\end{proof}
		
		The above lemma generalizes Theorem~\ref{thm:mpdag} and Corollary~\ref{coro:maximaloc}. When $\kappa$ is minimally redundant with respect to $[\mathcal{G}^*, {\cal K}]$, we have the following corollary.
		
		\begin{corollary}\label{coro:moc_of_dcc}
			Let $\mathcal{G}^*$ be a CPDAG and $\cal K$ be a set of DCCs consistent with $\mathcal{G}^*$, then a DCC $\kappa$ is minimally redundant with respect to $[\mathcal{G}^*, {\cal K}]$ if and only if either $\kappa_h=\{s\}$ is a singleton set and $\kappa_t\rightarrow s$ appears in $\mathcal{G}^*$, or $\kappa_h \subseteq sib(\kappa_t, \mathcal{G}^*)$ and there exists a connected undirected induced subgraph $\cal U$ of $\mathcal{G}^*$ containing the vertices in $\kappa_h$ and $\kappa_t$ such that with respect to $\cal K$ and $\mathcal{G}^*$, $\kappa_{h}$ are all and only potential leaf nodes in ${\cal U}$.
		\end{corollary}
		
		The following lemma is the key to prove Theorem~\ref{thm:equal_card}.
		
		\begin{lemma}\label{lemma:minimal_dcc}
			Suppose that ${\cal G}^*$ is a CPDAG and ${\cal K}$, ${\cal K}'$ are two non-redundant consistent DCC sets over ${\bf V}(\mathcal{G}^*)$, and $\kappa'$ is minimally redundant with respect to $[{\cal G}^*,{\cal K}]$ for any $\kappa'\in {\cal K}'$.
			\begin{enumerate}
				\item[(i)] If $\kappa\in{\cal K}$ is redundant with respect to $[{\cal G}^*,{\cal K}'\cup({\cal K}\setminus\{\kappa\})]$, then there exists a unique DCC $\kappa'\in {\cal K}'$ such that $\kappa'_t=\kappa_t$ and $\kappa'_h\subseteq\kappa_h$. Moreover, for any DCC $\gamma\in{\cal K}$ such that $\gamma\neq\kappa$, either $\kappa'_t\neq\gamma_t$ or $\kappa'_h\nsubseteq\gamma_h$.
				\item[(ii)] If $[{\cal G}^*,{\cal K}] = [{\cal G}^*,{\cal K}']$, then $|{\cal K}|=|{\cal K}'|$.
			\end{enumerate}
		\end{lemma}
		
		\begin{proof}[Proof of Lemma~\ref{lemma:minimal_dcc}] By assumption, ${\cal K}'\cup{\cal K}$ and all of its subset are consistent with ${\cal G}^*$.
			We first prove statement (i). Without loss of generality, in the following, we assume that ${\cal K}=\{\kappa_1,\cdots,\kappa_n\}$, ${\cal K}'=\{\kappa'_1,\cdots,\kappa'_m\}$, and $\kappa_n$ is redundant with respect to $[{\cal G}^*, {\cal K}'\cup({\cal K}\setminus\{\kappa_n\})]$. Since $\kappa'_i$ is not redundant with respect to ${\cal K}'\setminus\{\kappa'_i\}$, $\kappa'_{ih}\cap ch(\kappa'_{it}, {\cal G}^*)=\varnothing$. Moreover, since $\kappa'_i$ is minimally redundant with respect to $[{\cal G}^*, {\cal K}]$,  by Corollary~\ref{coro:moc_of_dcc}, for every $\kappa'_i$, there exists a connected undirected induced subgraph ${\cal U}'_i$ of $\mathcal{G}^*$ containing the vertices in $\kappa'_{ih}$ and $\kappa'_{it}$ such that with respect to $\cal K$ and $\mathcal{G}^*$, $\kappa'_{ih}$ are all and only potential leaf nodes in ${\cal U}'_i$.
			
			Assuming that the corresponding connected undirected induced subgraphs of $\kappa'_1,\cdots,\kappa'_k$, which are denoted by ${\cal U}'_1,\cdots,{\cal U}'_k$, contain all vertices in $\kappa_{nt}$ and $\kappa_{nh}\cap sib(\kappa_{nt}, {\cal G}^*)$, while the remaining induced subgraphs ${\cal U}'_{k+1},\cdots,{\cal U}'_m$ do not contain all vertices in $\kappa_{nt}$ and $\kappa_{nh}$. Since  $\kappa_n$ is redundant with respect to $[{\cal G}^*, {\cal K}'\cup({\cal K}\setminus\{\kappa_n\})]$, by Lemma~\ref{lem:moc_of_dcc} and the fact that $\kappa_n$ is not redundant with respect to $[{\cal G}^*, {\cal K}\setminus\{\kappa_n\}]$, there is a connected undirected induced subgraph ${\cal U}_n$ of $\mathcal{G}^*$ containing $\kappa_{nt}$ and a subset ${\bf s}_n$ of $\kappa_{nh}$ such that all potential leaf nodes in ${\cal U}_n$, with respect to ${\cal K}'\cup({\cal K}\setminus\{\kappa_n\})$, are in ${\bf s}_n$. It is clear that, there is at least one DCC $\kappa'^{*}\in{\cal K}'$ satisfying that $\kappa'^{*}_{t}$ and all vertices in $\kappa'^{*}_{h}\cap sib(\kappa'^{*}_{t}, {\cal G}^*)$ are in ${\cal U}_n$ (this condition is denoted by C1).
			Since otherwise, the restriction subset of  ${\cal K}'\cup({\cal K}\setminus\{\kappa_n\})$ on ${\cal U}_n$ is identical to that of ${\cal K}\setminus\{\kappa_n\}$, which violates the assumption that ${\cal K}$ is not redundant.
			
			Denote by $\{\kappa'^{*}_1,\cdots,\kappa'^{*}_x\}$ the set of all DCCs in ${\cal K}'$ satisfying the above condition C1. We claim that $\{\kappa'^{*}_1,\cdots,\kappa'^{*}_x\}\cap \{\kappa'_1,\cdots,\kappa'_k\}\neq\varnothing$.
			In fact, if $\{\kappa'^{*}_1,\cdots,\kappa'^{*}_x\}\subseteq \{\kappa'_{k+1},\cdots,\kappa'_m\}$, by the assumption that ${\cal U}'_l$ for every $l=k+1,\cdots,m$ does not contain all vertices in $\kappa_{nt}$ and $\kappa_{nh}\cap sib(\kappa_{nt}, {\cal G}^*)$, the restriction subset of ${\cal K}$ on ${\cal U}'_l$ is identical to that of ${\cal K}\setminus\{\kappa_n\}$, and hence $\kappa'^{*}_1,\cdots, \kappa'^{*}_x$ are redundant with respect to  $[{\cal G}^*, {\cal K}\setminus\{\kappa_n\}]$. On the other hand, the restriction subset of ${\cal K}'\cup{\cal K}\setminus\{\kappa_n\}$ on ${\cal U}_n$ is identical to that of $\{\kappa'^{*}_1,\cdots, \kappa'^{*}_x\}\cup({\cal K}\setminus\{\kappa_n\})$, implying that $\kappa_n$ is redundant with respect to $[{\cal G}^*,\{\kappa'^{*}_1,\cdots, \kappa'^{*}_x\}\cup({\cal K}\setminus\{\kappa_n\})]$. As a result, $\kappa_n$ is redundant with respect to $[{\cal G}^*, {\cal K}\setminus\{\kappa_n\}]$, which contradicts the non-redundancy of ${\cal K}$. This completes the proof of the claim.
			
			Without loss of generality, we assume that $\{\kappa'^{*}_1,\cdots,\kappa'^{*}_x\}\cap \{\kappa'_1,\cdots,\kappa'_k\}=\{\kappa'_1,\cdots,\kappa'_d\}$. We next show that there is at least one $\kappa'_i\in\{\kappa'_1,\cdots,\kappa'_d\}$ such that $\kappa'_{ih}\cap {\bf s}_n \neq \varnothing$.
			Consider the undirected induced subgraph over ${\bf V}({\cal U}'_1),\cdots,{\bf V}({\cal U}'_d)$ and ${\bf V}({\cal U}_n)$, which is also the union graph of ${\cal U}'_1,\cdots,{\cal U}'_d$ and ${\cal U}_n$. With respect to ${\cal K}'\cup({\cal K}\setminus\{\kappa_n\})$, if a vertex is not a potential leaf node in one of ${\cal U}'_1,\cdots,{\cal U}'_d$ and ${\cal U}_n$, then it is not a potential leaf node in the union graph either. If a vertex is a potential leaf node in ${\cal U}'_i$ ($i=1,\cdots,d$) with respect to ${\cal K}'\cup({\cal K}\setminus\{\kappa_n\})$, it must be in $\kappa'_{ih}\cup\{\kappa_{nt}\}$. However, if $\kappa'_{ih}\cap {\bf s}_n = \varnothing$,  the vertices in $\kappa'_{ih}\cup\{\kappa_{nt}\}$ are not potential leaf nodes in ${\cal U}_n$ with respect to ${\cal K}'\cup{\cal K}\setminus\{\kappa_n\}$, and thus are not potential leaf nodes in the union graph. Similarly, $\kappa'_{ih}\cap {\bf s}_n = \varnothing$ for any $i=1,\cdots,d$ implies that the vertices in ${\bf s}_n$ are not potential leaf nodes in any ${\cal U}'_i$, and thus are not potential leaf nodes in the union graph. As a result, no vertex in the union graph is a potential leaf node under the condition that $\kappa'_{ih}\cap {\bf s}_n = \varnothing$ for any $i=1,\cdots,d$ , which is impossible as ${\cal K}'\cup{\cal K}\setminus\{\kappa_n\}$ is consistent.
			
			Assuming that $\kappa'_1,\cdots \kappa'_l$ ($1\leq l\leq d$) are all and only the DCCs such that $\kappa'_{ih}\cap {\bf s}_n \neq \varnothing$ ($i=1,\cdots,l$). We will show that
			\[{\bf s}\coloneqq\bigcap_{i=1}^l\kappa'_{ih}\cap {\bf s}_n \neq \varnothing.\]
			Assume, for contradiction, that $\bigcap_{i=1}^l\kappa'_{ih}\cap {\bf s}_n = \varnothing$, and consider again the union graph of ${\cal U}'_1,\cdots,{\cal U}'_d$ and ${\cal U}_n$. Using the same argument given above, we can prove that with respect to ${\cal K}'\cup({\cal K}\setminus\{\kappa_n\})$, any potential leaf node in ${\cal U}'_i$ ($i=l+1,\cdots,d$) is not a potential leaf node in the union graph. If a vertex $h_i$ is a potential leaf node in ${\cal U}'_i$ $(i=1,\cdots,l)$ but not in ${\bf s}_n$, it is not a potential leaf node in the union graph either. However,  if $h_i\in{\bf s}_n$, by assumption, the exists a $\kappa'_j$, $j\neq i$ and $j\in\{1,...,l\}$, such that $h_i\notin\kappa'_{jh}$. Note that, $h_i\in{\bf s}_n\subseteq\kappa_{nh}$, meaning that $h_i\neq \kappa_{nt}$. Therefore, in ${\cal U}'_j$, with respect to ${\cal K}'\cup({\cal K}\setminus\{\kappa_n\})$, $h_i$ is not a potential leaf node as  $h_i\notin \kappa'_{jh}\cup\{\kappa_{nt}\}$, and thus $h_i$ is not a potential leaf node in the union graph. Similarly, any potential leaf node in ${\cal U}_n$ is not a potential leaf node in at least one ${\cal U}'_i$, $i\in\{1,...,l\}$, and thus is not a potential leaf node in the union graph. As a consequence, no vertex in the union graph is a potential leaf node, which leads to a contradiction.
			
			Since ${\bf s}\neq \varnothing$, by the similar argument we can prove that every potential leaf node in the union graph of ${\cal U}'_1,\cdots,{\cal U}'_d$ and ${\cal U}_n$ with respect to ${\cal K}'\cup({\cal K}\setminus\{\kappa_n\})$ is in ${\bf s}$. According to Lemma~\ref{lem:moc_of_dcc}, $\kappa'_{it}\tor {\bf s}$ is redundant with respect to $[{\cal G}^*, {\cal K}'\cup({\cal K}\setminus\{\kappa_n\})]$ for any $i=1,\cdots,l$, and thus redundant with respect to $[{\cal G}^*, {\cal K}'\cup{\cal K}]$. Note that, ${\cal K}'\cup{\cal K}$ is equivalent to ${\cal K}$ due to the minimal redundancy of the DCCs in ${\cal K}'$, $\kappa'_{it}\tor {\bf s}$ is redundant with respect to $[{\cal G}^*, {\cal K}]$ for any $i=1,\cdots,l$. Therefore, $\kappa'_{ih}={\bf s}$ for any $i=1,\cdots,l$, since otherwise ${\bf s}\subsetneq \kappa'_{ih}$ contradicts to the minimal redundancy of $\kappa'_{i}$.
			
			With $\kappa'_{ih}={\bf s}$ for any $i=1,\cdots,l$, we claim that $l=d$. Suppose that $l<d$, then with respect to ${\cal K}'\cup({\cal K}\setminus\{\kappa_n\})$, any potential leaf node in ${\cal U}'_i$ ($i=l+1,\cdots,d$) is not a potential leaf node in the union graph of ${\cal U}'_1,\cdots,{\cal U}'_d$ and ${\cal U}_n$. On the other hand, any potential leaf node in ${\cal U}_n$ is not a potential leaf node in ${\cal U}'_{l+1},\cdots,{\cal U}'_d$ as it is in ${\bf s}_n$. Since $\kappa'_{ih}={\bf s}\subseteq {\bf s}_n$ for $i=1,\cdots,l$, any potential leaf node in ${\cal U}'_i$, $i=1,\cdots,l$, is not a potential leaf node in ${\cal U}'_{l+1},\cdots,{\cal U}'_d$, and ${\cal U}_n$. Therefore, no vertex in the union graph is a potential leaf node, which leads to a contradiction.
			
			The next step is to prove that there exists a DCC $\kappa'\in \{\kappa'_1,\cdots,\kappa'_d\}$ such that $\kappa'_t=\kappa_{nt}$ and $\kappa'_h\subseteq\kappa_{nh}$. Assume, for contradiction, that none of the DCC in $\{\kappa'_1,\cdots,\kappa'_d\}$ satisfies the condition. Since we proved that $\kappa'_{ih}={\bf s}\subseteq\kappa_{nh}$ for any $i=1,\cdots,d$, by assumption, $\kappa'_{it}\neq\kappa_{nt}$ for any $i=1,\cdots,d$. Firstly, consider ${\cal U}_n$. With respect to ${\cal K}'\cup({\cal K}\setminus\{\kappa_n\})$, since $\{\kappa'_{it}\}$ are not potential leaf nodes, it holds that $\kappa'_{it}\notin{\bf s}_n$. Moreover, the restriction subset of ${\cal K}'\cup({\cal K}\setminus\{\kappa_n\})$ on ${\cal U}_n$ is identical to that of $\{\kappa'^{*}_1,\cdots,\kappa'^{*}_x\}\cup({\cal K}\setminus\{\kappa_n\})$, which is further identical to that of $\{\kappa'_1,\cdots,\kappa'_d\}\cup({\cal K}\setminus\{\kappa_n\})$. The latter hols because ${\cal U}'_i$ for $\kappa'_i\in\{\kappa'^{*}_1,\cdots,\kappa'^{*}_x\}\setminus\{\kappa'_1,\cdots,\kappa'_d\} $ does not contain all vertices in $\kappa_{nt}$ and $\kappa_{nh}\cap sib(\kappa_{nt}, {\cal G}^*)$, implying that $\kappa'_i$ is redundant with respect to $[{\cal G}^*, {\cal K}\setminus\{\kappa_n\}]$. Consequently, with respect to   $\{\kappa'_1,\cdots,\kappa'_d\}\cup({\cal K}\setminus\{\kappa_n\})$, ${\bf s}_n$ are all and only potential leaf nodes in ${\cal U}_n$. Since $\kappa'_{it}\neq\kappa_{nt}$ for any $i=1,\cdots,d$, with respect to ${\cal K}\setminus\{\kappa_n\}$, $\kappa_{nt}$ is not a potential leaf node in ${\cal U}_n$. Next, consider ${\cal U}'_i$ for $i=1,\cdots,d$. With respect to ${\cal K}\setminus\{\kappa_n\}$, $\{\kappa'_{it}\}$ are not potential leaf nodes. Therefore, in the union graph of ${\cal U}'_1,\cdots,{\cal U}'_d$ and ${\cal U}_n$, only the vertices in ${\bf s}_n$ are potential leaf nodes with respect to ${\cal K}\setminus\{\kappa_n\}$. This implies that $\kappa_n$ is redundant with respect to $[{\cal G}^*, {\cal K}\setminus\{\kappa_n\}]$, leading to a contradiction.
			
			Finally, to prove the uniqueness in statement (i), it suffices to show that none of the DCC $\kappa'\in {\cal K}'\setminus \{\kappa'_1,\cdots,\kappa'_d\}$ satisfies $\kappa'_t=\kappa_{nt}$ and $\kappa'_h\subseteq\kappa_{nh}\cap sib(\kappa_{nt}, {\cal G}^*)$. Assume, for contradiction, that such a DCC $\kappa'$ exists. Since ${\cal U}_n$ contains all vertices in $\kappa_{nh}\cap sib(\kappa_{nt}, {\cal G}^*)$ and $\kappa_{nt}$, ${\cal U}_n$ contains $\kappa'_h$ and $\kappa'_t$, meaning that $\kappa'\notin\{\kappa'_{d+1},\cdots,\kappa'_{k}\}$. However,      if $\kappa'\in\{\kappa'_{k+1},\cdots,\kappa'_{m}\}$, since ${\cal U}'_i$ does not contain all vertices in $\kappa_{nt}$ and $\kappa_{nh}\cap sib(\kappa_{nt}, {\cal G}^*)$, $\kappa'$ is redundant with respect to $[{\cal G}^*, {\cal K}\setminus\{\kappa_n\}]$, and hence $\kappa_n$ is redundant with respect to $[{\cal G}^*, {\cal K}\setminus\{\kappa_n\}]$, which leads to a contradiction.
			
			To complete the proof of statement (i), we assume, for contradiction, that there is another DCC $\kappa_i\in{\cal K}$, $i\neq n$, such that $\kappa'_t=\kappa_{it}$ and $\kappa'_h\subseteq\kappa_{ih}$. It is clear that ${\cal U}'$ contains all vertices in $\kappa_{ih}$ and $\kappa_{it}$, since otherwise $\kappa_i$ is redundant with respect to $[{\cal G}^*, {\cal K}\setminus\{\kappa_i\}]$. Now, consider ${\cal U}'$ and $\bigcup_{s\in\kappa_{nh}}\{s\to\kappa_{nt}\}\cup({\cal K}\setminus\{\kappa_n\})$. It is easy to verify that none of the vertices in ${\cal U}'$ is a potential leaf node. Therefore, $\kappa_n$ is redundant with respect to $[{\cal G}^*, {\cal K}\setminus\{\kappa_n\}]$, which is contradicted to the non-redundancy $\kappa_n$.
			
			We then prove statement (ii). If $[{\cal G}^*,{\cal K}] = [{\cal G}^*,{\cal K}']$, then any $\kappa\in{\cal K}$ is redundant with respect to $[{\cal G}^*, {\cal K}']$, and thus redundant with respect to $[{\cal G}^*, {\cal K}'\cup({\cal K}\setminus\{\kappa\})]$.
			By statement (i), $n\leq m$. If $n<m$, assume, without generality, that the unique DCC satisfies the condition in statement (i) for $\kappa_i$ is $\kappa'_i$, $i=1,2,\cdots,n$, then $\kappa_i$ is redundant with respect to $[{\cal G}^*, {\kappa'_i}]$. As a result, $\kappa'_m$ is redundant with respect to $[{\cal G}^*, \{\kappa'_1,\cdots, \kappa'_n\}]$, leading to a contradiction. Therefore, $n = m$.     	
		\end{proof}

		\begin{proof}[Proof of Theorem~\ref{thm:equal_card}]
			We first prove statement (i). For any DCC $\kappa\in{\cal K}$, let $\kappa'$ be the DCC such that (1) $\kappa'_t=\kappa_t$, (2) $\kappa'_h\subseteq\kappa_h$, and (3) $\kappa'$ is minimally redundant with respect to  $[{\cal G}^*, {\cal K}]$. Such a DCC always exists. In fact, one can first check whether $\kappa_t\to s$ is redundant for some $s\in\kappa_h$. If exists, then $\kappa_t\to s$ is minimally redundant, and if not exists, one can further enumerate and check all 2-elements subsets of $\kappa_h$. The enumeration will end because $\kappa'=\kappa$ is redundant with respect to  $[{\cal G}^*, {\cal K}]$.
			
			Let ${\cal K}'$ denote the set of all $\kappa'$ constructed above. It is clear that ${\cal K}'$ is equivalent to ${\cal K}$ and $|{\cal K}'| \leq |{\cal K}|$. If ${\cal K}'$ is redundant, then its non-redundant proper subset ${\cal K}''\subsetneq{\cal K}'$ is equivalent to ${\cal K}$. By statement (ii) of Lemma~\ref{lemma:minimal_dcc}, $|{\cal K}''| = |{\cal K}|$. This is not possible as $|{\cal K}''| < |{\cal K}'|$. Therefore, ${\cal K}'$ is not redundant.
			
			Assume that there is another non-redundant DCC set ${\cal K}'''$ such that ${\cal K}$ is equivalent to ${\cal K}'''$ and $\kappa'''$ is minimally redundant with respect to $[{\cal G}^*,{\cal K}]$ for any $\kappa'''\in {\cal K}'''$. Then,  ${\cal K}'''$ is also equivalent to ${\cal K}'$ and $\kappa'''$ is minimally redundant with respect to $[{\cal G}^*,{\cal K}']$ for any $\kappa'''\in {\cal K}'''$. By the first statement of Lemma~\ref{lemma:minimal_dcc}, for any $\kappa'\in{\cal K}'$, there is a unique DCC $\kappa'''\in {\cal K}'''$ such that $\kappa'''_t=\kappa'_t$ and $\kappa'''_h\subseteq\kappa'_h$. Since $\kappa'$ is also minimally redundant, $\kappa'''_h=\kappa'_h$. Therefore, By the first statement of Lemma~\ref{lemma:minimal_dcc}, ${\cal K}'''={\cal K}'$.
			
			Finally, by the first statement of Lemma~\ref{lemma:minimal_dcc}, for every $\kappa\in{\cal K}$, there exists a unique $\kappa'\in{\cal K}'$ such that $\kappa'_t=\kappa_t$ and $\kappa'_h\subseteq\kappa_h$; and for every $\kappa'\in{\cal K}'$, there exists a unique $\kappa\in{\cal K}$ such that $\kappa'_t=\kappa_t$ and $\kappa'_h\subseteq\kappa_h$.
			
			We next prove statement (ii). By statement (i), there is a  unique non-redundant DCC set ${\cal K}''$ over ${\bf V}(\mathcal{G}^*)$ such that ${\cal K}$ is equivalent to ${\cal K}''$ and $\kappa''$ is minimally redundant with respect to $[{\cal G}^*,{\cal K}]$ for any $\kappa''\in {\cal K}''$. Since ${\cal K}'$ is equivalent to ${\cal K}$, ${\cal K}''$ is also the unique non-redundant DCC set such that ${\cal K}'$ is equivalent to ${\cal K}''$ and $\kappa''$ is minimally redundant with respect to $[{\cal G}^*,{\cal K}']$ for any $\kappa''\in {\cal K}''$. By Lemma~\ref{lemma:minimal_dcc}, $|{\cal K}|=|{\cal K}''|=|{\cal K}'|$. Moreover, statement (i) of Lemma~\ref{lemma:minimal_dcc} indicates  that $\{\kappa_t\mid \kappa\in{\cal K}\}=\{\kappa''_t\mid \kappa''\in{\cal K}''\}$. therefore, $\{\kappa_t\mid \kappa\in{\cal K}\}=\{\kappa''_t\mid \kappa''\in{\cal K}''\}=\{\kappa'_t\mid \kappa'\in{\cal K}'\}$.    	
		\end{proof}

		\subsection{Proof of Theorem \ref{prop:eqdecomp}}
		\begin{proof}
			We first prove statement (i). Theorem~\ref{thm:nbr_set_const} proves that $\cal B$ can be equivalently represented by a set of DCCs $\cal K$. By Definition~\ref{interpretbg:mpdagresclass} of MPDAGs, $\cal H$ represents all common direct causal relations shared by all DAGs in  $[\mathcal{G}^*,{\cal B}]$. Due to the equivalence, $[\mathcal{G}^*, {\cal B}]=[\mathcal{G}^*, {\cal K}]=[{\cal H}, {\cal K}]$. Now consider ${\bf E}_d({\cal H})\cup{\cal K}$. For any $\kappa \in {\cal K}$, if $\kappa$ is redundant with respect to ${\bf E}_d({\cal H}) \cup ({\cal K} \setminus {\kappa})$, it can be removed from ${\cal K}$, yielding a set ${\bf E}_d({\cal H}) \cup ({\cal K} \setminus {\kappa})$ that is equivalent to ${\bf E}_d({\cal H}) \cup {\cal K}$. By repeating this procedure, we can iteratively remove all redundant DCCs from $\cal K$ until no DCC remains redundant given the others and ${\bf E}_d({\cal H})$. The remaining DCCs in ${\cal K}$ constitute the desired residual set ${\cal R}$.
			
			We next prove statement (ii). Let ${\cal R}_1$ and ${\cal R}_2$ be two residual DCC sets. By assumption, every DCC $\kappa_1\in{\cal R}_1$, if exists, is not redundant with respect to ${\bf E}_d({\cal H})\cup({\cal R}_1\setminus\{\kappa_1\})$, and hence ${\cal R}_1\cap {\bf E}_d({\cal H})=\varnothing$. Following the similar pruning procedure in the proof of (i), we can further remove redundant DCCs from ${\bf E}_d({\cal H})\cup{\cal R}_1$ and obtain a non-redundant DCC set ${\cal A}_1\subseteq {\bf E}_d({\cal H})\cup{\cal R}_1$ with respect to $\mathcal{G}^*$. Clearly, ${\cal A}_1$ is equivalent to ${\bf E}_d({\cal H})\cup{\cal R}_1$, and ${\cal A}_1$ can be written as a union of two sets, say ${\cal E}_1$ and ${\cal R}'_1$, where ${\cal E}_1\subseteq {\bf E}_d({\cal H})$ and ${\cal R}'_1\subseteq {\cal R}_1$. However, if there is a $\kappa \in {\cal R}_1$ but $\kappa \notin {\cal R}'_1$ then $\kappa$ is redundant with respect to ${\cal E}_1\cup{\cal R}'_1$, and therefore is redundant with respect to ${\bf E}_d({\cal H})\cup{\cal R}_1\setminus\{\kappa\}$, which contradicts the assumption. Therefore, ${\cal R}'_1 = {\cal R}_1$ and  ${\cal A}_1 = {\cal E}_1\cup{\cal R}_1$. By Theorem~\ref{thm:equal_card}, there is a unique DCC set ${\cal A}^*$ such that ${\cal A}^*$ is equivalent to ${\cal A}_1$ and every DCC in ${\cal A}^*$ is minimally redundant with respect to $[\mathcal{G}^*, {\cal A}_1]$, which also means ${\cal A}^*$ is minimally redundant with respect to $[\mathcal{G}^*, {\cal B}]$. Based on statement (i) of Theorem~\ref{thm:equal_card}, for any $\kappa_1 \in {\cal R}_1$, if exists, there is a unique $\kappa^*\in {\cal A}^*$ such that $\kappa^*_{t}=\kappa_{1,t}$ and $\kappa^*_{h}\subseteq\kappa_{1,h}$. If $|\kappa^*_{h}|=1$, then $\kappa^*$ is a directed edge and $\kappa^*\in {\bf E}_d({\cal H})$. This means $\kappa_1 \in {\cal R}_1$ is redundant with respect to ${\bf E}_d({\cal H})$, which is contradicted to our assumption. Therefore, it holds that $|\kappa^*_{h}|>1$. Similarly, all directed edges in ${\cal E}_1$ are in ${\cal A}^*$, and hence ${\cal A}^*={\cal E}_1\cup {\cal R}^*$ where ${\cal R}^*$ is the set of DCCs each of which has more than one head and uniquely corresponds to a DCC in ${\cal R}_1$. By statement (i) of Theorem~\ref{thm:equal_card}, $|{\cal R}^*| = |{\cal R}_1|$. Due to the uniqueness of the element-wise head-minimal DCC set, applying the same technique given above, we can prove that $|{\cal R}^*| = |{\cal R}_2|$. This completes the proof of statement (ii).
			
			For statement (iii), notice that the DCC set ${\cal R}^*$ in the proof of statement (ii) is a DCC set satisfying conditions (1) to (3). The uniqueness comes from statement (i) of Theorem~\ref{thm:equal_card}.
			
			Finally, statement (iv) comes from the proof of statement (ii).
		\end{proof}
		
	}
	
	\subsection{Proof of Theorem \ref{thm:h_equals_n}}\label{app:proof:h_equals_n}
	
	The proof of Theorem~\ref{thm:h_equals_n} requires the following lemma.
	
	\begin{lemma}\label{lem:app:clique}
		Let $\mathcal{H}$ be a causal MPDAG. For any vertex $X$ and $\textbf{S}\subseteq sib(X, \mathcal{H})$, if $\textbf{S}$ induces a complete subgraph of $\mathcal{H}$, then there is a DAG ${\cal G} \in [\mathcal{H}]$ in which $\textbf{S} \subseteq pa(X, {\cal G})$.
	\end{lemma}
	\begin{proof}
		Without loss of generality, we assume that the skeleton of $\cal H$, denoted by $\cal C$, is a connected chordal graph, and $\cal H$ is the MPDAG representing $[{\cal C}, {\cal B}_d]$ for some direct causal constraints ${\cal B}_d$ (Proposition \ref{nonemptyMPDAG}). That is, $[{\cal H}]=[{\cal C}, {\cal B}_d]$. We first show that $pa(X, {\cal H})\cup \textbf{S}\cup\{X\}$ induces a complete subgraph of $\mathcal{H}$. In fact, if $pa(X, {\cal H})=\varnothing$, then $pa(X, {\cal H})\cup \textbf{S}\cup\{X\}$ is complete. If $pa(X, {\cal H})\neq\varnothing$, then for any $p\in pa(X, {\cal H})$ and $S \in \textbf{S}$, $p$ and $S$ are adjacent in $\cal H$, since otherwise $X\to S$ should appear in $\cal H$ due to the maximality of $\cal H$. Thus, $pa(X, {\cal H})\cup \textbf{S}\cup\{X\}$ is also complete.
		
		\citet[Theorem~1]{fang2020bgida} proved that for any vertex $X$ and $\textbf{S}\subseteq sib(X, \mathcal{H})$, the following three statements are equivalent.
		\begin{enumerate}
			\item[(i)]  There is a DAG $\mathcal{G}$ in  $[\mathcal{H}]$ such that $pa(X, \mathcal{G}) = \textbf{S}\cup pa(X, \mathcal{H})$ and $ch(X, \mathcal{G}) = sib(X, \mathcal{H})\cup ch(X, \mathcal{H})\setminus \textbf{S}$.
			\item[(ii)]  Orienting $\textbf{S}\rightarrow X$ and $X\to sib(X, \mathcal{H})\setminus \textbf{S}$ in $\mathcal{H}$  does not introduce any new v-structure collided on $X$ or any directed triangle containing $X$.
			\item[(iii)]  The induced subgraph of $\mathcal{H}$ over $\textbf{S}$ is complete, and there does not exist an $S\in \textbf{S}$ and a $C\in sib(X, \mathcal{H})\setminus\textbf{S}$ such that $C\rightarrow S$ is in $\cal H$.
		\end{enumerate}
		Denote by $\textbf{M}$ a maximal clique containing $pa(X, {\cal H})\cup \textbf{S}\cup\{X\}$ (which definitely exists but may not be unique) and let $\textbf{M}_p=\textbf{M}\cap pa(X, {\cal H})=pa(X, {\cal H})$, $\textbf{M}_c=\textbf{M}\cap ch(X, {\cal H})$ and $\textbf{M}_s=\textbf{M}\cap sib(X, {\cal H})$. We first show that if $\textbf{M}_c=\varnothing$, then there is a DAG $\mathcal{G}\in[\mathcal{H}]$ such that $pa(X, \mathcal{G}) = \textbf{M}_s\cup pa(X, \mathcal{H})$ and $ch(X, \mathcal{G}) = sib(X, \mathcal{H})\cup ch(X, \mathcal{H})\setminus \textbf{M}_s$. As $\textbf{M}_s\subseteq \textbf{M}$, $\textbf{M}_s$ induces a complete subgraph of $\cal H$. {Suppose that there is a $P\in \textbf{M}_s$ and a $C\in sib(X, \mathcal{H})\setminus\textbf{M}_s$ such that $C\rightarrow P$, then $C$ is adjacent to every vertex in $pa(X, {\cal H})\cup \textbf{M}_s$, since otherwise $X\to P'$ is in $\cal H$ for any $P'\in pa(X, {\cal H})\cup \textbf{M}_s$ which is not adjacent to $C$,} and contradicts the assumption that $ pa(X, {\cal H})\cup \textbf{M}_s\subseteq  pa(X, {\cal H})\cup sib(X, {\cal H})$. However, $C$ is adjacent to every vertex in $pa(X, {\cal H})\cup \textbf{M}_s$ implies that $C$ is adjacent to every vertex in $\textbf{M}=\textbf{M}_s\cup\textbf{M}_p$, meaning that $\textbf{M}$ is not a maximal clique. This leads to a contradiction. The desired result then follows from \citet[Theorem~1]{fang2020bgida}.
		
		On the other hand, suppose that $\textbf{M}_c\neq\varnothing$ for any maximal clique $\textbf{M}$ containing $pa(X, {\cal H})\cup \textbf{S}\cup\{X\}$. In the following, we will construct a maximal clique $\textbf{M}$ containing $pa(X, {\cal H})\cup \textbf{S}\cup\{X\}$, such that orienting $\textbf{M}_s\to X$ and $X\to sib(X, {\cal H})\setminus\textbf{M}_s$ does not violate statement (iii) of \citet[Theorem~1]{fang2020bgida}.
		
		Let ${\bf M}^0$ be an arbitrary maximal clique containing $pa(X, {\cal H})\cup \textbf{S}\cup\{X\}$. If orienting $\textbf{M}^0_s\to X$ and $X\to sib(X, {\cal H})\setminus\textbf{M}^0_s$ does not violate statement (iii) of \citet[Theorem~1]{fang2020bgida}, then the proof is completed. If otherwise, let $\mathbf{C}^0\subseteq sib(X, {\cal H})\setminus\textbf{M}^0_s$ be the set of vertices such that for any $C\in\mathbf{C}^0$, $C\to P$ for some $P\in \textbf{M}^0_s$. Using the same argument given in the last paragraph, we can prove that $C$ is adjacent to every vertex in $pa(X, {\cal H})\cup\textbf{M}^0_s$ for any $C\in\mathbf{C}^0$. Likewise, it can be shown that any two distinct vertices in $\mathbf{C}^0$, if exist, are adjacent. Hence, $\mathbf{C}^0\cup \textbf{M}^0_s\cup pa(X, {\cal H})$ is a clique.
		
		Let ${\bf M}^1$ be a maximal clique containing $\{X\}\cup\mathbf{C}^0\cup \textbf{M}^0_s\cup pa(X, {\cal H})$. By assumption, ${\bf M}^1_c\neq\varnothing$ and ${\bf S}\subseteq \textbf{M}^0_s\subsetneq \textbf{M}^1_s$. If orienting $\textbf{M}^1_s\to X$ and $X\to sib(X, {\cal H})\setminus\textbf{M}^1_s$ does not violate statement (iii) of \citet[Theorem~1]{fang2020bgida}, then the proof is completed. Otherwise, following the above procedure we can find a new maximal clique ${\bf M}^2$ containing $\{X\}\cup\mathbf{C}^1\cup \textbf{M}^1_s\cup pa(X, {\cal H})$, where $\mathbf{C}^1\subseteq sib(X, {\cal H})\setminus\textbf{M}^1_s$ be the set of vertices such that for any $C\in\mathbf{C}^1$, $C\to P$ for some $P\in \textbf{M}^1_s$. Note that, $\textbf{S}\subseteq {\bf M}^0_s\subsetneq {\bf M}^1_s\subsetneq\cdots\subsetneq {\bf M}^i_s\subsetneq\cdots\subseteq sib(X, {\cal H})$ implies that the above construction will eventually stop as $sib(X, {\cal H})$ is a finite set. That is, we will finally find an ${\bf M}={\bf M}^i_s$ containing $pa(X, {\cal H})\cup \textbf{S}\cup\{X\}$ such that orienting $\textbf{M}_s\to X$ and $X\to sib(X, {\cal H})\setminus\textbf{M}_s$ does not violate statement (iii) of \citet[Theorem~1]{fang2020bgida}, which completes the proof.
	\end{proof}

	\begin{proof}[Proof of Theorem~\ref{thm:h_equals_n}] By definition, (i) is equivalent to (ii), and thus we only prove (i) is equivalent to (iii) in the following.
		
		By definition, $\cal H$ is fully informative with respect to ${\cal K}$ and ${\cal G}^*$ if and only if $\kappa_t\tor \kappa_h$ holds for all DAGs in $[\mathcal{H}]$ for any $\kappa\in {\cal K}$, where $\kappa\coloneqq \kappa_t\tor \kappa_h$.
		
		If $\kappa_h\cap ch(\kappa_t, {\cal H})\neq\varnothing$, then $\kappa_t\tor \kappa_h$ holds for all DAGs in $[\mathcal{H}]$. If $\kappa_h\cap sib(\kappa_t, {\cal H})$ induces an incomplete subgraph of ${\cal H}$, then there exist $V_1, V_2 \in \kappa_h\cap sib(\kappa_t, {\cal H})$ such that $V_1, V_2$ are not adjacent. By the first rule of Meek's rules, for any DAG in $[\mathcal{H}]$, either $V_1$ or $V_2$ is a child of $\kappa_t$. Thus, $\kappa_t\tor \kappa_h\cap sib(\kappa_t, {\cal H})$ holds for all DAGs in $[\mathcal{H}]$.
		
		Conversely, if $\kappa_h\cap ch(\kappa_t, {\cal H})=\varnothing$ and the induced subgraph of ${\cal H}$ over $\kappa_h\cap sib(\kappa_t, {\cal H})$ is complete, then by Lemma \ref{lem:app:clique}, there is a DAG ${\cal G} \in [\mathcal{H}]$ in which $\kappa_h\cap sib(\kappa_t, {\cal H})\subseteq pa(\kappa_t, {\cal G})$. Therefore, $\kappa_t\tor \kappa_h$ does not hold for ${\cal G}$.
	\end{proof}

	\subsection{Proof of Corollary \ref{coro:fullyinfo}}
	\begin{proof}
		Corollary \ref{coro:fullyinfo} follows from Theorem \ref{thm:h_equals_n} and Equation~\eqref{eq:n_u^c}.
	\end{proof}

	\subsection{Proof of Proposition \ref{prop:suff_oc}}\label{app:proof:suff_oc}
	\begin{proof}
		By Lemma~\ref{lem:app-pln}, for any DAG ${\cal G} \in [{\cal G}^*, {\cal K}]$, the leaf node in the induced subgraph of $\cal G$ over $\mathbf{V}(\mathcal{U})$, denoted by ${\cal G}_{\mathrm{sub}}$, must be $X$, as $X$ is the unique potential leaf node in $\cal U$ with respect to ${\cal K}$ and ${\cal G}^*$. The result comes from the fact that $adj(X, {\cal G}_{\mathrm{sub}}) = adj(X, {\cal U})$.
	\end{proof}
	
	\subsection{Proof of Theorem \ref{thm:mpdag}}\label{app:proof:mpdag}
	\begin{proof}
		The sufficiency follows from Proposition \ref{prop:suff_oc}, and below we prove  the necessity.	Since $\cal H$ has the same skeleton as ${\cal G}^*$, it suffices to prove that, for two adjacent variables $X$ and $Y$, if $X\to Y$ is not in ${\cal G}^*$ and there is no orientation component for $Y$ containing $X$ with respect to $\cal K$ and $\mathcal{G}^*$, then $X\to Y$ is not in ${\cal H}$.
		
		As $X\to Y$ is not in ${\cal G}^*$ but $X$ and $Y$ are adjacent, $Y$ is either a parent or a sibling of $X$ in ${\cal G}^*$. If $Y \to X$ is in ${\cal G}^*$, then $Y \to X$ is in ${\cal H}$, which completes the proof. Now consider the case where $X-Y$ is in ${\cal G}^*$. According to our assumption, with respect to $\cal K$ and $\mathcal{G}^*$, every connected undirected induced subgraph containing $X$ and $Y$ is not an orientation component for $Y$. If there is a such subgraph which is an orientation component for $X$, then by Proposition~\ref{prop:suff_oc} and the definition of an MPDAG,  $Y\to X$ is in ${\cal H}$.
		
		On the other hand, with respect to $\cal K$ and $\mathcal{G}^*$, if every connected undirected induced subgraph containing $X$ and $Y$ is neither an orientation component for $X$ nor an orientation component for $Y$, then by Theorem~\ref{thm:consistency}, it either has a potential leaf node which is neither $X$ nor $Y$, or has exactly two potential leaf nodes which are $X$ and $Y$. In the following, we will prove that there is a DAG ${\cal G}_1\in [{\cal G}^*, {\cal K}]$ in which $X\to Y$ and there is also a DAG ${\cal G}_2\in [{\cal G}^*, {\cal K}]$ in which $Y\to X$. According to symmetry, we need only to prove that there is a DAG ${\cal G}\in[{\cal G}^*, {\cal K}]$ in which $X\to Y$.
		
		To prove this claim, it suffices to show that ${\cal K}\cup \{X\to Y\}$ is consistent with ${\cal G}^*$. By  Theorem~\ref{thm:consistency}, any connected undirected induced subgraph $\cal U$ has a potential leaf node with respect to ${\cal K}$ and ${\cal G}^*$. If $\cal U$ does not contain $X$ or $Y$, then it can be easily verify that the potential leaf node of $\cal U$ with respect to ${\cal K}$ and ${\cal G}^*$ is still a potential lead node with respect to ${\cal K}\cup \{X\to Y\}$ and ${\cal G}^*$. If $\cal U$ contains both $X$ and $Y$, then we have that,
		\begin{enumerate}
			\item[(i)]  if $\cal U$ has a potential leaf node with respect to ${\cal K}$ and ${\cal G}^*$ which is neither $X$ nor $Y$, then by the definition of a potential leaf node, such a vertex is still a potential leaf node with respect to ${\cal K}\cup \{X\to Y\}$ and ${\cal G}^*$;
			\item[(ii)]  if $\cal U$ had exactly two potential leaf nodes which are $X$ and $Y$, then with respect to ${\cal K}\cup \{X\to Y\}$ and ${\cal G}^*$, $Y$ is still a potential leaf node in $\cal U$.
		\end{enumerate}
		Therefore, any connected undirected induced subgraph of ${\cal G}^*$ has a potential leaf node with respect to ${\cal K}\cup \{X\to Y\}$ and ${\cal G}^*$. The desired result follows from Theorem~\ref{thm:consistency}.
	\end{proof}
	
	\subsection{Proof of Proposition \ref{prop:union_of_oc}}\label{app:proof:union_of_oc}
	\begin{proof}
		Since $\mathcal{U}_1$ is an orientation components for $X$ with respect to ${\cal K}$ and ${\cal G}^*$, for any $Y_1\in\mathbf{V}(\mathcal{U}_1)$ such that $Y_1\neq X$, either $Y_1$ is a non-simplicial node in $\mathcal{U}_1$, or there is a $\kappa\in{\cal K}(\mathcal{U}_1)$ such that $\kappa_t=Y_1$. Denote by $\mathcal{U}_{12}$ the undirected induced subgraph of $\mathcal{G}^*$ over $\mathbf{V}(\mathcal{U}_1)\cup\mathbf{V}(\mathcal{U}_2)$. If $Y_1$ is a non-simplicial node in $\mathcal{U}_1$, then it is a non-simplicial node in $\mathcal{U}_{12}$, as two non-adjacent vertices in $\mathcal{U}_1$ remains non-adjacent in $\mathcal{U}_{12}$. Moreover, if there exists a $\kappa\in{\cal K}(\mathcal{U}_1)$ such that $\kappa_t=Y_1$, then $\kappa\in{\cal K}(\mathcal{U}_{12})$. Hence, $Y_1$ is not a potential leaf node in $\mathcal{U}_{12}$ with respect to ${\cal K}$ and ${\cal G}^*$. Following the same argument, none of the vertices in $\mathbf{V}(\mathcal{U}_2)\setminus\{X\}$ is a potential leaf node of $\mathcal{U}_{12}$ with respect to ${\cal K}$ and ${\cal G}^*$. However, due to the consistency,  $\mathcal{U}_{12}$ should have at least one potential leaf node. Therefore, $X$ is the only potential leaf node of $\mathcal{U}_{12}$ with respect to ${\cal K}$ and ${\cal G}^*$, which completes the proof.
	\end{proof}
	
	\subsection{Proof of Corollary \ref{coro:maximaloc}}\label{app:proof:maximaloc}
	\begin{proof}
		The conclusion follows directly from Proposition~\ref{prop:union_of_oc} and Theorem~\ref{thm:mpdag}.
	\end{proof}
	
	\subsection{Proof of Theorem \ref{thm:moc}}\label{app:proof:moc}
	\begin{proof}
		It is easy to see that the outputted graph of Algorithm~\ref{algo:moc}, denoted by ${\cal U}_\mathrm{out}$, is an orientation component for $X$ with respect to ${\cal K}$ and ${\cal G}^*$. Assuming that ${\cal U}_\mathrm{out}$ is not maximal, then the true maximal orientation component for $X$ with respect to ${\cal K}$ and ${\cal G}^*$, denoted by ${\cal U}_\mathrm{true}$, is a proper super graph of ${\cal U}_\mathrm{out}$. Recall that Algorithm~\ref{algo:moc} removes one vertex from the current graph in each while loop until the current graph is ${\cal U}_\mathrm{out}$. Since $\mathbf{V}({\cal U}_\mathrm{out})\subsetneq \mathbf{V}({\cal U}_{\mathrm{true}})$, let $\{Y_{{i_1}}, Y_{i_2}, \cdots, Y_{i_k}\} = \mathbf{V}({\cal U}_{\mathrm{true}})\setminus \mathbf{V}({\cal U}_\mathrm{out})$, where the subscript of each $Y_i$ indicates the number of loops when it is removed. Without loss of generality we can assume that $i_1<i_2<\cdots<i_k$. Based on Algorithm~\ref{algo:moc}, the vertices that are removed before $Y_{i_1}$, if exist, are not in $\mathbf{V}({\cal U}_\mathrm{out})$, since otherwise the while loop ends before considering $Y_{{i_1}}$. Let ${\cal U}_1$ be the undirected graph right before removing $Y_{{i_1}}$. By Algorithm~\ref{algo:moc}, with respect to ${\cal K}$ and ${\cal G}^*$, ${\cal U}_1$ is not an orientation component for $X$ and $Y_{{i_1}}$ is a potential lead node in ${\cal U}_1$. This means $Y_{{i_1}}$ is a simplicial node in ${\cal U}_1$ and there is no DCC in ${\cal K}(\mathcal{U}_1)$ in which $Y_{{i_1}}$ is the tail. Note that ${\cal U}_\mathrm{true}$ is the induced subgraph of ${\cal U}_1$ over $\mathbf{V}({\cal U}_{\mathrm{true}})$ and $Y_{{i_1}}\in \mathbf{V}({\cal U}_{\mathrm{true}})$, $Y_{{i_1}}$ is also a simplicial node in $\mathbf{V}({\cal U}_{\mathrm{true}})$. Moreover, there is no DCC in ${\cal K}(\mathcal{U}_{\mathrm{true}})$ in which $Y_{{i_1}}$ is the tail, as ${\cal K}(\mathcal{U}_{\mathrm{true}}) \subseteq {\cal K}(\mathcal{U}_1)$. Thus, $Y_{{i_1}}$ is a potential leaf node in ${\cal U}_\mathrm{true}$ with respect to ${\cal K}$ and ${\cal G}^*$, meaning that ${\cal U}_\mathrm{true}$ is not an orientation component for $X$. This leads to a contradiction, and hence ${\cal U}_\mathrm{out}$ is maximal.
	\end{proof}

	\subsection{Proof of Theorem \ref{thm:id}}\label{app:proof:id}
	
	The key ingredient for proving the identifiability is the following lemma.
	
	\begin{lemma}\label{app:lem:out}
		Let $\mathcal{H}$ be an MPDAG representing $[{\cal G}^*, {\cal K}]$ induced by a CPDAG ${\cal G}^*$ and a set of consistent DCCs ${\cal K}$. Then, for any vertex $X$, there is a DAG $\mathcal{G}$ in $[{\cal G}^*, {\cal K}]$ such that  $ch(X, \mathcal{G}) = sib(X, \mathcal{H}) \cup ch(X, \mathcal{H})$.
	\end{lemma}
	
	\begin{proof}
		It suffices to show that ${\cal K}\cup\{X\to W\mid W\in sib(X, \mathcal{H})\}$ is consistent with ${\cal G}^*$. For any connected undirected induced subgraph $\cal U$ of ${\cal G}^*$,  a potential leaf node in $\cal U$ with respect to ${\cal K}$ and ${\cal G}^*$, if it is not $X$, is still a potential leaf node in $\cal U$ with respect to ${\cal K}\cup\{X\to W\mid W\in sib(X, \mathcal{H})\}$ and ${\cal G}^*$. On the other hand, if $X$ is the only potential leaf node in $\cal U$ with respect to ${\cal K}$ and ${\cal G}^*$, then $adj(X, {\cal U})\to X$ are in $\cal H$. Consequently, $W$ is not in $\cal U$ since $W\in sib(X, \mathcal{H})$, and hence $X$ is still the only potential leaf node in $\cal U$ with respect to ${\cal K}\cup\{X\to W\mid W\in sib(X, \mathcal{H})\}$ and ${\cal G}^*$. The desired result follows from Theorem~\ref{thm:consistency}.
	\end{proof}
	
	\citet[Proposition~3.2]{Perkovic2020mpdag} showed that,
	
	\begin{lemma}[\citealt{Perkovic2020mpdag}, Proposition~3.2]\label{app:lem:two-special-p-origin}
		Let $\mathcal{H}$ be a causal MPDAG. If there is a proper possibly causal path $\pi=(X, W, U,...,Y)$ from $X\in \mathbf{X}$ to $Y\in \mathbf{Y}$ such that $X\to W\to U\to\cdots\to Y$ is in one DAG ${\cal G}_1\in[{\cal H}]$ and  $X\leftarrow W\to U\to\cdots\to Y$ is in another DAG ${\cal G}_2\in[{\cal H}]$, then there exists a multivariate Gaussian density $f$ over $\mathbf{V}({\cal G}^*)$ such that $f_1(y\mid do(\mathbf{x}))\neq f_2(y\mid do(\mathbf{x}))$, where $f_1(y \,|\, do(\mathbf{x}))$ and $f_2(y \,|\, do(\mathbf{x}))$ are interventional distributions computed from two causal models $({\cal G}_1, f(\mathbf{V}))$ and $({\cal G}_2, f(\mathbf{V}))$, respectively.
	\end{lemma}
	
	Lemma~\ref{app:lem:two-special-p-origin} results the following corollary.
	
	\begin{corollary}\label{app:lem:two-special-p}
		Let $\mathcal{H}$ be an MPDAG representing $[{\cal G}^*, {\cal K}]$ for a CPDAG ${\cal G}^*$ and a set of consistent DCCs ${\cal K}$. If there is a proper possibly causal path $\pi=(X, W, U,...,Y)$ from $X\in \mathbf{X}$ to $Y\in \mathbf{Y}$ such that $X\to W\to U\to\cdots\to Y$ is in one DAG ${\cal G}_1\in[{\cal G}^*, {\cal K}]$ and  $X\leftarrow W\to U\to\cdots\to Y$ is in another DAG ${\cal G}_2\in[{\cal G}^*, {\cal K}]$, then there exists a multivariate Gaussian density $f$ over $\mathbf{V}({\cal G}^*)$ such that $f_1(y\mid do(\mathbf{x}))\neq f_2(y\mid do(\mathbf{x}))$, where $f_1(y \,|\, do(\mathbf{x}))$ and $f_2(y \,|\, do(\mathbf{x}))$ are interventional distributions computed from two causal models $({\cal G}_1, f(\mathbf{V}))$ and $({\cal G}_2, f(\mathbf{V}))$, respectively.
	\end{corollary}
	
	Corollary~\ref{app:lem:two-special-p} holds because $[{\cal G}^*, {\cal K}]\subseteq [{\cal H}]$, Thus, we omit the proof. Analogue to Corollary~\ref{app:lem:two-special-p}, we have,
	
	\begin{lemma}\label{app:lem:two-special-p1}
		Let $\mathcal{H}$ be an MPDAG representing $[{\cal G}^*, {\cal K}]$ for a CPDAG ${\cal G}^*$ and a set of consistent DCCs ${\cal K}$. If there is a proper possibly causal path $\pi=(X, W, U,...,Y)$ from $X\in \mathbf{X}$ to $Y\in \mathbf{Y}$ such that $X\to W\leftarrow U\to\cdots\to Y$ and $X\to U$ are in one DAG ${\cal G}_1\in[{\cal G}^*, {\cal K}]$, and  $X\leftarrow W\to U\to\cdots\to Y$ and $X\to U$ are in another DAG ${\cal G}_2\in[{\cal G}^*, {\cal K}]$, then there exists a multivariate Gaussian density $f$ over $\mathbf{V}({\cal G}^*)$ such that $f_1(y\mid do(\mathbf{x}))\neq f_2(y\mid do(\mathbf{x}))$, where $f_1(y \,|\, do(\mathbf{x}))$ and $f_2(y \,|\, do(\mathbf{x}))$ are interventional distributions computed from two causal models $({\cal G}_1, f(\mathbf{V}))$ and $({\cal G}_2, f(\mathbf{V}))$, respectively.
	\end{lemma}
	
	\begin{proof}
		Let $f$ be a multivariate Gaussian density determined by the following linear Gaussian structural equation model,
		\[X_i=\sum_{X_j\in pa(X_i, {\cal G}_1)} \beta_{ji}X_j + \epsilon_i,\]
		where $\beta_{ji}=0$ if $X_j\to X_i$ is neither $X\to U$ nor the edges on the corresponding path of $\pi$ in ${\cal G}_2$ and $0<\beta_{ji}<1$ otherwise, $\epsilon_i$ are Gaussian noises with zero means and variances that makes every variable has variance one. (This is possible if we set $\beta_{ji}$'s small enough.) It is clear that $f$ is Markovian to ${\cal G}_2$, and thus Markovian to ${\cal G}_1$ as ${\cal G}_1$ and ${\cal G}_2$ are Markov equivalent. Moreover, denote by ${\cal G}'_i$ the DAGs obtained by removing from ${\cal G}_i$ the edges that are neither on the corresponding path of $\pi$ nor $X\to U$ ($i=1,2$), it can be checked that $f$ is Markovian to ${\cal G}'_i$ and $f_i(y\mid do(\mathbf{x})) = f'_i(y\mid do(\mathbf{x}))$, where $f'_i(y\mid do(\mathbf{x}))$ is the interventional distributions computed from the causal model $({\cal G}'_i, f(\mathbf{V}))$, $i=1,2$.
		
		Using the backdoor adjustment, it can be verified that $E_1(y\mid do(X = 1))=\sigma_{xy}$ and $E_2(y\mid do(X = 1))={(\sigma_{xy}-\sigma_{xw}\sigma_{wy})}/{(1-\sigma_{xw}^2)}$.  By Wright's rule~\citep{wright1921}, it can be checked that $\sigma_{xy}-{(\sigma_{xy}-\sigma_{xw}\sigma_{wy})}/{(1-\sigma_{xw}^2)}$ equals to the product of the edge weights along the path $X\leftarrow W\to U\to\cdots\to Y$ in ${\cal G}_2$. By our assumption, the edge weights are non-zero, $E_1(y\mid do(X = 1))\neq E_2(y\mid do(X = 1))$, and consequently, $f_1(y\mid do(\mathbf{x}))\neq f_2(y\mid do(\mathbf{x}))$.
	\end{proof}
	
	Finally, the proof of Theorem \ref{thm:id} follows a similar argument to that for \citet[Lemma~C.1]{perkovic2017interpreting}.
	
	
	\begin{proof}[Proof of Theorem \ref{thm:id}]
		Based on \citet[Theorem~3.6]{Perkovic2020mpdag}, the sufficiency of the identification condition holds, since $[{\cal G}^*, {\cal K}] \subseteq [\mathcal{H}]$ and the causal effect of $\mathbf{X}$ on $\mathbf{Y}$ is identifiable from $\cal H$ when the condition holds.
		
		To prove the necessity, let $\pi$ be a possibly causal path from $X\in \mathbf{X}$ to $Y\in \mathbf{Y}$ in $\cal H$ where the first edge from the side of $X$ is undirected. Denote by $\pi^*=(X, W, U,...,Y)$, a shortest subsequence of  $\pi$ with length at least 1, such that $\pi^*$ is also a possibly causal path from $X$ to $Y$, where the first edge from the side of $X$ is undirected. It is clear that $\pi^*(W,Y)$ is unshielded. By Lemma~\ref{app:lem:out}, there is a DAG $\mathcal{G}_2\in[{\cal G}^*, {\cal K}]$ such that  $ch(W, \mathcal{G}_2) = sib(W, \mathcal{H}) \cup ch(W, \mathcal{H})$. Hence, the corresponding path of $\pi^*$ in ${\cal G}_2$ is $X\leftarrow W\to U\to\cdots\to Y$, according to the first Meek's rule.
		
		If $X$ is not adjacent to $U$ in $\cal H$, then we consider a DAG $\mathcal{G}_1\in[{\cal G}^*, {\cal K}]$ where  $ch(X, \mathcal{G}_1) = sib(X, \mathcal{H}) \cup ch(X, \mathcal{H})$. By Lemma~\ref{app:lem:out}, such a DAG exists. Since $X$ is not adjacent to $U$, $\pi^*$ is unshielded, and thus the corresponding path of $\pi^*$ in ${\cal G}_1$ is $X\to W\to U\to\cdots\to Y$. By Corollary~\ref{app:lem:two-special-p}, the causal effect of $\mathbf{X}$ on $\mathbf{Y}$ is not identifiable.
		
		If $X$ is adjacent to $U$, then $X\to U$ is in $\cal H$, since otherwise $X-U$ and $\pi^*(U,Y)$ form a possibly causal path shorter than $\pi^*$, which contradicts our assumption. If $W\to U$ is in $\cal H$, then we again consider a DAG $\mathcal{G}_1\in[{\cal G}^*, {\cal K}]$ where  $ch(X, \mathcal{G}_1) = sib(X, \mathcal{H}) \cup ch(X, \mathcal{H})$. In ${\cal G}_1$, the corresponding path of $\pi^*$ is $X\to W\to U\to\cdots\to Y$. If $W-U$ is in $\cal H$, then we consider the DAG $\mathcal{G}_1\in[{\cal G}^*, {\cal K}]$ where $ch(U, \mathcal{G}_1) = sib(U, \mathcal{H}) \cup ch(U, \mathcal{H})$. Since $X\to U$ and $U\to W$ are in $\mathcal{G}_1$, $X\to W$ is in ${\cal G}_1$. Thus, according to Lemma~\ref{app:lem:two-special-p1}, the causal effect of $\mathbf{X}$ on $\mathbf{Y}$ is not identifiable.
	\end{proof}

	\subsection{Proof of Corollary \ref{coro:id}}
	\begin{proof}
		We first prove that (ii) $\Rightarrow$ (i). Assume that $X\to Y$ exists in every DAG in $[\mathcal{G}^*, {\cal K}]$, but $X-Y$ is in $\mathcal{G}^*$, then the causal effect of $Y$ on $X$ is not identifiable in $[\mathcal{G}^*]$ but becomes identifiable in $[\mathcal{G}^*, {\cal K}]$. In fact, the causal effect of $Y$ on $X$ is $0$ in every DAG in $[\mathcal{G}^*, {\cal K}]$.
		
		Conversely, if the common directed causal relations of the DAGs in $[\mathcal{G}^*, {\cal K}]$ are all encoded by directed edges in $\mathcal{G}^*$, then by the definition of a causal MPDAG, ${\cal H} = \mathcal{G}^*$. By Theorem~\ref{thm:id}, an effect is identifiable in $[\mathcal{G}^*]$ if and only if it is identifiable in $[\mathcal{G}^*, {\cal K}]$. This completes the proof of (i) $\Rightarrow$ (ii).

		Next, suppose that $\cal K$ is derived from a consistent pairwise causal background knowledge set $\cal B$. If there is a direct causal constraint in $\cal B$ which does not hold for all DAGs in $[\mathcal{G}^*]$, then it is clear that an unidentifiable effect becomes identifiable  in $[\mathcal{G}^*, {\cal K}]$, as statement (ii) holds.  If there is a non-ancestral causal constraint in $\cal B$ which does not hold for all DAGs in $[\mathcal{G}^*]$, then by Theorem~\ref{thm:nbr_set_const}, statement (ii) also holds. Finally, since $X\dashrightarrow Y$ implies $Y\longarrownot\dashrightarrow X$, $\cal B$ is equivalent to $(Y\longarrownot\dashrightarrow X) \cup {\cal B}$ with respect to $\mathcal{G}^*$. Thus, if $Y\longarrownot\dashrightarrow X$ does not hold for all DAG in $[\mathcal{G}^*]$, then statement (ii) holds.
	\end{proof}

	\subsection{Proof of Theorem \ref{thm:adjustment}}\label{app:proof:adjustment}
	
	Before proving Theorem \ref{thm:adjustment}, we first introduce some technical lemmas.
	
	\begin{lemma}[\citealt{perkovic2017interpreting}, Lemma~C.2]\label{app:lem:c2o}
		Let $\mathcal{H}$ be a causal MPDAG, and $\mathbf{X}, \mathbf{Y}$ are disjoint subsets of vertices of $\mathcal{H}$. Suppose that the causal effect of $\mathbf{X}$ on $\mathbf{Y}$ is identifiable in $[\cal H]$, then
		\begin{enumerate}
			\item[(i)]  $\mathbf{Z}\cap \mathrm{Forb}(\mathbf{X},\mathbf{Y}, {\cal H})=\varnothing$ implies that $\mathbf{Z}\cap \mathrm{Forb}(\mathbf{X},\mathbf{Y}, {\cal G})=\varnothing$ for any DAG ${\cal G}\in[{\cal H}]$, and
			\item[(ii)]  $\mathbf{Z}\cap \mathrm{Forb}(\mathbf{X},\mathbf{Y}, {\cal H})\neq\varnothing$ implies that in $\cal H$ there exist $X\in \mathbf{X}$, $Y\in \mathbf{Y}$, $U\in \mathbf{Z}\cap \mathrm{Forb}(\mathbf{X},\mathbf{Y}, {\cal H})$ and $W$ such that there are a directed path from $X$ to $W$ and unshielded possibly causal paths from $W$ to $Y$ and $U$, respectively.
		\end{enumerate}
	\end{lemma}
	
	\begin{lemma}[\citealt{perkovic2017interpreting}, Lemma~C.3]\label{app:lem:c3o}
		Let $\mathcal{H}$ be a causal MPDAG, and $\mathbf{X}, \mathbf{Y}$ are disjoint subsets of vertices of $\mathcal{H}$. Suppose that the causal effect of $\mathbf{X}$ on $\mathbf{Y}$ is identifiable in $[\cal H]$ and $\mathbf{Z}\cap \mathrm{Forb}(\mathbf{X},\mathbf{Y}, {\cal H})=\varnothing$, then there is a proper definite status non-causal path from $X\in\mathbf{X}$ to $Y\in\mathbf{Y}$ which is not blocked by $\mathbf{Z}$ in ${\cal H}$ implies that $\mathbf{Z}$ is not an adjustment set for $(\mathbf{X}, \mathbf{Y})$ in any DAG ${\cal G} \in [{\cal H}]$.
	\end{lemma}
	
	Lemmas~\ref{app:lem:c2o} and~\ref{app:lem:c3o} result the following Corollaries~\ref{app:lem:c2} and~\ref{app:lem:c3}, respectively.
	
	\begin{corollary}\label{app:lem:c2}
		Let $\mathcal{H}$ be an MPDAG representing $[{\cal G}^*, {\cal K}]$ for a CPDAG ${\cal G}^*$ and a set of consistent DCCs ${\cal K}$, and $\mathbf{X}, \mathbf{Y}$ are disjoint subsets of vertices of $\mathcal{H}$. Suppose that the causal effect of $\mathbf{X}$ on $\mathbf{Y}$ is identifiable in $[{\cal G}^*, {\cal K}]$, then
		\begin{enumerate}
			\item[(i)]  $\mathbf{Z}\cap \mathrm{Forb}(\mathbf{X},\mathbf{Y}, {\cal H})=\varnothing$ implies that $\mathbf{Z}\cap \mathrm{Forb}(\mathbf{X},\mathbf{Y}, {\cal G})=\varnothing$ for any DAG ${\cal G}\in[{\cal G}^*, {\cal K}]$, and
			\item[(ii)]  $\mathbf{Z}\cap \mathrm{Forb}(\mathbf{X},\mathbf{Y}, {\cal H})\neq\varnothing$ implies that in $\cal H$ there exist $X\in \mathbf{X}$, $Y\in \mathbf{Y}$, $U\in \mathbf{Z}\cap \mathrm{Forb}(\mathbf{X},\mathbf{Y}, {\cal H})$ and $W$ such that there are a directed path from $X$ to $W$ and unshielded possibly causal paths from $W$ to $Y$ and $U$, respectively.
		\end{enumerate}
	\end{corollary}
	
	\begin{corollary}\label{app:lem:c3}
		Let $\mathcal{H}$ be an MPDAG representing $[{\cal G}^*, {\cal K}]$ for a CPDAG ${\cal G}^*$ and a set of consistent DCCs ${\cal K}$, and $\mathbf{X}, \mathbf{Y}$ are disjoint subsets of vertices of $\mathcal{H}$. Suppose that the causal effect of $\mathbf{X}$ on $\mathbf{Y}$ is identifiable in $[{\cal G}^*, {\cal K}]$ and $\mathbf{Z}\cap \mathrm{Forb}(\mathbf{X},\mathbf{Y}, {\cal H})=\varnothing$, then there is a proper definite status non-causal path from $X\in\mathbf{X}$ to $Y\in\mathbf{Y}$ which is not blocked by $\mathbf{Z}$ in ${\cal H}$ implies that $\mathbf{Z}$ is not an adjustment set for $(\mathbf{X}, \mathbf{Y})$ in any DAG ${\cal G} \in [{\cal G}^*, {\cal K}]$.
	\end{corollary}
	
	The above corollaries hold since $[{\cal G}^*, {\cal K}]\subseteq [{\cal H}]$ and Theorem~\ref{thm:id} proves that the causal effect of $\mathbf{X}$ on $\mathbf{Y}$ is identifiable in $[{\cal G}^*, {\cal K}]$ if and only if it is identifiable in $[\cal H]$. We omit the proofs here.
	
	The proof of Theorem~\ref{thm:adjustment} is analogue to that of \citet[Theorem~4.4]{perkovic2017interpreting}.
	
	\begin{proof}[Proof of Theorem~\ref{thm:adjustment}]
		Based on Theorem~4.4 in \citet{perkovic2017interpreting} and the fact that $[{\cal G}^*, {\cal B}]\subseteq [{\cal H}]$, it is clear that $\mathbf{Z}$ satisfies the b-adjustment criterion relative to $(\mathbf{X}, \mathbf{Y})$ in $\cal H$ implies that $\mathbf{Z}$ is an adjustment set for $(\mathbf{X}, \mathbf{Y})$ with respect to ${\cal G}^*$ and $\cal B$. To prove the other direction, we use the proof by contradiction. First, Theorem~\ref{thm:id} indicates that the first condition holds. Assuming the first condition holds but the second condition fails to hold, then by Corollary~\ref{app:lem:c2}, there exist $X\in \mathbf{X}$, $Y\in \mathbf{Y}$, $U\in \mathbf{Z}\cap \mathrm{Forb}(\mathbf{X},\mathbf{Y}, {\cal H})$ and $W$ such that there are a directed path from $X$ to $W$ and unshielded possibly causal paths from $W$ to $Y$ and $U$, respectively. According to Lemma~\ref{app:lem:out}, there is a DAG ${\cal G} \in [{\cal G}^*, {\cal B}]$ such that every sibling of $W$ is a child of $W$ in  ${\cal G}$. Thus, the unshielded possibly causal paths from $W$ to $Y$ and $U$ are directed in ${\cal G}$, which means $\mathbf{Z}\cap \mathrm{Forb}(\mathbf{X},\mathbf{Y}, {\cal G})\neq\varnothing$ for ${\cal G}$. Consequently, by \citet[Theorem~4.4]{Perkovic2020mpdag}, $\mathbf{Z}$ is not an adjustment set for $(\mathbf{X}, \mathbf{Y})$ in ${\cal G}$. Finally, the necessity of the third condition is guaranteed by Corollary~\ref{app:lem:c3}.
	\end{proof}
	
	\subsection{Proof of Theorem \ref{thm:local}}\label{app:proof:local}
	
	\citet{maathuis2009estimating} proved that,
	
	\begin{lemma}[\citealt{maathuis2009estimating}, Lemma~3.1]\label{lem:local-valid-nbg}
		Given a CPDAG $\mathcal{G}^*$, a treatment $X$, and $\textbf{S}\subseteq sib(X, \mathcal{G}^*)$, there is a DAG $\mathcal{G}\in [{\cal G}^*]$ such that $pa(X, \mathcal{G}) = \textbf{S}\cup pa(X, {\cal G}^*)$ and $ch(X, \mathcal{G}) = sib(X, {\cal G}^*)\cup ch(X, {\cal G}^*)\setminus \textbf{S}$ if and only if the induced subgraph of ${\cal G}^*$ over $\textbf{S}$ is complete.
	\end{lemma}
	
	Given a CPDAG ${\cal G}^*$, a treatment $X$ and $\textbf{S}\subseteq sib(X, \mathcal{G}^*)$, and suppose that there is a DAG $\mathcal{G}\in [{\cal G}^*]$ such that $pa(X, \mathcal{G}) = \textbf{S}\cup pa(X, {\cal G}^*)$ and $ch(X, \mathcal{G}) = sib(X, {\cal G}^*)\cup ch(X, {\cal G}^*)\setminus \textbf{S}$. Regarding $\textbf{S}\to X$ and $X\to sib(X, {\cal G}^*)\setminus \textbf{S}$ as direct causal constraints and denote them by ${\cal K}$, the MPDAG $\cal H$ of $[{\cal G}^*, {\cal K}]$ is a chain graph~\citep[Theorem~6]{he2008active}. Suppose that, apart from ${\cal K}$, we have another DCC set ${\cal K}'$. The following lemma extends Theorem~\ref{thm:consistency} and gives a sufficient and necessary condition to check whether ${\cal K}'$ is consistent with ${\cal G}^*$ given ${\cal K}$ (that is, ${\cal K}'\cup{\cal K}$ is consistent with ${\cal G}^*$).
	
	\begin{lemma}\label{lemma:consist}
		Let ${\cal G}^*$ be a CPDAG, $X$ be a variable in ${\cal G}^*$ and $\textbf{S}\subseteq sib(X, \mathcal{G}^*)$. Let ${\cal K}=\{S\to X\mid S\in \textbf{S}\}\cup \{X\to C\mid C\in sib(X, {\cal G}^*)\setminus \textbf{S}\}$ and assume that ${\cal K}$ is consistent with ${\cal G}^*$. Denote by $\cal H$ the MPDAG of $[{\cal G}^*, {\cal K}]$. For any DCC set ${\cal K}'$, the following two statements are equivalent.
		\begin{enumerate}
			\item[(i)]  ${\cal K}'$ is consistent with $\mathcal{G}^*$ given ${\cal K}$.
			\item[(ii)]  Any connected undirected induced subgraph of $\cal H$ has a potential leaf node with respect to ${\cal K}'$ and $\cal H$.
		\end{enumerate}
	\end{lemma}
	
	Note that, in the main text the potential leaf node is defined over a CPDAG instead of an MPDAG. Therefore, for the sake of rigor, we extend the related definitions to chain graph causal MPDAG in the following, where a chain graph causal MPDAG is a causal MPDAG which itself is a chain graph.
	
	\begin{definition}\label{def:reducedf:2}
		Given a chain graph causal MPDAG $\cal H$ and a set  of DCCs $\cal K$ over ${\bf V}({\cal H})$, a reduced form of $\cal K$ with respect to $\cal H$, denoted by ${\cal K}({\cal H})$, is defined as as follows.
		\begin{equation}\label{eq:n_h}
			{\cal K}({\cal H})\coloneqq\{\kappa_t\tor \left(\kappa_h \cap sib(\kappa_t, {\cal H})\right)  \mid  \kappa \in {\cal K} \;\text{and}\; \kappa_h\cap ch(\kappa_t, {\cal H})=\varnothing\}.
		\end{equation}
	\end{definition}
	
	Similar to Proposition~\ref{pro:reducedf}, it is easy to verify that a DAG in $[\cal H]$ satisfies all constraints in ${\cal K}({\cal H})$ if and only if it satisfies all constraints in ${\cal K}$.
	
	\begin{definition}\label{def:restrict:2}
		Given an undirected induced subgraph $\mathcal{U}$ of a chain graph causal MPDAG $\cal H$ over $\mathbf{V}(\mathcal{U})\subseteq {\bf V}({\cal H})$, and a set of DCCs $\cal K$ over ${\bf V}({\cal H})$, the restriction subset of $\cal K$   on $\mathcal{U}$ given $\cal H$ is defined by
		\begin{equation}\label{eq:n_hh}
			{\cal K}(\mathcal{U}\mid{\cal H})\coloneqq\{\kappa\in {\cal K}({\cal H}) \mid \{\kappa_t\}\cup\kappa_h\subseteq \mathbf{V}(\mathcal{U})\}.
		\end{equation}
	\end{definition}

	\begin{definition}\label{def:pln:2}
		Let $\cal H$ be a chain graph causal MPDAG and $\cal K$ be a set of DCCs over ${\bf V}({\cal H})$. Given an undirected induced subgraph $\mathcal{U}$ of $\cal H$ and a vertex $X$ in $\mathcal{U}$, $X$ is called a potential leaf node in $\mathcal{U}$ with respect to $\cal K$ and $\cal H$, if $X$ is a simplicial vertex in $\mathcal{U}$ and $X$ is not the tail of any clause in ${\cal K}(\mathcal{U}\mid{\cal H})$.
	\end{definition}
	
	\begin{proof}[proof of Lemma~\ref{lemma:consist}]
		The proof is similar to that of Theorem~\ref{thm:consistency}.	We first prove the necessity.  If ${\cal K}'$ is consistent with $\mathcal{G}^*$ given ${\cal K}$, then there exists a DAG ${\cal G}\in[{\cal H}]$ satisfying all constraints in ${\cal K}'$. Let $\mathcal{U}$ be an arbitrary connected undirected induced subgraph of ${\cal H}$, and denote the induced subgraph of $\cal G$ over $\mathbf{V}(\mathcal{U})$ by ${\cal G}_{\mathrm{sub}}$. Since any induced subgraph of a DAG is still a DAG, ${\cal G}_{\mathrm{sub}}$ is a DAG, and thus it must have a leaf node $V_{\mathrm{leaf}}$. As we assume that ${\cal H}$ is a chain graph, following exactly the same argument for proving Lemma~\ref{lem:app-pln}, we can show that  $V_{\mathrm{leaf}}$ is a potential leaf node in $\mathcal{U}$ with respect to ${\cal K}'$ and ${\cal H}$.
		
		We next prove the sufficiency. Since ${\cal H}$ is a chain graph, no orientation of the edges not oriented in ${\cal H}$ will create a directed cycle which includes an edge or edges that were oriented in ${\cal H}$. Moreover, based on statement (iii) of Theorem~\ref{the:MPDAG}, no orientation of an edge not	directed in ${\cal H}$ can create a new v-structure with an edge that was oriented in ${\cal H}$. Then, following the same argument for proving Lemma~\ref{lem:app-equi}, we can show that ${\cal K}'$ is consistent with ${\cal G}^*$ given ${\cal K}$ if and only if ${\cal K}'({\cal C}\mid{\cal H})$ is consistent with ${\cal C}$ for any chain component ${\cal C}$ of ${\cal H}$. The desired result comes from the same construction of PEO given in the proof of Theorem~\ref{thm:consistency}.    	
	\end{proof}
	
	In order to prove Theorem \ref{thm:local}, we prove the following Theorem~\ref{thm:local:2},  which includes the result provided in Theorem \ref{thm:local}.
	
	\begin{thmbis}{thm:local}\label{thm:local:2}
		Let $\cal K$ be a set of DCCs consistent with a CPDAG ${\cal G}^*$, and $\mathcal{H}$ be the MPDAG of $[{\cal G}^*, {\cal K}]$. For any vertex $X$ and $\textbf{S}\subseteq sib(X, \mathcal{H})$, let
		\[\mathbf{T}=\{X\}\cup \left( (pa(X, \mathcal{H}) \cup \textbf{S})\cap sib(X, {\cal G}^*)\right),\]
		and
		\[D_X=\{u\to X \mid u\in pa(X, \mathcal{H}) \cup \textbf{S}\} \cup \{X\to v \mid v\in sib(X, \mathcal{H})\cup ch(X, \mathcal{H})\setminus \textbf{S}\}.\]
		Then, the following statements are equivalent.
		\begin{enumerate}
			\item[(i)]  There is a DAG $\mathcal{G}$ in $[{\cal G}^*, {\cal K}]$ such that $pa(X, \mathcal{G}) = \textbf{S}\cup pa(X, \mathcal{H})$ and $ch(X, \mathcal{G}) = sib(X, \mathcal{H})\cup ch(X, \mathcal{H})\setminus \textbf{S}$.
			
			\item[(ii)]  The induced subgraph of $\mathcal{H}$ over $\textbf{S}$ is complete and the restriction subset of ${\cal K}\cup D_X$ on ${\cal G}^*({\bf M}_{\mathbf{T}})$ given ${\cal G}^*$ is consistent with ${\cal G}^*({\bf M}_{\mathbf{T}})$ for all maximal clique ${\bf M}_{\mathbf{T}}$ of ${\cal G}^*$ containing $\bf T$.
			
			\item[(iii)] Given ${\cal G}^*$, the restriction subset of ${\cal K}\cup D_X$ on ${\cal G}^*(\{X\}\cup sib(X, {\cal G}^*))$ is consistent with ${\cal G}^*(\{X\}\cup sib(X, {\cal G}^*))$.
		\end{enumerate}
	\end{thmbis}
	
	\begin{proof}[Proof of Theorem \ref{thm:local:2}]
		By Theorem~\ref{thm:consistency}, statement (i) implies statement (iii) and statement (iii) $\Rightarrow$ (ii). Thus, we only prove (ii) $\Rightarrow$ (i) in the following.
		
		To prove (ii) $\Rightarrow$ (i), it suffices to show that ${\cal K}\cup D_X$ is consistent with ${\cal G}^*$. According to Lemma~\ref{lem:app-equi}, we can consider each chain component of ${\cal G}^*$ separately. Therefore, without loss of generality, we can assume that ${\cal G}^*$ is a connected chordal graph, and that ${\cal K}$ is already in its reduced form with respect to ${\cal G}^*$. This means we have that,
		\begin{enumerate}
			\item[(P1)] for any $\kappa\in {\cal K}$, $\kappa_h\subseteq sib(\kappa_t, {\cal G}^*)= adj(\kappa_t, {\cal G}^*)$, and the consistency of ${\cal K}$ with ${\cal G}^*$ indicates that $\kappa_h\neq\varnothing$.
		\end{enumerate}

		It can be seen from the definition of consistency that checking the consistency of ${\cal K}\cup D_X$ with  ${\cal G}^*$ is equivalent to the following two-steps procedure:
		\begin{enumerate}[label=\arabic*., leftmargin=0pt, labelwidth=*, align=left]
			\item[{\bf Step 1.}] checking whether $D_X$ is consistent with ${\cal G}^*$, and if the consistency holds, then
			\item[{\bf Step 2.}] checking whether ${\cal K}\setminus D_X$ is consistent with ${\cal G}^*$ given $D_X$.
		\end{enumerate}
		Note that, the intersection of ${\cal K}$ and $D_X$ may not be empty. Clearly, ${\cal K}\cup D_X$ is consistent if and only if neither of the above two steps returns a negative answer.

		We first prove that $D_X$ is consistent with ${\cal G}^*$, meaning that step 1 returns a positive answer. Based on statement (ii), the induced subgraph of $\cal H$ over $\textbf{S}$ is complete, indicating that the induced subgraph of ${\cal G}^*$ over $\textbf{S}$ is complete as $\cal H$ and ${\cal G}^*$ has the same skeleton. If $pa(X, {\cal H})=\varnothing$, then $pa(X, {\cal H})\cup \textbf{S}$ induces a complete subgraph. If $pa(X, {\cal H})\neq\varnothing$, then for any $p\in pa(X, {\cal H})$ and $S \in \textbf{S}$, $p$ and $S$ are adjacent in $\cal H$, since otherwise $X\to S$ should appear in $\cal H$ due to the maximality of $\cal H$. Moreover, since we have assumed that ${\cal G}^*$ is a connected chordal graph, $\cal H$ has no v-structure as discussed in Appendix~\ref{app:proof:pre}, and thus, $pa(X, {\cal H})$ induces a complete subgraph. Therefore, it holds that,
		\begin{enumerate}
			\item[(P2)] $pa(X, {\cal H})\cup \textbf{S}$ induces a complete subgraph of ${\cal G}^*$.
		\end{enumerate}
		By Lemma \ref{lem:local-valid-nbg} and the definition of $D_X$, $D_X$ is consistent with ${\cal G}^*$.
		
		Below we will show that ${\cal K}\setminus D_X$ is consistent with ${\cal G}^*$ given $D_X$. Note that, since $D_X$ contains direct causal constraints only, if we denote the MPDAG representing $[{\cal G}^*, D_X]$ by ${\cal C}^{*}$, then ${\cal C}^{*}$ is a chain graph MPDAG and $[{\cal C}^{*}]=[{\cal G}^*, D_X]$~\citep[Theorem~6]{he2008active}. Now, consider the following four subsets of ${\cal K}\setminus D_X$:
		\begin{align*}
			{\cal K}_p &= \{ \kappa \in {\cal K}\setminus D_X \mid \exists d\in \kappa_h, \kappa_t\to d \; \text{is in}\; {\cal C}^{*} \}, \\
			{\cal K}_u &= \{ \kappa\in {\cal K}\setminus D_X \mid \forall d\in \kappa_h, \kappa_t- d \; \text{is in}\; {\cal C}^{*}\}, \\
			{\cal K}_c  &= \{ \kappa\in {\cal K}\setminus D_X \mid \forall d\in \kappa_h, d\to \kappa_t \; \text{is in}\; {\cal C}^{*}\}, \\
			{\cal K}_r & = {\cal K}\setminus D_X \setminus \left( {\cal K}_p \cup {\cal K}_u \cup {\cal K}_c \right).
		\end{align*}
		It is easy to verify that ${\cal K}_p, {\cal K}_u, {\cal K}_c$ and ${\cal K}_r$ are disjoint and
		\[{\cal K}_r  = \{\kappa \in {\cal K}\setminus D_X \setminus {\cal K}_p \mid \exists d\in \kappa_h, d\to \kappa_t \; \text{is in}\; {\cal C}^{*} \; \text{and}\; \exists d\in \kappa_h, d- \kappa_t \; \text{is in}\; {\cal C}^{*}\}.\]
		Since the DCCs in ${\cal K}_p$ are already satisfied given ${\cal C}^{*}$, to prove that ${\cal K}\setminus D_X$ is consistent with ${\cal G}^*$ given $D_X$, it suffices to show that ${\cal K}_u \cup {\cal K}_c \cup {\cal K}_r$ is consistent with ${\cal G}^*$ given $D_X$.
		
		We will give a proof by contradiction. Assuming that ${\cal K}_u \cup {\cal K}_c \cup {\cal K}_r$ is not consistent with ${\cal G}^*$ given $D_X$, then either ${\cal K}_c \neq \varnothing$ or there exists a connected undirected induced subgraph $\cal W$ of ${\cal C}^{*}$ such that $\cal W$ has no potential lead node with respect to ${\cal K}_u \cup {\cal K}_r$ and ${\cal C}^{*}$, as indicated by Lemma~\ref{lemma:consist}.
		
		\vspace{0.5em}
		
		\textbf{Case 1.} Suppose that ${\cal K}_c \neq \varnothing$. Let $\kappa\in {\cal K}_c$ be an arbitrary DCC. By the definitions of ${\cal K}_c$ and ${\cal C}^*$ as well as the maximality of ${\cal C}^*$, it holds that,
		\begin{enumerate}
			\item[(P3)] the maximal orientation component for $\kappa_t$ with respect to ${\cal G}^*$ and $D_X$, denoted by $\mathcal{U}_t$, satisfies that $\kappa_h \subseteq pa(\kappa_t, {\cal C}^*) = adj(\kappa_t, \mathcal{U}_t)$.
		\end{enumerate}
		$pa(\kappa_t, {\cal C}^*) = adj(\kappa_t, \mathcal{U}_t)$ is because we have assumed that ${\cal G}^*$ does not contain any directed edge. By (P1), $\kappa_h\neq \varnothing$, and thus $adj(\kappa_t, \mathcal{U}_t)\neq \varnothing$ and $\mathcal{U}_t$ is not a singleton graph (that is,  $\mathcal{U}_t$ has at least two vertices).
		
		We first show that $X\in \mathbf{V}(\mathcal{U}_t)$. Assume, for the sake of contradiction, that $X\notin \mathbf{V}(\mathcal{U}_t)$, then the restriction subset of $D_X$ on $\mathcal{U}_t$ given ${\cal G}^*$ is empty, as every DCC in $D_X$ has $X$ as its tail or head. Thus, $\mathcal{U}_t$ is not a maximal orientation component for $\kappa_t$ with respect to $D_X$ and ${\cal G}^*$, since $\mathcal{U}_t$ is chordal and any connected chordal graph has at least two simplicial vertices. This leads to a contradiction.

		Recall that $\mathcal{U}_t$ is connected, hence $X\in \mathbf{V}(\mathcal{U}_t)$ implies that $adj(X, \mathcal{U}_t)\neq\varnothing$. The remaining proof is quite lengthy. We will consider two subcases: $adj(X, \mathcal{U}_t) \subseteq pa(X, {\cal H})\cup \mathbf{S}=pa(X, {\cal C}^*)$ and $adj(X, \mathcal{U}_t) \nsubseteq pa(X, {\cal H})\cup \mathbf{S}=pa(X, {\cal C}^*)$, and show that in both subcases statement (ii) does not hold. That is, the restriction subset of ${\cal K}\cup D_X$ on ${\cal G}^*({\bf M}_{\mathbf{T}})$ given ${\cal G}^*$ is not consistent with ${\cal G}^*({\bf M}_{\mathbf{T}})$ for some ${\bf M}_{\mathbf{T}}$. This completes the proof for case 1.
		
		\vspace{0.5em}
		
		\textbf{Case 1-1.} If $adj(X, \mathcal{U}_t) \subseteq pa(X, {\cal H})\cup \mathbf{S}=pa(X, {\cal C}^*)$, then $adj(X, \mathcal{U}_t)$ is a clique in $\mathcal{U}_t$ based on (P2). Thus, $X$ is a simplicial vertex in $\mathcal{U}_t$. On the other hand, $adj(X, \mathcal{U}_t) \subseteq pa(X, {\cal H})\cup \mathbf{S}$ implies that $D_X(\mathcal{U}_t)=\{u\to X \mid u\in adj(X, \mathcal{U}_t)\}$, hence $X$ is a potential leaf node in $\mathcal{U}_t$ with respect to $D_X$ and ${\cal G}^*$. By the definition of $\mathcal{U}_t$, it holds that $X=\kappa_t$. Thus, $\{X\}\cup \kappa_h \subseteq \{X\}\cup pa(X, {\cal H})\cup \mathbf{S}\subseteq \mathbf{M}_{\mathbf{T}}$ for some maximal clique $\mathbf{M}_{\mathbf{T}}$ of ${\cal G}^*$ containing $\mathbf{T}$. Recall that $\kappa\in {\cal K}_c \subseteq {\cal K}$, the restriction subset of ${\cal K}\cup D_X$ on ${\cal G}^*({\bf M}_{\mathbf{T}})$ given ${\cal G}^*$  contains $X\tor \kappa_h$ and $\kappa_h\to X$ both, which is not consistent with ${\cal G}^*({\bf M}_{\mathbf{T}})$.
		
		\vspace{0.5em}
		
		\textbf{Case 1-2.} If $adj(X, \mathcal{U}_t) \nsubseteq pa(X, {\cal H})\cup \mathbf{S}=pa(X, {\cal C}^*)$, then $adj(X, \mathcal{U}_t)\cap ch(X, {\cal C}^*)\neq\varnothing$. We first claim that,
		\begin{enumerate}
			\item[(P4)] none  of the vertices in $\mathbf{V}(\mathcal{U}_t)\setminus (\{\kappa_t, X\}\cup pa(X, {\cal C}^*))$ (possibly empty) is simplicial in $\mathcal{U}_t$.
		\end{enumerate}
		In fact, with respect to $D_X$ and ${\cal G}^*$, only $X$ and the vertices in $adj(X, \mathcal{U}_t)\cap pa(X, {\cal C}^*)$ (possibly empty) can be the tails of DCCs in $D_X(\mathcal{U}_t)$. By the definition of a potential leaf node and $\mathcal{U}_t$, the vertices in $\mathbf{V}(\mathcal{U}_t)\setminus (\{\kappa_t, X\}\cup pa(X, {\cal C}^*))$ are non-simplicial as they are not potential leaf nodes in $\mathcal{U}_t$ with respect to $D_X$ and ${\cal G}^*$.
		
		Now, denote by $\mathbf{L}$ the set of potential leaf nodes of $\mathcal{U}_t$ with respect to ${\cal K}_c$ and ${\cal G}^*$, then $\mathbf{L}\subseteq \{\kappa_t, X\}\cup pa(X, {\cal C}^*)$ based on (P4) as the vertices in $\mathbf{L}$ are simplicial in $\mathcal{U}_t$. However, as $\kappa_h \subseteq pa(\kappa_t, {\cal C}^*) = adj(\kappa_t, \mathcal{U}_t)$, if $X=\kappa_t$, then $adj(X, \mathcal{U}_t) = pa(X, {\cal C}^*)$, contradicted to the assumption of case 1-2. Hence, $X\neq \kappa_t$. Since $\kappa \in {\cal K}_c(\mathcal{U}_t)$, $\kappa_t$ is not a potential leaf node in $\mathcal{U}_t$ with respect to ${\cal K}_c$ and ${\cal G}^*$. Thus, it holds that,
		\begin{enumerate}
			\item[(P5)] the set of potential leaf nodes of $\mathcal{U}_t$ with respect to ${\cal K}_c$ and ${\cal G}^*$, denoted by $\mathbf{L}$, satisfies that $\mathbf{L}\subseteq \{X\}\cup pa(X, {\cal C}^*)$.
		\end{enumerate}

		\textbf{Case 1-2-1.} If $X$ is simplicial in $\mathcal{U}_t$, then $adj(X, \mathcal{U}_t)$ is a clique in $\mathcal{U}_t$. Recall that $adj(X, \mathcal{U}_t)\cap ch(X, {\cal C}^*)\neq\varnothing$, we next consider two possibilities depends on whether a vertex in $adj(X, \mathcal{U}_t)\cap ch(X, {\cal C}^*)$ is simplicial.
		
		\vspace{0.5em}
		First, suppose that there is a $c\in adj(X, \mathcal{U}_t)\cap ch(X, {\cal C}^*)$ which is simplicial in $\mathcal{U}_t$, then by (P4), $c\notin \mathbf{V}(\mathcal{U}_t)\setminus (\{\kappa_t, X\}\cup pa(X, {\cal C}^*))$, indicating that $c=\kappa_t$. Thus, $\kappa_t\in adj(X, \mathcal{U}_t)$.
		
		\begin{enumerate}
			\item[(i)] $adj(X, \mathcal{U}_t)=\{\kappa_t\}$. We claim that this is an impossible case. In fact, it holds that $adj(\kappa_t, \mathcal{U}_t)=\{X\}$, since otherwise the vertices in $adj(\kappa_t, \mathcal{U}_t)$ are adjacent to $X$ due to the simplicity of $\kappa_t$, which contradicts $|adj(X, \mathcal{U}_t)|=1$. Thus, $\kappa_h=\{X\}$ and $\kappa_t\to X$ is in $\cal H$, which contradicts the assumption that $\kappa_t\in ch(X, {\cal C}^*) = ch(X, {\cal H})\cup sib (X, {\cal H}) \setminus \mathbf{S}$.
			
			\item[(ii)] $adj(X, \mathcal{U}_t)\setminus\{\kappa_t\}\neq\varnothing$ and $adj(X, \mathcal{U}_t)\setminus\{\kappa_t\}= pa(X, {\cal C}^*)$. Since $\kappa_t, X$ are both simplicial in $\mathcal{U}_t$, $\kappa_h \subseteq adj(\kappa_t, \mathcal{U}_t) = \{X\}\cup adj(X, \mathcal{U}_t)\setminus\{\kappa_t\}= \{X\}\cup pa(X, {\cal C}^*)={\mathbf{T}}$. Hence, $\{\kappa_t\}\cup {\mathbf{T}}$ is a clique in ${\cal G}^*$. It is then straightforward to check that the restriction subset of ${\cal K}\cup D_X$ on ${\cal G}^*({\bf M}_{\mathbf{T}})$ given ${\cal G}^*$ is not consistent with ${\cal G}^*({\bf M}_{\mathbf{T}})$ for any ${\bf M}_{\mathbf{T}}$ containing $\{\kappa_t\}\cup {\mathbf{T}}$.
			
			\item[(iii)] $adj(X, \mathcal{U}_t)\setminus\{\kappa_t\}\neq\varnothing$ and $adj(X, \mathcal{U}_t)\setminus\{\kappa_t\} \subsetneq pa(X, {\cal C}^*)$. Below we show that this is also an impossible case. For any $p \in pa(X, {\cal C}^*)$ which is not in $\mathcal{U}_t$, $p$ is not adjacent to $\kappa_t$, since otherwise $p\to \kappa_t$ is in ${\cal C}^*$ by Rule 1 of Meek's rules and $p$ should be included in $\mathcal{U}_t$. Denote by $\cal T$ the induced subgraph of ${\cal G}^*$ over $\{p, \kappa_t\}\cup adj(\kappa_t, \mathcal{U}_t)$. As $adj(\kappa_t, \mathcal{U}_t)\subseteq \{X\}\cup pa(X, {\cal C}^*)$ and $\{X\}\cup pa(X, {\cal C}^*)$ induces a complete subgraph of ${\cal G}^*$, $p\in pa(X, {\cal C}^*)$ is adjacent to every vertex in $adj(\kappa_t, \mathcal{U}_t)$. Therefore, $p$ and $\kappa_t$ are all and only simplicial vertices in $\cal T$. Moreover, due to the consistency of ${\cal K}_c$ and the fact that $\kappa\in {\cal K}_c$, $\cal T$ is an orientation component for $p$ with respect to ${\cal K}_c$ and ${\cal G}^*$. Therefore, $adj(\kappa_t, \mathcal{U}_t) \to p$ are in $\cal H$. In particular, $X\to p$ is in $\cal H$, which is contrary to the assumption that $p\in pa(X, {\cal C}^*) \subseteq pa(X, {\cal H})\cup sib(X, {\cal H})$.
			
			\item[(iv)] $adj(X, \mathcal{U}_t)\setminus\{\kappa_t\}\neq\varnothing$ and $adj(X, \mathcal{U}_t)\cap ch(X, {\cal C}^*)\setminus\{\kappa_t\} \neq\varnothing$. This case is not possible either. Notice that none of the vertices in $adj(X, \mathcal{U}_t)\cap ch(X, {\cal C}^*)\setminus\{\kappa_t\}$ is simplicial in $\mathcal{U}_t$, we have that $\mathcal{U}_t$ is not complete, and thus $\mathcal{U}_t$ must have two non-adjacent simplicial vertices. Since $X$ and $\kappa_t$ are two adjacent simplicial vertices in $\mathcal{U}_t$, there must exist another simplicial vertex $w$ in $\mathcal{U}_t$ which is not adjacent to $X$ and not adjacent to $\kappa_t$. However, as $w\neq X$ and $w$ is not adjacent to $X$, $w$ is not the tail of any DCC in $D_X$. It implies that $w\neq \kappa_t$ is a potential leaf node of $\mathcal{U}_t$ with respect to $D_X$ and ${\cal G}^*$. This leads to a contradiction since we assume that $\mathcal{U}_t$ is the maximal orientation component for $\kappa_t$ with respect to $D_X$ and ${\cal G}^*$.
		\end{enumerate}
		
		\vspace{0.5em}
		
		Next, suppose that none of the vertices in $adj(X, \mathcal{U}_t)\cap ch(X, {\cal C}^*)$ is simplicial in $\mathcal{U}_t$. Denote by $\mathbf{R}$ (possibly empty) the set of simplicial vertices in $\mathcal{U}_t$ which are adjacent to $X$.
		Let $\cal T$ be the induced subgraph of $\mathcal{U}_t$ over $\mathbf{V}(\mathcal{U}_t)\setminus \mathbf{R}$. We will show that,
		\begin{enumerate}
			\item[(P6)] $\cal T$ is an orientation component for $X$ with respect to ${\cal K}_c$ and ${\cal G}^*$.
		\end{enumerate}
		Since $X$ is a simplicial node in $\mathcal{U}_t$, $X$ must be simplicial in $\cal T$. Hence, by Theorem~\ref{thm:consistency} and the consistency of ${\cal K}_c$, it suffices to show that for any $Y\in \mathbf{V}({\cal T})\setminus\{X\}$, $Y$ is not a potential leaf node in ${\cal T}$ with respect to ${\cal K}_c$ and ${\cal G}^*$. The proof consists of two claims:
		
		\begin{enumerate}
			\item [(i)]  For any $Y\in \mathbf{V}({\cal T})\setminus\{X\}$, $Y$ is not simplicial in $\mathcal{U}_t$ implies that $Y$ is not simplicial in $\cal T$. Suppose that there exists a $Y\in \mathbf{V}({\cal T})\setminus\{X\}$ which is not simplicial in $\mathcal{U}_t$ but simplicial in $\cal T$, then $adj(Y, \mathcal{U}_t)\setminus adj(Y, \mathcal{T}) \neq\varnothing$. This implies that $\mathbf{R}\neq \varnothing$ as $adj(Y, \mathcal{U}_t)\setminus adj(Y, \mathcal{T}) \subseteq \mathbf{V}(\mathcal{U}_t)\setminus \mathbf{V}({\cal T})=\mathbf{R}$. On the other hand, since every vertex in  $\mathbf{R}$ is simplicial in $\mathcal{U}_t$, $\mathbf{R} \subseteq adj(Y, \mathcal{U}_t)$, meaning that $adj(Y, \mathcal{U}_t)\setminus adj(Y, \mathcal{T}) = \mathbf{R}$. Moreover, $\mathbf{R} \subseteq adj(X, \mathcal{U}_t)$ implies that $Y$ is also adjacent to $X$. Notice that $adj(Y, {\cal T})$ is a clique and $X\in adj(Y, {\cal T})$, $adj(Y, \mathcal{U}_t)$ is also a clique, which contradicts our assumption.
			
			\item[(ii)] For any $Y\in \mathbf{V}({\cal T})\setminus\{X\}$, $Y$ is simplicial in $\mathcal{U}_t$ implies that $Y$ is the tail of some DCC in ${\cal K}_c(\cal T)$. In fact, if $Y\notin \mathbf{R}\cup \{X\}$ is simplicial in $\mathcal{U}_t$, then it is simplicial in $\cal T$, and there is a DCC in ${\cal K}_c(\mathcal{U}_t)$ whose tail is $Y$, since otherwise $Y$ is a potential leaf node in $\mathcal{U}_t$ with respect to ${\cal K}_c$ and ${\cal G}^*$, which means $Y\in adj(X, \mathcal{U}_t)\cap pa(X, {\cal C}^*)$ and thus $Y\in \mathbf{R}$, according to (P5). If $Y$ is not the tail of any DCC in ${\cal K}_c(\cal T)$, then $adj(Y, {\cal T})\subsetneq adj(Y, {\mathcal{U}_t})$, meaning that $adj(Y, {\mathcal{U}_t})\cap \mathbf{R}\neq\varnothing$. Since the vertices in $\mathbf{R}$ are simplicial in $\mathcal{U}_t$ and $\mathbf{R}\subseteq adj(X, \mathcal{U}_t)$, $Y$ is adjacent to $X$ in $\mathcal{U}_t$. Finally, as $Y$ is simplicial in $\mathcal{U}_t$, by the construction, we have $Y\in \mathbf{R}$ and hence $Y\notin\mathbf{V}({\cal T})$, which contradicts our assumption.
		\end{enumerate}

		In conclusion, $\cal T$ is an orientation component for $X$ with respect to ${\cal K}_c$ and ${\cal G}^*$. Note that $adj(X, \mathcal{U}_t)\cap ch(X, {\cal C}^*)\neq\varnothing$ and $adj(X, \mathcal{U}_t)\cap ch(X, {\cal C}^*)\cap \mathbf{R}=\varnothing$, $adj(X, \mathcal{U}_t)\cap ch(X, {\cal C}^*) \subseteq adj(X, {\cal T})$. Hence, $adj(X, \mathcal{U}_t)\cap ch(X, {\cal C}^*) \to X$ are in $\cal H$, which contradicts the assumption that $ch(X, {\cal C}^*) \subseteq ch(X, {\cal H})\cup sib (X, {\cal H}) \setminus \mathbf{S}$.
		
		\vspace{0.5em}
		
		\textbf{Case 1-2-2.} Suppose that $X$ is not simplicial in $\mathcal{U}_t$. We first show that there is a $\phi\in adj(X, \mathcal{U}_t)\cap pa(X, {\cal C}^*)$ which is a potential leaf node in $\mathcal{U}_t$ with respect to ${\cal K}_c$ and ${\cal G}^*$. In fact, with respect to ${\cal K}_c$ and ${\cal G}^*$, $\mathcal{U}_t$, which is a connected undirected induced subgraph of ${\cal G}^*$, must have a potential leaf node, since ${\cal K}_c\subseteq {\cal K}$ and ${\cal K}$ is consistent with ${\cal G}^*$. By (P5), these potential leaf nodes are in $\{X\}\cup pa(X, {\cal C}^*)$. However, as $X$ is not simplicial in $\mathcal{U}_t$ by our assumption, the potential leaf nodes are all in $pa(X, {\cal C}^*)$.
		
		Now let $\phi\in adj(X, \mathcal{U}_t)\cap pa(X, {\cal C}^*)$ be a potential leaf node in $\mathcal{U}_t$ with respect to ${\cal K}_c$ and ${\cal G}^*$. Denote by $\mathbf{R}$ (possibly empty) the set of simplicial vertices in $\mathcal{U}_t$ which are adjacent to $\phi$. Clearly, $\{\phi\}\cup \mathbf{R}$ is a clique of $\mathcal{U}_t$. Let $\cal T$ be the induced subgraph of $\mathcal{U}_t$ over $\mathbf{V}(\mathcal{U}_t)\setminus \mathbf{R}$. By the similar argument given to prove (P6), we can prove that $\cal T$ is an orientation component for $\phi$ with respect to ${\cal K}_c$ and ${\cal G}^*$. Note that $X\in adj(\phi, \mathcal{U}_t)$ and $X\notin \mathbf{R}$, $X\in adj(\phi, \mathcal{T})$. Hence, $X \to \phi$ is in $\cal H$, which contradicts the assumption that $\phi \in pa(X, {\cal H})\cup \mathbf{S}$.
		
		\vspace{0.5em}
		
		\textbf{Case 2.} Suppose that ${\cal K}_c = \varnothing$, then there exists a connected undirected induced subgraph $\cal W$ of ${\cal C}^{*}$ such that $\cal W$ has no potential lead node with respect to ${\cal K}_u \cup {\cal K}_r$ and ${\cal C}^{*}$ (Lemma~\ref{lemma:consist}). Note that, since ${\cal C}^{*}$ is a chain graph, for any DCC in ${\cal K}_u$, its heads and tail are in the same chain component of ${\cal C}^{*}$. As $\cal K$ is consistent with ${\cal G}^*$, ${\cal K}_u \subseteq {\cal K}$ is also consistent with ${\cal G}^*$. Then, by Lemma~\ref{lemma:consist}, we can conclude that,
		\begin{enumerate}
			\item[(P7)] ${\cal K}_u$ is consistent with ${\cal G}^*$ given $D_X$ because ${\cal K}_u({\cal C}^{*})={\cal K}_u={\cal K}_u({\cal G}^{*})$ and any connected undirected induced subgraph of ${\cal C}^{*}$ is also a connected undirected induced subgraph of ${\cal G}^{*}$.
		\end{enumerate}
		Since $\cal W$ is undirected and connected, $\cal W$ must be an induced subgraph of some chain component of ${\cal C}^{*}$.
		Denote by $\mathbf{A}$ the set of all potential leaf nodes of $\cal W$ with respect to ${\cal K}_u$ and ${\cal C}^{*}$. As ${\cal K}_u$ is consistent with ${\cal G}^*$ given $D_X$, $\mathbf{A}\neq\varnothing$. Notice that, $({\cal K}_u \cup {\cal K}_r)({\cal W}\mid {\cal C}^{*}) = {\cal K}_u({\cal W}\mid {\cal C}^{*}) \cup {\cal K}_r({\cal W}\mid {\cal C}^{*})$, $\cal W$ has no potential leaf node with respect to ${\cal K}_u \cup {\cal K}_r$ and ${\cal C}^{*}$ implies that every vertex in $\mathbf{A}$ is the tail of some DCC in ${\cal K}_r({\cal W}\mid {\cal C}^{*})$. That is,
		\begin{enumerate}
			\item[(P8)] for any $a \in \mathbf{A}$ there is a DCC $\kappa\coloneqq(a\tor\kappa_h)\in {\cal K}_r$  such that (i) $\kappa_h\cap sib(a, {\cal C}^*) \neq \varnothing$ and  $\kappa_h\cap sib(a, {\cal C}^*) \subseteq \mathbf{V}(\cal W)$,  (ii) $\kappa_h\setminus\mathbf{V}(\cal W)\neq\varnothing$ and $\kappa_h\setminus\mathbf{V}({\cal W})\subseteq pa(a, {\cal C}^*)$, and (iii) $\kappa_h\subseteq \mathbf{V}({\cal W})\cup pa(a, {\cal C}^*)$.
		\end{enumerate}

		Now consider ${\cal G}^*$ and $D_X$. For any $a\in \mathbf{A}$, let $\mathcal{U}_a$ be the maximal orientation component for $a$ with respect to $D_X$ in ${\cal G}^*$. It is clear that $pa(a, {\cal C}^*) = adj(a, \mathcal{U}_a)$. Since ${\cal C}^*$ is a chain graph causal MPDAG and $\mathbf{A}\subseteq \mathbf{V}(\cal W)$ and $\cal W$ is an induced subgraph of some chain component of ${\cal C}^{*}$, by Theorem~\ref{the:MPDAG}, $pa(a, {\cal C}^*) = pa(a', {\cal C}^*)$ for any $a, a'\in \mathbf{A}$. Thus, $adj(a, \mathcal{U}_a) = adj(a', \mathcal{U}_{a'})$ for any $a, a'\in \mathbf{A}$.
		
		On the other hand, following the same argument for case 1, it can be shown that $X\in \mathbf{V}(\mathcal{U}_a)$ for any $a\in \mathbf{A}$. Moreover, if $adj(X, \mathcal{U}_a) \cap ch(X, {\cal C}^*)=\varnothing$ for some $a\in \mathbf{A}$, then we can prove that $X$ is a potential leaf node in $\mathcal{U}_a$ with respect to $D_X$ and ${\cal G}^*$, which means $X=a$, and consequently,  $X \in \mathbf{V}(\cal W)$. As $\cal W$ is an induced subgraph of some chain component of ${\cal C}^*$ but $X$ has no siblings in  ${\cal C}^*$,  $\mathbf{V}({\cal W})=\{X\}$. This is impossible since $\mathbf{V}({\cal W})=\{X\}$ implies that $\kappa_h\cap sib(a, {\cal C}^*)=\varnothing$, contrary to (P8). Therefore, we have that,
		\begin{enumerate}
			\item[(P9)] for all $a\in \mathbf{A}$, $adj(X, \mathcal{U}_a) \cap ch(X, {\cal C}^*)\neq\varnothing$.
		\end{enumerate}
		Moreover, following the same argument for proving (P4), it can be shown that,
		\begin{enumerate}
			\item[(P10)] none of the vertices in $\mathbf{V}(\mathcal{U}_a)\setminus (\{a, X\}\cup pa(X, {\cal C}^*))$ (possibly empty) is simplicial in $\mathcal{U}_a$.
		\end{enumerate}

		The rest of the proof is similar to that for case 1-2. Let
		\[\mathbf{F}=\mathbf{V}({\cal W})\bigcup_{a\in \mathbf{A}} \mathbf{V}(\mathcal{U}_a),\]
		and
		\[\mathbf{N}=\bigcup_{a\in \mathbf{A}} \mathbf{V}(\mathcal{U}_a) \setminus \left( \{X\}\cup \mathbf{A} \cup pa(X, {\cal C}^*) \right).\]
		Denote by ${\cal F}$ the induced subgraph of ${\cal G}^*$ over $\mathbf{F}$. Firstly, since $\mathbf{V}(\mathcal{U}_a)\subseteq \mathbf{F}$ and $\mathbf{N}\subseteq \mathbf{F}$, by (P10), it holds that
		\begin{enumerate}
			\item[(P11)] none of the vertices in $\mathbf{N}$ is simplicial in ${\cal F}$.
		\end{enumerate}
		Moreover, by the definition of $\mathbf{A}$, every vertex in $\mathbf{V}({\cal W})\setminus\mathbf{A}$ is either non-simplicial in $\cal W$, or the tail of a DCC in ${\cal K}_u({\cal W}\mid {\cal C}^*)$. Thus,
		\begin{enumerate}
			\item[(P12)] $w\in\mathbf{V}({\cal W})\setminus\mathbf{A}$ is non-simplicial in $\cal W$ implies that $w$ is non-simplicial in $\cal F$, and $w\in\mathbf{V}({\cal W})\setminus\mathbf{A}$ is the tail of a DCC in ${\cal K}_u({\cal W}\mid {\cal C}^*)$ implies that $w$ is also the tail of a DCC in ${\cal K}_u({\cal W}\mid {\cal G}^*)$ based on (P7).
		\end{enumerate}
		Finally, for set $\mathbf{A}$, we have that,
		\begin{enumerate}
			\item[(P13)] for any $a \in \mathbf{A}$, $a$ is simplicial in $\cal F$, and is also the tail of some DCC in ${\cal K}_r({\cal F}\mid {\cal G}^*)$.
		\end{enumerate}
		The first claim comes from the simplicity of $a$ in $\cal W$ as well as the fact that
		\[adj(a, {\cal F})=adj(a, {\cal W})\bigcup_{a'\in \mathbf{A}} adj(a', {\cal U}_{a'})\\ =adj(a, {\cal W})\cup adj(a, {\cal C}^*)\\ =  adj(a, {\cal W})\cup pa(a, {\cal C}^*),\]
		where the second equality holds because $adj(a, \mathcal{U}_a) = adj(a', \mathcal{U}_{a'})$ for any $a, a'\in \mathbf{A}$, and the third equality holds because of the definition of ${\cal U}_a$. The second claim holds because of the above equation and (P8)-(iii).

		Below we will consider two subcases depending on whether $X$ is simplicial in ${\cal F}$.

		\vspace{0.5em}
		
		\textbf{Case 2-1 (analogue to case 1-2-1).} If $X$ is simplicial in $\mathcal{F}$, then $adj(X, \mathcal{F})$ is a clique in $\mathcal{F}$. Recall that (P9) says that $adj(X, \mathcal{U}_a) \cap ch(X, {\cal C}^*)\neq\varnothing$ for all $a\in \mathbf{A}$, we have $adj(X, \mathcal{F})\cap ch(X, {\cal C}^*)\neq\varnothing$ based on the definition of $\cal F$.
		
		If there is a $c\in adj(X, \mathcal{F})\cap ch(X, {\cal C}^*)$ which is simplicial in $\mathcal{F}$, then $c\in \mathbf{V}({\cal W})$ based on (P11), (P12) and (P13). Since ${\cal C}^*$ is a chain graph and $\cal W$ is an induced subgraph of some chain component in ${\cal C}^*$, $X\to \mathbf{V}({\cal W })$ in ${\cal C}^*$. That is, $\mathbf{V}({\cal W}) \subseteq ch(X, {\cal C}^*) \cap adj(X, \mathcal{F})$. As $X$ is simplicial, $\mathbf{V}({\cal W})$ is a clique in $\mathcal{F}$.
		
		\begin{enumerate}
			\item[(i)]  $adj(X, \mathcal{F}) = \mathbf{V}({\cal W})$. If there is an $a\in \mathbf{A}\subseteq \mathbf{V}({\cal W})$ such that $pa(a, {\cal C}^*)\setminus\{X\}\neq\varnothing$, then the vertices in $pa(a, {\cal C}^*)\setminus\{X\}$ are adjacent to $X$, since $a$ is simplicial in $\cal F$ and $pa(a, {\cal C}^*) \subseteq \mathbf{V}({\cal U}_a) \subseteq \mathbf{V}({\cal F})$. This contradicts $adj(X, \mathcal{F}) = \mathbf{V}({\cal W})$, and thus, $pa(a, {\cal C}^*)=\{X\}$. Since $pa(w, {\cal C}^*) = pa(w', {\cal C}^*)$ for any $w, w'\in \mathbf{V}({\cal W})$, we have $pa(\mathbf{V}({\cal W}), {\cal C}^*)=\{X\}$. It can be shown by (P12) and (P13) that the induced subgraph of ${\cal G}^*$ over $\{X\}\cup \mathbf{V}({\cal W})$ is an orientation component for $X$ with respect to ${\cal K}_u\cup {\cal K}_r$ and ${\cal G}^*$, thus $\mathbf{V}({\cal W})\to X$ are in $\cal H$, which contradicts the assumption that $\mathbf{V}({\cal W})\subseteq ch(X, {\cal C}^*) \subseteq ch(X, {\cal H})\cup sib (X, {\cal H}) \setminus \mathbf{S}$.
			
			\item[(ii)]  $adj(X, \mathcal{F}) \setminus \mathbf{V}({\cal W}) = pa(X, {\cal C}^*)$. We claim that the restriction subset of ${\cal K}\cup D_X$ on ${\cal G}^*({\bf M}_{\mathbf{T}})$ given ${\cal G}^*$ is not consistent with ${\cal G}^*({\bf M}_{\mathbf{T}})$ for any ${\bf M}_{\mathbf{T}}$ containing $\mathbf{T}\cup \mathbf{V}({\cal W})$. In fact, every vertex in $\mathbf{T}\cup \mathbf{V}({\cal W})=\{X\}\cup pa(X, {\cal C}^*)\cup \mathbf{V}({\cal W})$ is the tail of some DCC in the restriction subset of
			${\cal K}\cup D_X$ on ${\cal G}^*(\mathbf{T}\cup \mathbf{V}({\cal W}))$ given ${\cal G}^*$, because (1) $pa(X,{\cal C}^*)\to X$ and $X\to \mathbf{V}({\cal W})$ are in the restriction subset of	$D_X$ on ${\cal G}^*(\mathbf{T}\cup \mathbf{V}({\cal W}))$ given ${\cal G}^*$, (2) every vertex in $\mathbf{V}({\cal W})\setminus \mathbf{A}$ is the tail of some DCC in the restriction subset of ${\cal K}_u$ on ${\cal G}^*(\mathbf{T}\cup \mathbf{V}({\cal W}))$ given ${\cal G}^*$ according to (P12), and (3) every vertex in $\mathbf{A}$ is the tail of some DCC in the restriction subset of ${\cal K}_r$ on ${\cal G}^*(\mathbf{T}\cup \mathbf{V}({\cal W}))$ given ${\cal G}^*$ according to (P13) and (P8).
			
			\item[(iii)]  $adj(X, \mathcal{F}) \setminus \mathbf{V}({\cal W}) \subsetneq pa(X, {\cal C}^*)$. Let $p\in pa(X, {\cal C}^*)$ such that $p\notin adj(X, \mathcal{F})$. It is clear that none of the vertices in ${\cal W}$ is adjacent to $p$, since otherwise $p\to \mathbf{V}({\cal W})$ are in ${\cal C}^*$, and in particular, $p\to \mathbf{A}$ are in ${\cal C}^*$ and $p$ should be included in $\cal F$. By the similar argument for proving case 1-2-1 we can show that the induced subgraph of ${\cal G}^*$ over $\{p, X\}\cup adj(X, \mathcal{F})$ is an orientation component for $p$ with respect to ${\cal K}_u\cup {\cal K}_r$. In fact, (1) $p$ is simplicial in the induced subgraph of ${\cal G}^*$ over $\{p, X\}\cup adj(X, \mathcal{F})$ since $p$ is adjacent to the vertices in $\{X\}\cup adj(X, \mathcal{F}) \setminus \mathbf{V}({\cal W})$ and $\{X\}\cup adj(X, \mathcal{F})$ is a clique, (2) every vertex in $\{X\}\cup adj(X, \mathcal{F}) \setminus \mathbf{V}({\cal W})$ is non-simplicial in the induced subgraph of ${\cal G}^*$ over $\{p, X\}\cup adj(X, \mathcal{F})$ since $p$ is not adjacent to any $w\in\mathbf{V}({\cal W})$ but both $p$ and $w$ are neighbors of the vertices in $\{X\}\cup adj(X, \mathcal{F}) \setminus \mathbf{V}({\cal W})$, (3) every vertex in $\mathbf{V}({\cal W})\setminus \mathbf{A}$ is the tail of some DCC in ${\cal K}_u({\cal W}\mid {\cal G}^*)$, and (4) $\mathbf{V}({\cal W})\cup pa(a, {\cal C}^*)\subseteq \{p, X\}\cup adj(X, \mathcal{F})$ for every $a\in\mathbf{A}$ (since the simplicity of $a$ in $\cal F$ implies that $pa(a, \mathcal{F})\subseteq adj(X, \mathcal{F})$) and (P8) implies that every $a\in\mathbf{A}$ is the tail of some DCC in the restriction subset of ${\cal K}_r$ on ${\cal G}^*(\{p, X\}\cup adj(X, \mathcal{F}))$ given ${\cal G}^*$. Therefore, $X\to p$ is in $\cal H$, which leads to a contradiction.
			
			\item[(iv)]  $adj(X, \mathcal{F}) \cap ch(X, {\cal C}^*) \setminus \mathbf{V}({\cal W})\neq \varnothing$. Let $c\in adj(X, \mathcal{F}) \setminus \mathbf{V}({\cal W})$ such that $c\in ch(X, {\cal C}^*)$. It is clear that $c\in \mathbf{N}$, and consequently, $c$ is not simplicial in $\cal F$ by (P11). This indicates that $\cal F$ is not complete, and thus, there is a simplicial vertex $w$ in $\cal F$ which is not adjacent to $X$. However, this is impossible since every simplicial vertex in $\cal F$ should be in $\mathbf{F}\setminus\mathbf{N}\subseteq \mathbf{V}({\cal W})\cup pa(X, {\cal C}^*) \cup\{X\}$, which is either $X$ or adjacent to $X$.
		\end{enumerate}

		Now assume that none of the vertices in $adj(X, \mathcal{F})\cap ch(X, {\cal C}^*)$ is simplicial in $\mathcal{F}$. Denote by $\mathbf{R}$ (possibly empty) the set of simplicial vertices in $\mathcal{F}$ which are adjacent to $X$. Following the same argument for proving (P6), we can show that the induced subgraph of $\mathcal{F}$ over $\mathbf{F}\setminus \mathbf{R}$ is an orientation component for $X$ with respect to ${\cal K}_u\cup {\cal K}_r$ and ${\cal G}^*$. Hence, $adj(X, \mathcal{F})\cap ch(X, {\cal C}^*) \to X$ are in $\cal H$, which leads to a contradiction.
		
		\vspace{0.5em}
		
		\textbf{Case 2-2 (analogue to case 1-2-2).} Suppose that $X$ is not simplicial in $\mathcal{F}$. Denote by $\bf L$ the set of potential leaf nodes in $\cal F$ with respect to ${\cal K}_u \cup {\cal K}_r$ and ${\cal G}^*$. Since ${\cal K}_u \cup {\cal K}_r$ is consistent with ${\cal G}^*$, $\mathbf{L}\neq\varnothing$. Based on (P11), (P12), (P13) and the definition of a potential leaf node,  $\mathbf{L}\subseteq \{X\}\cup pa(X, {\cal C}^*)$. Moreover, since $X$ is not simplicial in $\mathcal{F}$, $\mathbf{L}\subseteq pa(X, {\cal C}^*)$. Therefore, there is a $\phi\in adj(X, \mathcal{F})\cap pa(X, {\cal C}^*)$ which is a potential leaf node in $\mathcal{F}$ with respect to ${\cal K}_u \cup {\cal K}_r$. Denote by $\mathbf{R}$ (possibly empty) the set of simplicial vertices in $\mathcal{F}$ which are adjacent to $\phi$. Following the same proof for case 1-2-2, it can be checked that the induced subgraph of $\mathcal{F}$ over $\mathbf{F}\setminus \mathbf{R}$ is an orientation component for $\phi$ with respect to ${\cal K}_u \cup {\cal K}_r$ and ${\cal G}^*$. Hence, $X \to \phi$ is in $\cal H$, which leads to a contradiction.
	\end{proof}
	
	\section{Simulations}\label{app:sim}
	{
		Some details of the simulations are presented in this section.
		\subsection{Generating Chordal Graphs}\label{app:app:sim}
		To generate a chordal graph with $n$ nodes and $e$ edges, where $e\geq n$, we first randomly generate a connected undirected tree with $n$ nodes using the following Algorithm~\ref{algo: tree_construct}. Algorithm~\ref{algo: tree_construct} successively adds $n-1$ edges to a graph with $n$ nodes but without any edge, and every edge Algorithm~\ref{algo: tree_construct} adds except the first one makes a singleton node connect to the existing non-singleton connected component. Next, we sequentially and randomly add $e-n+1$ undirected edges to $\cal T$. Every time an undirected edge is added to $\cal T$, we check whether the resulting graph is chordal, by calling $\tt networkx.is\_chodal$ from the Python package $\tt networkx$. If the resulting graph is chordal, we accept the added edge, otherwise we reject the added edge, re-sample an edge, and check the chordality again. The complete procedure is summarized in Algorithm~\ref{algo: gen_chordal}. In our implementation, we keep track of the number of iterations of the while loop. If it exceeds a predefined maximum threshold, the loop is terminated and FAIL is returned. In such case, we restart the process to sample a new chordal graph.

		\begin{algorithm}[!t]
			\caption{Generating a connected undirected tree.}
			\label{algo: tree_construct}
			\begin{algorithmic}[1]
				\REQUIRE
				$X_1,\cdots,X_n$, which are $n$ nodes.
				\ENSURE
				A connected undirected tree over $X_1,\cdots, X_n$.
				\STATE {Set $Q=(X_1)$ and $P=(X_2, \cdots X_n)$.}
				\STATE {Set $\cal T$ be a graph with nodes $X_1,\cdots,X_n$ and empty edge set.}
				\FOR {$i$ in $2,\cdots,n$,}
				\STATE {Randomly sample a node $q$ from Q and a node $p$ from $P$.}
				\STATE {Add the undirected edge $q-p$ to the graph $\cal T$.}
				\STATE {Adding $p$ to $Q$ and removing $p$ from $P$.}
				\ENDFOR
				\RETURN {$\cal T$.}
			\end{algorithmic}
		\end{algorithm}
		
		\begin{algorithm}[!t]
			\caption{Generating a chordal graph.}
			\label{algo: gen_chordal}
			\begin{algorithmic}[1]
				\REQUIRE
				$X_1,\cdots,X_n$, which are $n$ nodes, and $e$ representing the number of edges.
				\ENSURE
				A connected chordal graph over $X_1,\cdots, X_n$ with $e$ edges.
				\STATE {Sample a connected undirected tree $\cal T$ using Algorithm~\ref{algo: tree_construct}.}
				\STATE {Set $r=e-n+1$.}
				\WHILE {$r\neq 0$,}
				\STATE {Randomly sample two non-adjacent nodes $p$ and $q$.}
				\STATE {Add the undirected edge $q-p$ to $\cal T$ and denote the resulting graph by ${\cal T}_{\rm tmp}$.}
				\IF {${\cal T}_{\rm tmp}$ is chordal,}
				\STATE {Set $r = r-1$.}
				\STATE {Set ${\cal T} = {\cal T}_{\rm tmp}$.}
				\ENDIF
				\ENDWHILE
				\RETURN {$\cal T$.}
			\end{algorithmic}
		\end{algorithm}
		
		\subsection{Additional Results}\label{app:app:exp}
		
		Figure~\ref{fig:4metrics} shows the results of three metrics in the settings where $n=30$, where the metrics are the causal mean squared error (CMSE) introduced by~\citep{Tsirlis2018scoring, liu2020local}, the number of possible causal effects, and the length of the interval determined by the minimum and maximum values of a set of possible effects. All scores are rescaled, as described in Section~\ref{sec:sec:sim}.
		
		The results of the above three metrics are similar, as shown in Figure~\ref{fig:4metrics}. All scores decrease rapidly as the number of constraints increases. Moreover, for the same number of constraints, the scores of providing ancestral causal constraints are much lower than
		those of providing direct causal constraints, which in turn are significantly lower than those for non-ancestral constraints. This phenomenon arises because ancestral causal constraints are more informative than non-ancestral causal constraints. Specifically, knowing that $X$ is a cause of $Y$ implies that $Y$ is not a cause of $X$, but the reverse implication does not hold. In conclusion, pairwise causal background knowledge---particularly ancestral causal background knowledge---can greatly enhance the identifiability of a causal effect.
		%

		\begin{figure}[!t]
			\centering	
			\subfloat[CMSE, $e=30$]{
				\begin{minipage}[t]{0.22\linewidth}
					\centering
					\includegraphics[width=1\linewidth]{fig/edge30number30_cmse.pdf}
				\end{minipage}%
			}%
			\hspace{0.01\linewidth}
			\subfloat[CMSE, $e=45$]{
				\begin{minipage}[t]{0.22\linewidth}
					\centering
					\includegraphics[width=1\linewidth]{fig/edge45number30_cmse.pdf}
				\end{minipage}%
			}%
			\hspace{0.01\linewidth}
			\subfloat[CMSE, $e=60$]{
				\begin{minipage}[t]{0.22\linewidth}
					\centering
					\includegraphics[width=1\linewidth]{fig/edge60number30_cmse.pdf}
				\end{minipage}%
			}%
			\hspace{0.01\linewidth}
			\subfloat[CMSE, $e=75$]{
				\begin{minipage}[t]{0.22\linewidth}
					\centering
					\includegraphics[width=1\linewidth]{fig/edge75number30_cmse.pdf}
				\end{minipage}%
			}%
			
			\subfloat[Number of Possible effects, $e=30$]{
				\begin{minipage}[t]{0.22\linewidth}
					\centering
					\includegraphics[width=1\linewidth]{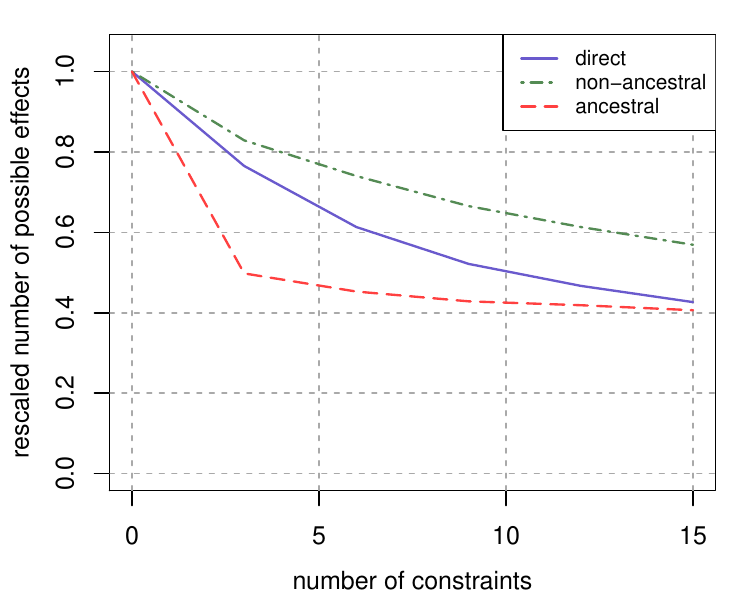}
				\end{minipage}%
			}%
			\hspace{0.01\linewidth}
			\subfloat[Number of Possible effects, $e=45$]{
				\begin{minipage}[t]{0.22\linewidth}
					\centering
					\includegraphics[width=1\linewidth]{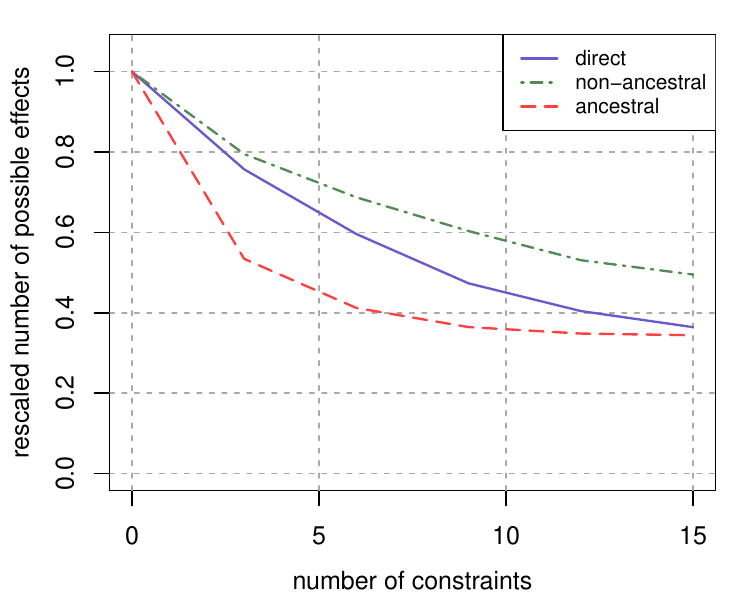}
				\end{minipage}%
			}%
			\hspace{0.01\linewidth}
			\subfloat[Number of Possible effects, $e=60$]{
				\begin{minipage}[t]{0.22\linewidth}
					\centering
					\includegraphics[width=1\linewidth]{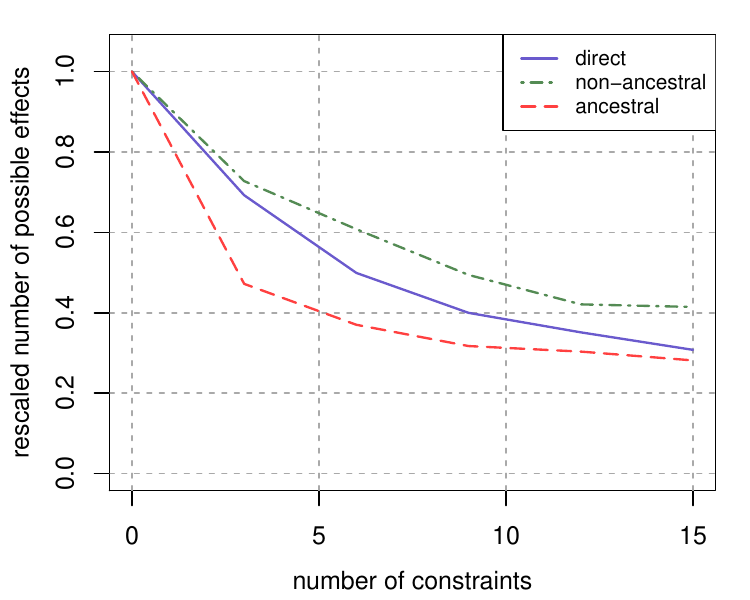}
				\end{minipage}%
			}%
			\hspace{0.01\linewidth}
			\subfloat[Number of Possible effects, $e=75$]{
				\begin{minipage}[t]{0.22\linewidth}
					\centering
					\includegraphics[width=1\linewidth]{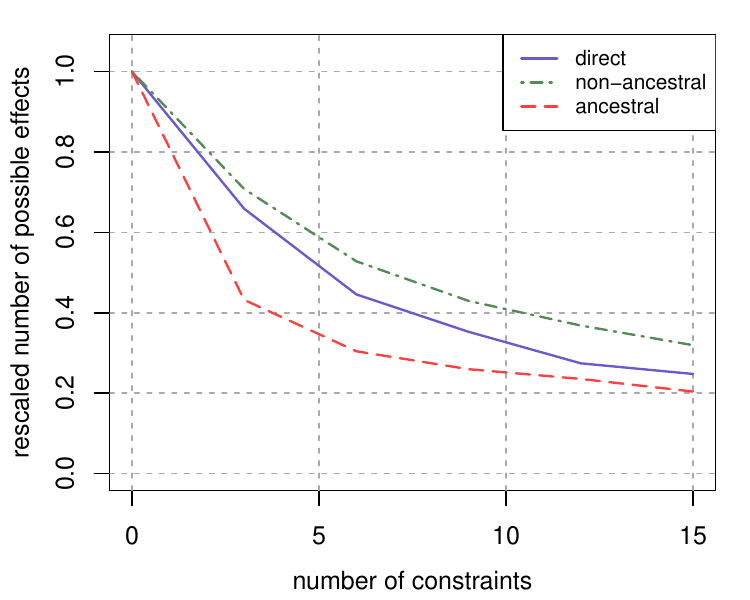}
				\end{minipage}%
			}%
			
			\subfloat[Length of the interval, $e=30$]{
				\begin{minipage}[t]{0.22\linewidth}
					\centering
					\includegraphics[width=1\linewidth]{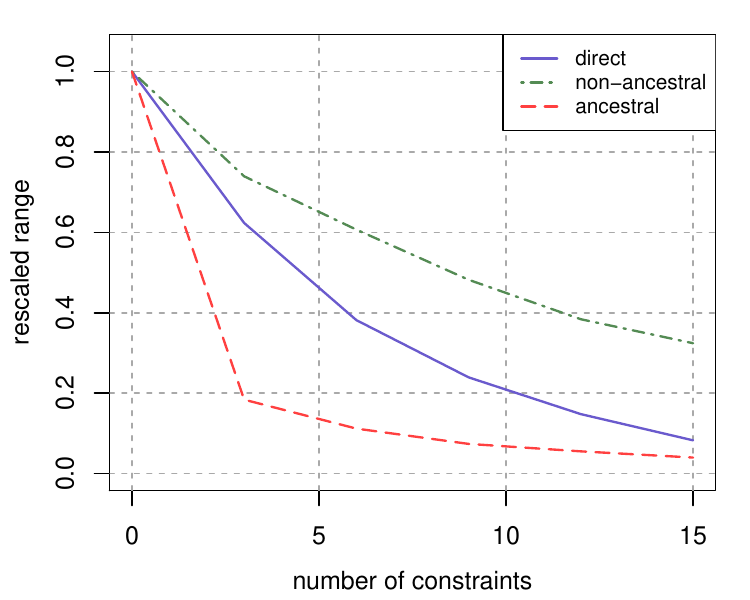}
				\end{minipage}%
			}%
			\hspace{0.01\linewidth}
			\subfloat[Length of the interval, $e=45$]{
				\begin{minipage}[t]{0.22\linewidth}
					\centering
					\includegraphics[width=1\linewidth]{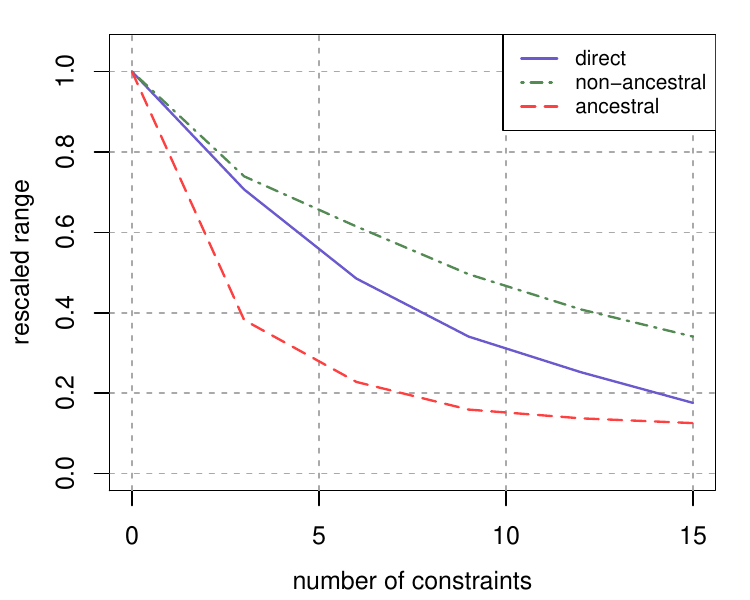}
				\end{minipage}%
			}%
			\hspace{0.01\linewidth}
			\subfloat[Length of the interval, $e=60$]{
				\begin{minipage}[t]{0.22\linewidth}
					\centering
					\includegraphics[width=1\linewidth]{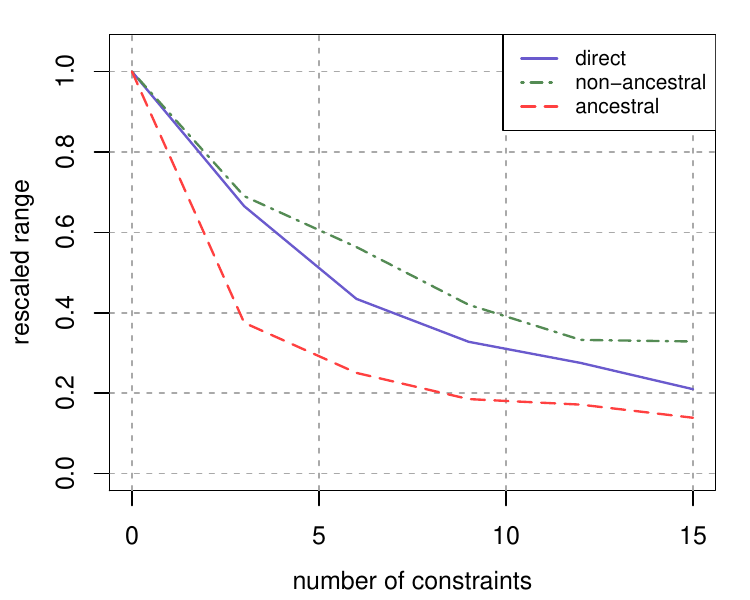}
				\end{minipage}%
			}%
			\hspace{0.01\linewidth}
			\subfloat[Length of the interval, $e=75$]{
				\begin{minipage}[t]{0.22\linewidth}
					\centering
					\includegraphics[width=1\linewidth]{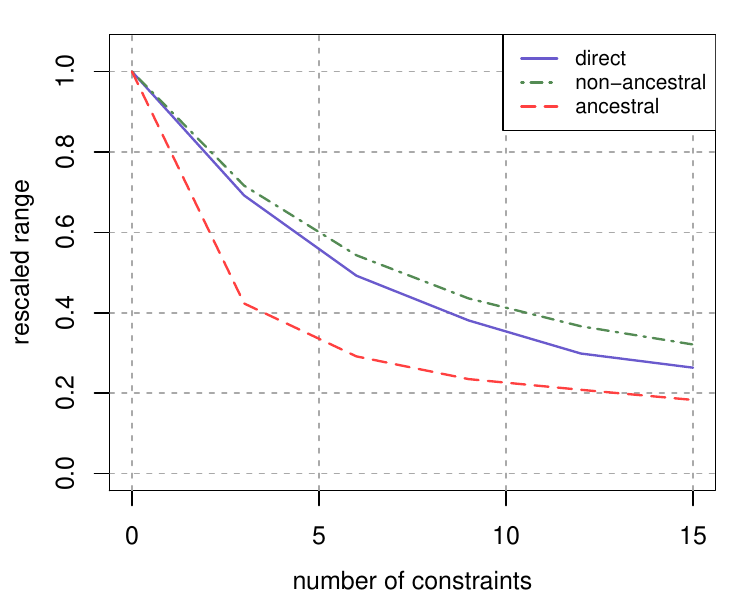}
				\end{minipage}%
			}%
			
			\caption{The results of the three metrics when $n=30$.}
			\label{fig:4metrics}
		\end{figure}
		
	}

	\newpage

	\vskip 0.2in
	\bibliography{ref220626}       
	
\end{document}